\documentclass{article}

\usepackage{microtype}
\usepackage{graphicx}
\usepackage{caption}
\usepackage{subcaption}
\usepackage{booktabs} 

\usepackage{ifthen}
\usepackage{svg}

\usepackage{longtable}

\usepackage{faktor}

\usepackage{makecell}

\usepackage{amsmath}
\usepackage{amssymb}
\usepackage{amsthm}

\usepackage{comment}

\usepackage{color}
\usepackage{url}

\usepackage{csquotes}
\usepackage{comment}
\usepackage{dsfont}

\usepackage{tikz}
\usetikzlibrary{bayesnet}
\usetikzlibrary{positioning}

\usepackage{textcomp}

\usepackage[ngerman]{datetime}

\DeclareMathOperator*{\argmax}{arg\,max}

\theoremstyle{plain}
\newtheorem{thm}{\protect\theoremname}

\theoremstyle{definition}
\newtheorem{defn}[thm]{\protect\definitionname}
\theoremstyle{definition}
\newtheorem{example}[thm]{Example}

\theoremstyle{plain}
\newtheorem{prop}[thm]{\protect\propositionname}

\theoremstyle{plain}
\newtheorem{lem}[thm]{\protect\lemmaname}
\theoremstyle{plain}
\newtheorem{cor}[thm]{Corollary}

\theoremstyle{plain}

\theoremstyle{definition}
\newtheorem{remark}[thm]{Remark}

\providecommand{\definitionname}{Definition}
\providecommand{\lemmaname}{Lemma}
\providecommand{\propositionname}{Proposition}
\providecommand{\theoremname}{Theorem}

\newcommand{\AOHistorySet}{\overline{\mathcal{A\!O}}}
\newcommand{\AOHistoryRV}{\overline{A\!O}}
\newcommand{\AOHistory}{\tau}

\newcommand{\ObservationSet}{\mathcal{O}}
\newcommand{\ObservationRV}{O}
\newcommand{\Observation}{o}

\newcommand{\HistorySet}{\mathcal{H}}
\newcommand{\HistoryRV}{H}

\newcommand{\History}{\tau}

\newcommand{\ActionSet}{\mathcal{A}}
\newcommand{\ActionRV}{A}
\newcommand{\Action}{a}

\newcommand{\StateSet}{\mathcal{S}}
\newcommand{\StateRV}{S}
\newcommand{\State}{s}

\newcommand{\RewardFunction}{\mathcal{R}}
\newcommand{\RewardRV}{R}
\newcommand{\Reward}{r}

\newcommand{\Aut}{\mathrm{Aut}}
\newcommand{\Sym}{\mathrm{Sym}}
\newcommand{\Iso}{\mathrm{Iso}}
\newcommand{\Auto}{g}
\newcommand{\AutoSecond}{\tilde{\Auto}}

\newcommand{\Isom}{f}
\newcommand{\IsoProfile}{\mathbf{\Isom}}
\newcommand{\SymProfile}{\mathbf{\Isom}}

\newcommand{\AutProfile}{\mathbf{\Auto}}

\newcommand{\DecP}{D}
\newcommand{\DecPSet}{{\mathcal{\DecP}}}
\newcommand{\DecPSetSecond}{{\mathcal{C}}}
\newcommand{\DecPTilde}{{\tilde{\DecP}}}

\newcommand{\Policy}{\pi}
\newcommand{\PolicySet}{\Pi}
\newcommand{\PolicySetDet}{\Pi^0}

\newcommand{\LA}{\sigma}
\newcommand{\LAProfile}{\boldsymbol{\sigma}}
\newcommand{\LASet}{\Sigma}

\newcommand{\OP}{\mathrm{OP}}
\newcommand{\Prob}{\mathbb{P}}
\newcommand{\E}{\mathbb{E}}
\newcommand{\PlayerSet}{\mathcal{N}}

\newcommand{\Proj}{\mathrm{proj}}

\newcommand{\Mixture}{\mu}

\newcommand{\Distr}{\nu}

\newcommand{\U}{\mathcal{U}} 

\newcommand{\PowerSet}{\mathcal{P}}
\newcommand{\PolicyLatent}{Z}

\newcommand{\SymNorm}{{\mathcal{K}}}

\newcommand{\PolicyTilde}{\tilde{\pi}}
\newcommand{\PolicyHat}{\hat{\pi}}

\newcommand{\Tmax}{T}

\newcommand{\Hash}{\#}

\newcommand{\IsomSecond}{\tilde{\Isom}}

\newcommand{\IsomHat}{\hat{\Isom}}
\newcommand{\DecPSecond}{E}
\newcommand{\DecPThird}{F}

\newcommand{\MixtureHat}{\hat{\Mixture}}

\newcommand{\Id}{e}

\newcommand{\Param}{\theta}

\newcommand{\NN}{\xi}

\newcommand{\Tie}{\chi}

\newcommand{\EquivMapping}{\phi}
\newcommand{\EquivSet}{\Phi}

\newboolean{commentsactivated}
\setboolean{commentsactivated}{false}
\newcommand{\co}[1]{\ifthenelse{\boolean{commentsactivated}}{{\color{red} {\em CO: #1 }}}{}}

\newcommand{\jt}[1]{\ifthenelse{\boolean{commentsactivated}}{{\color{blue} {\em JT: #1 }}}{}}

\mathchardef\ordinarycolon\mathcode`\:
\mathcode`\:=\string"8000
\begingroup \catcode`\:=\active
  \gdef:{\mathrel{\mathop\ordinarycolon}}
\endgroup

\usepackage{tabularx}

\usepackage{xtab}

\usepackage{hyperref}

\usepackage[accepted]{icml2021}

\usepackage{xargs}

\newcommandx{\parencite}[3][1, 2]{{\citep[#1][#2]{#3}}}
\newcommandx{\textcite}[3][1, 2]{{\citet[#1][#2]{#3}}}

\icmltitlerunning{A New Formalism, Method and Open Issues for Zero-Shot Coordination}

\usepackage{appendix}

\newif\ifappendix
\appendixtrue

\begin{document}

\twocolumn[
\icmltitle{A New Formalism, Method and Open Issues for Zero-Shot Coordination}



\icmlsetsymbol{equal}{\(\dagger\)}

\begin{icmlauthorlist}
\icmlauthor{Johannes Treutlein}{to,vector,equal}
\icmlauthor{Michael Dennis}{chai}
\icmlauthor{Caspar Oesterheld}{duke}
\icmlauthor{Jakob Foerster}{to,vector,face}
\end{icmlauthorlist}

\icmlaffiliation{to}{Department of Computer Science, University of Toronto, Toronto, Canada}
\icmlaffiliation{vector}{Vector Institute, Toronto, Canada}
\icmlaffiliation{duke}{Department of Computer Science, Duke University, Durham, USA}
\icmlaffiliation{face}{Facebook AI Research, USA}
\icmlaffiliation{chai}{Department of Electrical Engineering and Computer Science, University of California, Berkeley, USA}

\icmlcorrespondingauthor{Johannes Treutlein}{treutlein@cs.toronto.edu}

\icmlkeywords{Multi-agent reinforcement learning, zero-shot coordination, Dec-POMDP isomorphism}

\vskip 0.3in
]



\printAffiliationsAndNotice{\textsuperscript{\(\dagger\)}Part of the work was done as an intern at the Center for Human-Compatible AI, University of California, Berkeley}


\begin{abstract}
In many coordination problems, independently reasoning humans are able to discover mutually compatible policies. In contrast, independently trained self-play policies are often mutually incompatible.  
\emph{Zero-shot coordination} (ZSC) has recently been proposed as a new frontier in multi-agent reinforcement learning to address this fundamental issue. 
Prior work approaches the ZSC problem by assuming players can agree on a shared learning algorithm but not on labels for actions and observations, and proposes \emph{other-play} as an optimal solution. 
However, until now, this ``label-free'' problem has only been informally defined. 
We formalize this setting as the \emph{label-free coordination (LFC) problem} by
defining the \emph{label-free coordination game}.  We show that other-play is not an optimal solution to the LFC problem as it fails to consistently break ties between incompatible maximizers of the other-play objective.  We introduce an extension of the algorithm, \emph{other-play with tie-breaking}, and prove that it is optimal in the LFC problem and an equilibrium in the LFC game. Since arbitrary tie-breaking is precisely what the ZSC setting aims to prevent, we conclude that the LFC problem does not reflect the aims of ZSC. To address this, we introduce an alternative informal operationalization of ZSC as a starting point for future work.

\end{abstract}

\section{Introduction}
\label{section-introduction}

In multi-agent reinforcement learning (MARL), variations of the
\emph{self-play} (SP) regime \parencite[][]{tesauro1994td} have been successful in producing superhuman policies for two-player zero-sum games such as chess, go, and poker \parencite{campbell2002deep,silver2017mastering,brown2018superhuman}. SP leads to policies that are highly adapted to each other and thus often appear artifical.
This is not a problem in two-player zero-sum games as all optimal policies are interchangeable \parencite{nash1951non}, at least when considering optimal opponents.

In fully cooperative MARL, however, such arbitrary conventions can be undesirable, as they fail when paired with agents that were not present during SP training. For instance, consider a situation in which robots must avoid collisions, by either swerving right or left or slowing down to avoid the other robot. Here, robots trained via SP would randomly learn to swerve either left or right and thus crash half the time at test time when paired in \emph{cross-play (XP)} with agents from independent training runs. Similarly, the arbitrary conventions learned by agents, e.g., in the card-game Hanabi, can prevent successful human-AI coordination \parencite[][]{foerster2019bayesian,carroll2019utility}.

\begin{figure}
    \centering
    \includegraphics[width=0.475\columnwidth]{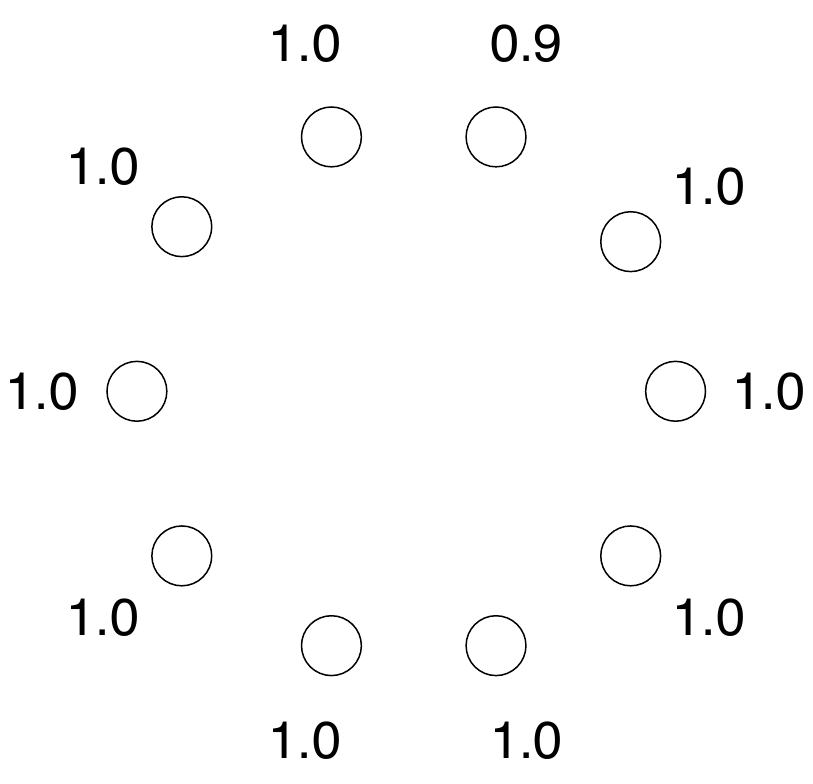}
    \caption{Rewards in the lever coordination game. Levers with equal rewards cannot be distinguished without labels.}
    \label{fig:1}
\end{figure}

This shortcoming of SP in fully cooperative problems motivates the study of the \emph{zero-shot coordination (ZSC) problem}, which \textcite{hu2020other} operationalize as finding a general-purpose learning algorithm that allows independently trained agents to coordinate successfully at test time. The independent training is a proxy for the independent decision making that humans have to undertake when solving coordination tasks, while the ability to agree on an algorithm corresponds to having a common high-level approach for solving these problems. 

More specifically, \textcite{hu2020other} assume that players only agree on a learning algorithm, but without sharing labels for observations, actions, and states in the environment. As an example, consider the lever coordination game in Figure~\ref{fig:1}. There are two agents, each having the choice between \(10\) different levers. If both agents choose the same lever, they receive rewards as specified in Figure~\ref{fig:1}. There is \(1\) lever with a reward of \(0.9\) and \(9\) levers with reward \(1\). If agents pull different levers, the reward is \(0\). Here, the SP algorithm will learn a joint policy wherein  one lever of the nine with a reward of \(1\) is played by both agents, but such a policy cannot be coordinated on without labels for levers.

\textcite{hu2020other} suggest the \emph{other-play} (OP) algorithm as a solution and give an informal optimality proof. The idea behind the algorithm is that it learns policies that are robust to permutations by symmetries of a given problem. For example, in the lever coordination game, the actions leading to a reward of \(1\) are all symmetric, so a randomly permuted policy will pick them with equal probability. In contrast, the lever leading to a payoff of \(0.9\) does not have symmetric counterparts, so it is possible to consistently choose that lever. \textcite{hu2020other} show that in Hanabi, OP improves performance over SP when playing with real humans, showing the benefits of ZSC to human-AI coordination. However, \textcite{hu2020other} do not formalize the ``no labels'' assumption, but rely on an intuitive notion in their proof. Moreover, \textcite{hu2020other}'s proof relies implicitly on the assumption that the  OP algorithm's objective has a unique maximizer.

\begin{figure}
    \centering
    \vspace{0pt}
    \includegraphics[width=0.95\columnwidth]{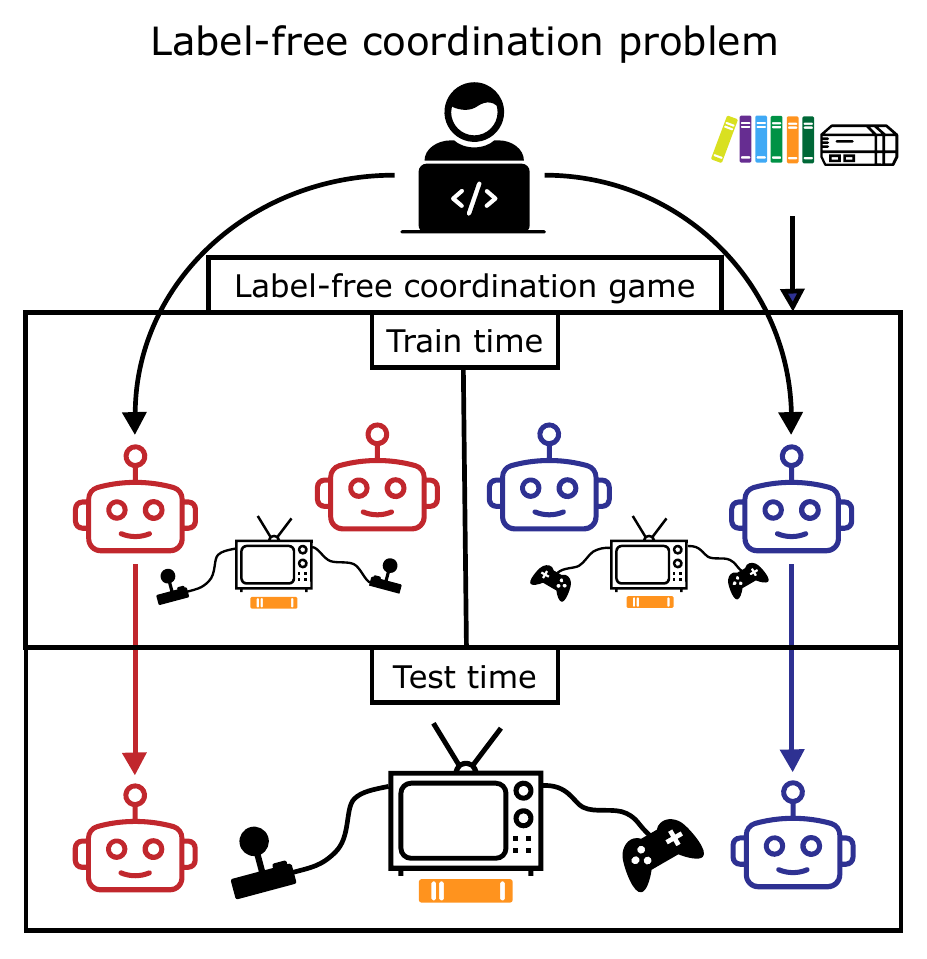}
    \vspace{0pt}
    \caption{Illustration of the LFC problem. A learning algorithm trains agents independently in a randomly chosen LFC game. The use of different controllers by red and blue agents symbolizes that, while the agents can separately coordinate on policies during train time, they do not know the labels used by agents of the other color and cannot coordinate with them before test time.}
 
    \label{fig:10}
\end{figure}

The first goal of this paper is a formalization of \textcite{hu2020other}'s problem setting and a theoretical analysis of OP (Section~\ref{formalism-and-analysis}). To do so, we introduce \emph{label-free coordination (LFC) games} and define the \emph{label-free coordination problem} as finding an optimal algorithm to recommend to players in a random LFC game (Section~\ref{section-the-zero-shot-coordination-game}). The formalization provides a rigorous optimality criterion that can be used to compare OP to other algorithms theoretically and to discuss whether \textcite{hu2020other}'s formulation is aligned with the goal of human-AI coordination. After introducing a generalized version of the OP algorithm (Section~\ref{section-other-play-generalization}), we show that OP can be suboptimal in the LFC problem, as \textcite{hu2020other}'s assumption of a unique maximizer is not always fulfilled (Section~\ref{section-other-play-not-optimal}).
Our findings suggest that the 8 point gap between the SP and XP scores for the vanilla version of OP in Hanabi \parencite{hu2020other} may be due to fundamental problem with the algorithm instead of an optimization issue.

Second, we fix this shortcoming by introducing an extension of the algorithm, \emph{OP with tie-breaking}, in which players use a tie-breaking function to consistently break ties between different OP-optimal policies (Section~\ref{section-other-play-with-tie-breaking}). We prove that this extension is an optimal solution to the LFC problem and that all players using the algorithm is a Nash equilibrium of any LFC game.

Third, we verify our results experimentally in two toy examples (Section~\ref{section-experiments}). Our examples are stylized coordination problems that abstractly model real-world coordination problems. 

Fourth, we argue that the operationalization of the ZSC problem by \textcite{hu2020other} does not reflect ZSC's aims, and we suggest a new operationalization as a starting point for future work (Section~\ref{section-open-issues}). Despite the ``no labels'' assumption, an optimal algorithm for the problem implements arbitrary tie-breaks. While there may be some settings where it is feasible to pre-coordinate on a tie-breaking function, in general, arbitrary tie-breaking is precisely what the ZSC setting aims to prevent. This shows that algorithmic advances towards \textcite{hu2020other} problem formulation are misaligned with the overarching goals of ZSC. We propose an improved informal operationalization in which players are allowed to coordinate only on high-level ideas for a learning algorithm but are prohibited from sharing implementation details such as random seeds, parameters, or code. We leave it to future work to refine and address this revised definition of ZSC.

To save space, we give informal statements and explanations of our theoretical results in the main text. A rigorous treatment of all results, including formal statements and proofs, can be found in the appendix. Our two main results are stated and proven in Appendices~\ref{appendix-proof-of-theorem-1} and \ref{appendix-proof-of-theorem-2}.

\section{Related work}
\label{related-work}

\label{equilibrium-selection-problems}


\paragraph{Game theory} A closely related problem to ZSC is the \emph{equilibrium selection problem} in game theory \parencite[see][]{harsanyi1988general}, which arises when there are different equilibria in a game. The ZSC problem arises when there is an equilibrium selection problem between different optimal policies in a fully cooperative game. Equilibrium selection problems can introduce the additional difficulty that players can have different preferences over the equilibria.

\citeauthor{harsanyi1988general} (\citeyear{harsanyi1988general}; see \citeauthor{harsanyi1975tracing}, \citeyear{harsanyi1975tracing}) introduce a general solution to the equilibrium selection problem in the framework of \emph{standard-form games}, and \textcite{herings2003equilibrium} have adapted it to stochastic games. One property of this solution is invariance to isomorphisms between games, which, like the LFC problem, is based on the idea that a solution should not depend on arbitrary labels \parencite[][ch.~3.4]{harsanyi1988general}. A difference to our setting is that we are interested in practical algorithms that can be run on large-scale games, for which computing \textcite{harsanyi1988general}'s solution would be infeasible \parencite[cf.][]{goldberg2013complexity,herings2002computation}. Another difference is that standard-form games have less structure than Dec-POMDPs. We explicate this and compare our approaches in more detail in Appendix~\ref{section-using-different-symmetry-concepts}.

Another game-theoretic approach to coordination problems is based on exogenous information about agents' options. For instance, consider the famous problem ``you lost your friend in New
York City, where are you going to meet?'' \parencite{schelling1980strategy,mehta1994nature}. Here, additional meaning is attached to each option, independent of dynamics and rewards of the problem, which allows for picking a unique option. Agents might also be able to choose options based on social conventions and norms, such as which side of the street to drive on \parencite{lewis2008convention}.
In this work, we instead restrict our attention to coordination-problem solutions based only on endogenous information present in the abstract structure of the problem. In many settings, conventions need to be introduced and adapted to within an episode at test time, rather than coordinating them beforehand, e.g., via joint training. 

\paragraph{Coordination without joint training} Some work looks at coordination problems in settings that do not assume  agents are trained together. For instance,
\textcite{boutilier1999sequential} introduces a dynamic programming algorithm for fully observable, fully cooperative stochastic games, where no prior coordination between agents is possible.
Agents randomize between different optimal actions in a given state until they succeed on coordinating on an optimal joint action. 
\textcite{goldman2007learning} consider the Dec-POMDP setting with a cheap-talk channel in which agents cannot pre-coordinate on strategies. They introduce an algorithm in which agents learn to interpret each others' messages and use them to communicate observations and to suggest actions to coordinate on. The idea behind ZSC is to learn joint policies that implement similarly robust strategies as the above approaches, without having to explicitly specify such behavior.

Another related approach is ad-hoc teamwork, wherein the goal is to train an agent to perform well in expectation when subbed into a randomly chosen team of agents \parencite{stone2010ad,barrett2011empirical,barrett2015cooperating}. Ultimately, this amounts to learning a best response to a distribution over team members. However, the team members’ policies may themselves be ill-suited for coordination, e.g., if they are obtained via SP. ZSC instead assumes that all agents are optimized for being able to coordinate well, in the absence of pre-established conventions (even though, as mentioned above, optimal ZSC policies in many settings introduce conventions \emph{within} an episode). It can thus find entirely different equilibria, ones that achieve good performance and can be consistently coordinated upon without SP training. As an example, consider the lever coordination game from the introduction. Optimal ad-hoc agents, trained as a best response to a population of SP agents would learn to uniformly randomize between the levers with a payoff of 1.0, while an optimal policy for ZSC always plays the unique lever with a payoff of 0.9.

Alternatively, an agent may be trained using data about other agents' behavior in order to learn a compatible strategy offline \parencite{lerer2019learning,tucker2020adversarially}. This again differs from ZSC in that it is concerned with learning a best response instead of optimizing all agents to find non-arbitrary equilibria. Similarly, human-AI coordination can be improved by training agents as a best response to a human model \parencite{carroll2019utility}. Human-AI coordination is also an aim of ZSC, but we try to uncover general principles behind a human-like learning algorithm instead of learning problem-specific policies from human models.



\section{Background}
\label{section-background}

\paragraph{Dec-POMDPs}
We consider decentralized partially observable Markov decision problems (Dec-POMDPs) \parencite{nair2003taming,oliehoek2016concise}. 
A (finite-horizon) Dec-POMDP \(\DecP\) is a tuple of 
a set of agents $\mathcal{N}=\{1,\dots,N\}$ where \(N\in\mathbb{N}\), a finite set of states $\StateSet$, a set of joint actions \(\ActionSet:=\prod_{i\in \PlayerSet}\ActionSet_{i}\), where $\ActionSet_{i}$ is a finite set of actions for agent $i\in\PlayerSet$, a transition probability kernel $P\colon\StateSet\times\ActionSet\rightarrow\Delta(\StateSet)$, where $\Delta(\StateSet)$ denotes the
set of probability distributions over $\StateSet$, a reward function $\RewardFunction\colon\StateSet\times\ActionSet\rightarrow\mathbb{R}$, a set of joint observations \(\ObservationSet:=\prod_{i\in \PlayerSet}\ObservationSet_{i}\), where $\ObservationSet_{i}$ is a finite set of observations for agent $i\in\PlayerSet$, an observation probability kernel $O\colon\StateSet\times\ActionSet\rightarrow\Delta(\ObservationSet)$, an initial state distribution $b_{0}\in\Delta(\StateSet)$, and a horizon $\Tmax\in\mathbb{N}_0$. 
We write \(\ActionSet^\DecP,\ActionSet^\DecPSecond\), etc.\ to indicate which Dec-POMDP \(\DecP\), \(\DecPSecond\), etc.\ a set or function belongs to. 

In a Dec-POMDP, at time step \(1\leq t\leq \Tmax\), the environment is in a state \(\StateRV_t\), agent \(i\in\PlayerSet\) receives observations \(\ObservationRV_{i,t}\) via \((\ObservationRV_{j,t})_{j\in\PlayerSet}\sim O(\cdot\mid\StateRV_{t},\ActionRV_{t-1})\) and
chooses an action \({\ActionRV_{i,t}\sim \Policy_i(\cdot\mid \AOHistoryRV_{i,t})}\) according to a \emph{local policy} \(\Policy_i\in\PolicySet_i\), where \(\AOHistoryRV_{i,t}:=(\ActionRV_{i,0},\ObservationRV_{i,1},\dotsc,\ActionRV_{i,t-1},\ObservationRV_{i,t})\) is a random variable for agent \(i\)'s action-observation history at step \(t\), with values \(\AOHistory_{i,t}\in\AOHistorySet_{i,t}:=\left(\ActionSet_{i}\times\ObservationSet_{i}\right)^t\).\footnote{In a slight abuse of notation, we use \(\ObservationRV\) for both observation probabilities and observation random variable.}
Agents receive a joint reward \(\RewardRV_t:=\RewardFunction(\StateRV_t,\ActionRV_t)\) and the environment transitions into a state \(\StateRV_{t+1}\sim P(\cdot\mid \StateRV_t,\ActionRV_t)\). The initial state is \(\StateRV_0\sim b_0\). We define the set of entire histories, containing tuples of all states, actions, rewards and observations until step \(\Tmax\) as \(\HistorySet\) and denote \(\HistoryRV\) as a random variable for the entire history.

Denote \(\Prob_\Policy\) for a probability measure on a space with the random variables defined above, where agents follow the \emph{(joint) policy} \(\Policy\in\PolicySet:=\prod_{i\in\PlayerSet}\PolicySet_i\), and let \(\E_\Policy\) be the expectation with respect to that measure. Given a Dec-POMDP \(\DecP\), 
the \emph{self-play (SP) objective} \(J^\DecP\colon \PolicySet^\DecP\rightarrow\mathbb{R}\) of \(\DecP\) is defined via
\(J^{\DecP}(\Policy):=\E_{\Policy}\left[\sum_{t=0}^{\Tmax}R_{t}\right]\)
for \(\Policy\in\PolicySet^\DecP\). Here, \(J^\DecP(\Policy)\) is called the expected return of the joint policy $\Policy$.


\paragraph{Zero-shot coordination and other-play} As explicated in the lever coordination problem, there can be different, incompatible SP-optimal joint policies. A SP algorithm tries to maximize the SP objective and will in general randomly learn any one of these policies. When two such independently trained joint policies \(\Policy^{(1)},\Policy^{(2)}\) are matched, this can yield bad \emph{XP values} \(J(\Policy_1^{(1)},\Policy_2^{(2)})\). 

To address this shortcoming, \textcite{hu2020other} introduce the ZSC problem. In spirit, the problem is to find a general-purpose learning algorithm for fully cooperative environments to train agents that are able to robustly coordinate with their teammates. It is assumed that teammates have also been optimized for ZSC, using a common high-level approach. However, arbitrarily co-adapting agents' policies, e.g., through joint training, is disallowed. \textcite{hu2020other} operationalize this as the problem of recommending one learning algorithm to players in a fully cooperative game. Each player trains a joint policy using the algorithm and discards all but one agent. The resulting agents from all players are then evaluated in XP over one episode. \textcite{hu2020other} assume players are able to coordinate on a common learning algorithm, but that they are unable to coordinate the learned policies based on common labels for the Dec-POMDP.

\textcite{hu2020other} propose the OP algorithm as a method for this setting. The algorithm's main idea is to train a joint policy to achieve high expected return when each local policy is randomly permuted to break symmetries in different ways. The hope is that this results in a unique joint policy, at the cost of a potentially suboptimal expected return. Informally, one can consider \emph{equivalence mappings} \(\EquivMapping\in\EquivSet\), which are maps that can be applied to actions, observations, and states, such that applying the map leaves the problem dynamics unchanged. Equivalence mappings can also be applied to a local policy \(\Policy_i\) to get a new policy \(\EquivMapping(\Policy_i)\). The OP objective \(J_\OP\) can then be defined via \(J_\OP(\Policy):=\E_{\EquivMapping\sim \U(\EquivSet)}[J(\EquivMapping(\Policy_1),\EquivMapping(\Policy_2))]\)
for a joint policy \(\Policy\in\PolicySet\),
where \(\U(\EquivSet)\) is a uniform distribution over \(\EquivSet\).\footnote{This is not \textcite{hu2020other}'s original definition, but it is equivalent, by \textcite{hu2020other}'s Proposition~2. The version presented here is the one we will generalize later.}
\textcite{hu2020other}'s definitions 
do not apply to Dec-POMDPs in which agents have different action or observation sets, and they do not account for symmetries between agents. We will formally define a more general version of OP in Section~\ref{section-other-play-generalization}.

\textcite[][]{hu2020other} provide an informal proof that both players using OP is an optimal equilibrium in the fully cooperative game described above. Using an appropriate formalization, we show in the next section that this is in general not correct. Contrary to \textcite{hu2020other}'s implicit assumption, there can be multiple, incompatible maximizers of the OP objective.

\section{Formalism and analysis of OP}
\label{formalism-and-analysis}

\subsection{Dec-POMDP isomorphisms}

\label{section-dec-pomdp-isomorphism}

To formalize the no-labels assumption, we introduce isomorphisms between Dec-POMDPs, which formalize the intuition that two Dec-POMDPs may represent the same problem using different labels. Our definition is a trivial generalization of \textcite{kang2012exploiting}'s  \emph{automorphisms} over partially observable stochastic games (POSGs)  to the concept of an \emph{isomorphism}, but restricted to fully cooperative problems. Analogous definitions of isomorphisms between games have been introduced before in different frameworks (e.g.,~\citeauthor{harsanyi1988general},~\citeyear{harsanyi1988general}, ch.~3.4; \citeauthor{peleg1999canonical},~\citeyear{peleg1999canonical}).

Let \(\DecP,\DecPSecond\) be two Dec-POMDPs. Consider a tuple of bijective maps
\(\Isom:=(\Isom_{N},\Isom_{S},(\Isom_{A_i})_{i\in\PlayerSet},(\Isom_{O_i})_{i\in\PlayerSet}),\) where 
\begin{alignat}{2}\label{eq:7}
&\Isom_{N}\colon &&\mathcal{N}^\DecP\rightarrow\PlayerSet^\DecPSecond\\
& \Isom_{S}\colon &&\StateSet^\DecP\rightarrow\StateSet^\DecPSecond\\
\forall i\in\PlayerSet\colon\quad&\Isom_{A_i}\colon &&\ActionSet_{i}^\DecP\rightarrow\ActionSet^\DecPSecond_{\Isom_N(i)}\\
\forall i \in\PlayerSet\colon\quad&\Isom_{O_i}\colon
&&\ObservationSet_{i}^\DecP\rightarrow\ObservationSet_{\Isom_N(i)}^\DecPSecond.\label{eq:8}
\end{alignat}
\jt{think about the following again}
Define a map \(\Isom_A\colon\ActionSet^\DecP\rightarrow\ActionSet^\DecPSecond\) via
\begin{equation}\Isom_A(\Action):=\left(\Isom_{A_{\Isom^{-1}_N(j)}}\left(\Action_{\Isom^{-1}_N(j)}\right)\right)_{j\in\PlayerSet^\DecPSecond},\end{equation}
for \(\Action\in\ActionSet^\DecP\),
and \(\Isom_O\) analogously for observations \(\Observation\in\ObservationSet^\DecP\). That is, in the joint action \(\Isom_A(\Action)\in\ActionSet^\DecPSecond\), agent \({j=\Isom_N(i)\in\PlayerSet^\DecPSecond}\) (where \(i\in\PlayerSet^\DecP\)) plays action \({\Isom_{A_i}(a_i)\in\ActionSet_j^\DecPSecond}\).

\begin{defn}[Dec-POMDP isomorphism]\label{dec-pomdp-isomorphism}
\label{def:Dec-POMDP-isomorphism} Let $\DecP,\DecPSecond$ be Dec-POMDPs such that both have the same horizon \({\Tmax^\DecP=\Tmax^\DecPSecond}\), and let \(\Isom\) be a tuple of bijective maps as defined in Equations~(\ref{eq:7})--(\ref{eq:8}). Then \(\Isom\) is an isomorphism from \(\DecP\) to \(\DecPSecond\) if
for any \(\Action\in\ActionSet^\DecP\), \(s,s'\in\StateSet^\DecP\), and \(o\in\ObservationSet^\DecP\), 
\begin{align}
P^\DecP(s'\mid s,a) &= P^\DecPSecond(f_{S}(s')\mid f_{S}(s),f_A(a))\label{eq:2}\\
O^\DecP(o\mid s,a) &= O^\DecPSecond(f_{O}(o)\mid f_{S}(s),f_A(a))\label{eq:3}\\
\RewardFunction^\DecP(s,a) &= \RewardFunction^\DecPSecond(f_{S}(s),f_A(a))\label{eq:4}\\
b^\DecP_{0}(s) &= b_{0}^\DecPSecond(f_{S}(s)).
\end{align}
We denote \(\Iso(\DecP,\DecPSecond)\) for the set of isomorphisms from \(\DecP\) to \(\DecPSecond\). If that set is non-empty, \(\DecP\) and \(\DecPSecond\) are called isomorphic.
\end{defn}
In the following, we adopt the convention to write \(fa\) instead of \(f_A(a)\) and \(f\Action_i\) instead of \(f_{A_i}\Action_i\), and we do the same for observations and states. We can also write \(f\AOHistory_{i,t}\) for \({\AOHistory_{i,t}\in\AOHistorySet_{i,t}}\), which is defined as the element-wise application of \(f\). Letting \(\Isom \Reward := \Reward\) for rewards, we can also define \(\Isom\History\) for entire histories \(\History\in\HistorySet\). One can show that this action of isomorphisms can be inverted by \(\Isom^{-1}\) (see Appendix~\ref{appendix-dec-pomdp-isomorphisms-and-automorphisms}).

A policy \(\Policy\in\PolicySet^\DecP\) can be transformed by an isomorphism \(\Isom\in\Iso(\DecP,\DecPSecond)\) into a policy for \(\DecPSecond\).
We call this operation the \emph{pushforward}, analogously, for instance, to the construction of pushforward measures, as precomposition of \(\Policy\) with \(\Isom^{-1}\) ``pushes'' the policy from one Dec-POMDP to another. The definition is analogous to that of applications of symmetries to policies in \textcite{hu2020other}.

\begin{defn}\label{definition-pushforward}
Let $\DecP,\DecPSecond$ be isomorphic Dec-POMDPs, let \(\Isom\in\Iso(\DecP,\DecPSecond)\), and let \(\Policy\in\PolicySet^\DecP\). 
Then we define the \emph{pushforward} \(\Isom^*\Policy\in\PolicySet^\DecPSecond\) of \(\Policy\) by \(\Isom\) via
\begin{equation}
(\Isom^{*}\Policy)_{j}(\Action_{j}\mid\AOHistory_{j,t}):=\Policy_{\Isom^{-1} j}(f^{-1}\Action_{j}\mid f^{-1}\AOHistory_{j,t})
\end{equation}
for all \(j\in\mathcal{N}^\DecPSecond,a_j\in\ActionSet_j^\DecPSecond,t\in\{0,\dotsc,\Tmax\},\) and \({\AOHistory_{j,t}\in\AOHistorySet_{j,t}^\DecPSecond}\). That is, in the joint policy \(\Isom^*\Policy\), agent \(j\in\PlayerSet^\DecPSecond\) gets assigned the local policy \(\Policy_i\) of agent \(i:=\Isom^{-1}j\in\PlayerSet^\DecP\), precomposed with \(\Isom^{-1}\).
\end{defn}
We show in Appendix~\ref{appendix-pushforward-policies} that \(\Policy\) and \(\Isom^*\Policy\) lead to the same expected return in their respective Dec-POMDPs.

For an example of an isomorphism and a pushforward policy, consider the \emph{two-stage lever game}, a stylized coordination problem like the lever coordination game, but with two rounds instead of one. We will use the example in Section~\ref{section-other-play-not-optimal} to show that OP is not optimal in the LFC problem.

\begin{example}[Two-stage lever game]
\label{example-two-stage-lever-game}


Consider the following variant of the lever coordination game, denoted by \(\DecP\). The problem has two agents,  \(\PlayerSet=\{1,2\}\), and rounds (\(\Tmax=1\)). Each round, each agent pulls a lever, \(\mathcal{A}_1=\mathcal{A}_2=\{1,2\}\). If both agents choose the same lever, the reward is \(1\), otherwise \(-1\). There are two observations, \(\ObservationSet_1=\ObservationSet_2=\{1,2\}\), and one state. In the second round (\(t=1\)), agents observe each other's previous action, so \(\ObservationRV_{i,1}=\ActionRV_{-i,0}\) for \(i=1,2\).

\begin{figure}
    \centering
    \begin{tikzpicture}
\node (init) {};
            \node[ text centered] (t1) {         \small\begin{tabular}{c|c}
        \(\Action\) & \(\RewardFunction^\DecP(\State,\Action)\)\\
        \hline
       (1,1)  & 1 \\
        (1,2) & -1 \\
       (2,1)  & -1 \\
        (2,2) & 1 
    \end{tabular}};
            \node[right=of t1, text centered] (t2) {
             \small\begin{tabular}{c|c}
        \(\Action\) & \(\RewardFunction^\DecPSecond(\State,\Action)\)\\
        \hline
       (1,1)  & -1 \\
        (1,2) & 1 \\
       (2,1)  & 1 \\
        (2,2) & -1 
    \end{tabular}
            };
            
            \path[draw, -latex'] (t1) to node [midway,above] {\(\Isom^{-1}\)} (t2); 

        \end{tikzpicture}

         \caption{Reward function \(\RewardFunction^\DecP\) in the two-stage lever game and reward function \(\RewardFunction^\DecPSecond=\RewardFunction^\DecP\circ\Isom^{-1}\) in an isomorphic problem.}
       
    \label{tab:4}
\end{figure}

Now consider an isomorphic problem \(\DecPSecond\) in which the labels for the actions of the second agent have been switched. A possible isomorphism \(\Isom\in\Iso(\DecP,\DecPSecond)\) is one consisting of identity maps, except for \(\Isom_{A_2}\), which switches the two actions of player \(2\). We give tables of rewards for both the original and the isomorphic problem in Figure~\ref{tab:4}. Applying \(\Isom^*\) to any policy \(\Policy\in\PolicySet^\DecP\) creates an equivalent policy \(\Isom^*\Policy\) for the Dec-POMDP \(\DecPSecond\), in which the actions of agent 2 in \(\DecP\) are replaced by the corresponding actions in \(\DecPSecond\).

\end{example}

\subsection{The LFC game and problem}
\label{section-the-zero-shot-coordination-game}

We begin by defining LFC games, which we then use to define the LFC problem. We formalize an LFC game as a fully cooperative game between \emph{principals} whose strategies are \emph{learning algorithms}, as defined below. The LFC game is defined for a specific ``ground truth'' problem \(\DecP\). The game's players, called principals, are the same as the agents in \(\DecP\). Each principal observes a randomly \emph{relabeled} but isomorphic Dec-POMDP and trains a joint policy on that problem using a learning algorithm. The policies are then pushed back to \(\DecP\) and evaluated in XP.

For a set of Dec-POMDPs \(\DecPSetSecond\), let \(\Delta(\PolicySet^\DecP)\) be a set of probability measures over \(\PolicySet^\DecP\) for \(\DecP\in\DecPSetSecond\). A learning algorithm for \(\DecPSetSecond\) is then defined as a map \(\LA\) that takes in Dec-POMDPs \(\DecPSecond\in\DecPSetSecond\) and outputs distributions 
\(\LA(\DecP)\in\Delta(\PolicySet^\DecP)\), and \(\LASet^\DecPSetSecond\) is defined as the set of learning algorithms for \(\DecPSetSecond\). For \(\Distr\in\Delta(\PolicySet^\DecP)\) and an isomorphism \(\Isom\in\Iso(\DecP,\DecPSecond)\), define the pushforward distribution \(\Isom^*\Distr:=\Distr \circ (\Isom^*)^{-1}\in\Delta(\PolicySet^\DecPSecond)\).

Now fix a Dec-POMDP \(\DecP\). For a given Dec-POMDP, we can create infinitely many different isomorphic problems, as we can use any set of labels, such as natural or real numbers, to define the problem. To describe the process of randomly sampling an isomorphic version of \(\DecP\), then, we restrict ourselves to a specific subset \(\DecPSet\) of isomorphic Dec-POMDPs in which the sets of states, actions, etc.\ are of the form \(\{1,2,\dotsc,k-1,k\}\subseteq\mathbb{N}\). \(\DecPSet\) is defined as the set of all relabeled Dec-POMDPs, i.e., all problems that are isomorphic to \(\DecP\) and have this canonical form (for a rigorous definition, see Appendix~\ref{appendix-dec-pomdp-labelings-and-relabeled-dec-pomdps}). One can interpret sampling from this set as principals coordinating on a canonical way to represent Dec-POMDPs, but each implementing the problem independently.

\begin{defn}[Label-free coordination game]
The \emph{Label-free coordination (LFC) game} for \(\DecP\) is defined as a game \(\Gamma^\DecP\) where the set of players (here called principals) is \(\PlayerSet^\DecP\), the set of strategies is \(\LASet^\DecPSet\), and the common payoff for the strategy profile \(\LAProfile_1,\dotsc,\LAProfile_N\in\LASet^\DecP\) is
    \begin{multline}U^\DecP(\LAProfile):=
    \E_{\DecP_i\sim \U(\DecPSet),\,\Isom_i\sim\U(\Iso(\DecP_i,\DecP)),\,i\in\PlayerSet}\Big[\\
    \E_{\Policy^{(j)}\sim\Isom_j^*\LAProfile_j(\DecP_j),\,j\in\mathcal{N}}
    \Big[
J^\DecP((\Policy^{(k)}_k)_{k\in\PlayerSet})\Big]\Big],\end{multline}
where \(\mathcal{U}(\DecPSet)\) is a uniform distribution over \(\DecPSet\).
\end{defn}
We show in Appendix~\ref{appendix-the-zero-shot-coordination-game} that the LFC games for two isomorphic Dec-POMDPs are equivalent, up to a potential permutation of the principals.

Turning to the LFC problem, the goal is to find a general learning algorithm to recommend to principals in any LFC game (see Figure~\ref{fig:10}). We hence formalize the problem here for a distribution over LFC games; however, since this does not change our theoretical analysis, we will only consider LFC problems for single LFC games afterwards. Let a set \(\DecPSetSecond\) of Dec-POMDPs be given, and denote \(\overline{\DecPSetSecond}:=\bigcup_{\DecPSecond\in\DecPSetSecond}\DecPSet^\DecPSecond\) where \(\DecPSet^\DecPSecond\) is the set of all relabeled problems of \(\DecPSecond\). The LFC problem for \(\DecPSetSecond\) is then defined as the problem of finding one learning algorithm \(\LA\in\LASet^{\overline{\DecPSetSecond}}\) to be used by principals in a randomly drawn game \(\Gamma^\DecPSecond\) for \(\DecPSecond\sim\U(\DecPSetSecond)\).

\begin{defn}[Label-free coordination problem]
We define the \emph{Label-free coordination (LFC) problem} for \(\DecPSetSecond\) as the optimization problem
\begin{equation}\max_{\LA\in\Sigma^{\overline{\DecPSetSecond}}}U^\DecPSetSecond(\LA)\end{equation}
where \(U^\DecPSetSecond(\LA):=\E_{\DecPSecond\sim\U(\DecPSetSecond)}\left[U^\DecPSecond(\LA,\dotsc,\LA)\right]\)
for
\(\LA\in\LASet^{\overline{\DecPSetSecond}}\). If \(\DecPSetSecond=\{\DecPSecond\}\), we refer to this as the LFC problem for \(\DecPSecond\) and write \(U^\DecPSecond(\LA):=U^{\{\DecPSecond\}}(\LA)=U^\DecPSecond(\LA,\dotsc,\LA)\).

\end{defn}

 \subsection{Generalization of OP}
\label{section-other-play-generalization}
\begin{figure*}
    \centering\begin{subfigure}{1\columnwidth}
            \includegraphics[width=1\textwidth]{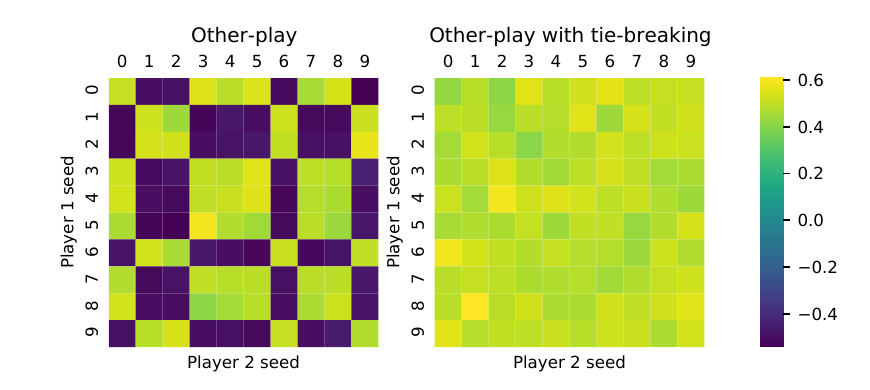}
    \caption{Two-stage lever game}
    \end{subfigure}
        \begin{subfigure}{1\columnwidth}
    \centering
    \includegraphics[width=1\textwidth]{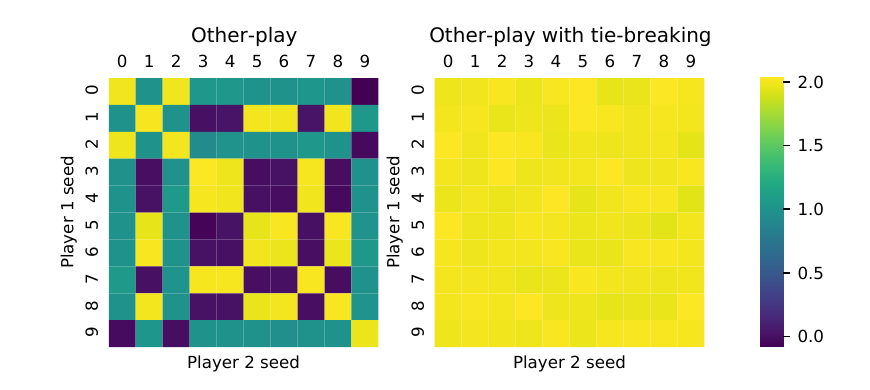}
        \caption{Asymmetric lever game}
    \end{subfigure}
    \caption{Heatmaps indicating XP values between policies from independent training runs. OP is on the left and OP with tie-breaking on the right, in which ties were broken between 32 different seeds. Each value has been averaged over 2048 episodes.}
    \label{fig:21}
\end{figure*}

Here, we introduce a generalized version of the OP algorithm by \textcite{hu2020other}. It is based on \emph{Dec-POMDP automorphisms}, which are isomorphisms $\Auto\in \Aut(\DecP):=\Iso(D,D)$ from a Dec-POMDP onto itself and can be thought of as representing symmetries of the problem. Our definition of an automorphism is equivalent to that of \textcite{kang2012exploiting} in the POSG framework\footnote{It is a straightforward consequence of \textcite{kang2012exploiting}'s results that the problem of finding Dec-POMDP automorphisms is graph isomorphism--complete.} and is a generalization of \textcite{hu2020other}'s equivalence mappings. Unlike the latter, automorphisms are defined for agents with potentially different action and observation sets, they can consist of different permutations for different agents, and they can incorporate permutations of the agents themselves.

As an example, consider an automorphism \(\Auto\in\Aut(\DecP)\) in the two-stage lever game. \(\Auto\) is an automorphism if precomposition with \(\Auto^{-1}\) does not change the reward function, observation probability kernel, etc. For instance, note that players are symmetric, so \(\Auto_N\) can be either the identity or it can switch both agents. In the latter case, applying \(\Auto^{-1}\) to joint actions switches the two players' actions, which does not have an effect on the relevant functions.

To define the OP objective, we use the pushforward by automorphisms, which is by definition a self-map on the set of joint policies \(\PolicySet^\DecP\). Similarly to equivalence mappings, we can randomly permute agents' local policies by different automorphisms, but we have to take into account potential permutations of agents. To that end, for a profile of automorphisms \(\AutProfile\in\Aut(\DecP)^\PlayerSet\), we define the joint policy \(\AutProfile^*\Policy:=\hat{\Policy}\), where the local policy \(\hat{\Policy}_i\) of agent \(i\in\PlayerSet\) is given by
\(\hat{\Policy}_i:= (\AutProfile^*_i\Policy)_i=
\Policy_{\AutProfile_i^{-1}i}(\AutProfile_i^{-1}\cdot\mid \AutProfile_i^{-1}\cdot)\).

\begin{defn}[Other-play]\label{defn-other-play-objective-general}For a Dec-POMDP \(\DecP\) and joint policy \(\Policy\in\PolicySet^\DecP\), define 
\begin{equation}
J^{D}_\OP\colon \Policy\mapsto\E_{\AutProfile\sim \U(\Aut(D)^\mathcal{N})}\left[J^{D}(\AutProfile^*\Policy)
\right].\end{equation}
We say that \(J^\DecP_\OP\) is the \emph{other-play (OP) objective} of \(\DecP\), and \(J^\DecP_\OP(\Policy)\) is the \emph{OP value} of \(\Policy\in\PolicySet^\DecP\).
Given a set of Dec-POMDPs \(\DecPSetSecond\), we define an OP learning algorithm as any learning algorithm \(\LA^\OP\in\Sigma^\DecPSetSecond\) such that \(\E_{\Policy\sim\LA^\OP(\DecP)}[J_\OP^\DecP(\Policy)]=\max_{\Policy\in\PolicySet^\DecP}J^\DecP_\OP(\Policy)\) for all \(\DecP\in\DecPSetSecond\).
\end{defn}
\textcite{hu2020other} show that their objective can be maximized in practice by considering a modified Dec-POMDP and applying any learning algorithm for Dec-POMDPs \parencite[e.g.,][]{sunehag2018value} to that problem. We show in Appendix~\ref{appendix-no-optimal-deterministic-policy} that this does not work for our objective, as optimal policies may need to be stochastic, while in Dec-POMDPs there always exist optimal deterministic policies \parencite[][sec.~2.4.4]{oliehoek2008optimal}. However, we can still apply a vanilla policy gradient method (see Appendix~\ref{appendix-op-implementation}). %

 \subsection{OP is not optimal in the LFC problem}
\label{section-other-play-not-optimal}

An OP learning algorithm may learn different, potentially incompatible OP-optimal policies in independent training runs. Hence, if the algorithm does not only learn compatible policies, it can be suboptimal in the LFC problem.

To see this, consider the two-stage lever game. In a simple game with one round, such as the lever coordination game, applying symmetries helps avoid arbitrary coordination on one lever. However, this changes in a game with two rounds. Since symmetries are not applied independently to the rounds, but they always apply to the whole episode, agents are able to coordinate in the second round if they coordinated by chance in the first round. This is advantageous for getting a higher return, but unfortunately, there are two different ways to coordinate in the second round.

Consider the two policies \(\Policy^R\) and \(\Policy^S\). In both policies, agents randomize uniformly between both levers in the first round. They also both randomize in the second round if coordination was unsuccessful in the first one. If coordination in the first round was successful, there are two different strategies: in \(\Policy^R\), both agents repeat their respective actions from round one. In \(\Policy^S\), both agents switch to the action they did not play in round one, which is unique, given there are only two levers. We show in Appendix~\ref{appendix-two-incompatible} that these policies are both optimal under OP.

Now suppose one agent chooses a local policy from \(\Policy^R\) and the other chooses a local policy from \(\Policy^S\). It is clear that this will yield a suboptimal expected return compared with \(\Policy^R\) or \(\Policy^S\) as agents will always fail to coordinate in the second round if they coordinated in the first round. Thus, in the LFC problem for the two-stage lever game, if a learning algorithm is not concentrated on only one of \(\Policy^S\) or \(\Policy^R\), but instead learns both policies (or potentially equivalent policies in relabeled problems), then that algorithm is suboptimal. We hence have the following result:

\begin{thm}[Informal]\label{thm-op-not-optimal-informal}
Any learning algorithm that learns both \(\Policy^R\) and \(\Policy^S\) in the two-stage lever game is an OP learning algorithm, but it is not optimal in the LFC problem for that game.
\end{thm}

\co{To save space, I usually title my appendix sections, ``Proof of such and such'', plus I mention in the introduction that formal details and proofs are given in the appendix and then I omit the ``Proof see appendix such and such'' lines, because these take significant space. Only semi-relatedly, I like using the restatable environment for repeating results in the appendix, see \url{https://tex.stackexchange.com/questions/51286/recalling-a-theorem}. (Though because you here give informal versions of results in the main text it's not as clear that this is useful.)}

\section{OP with tie-breaking}
\label{section-other-play-with-tie-breaking}

To fix OP's shortcoming outlined above, we introduce \emph{OP with tie-breaking}, which is based on the notion of a tie-breaking function that uniquely ranks the different OP-optimal policies in a given problem and thus allows for consistently choosing among them. A tie-breaking function could, for instance, compare the two incompatible policies, \(\Policy^R\) and \(\Policy^S\), in the two-stage lever game and choose the one under which actions are more highly correlated, which is \(\Policy^R\). A tie-breaking function has to satisfy certain conditions, e.g., it always must have a unique maximizer, and it must choose equivalent policies in isomorphic Dec-POMDPs. We define OP with tie-breaking as an algorithm that chooses an OP-optimal policy that maximizes a tie-breaking function (see Appendix~\ref{appendix-definition-op-tie-breaking}). We then have the following result:

\begin{thm}[Informal]\label{thm-op-with-tie-breaking-informal}
OP with tie-breaking is optimal in the LFC problem, and all principals using the algorithm is a Nash equilibrium of any LFC game.
In particular, the optimal value in the LFC problem for any Dec-POMDP \(\DecP\) is equal to the OP value of any OP-optimal policy, i.e.,
\(\max_{\LA\in\LASet^\DecPSet}U^{\DecP}(\LA)=\max_{\Policy}J_\OP^\DecP(\Policy)\).
\end{thm}

In practice, we can implement OP with tie-breaking by sampling, for a given Dec-POMDP \(\DecP\), \(K\in\mathbb{N}\) policies using an OP algorithm \(\LA^\OP\) and choosing the policy with the highest tie-breaking value. To compute tie-breaking values, we use a neural network, randomly initialized using a fixed random seed, to map histories deterministically to real numbers. We call these numbers ``hash values'' in analogy to the hash functions used in many areas of computer science to assign unique keys. The joint policy's tie-breaking value is then calculated as the expected hash value of histories under that policy. A few additional operations (randomly permuting policies, computing normal forms of histories, and summing over agent permutations) are required to ensure this works independently of labels.

Concretely, the tie-breaking function is computed as a Monte Carlo estimate of
\begin{equation}\label{eq:tie-breaking-function}
\frac{1}{N!}\sum_{\Isom_N\in\mathrm{Bij}(\PlayerSet)}\E_{\AutProfile\sim\U(\Aut(\DecP)^\PlayerSet)}\left[\E_{\AutProfile^*\Policy}\left[\Hash(\Isom_N(\iota(\HistoryRV)))\right]\right],
\end{equation}
where \(\Hash\) is the neural network, \(\mathrm{Bij}(\PlayerSet)\) is the set of permutations of \(\PlayerSet\), \(\iota(\History)\) is a \emph{normal form} of the history \(\History\), and for \(\iota(\History):=(\State_0,(\Action_{i,0})_{i\in\PlayerSet},\Reward_0,\dotsc)\), we define
\(\Isom_N(\iota(\History)):=(\State_0,(\Action_{\Isom_N^{-1}i,0})_{i\in\PlayerSet},\Reward_0,\dotsc)\). The normal form \(\iota(\History)\) is computed by replacing the first occurrence of each state, action, or observation in \(\History\) by a \(0\), the second occurrence by a \(1\), and so on, and repeating the number if an element repeats itself in the history. Together with the summation over permutations of agents, this achieves that
\(\iota(\History)\) does not depend on particular labels for agents, states, etc. Moreover, consistent tie-breaking is only possible between policies that are randomly permuted by applications of \(\AutProfile\). 
We prove in Appendix~\ref{appendix-random-tie-breaking-functions} that, using a suitable random function \(\Hash\), a modification of the tie-breaking function described here satisfies our formal requirements almost surely.

\begin{figure}
    \centering
    \begin{subfigure}{0.475\columnwidth}
    \centering
    \small
    \includegraphics[width=1\textwidth]{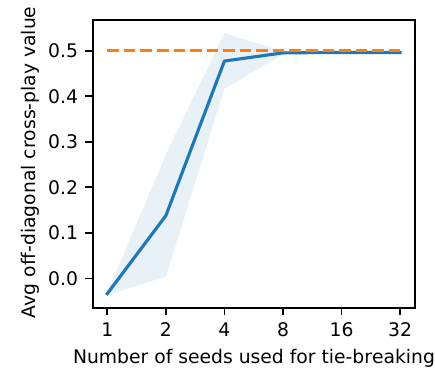}
    \caption{Two-stage lever game}
    \end{subfigure}
    \begin{subfigure}{0.475\columnwidth}
    \centering
    \includegraphics[width=1\textwidth]{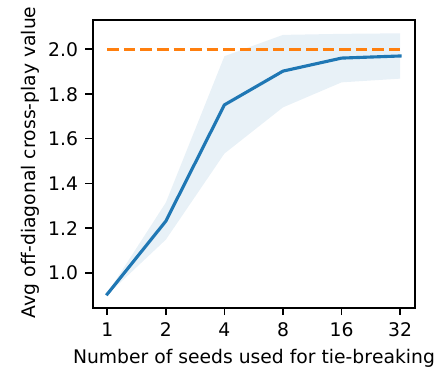}
    \caption{Asymmetric lever game}
    \end{subfigure}
    \caption{Plot of average off-diagonal XP value of the tie-breaking method, using different numbers of seeds for tie-breaking (using one seed is equivalent to OP and represents the baseline). The dashed orange line indicates the theoretical optimum. The shaded area indicates standard deviations across 20 different seeds used for the hash function.}
    \vspace{-0pt}
    \label{fig:avg-off-diagonal-XP}
\end{figure}

\section{Experiments}
\label{section-experiments}

\label{section-training-and-XP-evaluation}

\paragraph{MARL training and XP evaluation}
We use a vanilla policy gradient algorithm (\citeauthor{nguyen2017policy}, \citeyear{nguyen2017policy}; \citeauthor{williams1992simple}, \citeyear{williams1992simple}) to train recurrent neural network policies on the OP objective. We use a randomly initialized feedforward neural network as a hash function and implement a tie-breaking function as described in Section~\ref{section-other-play-with-tie-breaking}. To implement OP with tie-breaking and study the dependency of its performance on the number of policies used for tie-breaking, we apply the tie-breaking function to \(K\) learned policies and choose the one with maximal value, for \(K\in\{2,4,8,16,32\}\).

To evaluate a given learning algorithm \(\LA\), we simplify the objective of the LFC problem. Instead of using relabeled Dec-POMDPs, we evaluate policies from independent training runs on the same Dec-POMDP, permuted by random automorphisms. That is, we estimate
\begin{multline}
\E_{\Policy^{(i)}\sim\LA(\DecP),\,i=1,2}
    \Big[
    \E_{\AutProfile_i\sim \U(\Aut(\DecP)),\,i=1,2}\Big[
    \\
J^\DecP((\AutProfile_1^*\Policy^{(1)})_1,(\AutProfile_2^*\Policy^{(2)})_2)\Big]\Big]
\end{multline}
by computing average off-diagonal XP values between independently trained, permuted policies. We believe that this is a realistic estimate of the objective, as the distributions over policies produced by our learning algorithms should not depend on the used labels (see Appendix~\ref{appendix-equivariant-learning-algorithms}).

In both examples, we train 320 policies using different random seeds, and partition them into 10 sets of 32 policies, where the 32 policies can be used for tie-breaking and each set corresponds to one run to be used for computing XP values. We apply OP with tie-breaking using 20 different random seeds for the hash function to explore to what degree the quality of the tie-breaking function depends on the random initialization of the hash network. Additional experimental details and results are described in Appendix~\ref{appendix-further-training-details}. Code for our experiments can be found at \url{https://github.com/johannestreutlein/op-tie-breaking}.

\begin{table}
        \caption{Rewards in the asymmetric lever game, for \(t=0,1\) on the left, and for \(t=2\) on the right.}
            \vskip 0.15in
    \label{tab:rewards-example-3}
\centering
\small
\begin{subtable}{0.49\columnwidth}\centering
    \begin{tabular}{c|c|c|c}
         &  \(a_{2,1}\) & \(a_{2,2}\) & \(a_{2,3}\)\\
     \hline
      \(a_{1,1}\)& \(1\) & \(-1\) & \(-1\)\\
      \hline
      \(a_{1,2}\)&\(-1\) &  \(1\) & \(-1\)\\
      \hline
      \(a_{1,3}\) & \(-1\) & \(-1\) & \(-1\)
    \end{tabular}
    \end{subtable}
    \begin{subtable}{0.49\columnwidth}\centering
    \begin{tabular}{c|c|c|c}
         &  \(a_{2,1}\) & \(a_{2,2}\) & \(a_{2,3}\)\\
     \hline
      \(a_{1,1}\)&  \(0\) & \(0\) & \(0\)\\
     \hline
      \(a_{1,2}\) &   \(0\) & \(0\) & \(0\)\\
      \hline
      \(a_{1,3}\) &  \(1\) & \(1\) & \(1\)
    \end{tabular}
    \end{subtable}
\end{table}

\paragraph{Environments} As environments, we implement the two-stage lever game introduced in Example~\ref{example-two-stage-lever-game}, as well as the \emph{asymmetric lever game}. In the asymmetric lever game, there are two agents \(i=1,2\), which can pull one of three levers \(\{a_{i,1},a_{i,2},a_{i,3}\}\). There are three states \(\{0,1,2\}\) representing three rounds of this game (i.e., \(P(s+1\mid s,a)=1\) for \(s\in\{0,1\}\) and \(\Tmax=2\)). As before, both agents observe the previous action of the other agent, so \(O_{-i,t}=A_{i,t-1}\) for \(i=1,2\), \(t=1,2\). The reward function is given in Table~\ref{tab:rewards-example-3}. The first agent has an extra task in the third round, which makes the agents asymmetric. We choose this example since, unlike the two-stage lever game, it is one where no OP-optimal policy is an intuitively sensible solution to ZSC.

\paragraph{Results} All learned policies are close to optimal under the OP objective. In the two-stage lever game, both OP-optimal policies are learned equally well, while there is one dominant policy in the asymmetric lever game. In XP evaluation, OP with tie-breaking outperforms OP, achieving close to optimal XP values in both games.

We display XP matrices, which indicate XP values for any matching of two agents from 10 independent runs, in Figure~\ref{fig:21}. OP learns incompatible policies in different training runs, whereas policies chosen by OP with tie-breaking appear to be compatible. Average off-diagonal XP values for different numbers of policies used for tie-breaking are plotted in Figure~\ref{fig:avg-off-diagonal-XP}. Here, using only one policy for tie-breaking is equivalent to OP, as there is only one policy to choose from. 
We divide the policies into classes of policies with a high mutual XP score and show a histogram of the tie-breaking values for each class in Figure~\ref{fig:4}. 


\begin{figure}
\centering
    \centering
    \begin{subfigure}{0.475\columnwidth}
    \centering
    \small
    \includegraphics[width=1\textwidth]{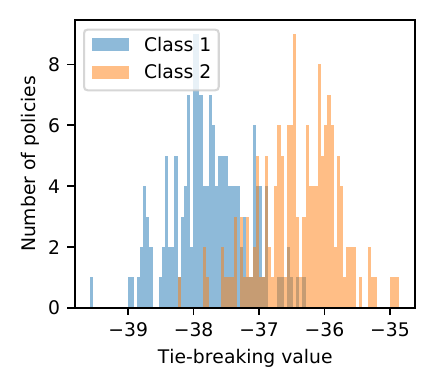}

    \caption{Two-stage lever game}
    \end{subfigure}
    \begin{subfigure}{0.475\columnwidth}
    \centering
    \includegraphics[width=1\textwidth]{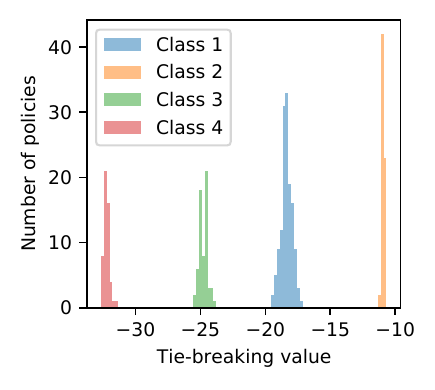}

    \caption{Asymmetric lever game}
    \end{subfigure}
    \vspace{-0pt}
    \caption{Histogram of the tie-breaking values of the learned policies, categorized into classes with mutually high XP values.}

    \label{fig:4}
\end{figure}

\section{Open issues}
\label{section-open-issues}

Tie-breaking may be a feasible solution in some coordination problems. For instance, if a central authority can recommend an arbitrary tie-breaking function to all principals, using OP with tie-breaking may provide an easy way to coordinate over a range of different problems, even if joint training is impossible and common labels for problems are unavailable. Moreover, there could be tie-breaking functions based on natural biases that are shared between all principals, such as a simplicity bias. In the two-stage lever game, such a bias could be used to justify the strategy in which agents both repeat their action if they succeeded in coordinating in the first round.

However, our work shows that the current operationalization of ZSC needs to be revised. This is because OP with tie-breaking, an optimal solution to the LFC problem, is ultimately unsatisfactory as a solution to ZSC. First, OP with tie-breaking allows for arbitrary tie-breaking functions. This goes against the spirit of the ZSC problem, which prohibits arbitrary co-adaptation of policies and allows only tie-breaking functions based on plausible, non-arbitrary meta conventions, such as repeating actions that have previously been coordinated on.

Second, sometimes there may be no plausible tie-breaking function; instead, an entirely different policy should be learned. For instance, in the asymmetric lever game, OP-optimal policies choose one agent to switch to a different lever if coordination failed in the first round. But this choice of agent appears arbitrary. Hence, it would be preferable to learn a policy that randomizes if the players could not coordinate in the first round, similar to the policies in the two-stage lever game. Since such a policy is not learned by OP, no appropriate policy can be chosen by OP with tie-breaking. In fact, no optimal solution to the LFC problem would be to able learn this policy, as doing so would lead to lower performance under the LFC problem's objective.

Our results imply that an operationalization of ZSC should preclude principals from sharing not only labels but also any other implementation details. This suggests the following improved setting: principals coordinate on high-level ideas for learning algorithms, but they cannot coordinate on specific implementation details, such as random seeds, parameters, or code. Each principal then implements their algorithm independently and trains an agent on a given randomly sampled environment. As prior coordination between principals is restricted and the algorithm must work in a range of environments, it can only rely on general high-level principles for coordination, not on arbitrary tie-breaking.

Unfortunately, the question of what counts as an implementation detail versus a high-level idea for an algorithm is vague, and thus, unlike the LFC problem, this operationalization does not have a straightforward formalization. Nevertheless, it better suits ZSC's spirit and thus serves as an improved problem setting. We leave it to future work to address and refine this new operationalization.

\section{Conclusion and future work}
\label{section-conclusion-future-work}

We formalized  \textcite{hu2020other}'s operationalization of ZSC as the LFC problem, showed that OP is not optimal in the problem, and introduced an extension, OP with tie-breaking, that is optimal. We supported our theoretical results experimentally in two toy examples.
Lastly, based on our findings, we concluded that the LFC problem is misaligned with ZSC's aims and suggested a revised intuitive operationalization of ZSC.


More work is required to devise formalisms and algorithms that suit this revised operationalization. One avenue may be different symmetry concepts \parencite[e.g.,][]{harsanyi1988general,neumann1947theory,nash1951non,peleg1999canonical,Casajus2001} using weaker notions of equivalence. Unfortunately, as we show in Appendix~\ref{section-using-different-symmetry-concepts}, considering Dec-POMDPs as standard-form games and applying symmetries in that formalism leads to too little possible coordination between agents. This raises the question whether there is a ``Goldilocks'' concept obviating the need for tie-breaks while allowing for maximal coordination.

\section*{Acknowledgements}
We are grateful to Leon Lang and Lennart Stern for help with proofs and writing, and Christian Schroeder de Witt for assistance with the PyMARL code repository. Johannes Treutlein would like to thank his readers Prof.~Gitta Kutyniok and Prof.~Reinhold Schneider at Technical University of Berlin, where he submitted a version of this paper as a BSc thesis. During his work on this paper, Johannes Treutlein was supported by Open Philanthropy, the Berkeley Existential Risk Initiative, and the Center on Long-Term Risk. Caspar Oesterheld is thankful for support by the National Science Foundation under Award IIS-1814056.

\bibliography{citations}

\begin{thebibliography}{50}
\providecommand{\natexlab}[1]{#1}
\providecommand{\url}[1]{\texttt{#1}}
\expandafter\ifx\csname urlstyle\endcsname\relax
  \providecommand{\doi}[1]{doi: #1}\else
  \providecommand{\doi}{doi: \begingroup \urlstyle{rm}\Url}\fi

\bibitem[Alon \& Spencer(2016)Alon and Spencer]{alon2016probabilistic}
Alon, N. and Spencer, J.~H.
\newblock \emph{The probabilistic method}.
\newblock John Wiley \& Sons, Hoboken, NJ, 4th edition, 2016.

\bibitem[Barrett \& Stone(2015)Barrett and Stone]{barrett2015cooperating}
Barrett, S. and Stone, P.
\newblock Cooperating with unknown teammates in complex domains: A robot soccer
  case study of ad hoc teamwork.
\newblock In \emph{Proceedings of the AAAI Conference on Artificial
  Intelligence}, volume~29, 2015.

\bibitem[Barrett et~al.(2011)Barrett, Stone, and Kraus]{barrett2011empirical}
Barrett, S., Stone, P., and Kraus, S.
\newblock Empirical evaluation of ad hoc teamwork in the pursuit domain.
\newblock In \emph{AAMAS}, pp.\  567--574, 2011.

\bibitem[Boutilier(1999)]{boutilier1999sequential}
Boutilier, C.
\newblock Sequential optimality and coordination in multiagent systems.
\newblock In \emph{Proceedings of the 16th International Joint Conference on
  Artificial Intelligence - Volume 1}, pp.\  478–485, 1999.

\bibitem[Brown \& Sandholm(2018)Brown and Sandholm]{brown2018superhuman}
Brown, N. and Sandholm, T.
\newblock Superhuman {AI} for heads-up no-limit poker: Libratus beats top
  professionals.
\newblock \emph{Science}, 359\penalty0 (6374):\penalty0 418--424, 2018.

\bibitem[Campbell et~al.(2002)Campbell, Hoane~Jr, and Hsu]{campbell2002deep}
Campbell, M., Hoane~Jr, A.~J., and Hsu, F.-h.
\newblock Deep {B}lue.
\newblock \emph{Artificial Intelligence}, 134\penalty0 (1-2):\penalty0 57--83,
  2002.

\bibitem[Carroll et~al.(2019)Carroll, Shah, Ho, Griffiths, Seshia, Abbeel, and
  Dragan]{carroll2019utility}
Carroll, M., Shah, R., Ho, M.~K., Griffiths, T., Seshia, S., Abbeel, P., and
  Dragan, A.
\newblock On the utility of learning about humans for human-{AI} coordination.
\newblock In \emph{Advances in Neural Information Processing Systems},
  volume~32, 2019.

\bibitem[Casajus(2001)]{Casajus2001}
Casajus, A.
\newblock Weak isomorphisms of extensive games.
\newblock In \emph{Focal Points in Framed Games: Breaking the Symmetry}, pp.\
  55--90. Springer Berlin Heidelberg, Berlin, Heidelberg, 2001.

\bibitem[Emmons et~al.(2021)Emmons, Oesterheld, Critch, Conitzer, and
  Russell]{emmons2021symmetry}
Emmons, S., Oesterheld, C., Critch, A., Conitzer, V., and Russell, S.
\newblock {S}ymmetry, equilibria, and robustness in common-payoff games.
\newblock In \emph{3rd Games, Agents, and Incentives Workshop (GAIW 2021). Held
  as part of the Workshops at the 20th International Conference on Autonomous
  Agents and Multiagent Systems.}, May 2021.

\bibitem[Foerster et~al.(2018)Foerster, Farquhar, Afouras, Nardelli, and
  Whiteson]{foerster2017counterfactual}
Foerster, J., Farquhar, G., Afouras, T., Nardelli, N., and Whiteson, S.
\newblock Counterfactual multi-agent policy gradients.
\newblock In \emph{AAAI Conference on Artificial Intelligence}, 2018.

\bibitem[Foerster et~al.(2019)Foerster, Song, Hughes, Burch, Dunning, Whiteson,
  Botvinick, and Bowling]{foerster2019bayesian}
Foerster, J., Song, F., Hughes, E., Burch, N., Dunning, I., Whiteson, S.,
  Botvinick, M., and Bowling, M.
\newblock {B}ayesian action decoder for deep multi-agent reinforcement
  learning.
\newblock In \emph{Proceedings of the 36th International Conference on Machine
  Learning}, volume~97 of \emph{Proceedings of Machine Learning Research}, pp.\
   1942--1951, 2019.

\bibitem[Gibbons(1992)]{gibbons1992game}
Gibbons, R.~S.
\newblock \emph{Game theory for applied economists}.
\newblock Princeton University Press, Princeton, NJ, 1992.

\bibitem[Glicksberg(1952)]{glicksberg1952generalization}
Glicksberg, I.~L.
\newblock A further generalization of the {K}akutani fixed point theorem, with
  application to {N}ash equilibrium points.
\newblock \emph{Proceedings of the American Mathematical Society}, 3\penalty0
  (1):\penalty0 170--174, 1952.

\bibitem[Goldberg et~al.(2013)Goldberg, Papadimitriou, and
  Savani]{goldberg2013complexity}
Goldberg, P.~W., Papadimitriou, C.~H., and Savani, R.
\newblock The complexity of the homotopy method, equilibrium selection, and
  {L}emke-{H}owson solutions.
\newblock \emph{ACM Transactions on Economics and Computation}, 1\penalty0
  (2):\penalty0 1--25, 2013.

\bibitem[Goldman et~al.(2007)Goldman, Allen, and
  Zilberstein]{goldman2007learning}
Goldman, C.~V., Allen, M., and Zilberstein, S.
\newblock Learning to communicate in a decentralized environment.
\newblock \emph{Autonomous Agents and Multi-Agent Systems}, 15\penalty0
  (1):\penalty0 47--90, 2007.

\bibitem[Harsanyi(1975)]{harsanyi1975tracing}
Harsanyi, J.~C.
\newblock The tracing procedure: a {B}ayesian approach to defining a solution
  for n-person noncooperative games.
\newblock \emph{International Journal of Game Theory}, 4\penalty0 (2):\penalty0
  61--94, 1975.

\bibitem[Harsanyi \& Selten(1988)Harsanyi and Selten]{harsanyi1988general}
Harsanyi, J.~C. and Selten, R.
\newblock \emph{A general theory of equilibrium selection in games}.
\newblock The MIT Press, Cambridge, MA, 1988.

\bibitem[Herings \& Peeters(2003)Herings and Peeters]{herings2003equilibrium}
Herings, P. J.-J. and Peeters, R.~J.
\newblock Equilibrium selection in stochastic games.
\newblock \emph{International Game Theory Review}, 5\penalty0 (4):\penalty0
  307--326, 2003.

\bibitem[Herings \& Van Den~Elzen(2002)Herings and Van
  Den~Elzen]{herings2002computation}
Herings, P. J.-J. and Van Den~Elzen, A.
\newblock Computation of the {N}ash equilibrium selected by the tracing
  procedure in n-person games.
\newblock \emph{Games and Economic Behavior}, 38\penalty0 (1):\penalty0
  89--117, 2002.

\bibitem[Hu et~al.(2020)Hu, Lerer, Peysakhovich, and Foerster]{hu2020other}
Hu, H., Lerer, A., Peysakhovich, A., and Foerster, J.
\newblock “{O}ther-play” for zero-shot coordination.
\newblock In \emph{Proceedings of the 37th International Conference on Machine
  Learning}, volume 119 of \emph{Proceedings of Machine Learning Research},
  pp.\  4399--4410, 2020.

\bibitem[Kang \& Kim(2012)Kang and Kim]{kang2012exploiting}
Kang, B.~K. and Kim, K.-E.
\newblock Exploiting symmetries for single-and multi-agent partially observable
  stochastic domains.
\newblock \emph{Artificial Intelligence}, 182:\penalty0 32--57, 2012.

\bibitem[Kuhn(1953)]{kuhn1953contributions}
Kuhn, H.~W.
\newblock Extensive games and the problem of information.
\newblock In \emph{Contributions to the Theory of Games: Volume II}, volume~28
  of \emph{Annals of Mathematics Studies}, pp.\  193--216. Princeton University
  Press, Princeton, NJ, 1953.

\bibitem[Lerer \& Peysakhovich(2019)Lerer and Peysakhovich]{lerer2019learning}
Lerer, A. and Peysakhovich, A.
\newblock Learning existing social conventions via observationally augmented
  self-play.
\newblock In \emph{Proceedings of the 2019 AAAI/ACM Conference on AI, Ethics,
  and Society}, pp.\  107--114, 2019.

\bibitem[Lewis(2008)]{lewis2008convention}
Lewis, D.
\newblock \emph{Convention: A philosophical study}.
\newblock John Wiley \& Sons, New York, NY, 2008.

\bibitem[Maschler et~al.(2013)Maschler, Solan, and
  Zamir]{maschler_solan_zamir_2013}
Maschler, M., Solan, E., and Zamir, S.
\newblock Behavior strategies and {K}uhn's theorem.
\newblock In \emph{Game Theory}, pp.\  219–250. Cambridge University Press,
  Cambridge, UK, 2013.

\bibitem[Mehta et~al.(1994)Mehta, Starmer, and Sugden]{mehta1994nature}
Mehta, J., Starmer, C., and Sugden, R.
\newblock The nature of salience: An experimental investigation of pure
  coordination games.
\newblock \emph{The American Economic Review}, 84\penalty0 (3):\penalty0
  658--673, 1994.

\bibitem[Mnih et~al.(2016)Mnih, Badia, Mirza, Graves, Lillicrap, Harley,
  Silver, and Kavukcuoglu]{mnih2016asynchronous}
Mnih, V., Badia, A.~P., Mirza, M., Graves, A., Lillicrap, T., Harley, T.,
  Silver, D., and Kavukcuoglu, K.
\newblock Asynchronous methods for deep reinforcement learning.
\newblock In \emph{Proceedings of The 33rd International Conference on Machine
  Learning}, volume~48 of \emph{Proceedings of Machine Learning Research}, pp.\
   1928--1937, 2016.

\bibitem[Nair et~al.(2003)Nair, Tambe, Yokoo, Pynadath, and
  Marsella]{nair2003taming}
Nair, R., Tambe, M., Yokoo, M., Pynadath, D., and Marsella, S.
\newblock Taming decentralized {POMDPs}: Towards efficient policy computation
  for multiagent settings.
\newblock In \emph{IJCAI'03: Proceedings of the 18th International Joint
  Conference on Artificial intelligence}, pp.\  705--711, 2003.

\bibitem[Nash(1951)]{nash1951non}
Nash, J.~F.
\newblock Non-cooperative games.
\newblock \emph{Annals of Mathematics}, 54\penalty0 (2):\penalty0 286--295,
  1951.

\bibitem[Nguyen et~al.(2017)Nguyen, Kumar, and Lau]{nguyen2017policy}
Nguyen, D.~T., Kumar, A., and Lau, H.~C.
\newblock Policy gradient with value function approximation for collective
  multiagent planning.
\newblock In \emph{Advances in Neural Information Processing Systems},
  volume~30, 2017.

\bibitem[Oliehoek et~al.(2006)Oliehoek, Vlassis, et~al.]{oliehoek2006dec}
Oliehoek, F., Vlassis, N., et~al.
\newblock Dec-{POMDP}s and extensive form games: equivalence of models and
  algorithms.
\newblock IAS technical report IAS-UVA-06-02, University of Amsterdam,
  Intelligent Systems Lab, Amsterdam, The Netherlands, 2006.

\bibitem[Oliehoek et~al.(2008)Oliehoek, Spaan, and
  Vlassis]{oliehoek2008optimal}
Oliehoek, F.~A., Spaan, M.~T., and Vlassis, N.
\newblock Optimal and approximate {Q}-value functions for decentralized
  {POMDP}s.
\newblock \emph{Journal of Artificial Intelligence Research}, 32:\penalty0
  289--353, 2008.

\bibitem[Oliehoek et~al.(2016)Oliehoek, Amato, et~al.]{oliehoek2016concise}
Oliehoek, F.~A., Amato, C., et~al.
\newblock \emph{A concise introduction to decentralized POMDPs}.
\newblock Springer, 2016.

\bibitem[Osborne \& Rubinstein(1994)Osborne and Rubinstein]{osborne1994course}
Osborne, M.~J. and Rubinstein, A.
\newblock \emph{A course in game theory}.
\newblock MIT Press, Cambridge, MA, 1994.

\bibitem[Paszke et~al.(2019)Paszke, Gross, Massa, Lerer, Bradbury, Chanan,
  Killeen, Lin, Gimelshein, Antiga, Desmaison, Kopf, Yang, DeVito, Raison,
  Tejani, Chilamkurthy, Steiner, Fang, Bai, and Chintala]{NEURIPS2019-pytorch}
Paszke, A., Gross, S., Massa, F., Lerer, A., Bradbury, J., Chanan, G., Killeen,
  T., Lin, Z., Gimelshein, N., Antiga, L., Desmaison, A., Kopf, A., Yang, E.,
  DeVito, Z., Raison, M., Tejani, A., Chilamkurthy, S., Steiner, B., Fang, L.,
  Bai, J., and Chintala, S.
\newblock {PyTorch}: An imperative style, high-performance deep learning
  library.
\newblock In \emph{Advances in Neural Information Processing Systems},
  volume~32, 2019.

\bibitem[Pearl(2009)]{Pearl2009-bb}
Pearl, J.
\newblock \emph{Causality: Models, Reasoning, and Inference}.
\newblock Cambridge University Press, New York, NY, 2nd edition, 2009.

\bibitem[Peleg et~al.(1999)Peleg, Rosenm{\"u}ller, and
  Sudh{\"o}lter]{peleg1999canonical}
Peleg, B., Rosenm{\"u}ller, J., and Sudh{\"o}lter, P.
\newblock The canonical extensive form of a game form: Symmetries.
\newblock In Alkan, A., Aliprantis, C.~D., and Yannelis, N.~C. (eds.),
  \emph{Current Trends in Economics}, pp.\  367--387. Springer, 1999.

\bibitem[Rotman(2012)]{rotman2012introduction}
Rotman, J.~J.
\newblock \emph{An Introduction to the Theory of Groups}.
\newblock Springer, New York, NY, 2012.

\bibitem[Samvelyan et~al.(2019)Samvelyan, Rashid, de~Witt, Farquhar, Nardelli,
  Rudner, Hung, Torr, Foerster, and Whiteson]{samvelyan19smac}
Samvelyan, M., Rashid, T., de~Witt, C.~S., Farquhar, G., Nardelli, N., Rudner,
  T. G.~J., Hung, C.-M., Torr, P. H.~S., Foerster, J., and Whiteson, S.
\newblock {The} {StarCraft} {Multi}-{Agent} {Challenge}.
\newblock arXiv preprint arXiv:1902.04043, 2019.

\bibitem[Schelling(1980)]{schelling1980strategy}
Schelling, T.~C.
\newblock \emph{The Strategy of Conflict}.
\newblock Harvard University Press, Cambridge, MA, 1980.

\bibitem[Schulman et~al.(2017)Schulman, Abbeel, and
  Chen]{schulman2017equivalence}
Schulman, J., Abbeel, P., and Chen, X.
\newblock Equivalence between policy gradients and soft {Q}-learning.
\newblock arXiv preprint arXiv:1704.06440, 2017.

\bibitem[Silver et~al.(2017)Silver, Schrittwieser, Simonyan, Antonoglou, Huang,
  Guez, Hubert, Baker, Lai, Bolton, Chen, Lillicrap, Hui, Sifre, van~den
  Driessche, Graepel, and Hassabis]{silver2017mastering}
Silver, D., Schrittwieser, J., Simonyan, K., Antonoglou, I., Huang, A., Guez,
  A., Hubert, T., Baker, L., Lai, M., Bolton, A., Chen, Y., Lillicrap, T., Hui,
  F., Sifre, L., van~den Driessche, G., Graepel, T., and Hassabis, D.
\newblock Mastering the game of go without human knowledge.
\newblock \emph{Nature}, 550\penalty0 (7676):\penalty0 354--359, 2017.

\bibitem[Stone et~al.(2010)Stone, Kaminka, Kraus, and Rosenschein]{stone2010ad}
Stone, P., Kaminka, G., Kraus, S., and Rosenschein, J.
\newblock Ad hoc autonomous agent teams: Collaboration without
  pre-coordination.
\newblock In \emph{Proceedings of the AAAI Conference on Artificial
  Intelligence}, volume~24, 2010.

\bibitem[Sunehag et~al.(2018)Sunehag, Lever, Gruslys, Czarnecki, Zambaldi,
  Jaderberg, Lanctot, Sonnerat, Leibo, Tuyls, and Graepel]{sunehag2018value}
Sunehag, P., Lever, G., Gruslys, A., Czarnecki, W.~M., Zambaldi, V., Jaderberg,
  M., Lanctot, M., Sonnerat, N., Leibo, J.~Z., Tuyls, K., and Graepel, T.
\newblock Value-decomposition networks for cooperative multi-agent learning
  based on team reward.
\newblock In \emph{Proceedings of the 17th International Conference on
  Autonomous Agents and MultiAgent Systems}, pp.\  2085–2087, 2018.

\bibitem[Sutton \& Barto(2018)Sutton and Barto]{sutton2018reinforcement}
Sutton, R.~S. and Barto, A.~G.
\newblock \emph{Reinforcement learning: An introduction}.
\newblock MIT Press, Cambridge, MA, 2018.

\bibitem[Tesauro(1994)]{tesauro1994td}
Tesauro, G.
\newblock {TD-Gammon}, a self-teaching backgammon program, achieves
  master-level play.
\newblock \emph{Neural Computation}, 6\penalty0 (2):\penalty0 215--219, 1994.

\bibitem[Tucker et~al.(2020)Tucker, Zhou, and Shah]{tucker2020adversarially}
Tucker, M., Zhou, Y., and Shah, J.
\newblock Adversarially guided self-play for adopting social conventions.
\newblock arXiv preprint arXiv:2001.05994, 2020.

\bibitem[von Neumann \& Morgenstern(1947)von Neumann and
  Morgenstern]{neumann1947theory}
von Neumann, J. and Morgenstern, O.
\newblock \emph{Theory of Games and Economic Behavior}.
\newblock Princeton University Press, Princeton, NJ, 1947.

\bibitem[Williams(1991)]{williams1991probability}
Williams, D.
\newblock \emph{Probability with Martingales}.
\newblock Cambridge University Press, Cambridge, UK, 1991.

\bibitem[Williams(1992)]{williams1992simple}
Williams, R.~J.
\newblock Simple statistical gradient-following algorithms for connectionist
  reinforcement learning.
\newblock \emph{Machine Learning}, 8\penalty0 (3-4):\penalty0 229--256, 1992.

\end{thebibliography}
\bibliographystyle{icml2021}

\onecolumn

\ifappendix

\newgeometry{}
\raggedbottom


\let\headwidth\textwidth




\appendix

\newpage

\part{Appendix}

In Appendix~\ref{section-using-different-symmetry-concepts}, we briefly compare our approach to \textcite{harsanyi1988general}. In Appendix~\ref{appendix-further-training-details}, we provide additional details about our implemented algorithms, cross-play evaluation, and further results. Everything afterwards, starting with Appendix~\ref{appendix-formalization-of-the-zero-shot-coordination-game}, is a self-contained rigorous treatment of the results that were informally stated in the paper. The two main theorems are subject of Appendices~\ref{appendix-proof-of-theorem-1} and \ref{appendix-proof-of-theorem-2}. In the following, we give a brief outline of Appendices~\ref{appendix-formalization-of-the-zero-shot-coordination-game}--\ref{appendix-proof-of-theorem-2}.

In Appendix~\ref{appendix-formalization-of-the-zero-shot-coordination-game}, we make rigorous the definition of an LFC game and LFC problem from Section~\ref{section-the-zero-shot-coordination-game}, and we provide auxiliary results required to prove our main theorems. Among those, we show in Section~\ref{appendix-optimal-symmetric-strategy-profiles} that any optimal symmetric profile of learning algorithms in an LFC game is a Nash equilibrium. In addition, in Appendix~\ref{appendix-equivariant-learning-algorithms}, we briefly discuss a condition under which the objective in the LFC problem is equivalent to the formulation used in our experiments as outlined in Section~\ref{section-training-and-XP-evaluation}.

In Appendix~\ref{appendix-characterization-of-other-play}, we provide characterizations of both the OP objective and the payoff in an LFC game, in terms of equivalence classes of policies under random permutations by automorphisms. This notion of equivalence makes it possible to analyze the OP-optimal policies in terms of representatives of equivalence classes that are invariant to automorphisms. We will use our results from this section for the proofs about the LFC problem, and for a proof about the existence of random tie-breaking functions.

In Appendix~\ref{appendix-proof-of-theorem-1}, we then turn to stating and proving a rigorous version of Theorem~\ref{thm-op-not-optimal-informal}. To that end, we show that there are two distinct OP-optimal equivalence classes in the two-stage lever game (Appendix~\ref{appendix-two-incompatible}), and then prove that any algorithm that learns both of these is an OP learning algorithm, but not optimal in the LFC problem of that game (Appendix~\ref{appendix-proof-op-suboptimal}).

Lastly, in Appendix~\ref{appendix-proof-of-theorem-2}, we provide theoretical results about OP with tie-breaking and state and prove a rigorous version of Theorem~\ref{thm-op-with-tie-breaking-informal}. First, we define OP with tie-breaking and discuss to what degree the formal definition is satisfied by our method (Appendix~\ref{appendix-definition-op-tie-breaking}). Second, we show that OP with tie-breaking is optimal in the LFC problem and that all principals using OP with tie-breaking is an optimal symmetric Nash equilibrium of any LFC game (Appendix~\ref{appendix-optimality-zero-shot-coordination-problem}). Third, we prove that a modification of the tie-breaking function introduced in Section~\ref{section-other-play-with-tie-breaking} satisfies our formal requirements (Appendix~\ref{appendix-random-tie-breaking-functions}).

\section*{List of Symbols}

\subsection*{General mathematical notation}

\begin{xtabular}{p{.15\columnwidth}  p{.85\columnwidth}}
        \(\mathbb{N}\) & natural numbers excluding \(0\) \\
        \(\mathbb{R}\) & real numbers\\
        \(\mathbb{N}_0\) & \(\mathbb{N}\cup\{0\}\) \\
        \(\PowerSet(\mathcal{X})\) & power set of the set \(\mathcal{X}\)\\
        \(\prod_{i=1}^N \mathcal{X}_i\) & Cartesian product \(\mathcal{X}_1\times\dotsb\times \mathcal{X}_N\) of sets \(\mathcal{X}_1,\dotsc,\mathcal{X}_N\)\\
        \(|\mathcal{X}|\) & cardinality of the set \(\mathcal{X}\)\\
        \(\mathcal{X}\setminus \mathcal{Y}\) &  set of elements of \(\mathcal{X}\) that are not in \(\mathcal{Y}\)\\
        \(X:=Y\) & \(X\) is defined as \(Y\) \\ 
        \(f\circ g\) & composition of two composable functions \(f\) and \(g\)\\
        \(\Proj_i(x)\) & projection on the \(i\)th component of the vector \(x=(x_i)_{i\in\mathcal{X}}\)\\
        \(x_{-i}\) & vector with the \(i\)th component removed, \(x_{-i}:=(x_1,\dotsc,x_{i-1},x_{i+1},\dotsc x_N)\)\\
        \((\tilde{x}_i,x_{-i})\) & vector with \(\tilde{x}_i\) as \(i\)th component, \((\tilde{x}_i,x_{-i}):=(x_1,\dots x_{i-1},\tilde{x}_i,x_{i+1},\dotsc x_N)\)\\
        \(\delta_{ij}\) & Kronecker delta\\
        \(\mathds{1}_\mathcal{X}\) & indicator function of set \(\mathcal{X}\) \\
        \(\Prob\) &probability measure\\
        \(\Delta(\mathcal{X})\) & set of probability mass functions or measures over the set \(\mathcal{X}\)\\
        \(\mathcal{E}_1\otimes\mathcal{E}_2\) & product-\(\sigma\)-Algebra of \(\mathcal{E}_1\) and \(\mathcal{E}_2\)\\
        \(\Mixture_1\otimes\Mixture_2\) & product measure of \(\Mixture_1\) and \(\Mixture_2\) \\
\(\E[X]\) & expectation of the random variable \(X\)\\
\(\E_{x\sim \Distr}[f(x)]\) & integral of \(f\) with respect to the measure \(\Distr\)\\
\(\U(\mathcal{X})\) & uniform distribution over the set \(\mathcal{X}\)\\
\(\delta_x\) & Dirac measure\\
\(\Param\) & parameter value\\
\(\NN\) & neural network\\
\(\mathcal{L}\) & loss function\\
\end{xtabular}

\subsection*{Dec-POMDPs}

\begin{xtabular}{p{.15\columnwidth}  p{.85\columnwidth}}
        \(\DecP,\DecPSecond,\DecPThird\) & Dec-POMDPs\\
        \(\mathcal{X}^\DecP\) & set belonging to the Dec-POMDP \(\DecP\) (e.g., \(\PolicySet^\DecP\) is the set of policies of \(\DecP\))\\
        \(\PlayerSet\) & set of agents or principals \(\PlayerSet=\{1,\dotsc, N\}\) \\
        \(i\) & agent or principal\\
        \(\StateSet\) & set of states \\
        \(\StateRV_t\)  & random variable for the state at step \(t\)\\
        \(\State\) & state\\
        \(\ActionSet\) & set of joint actions\\
        \(\ActionSet_i\) & set of actions of player \(i\) \\
        \(\ActionRV_{i,t}\) & random variable for the action of agent \(i\) at step \(t\)\\
        \(\Action\) & joint action \\
        \(\Action_i\) & action of agent \(i\)\\
        \(\ObservationSet\) & set of observations\\
        \(\ObservationSet_i\) & set of observations of agent \(i\) \\
        \(\ObservationRV_{i,t}\) & random variable for the observation of agent \(i\) at step \(t\)\\
        \(\Observation\) & joint observation\\
        \(\Observation_i\) & observation of agent \(i\)\\
        \(\RewardRV_t\) & random variable for the reward at step \(t\)\\
        \(\Reward\) & reward \\
        \(P(\State'\mid\State,\Action)\) & transition probability\\
        \(O(\Observation\mid\State,\Action)\) & observation probability \\
        \(\RewardFunction(\Action,\State)\) & reward given joint action \(\Action\) and state \(\State\) \\
        \(\Tmax\) & horizon \\
        \(\AOHistorySet\) & set of joint action-observation histories\\
        \(\AOHistorySet_{i,t}\) & set of local action-observation history of agent \(i\) of length \(t\)\\
        \(\AOHistoryRV_{i,t}\) & random variable for action-observation history of agent \(i\) of length \(t\)\\
        \(\AOHistory_{i,t}\) & local action-observation history of agent \(i\) of length \(t\)\\
        \(\HistorySet\) & set of histories\\
        \(\HistoryRV\) & random variable for the history\\
        \(\History\) & history\\
        \(\HistorySet_t\) & set of histories of length \(t\)\\
        \(\HistoryRV_t\) & random variable for the history of length \(t\)\\
        \(\History_t\) & history of length \(t\) \\
        \(\PolicySet\) & set of joint policies\\
        \(\PolicySet_i\) & set of local policies of agent \(i\)\\
        \(\Policy\) & joint policy \\
        \(\Policy_i\) & local policy of agent \(i\)\\
        \(\PolicySet^0\) & set of joint deterministic policies\\
        \(\PolicySet^0_i\) & set of local deterministic policies of agent \(i\)\\
        \((\Omega,\PowerSet(\Omega),\Prob_\Policy)\) & measure space for a Dec-POMDP environment induced by policy \(\Policy\)\\
        \(\E_\Policy\) & expectation with respect to \(\Prob_\Policy\)\\
        \(J^\DecP(\Policy)\) & expected return of policy \(\Policy\) in Dec-POMDP \(\DecP\)\\
    \end{xtabular}

\subsection*{Label-free coordination and other-play}

\begin{xtabular}{p{.15\columnwidth}  p{.85\columnwidth}}
\(\Aut(\DecP)\) & set of automorphisms \\
\(\Iso(\DecP,\DecPSecond)\) & set of isomorphisms from \(\DecP\) to \(\DecPSecond\)\\
\(\Sym(\DecP)\) & set of labelings of \(\DecP\)\\
\(\Isom\) & isomorphism or labeling\\
\(\Isom^*\DecP\) & relabeled Dec-POMDP\\
\(\Isom^*\Policy\) & pushforward policy\\
\(\IsoProfile\) & profile of isomorphisms\\
\(\Aut(\DecP)\) & set of automorphisms of \(\DecP\)\\
\(\Auto\) & automorphism\\
\(\Id\) & identity automorphism\\
\(\AutProfile\) & profile of automorphisms\\
\(J^\DecP_\OP(\Policy)\) & other-play value of \(\Policy\) in Dec-POMDP \(\DecP\)\\
\(\DecPSet,\DecPSetSecond\) & sets of Dec-POMDPs\\
\(\mathcal{F}_i\) & \(\sigma\)-Algebra over \(\PolicySet_i\)\\
\(\mathcal{F}\) & product-\(\sigma\)-Algebra over \(\PolicySet\) \\
\(\Distr\) & distribution over policies \\
\(\LASet^\DecPSet\) & set of learning algorithms for \(\DecPSet\)\\
\(\LA\) & learning algorithm\\
\(\LAProfile\) & profile of learning algorithms\\
\(U^\DecP(\LAProfile)\) & payoff in the label-free coordination game for \(\DecP\) given strategy profile \(\LAProfile\)\\
\(U^\DecP(\LA)\) & value of \(\LA\) in the label-free coordination problem for \(\DecP\)\\
\(\Mixture\) & distribution over policies with independent local policies\\
\(\PolicyLatent_i\) & latent variable for the policy of agent \(i\) \\
\(\mathcal{Z}\) & measurable set of policies\\
\(\Policy^\Mixture\) & policy corresponding to the distribution \(\Mixture\)\\
\(\Psi(\Policy)\) & policy corresponding to the other-play distribution of \(\Policy\)\\
\([\Policy]\) & equivalence class of policies\\
\(\Hash\) & hash function\\
\(\Tie\) & tie-breaking function\\
\end{xtabular}

\section{Comparison to the solution by Harsanyi \& Selten}
\label{section-using-different-symmetry-concepts}

\jt{possibly perform experiments}

Here, we compare the solution to the equilibrium selection problem provided by \textcite{harsanyi1988general} to our approach. It is unclear how to apply \textcite{harsanyi1988general}'s solution to Dec-POMDPs, and this would be an interesting area for future work. However, we can translate Dec-POMDPs into \textcite{harsanyi1988general}'s formalism of standard-form games, using similar constructions as the ones for normal-form games and extensive-form games by \textcite{oliehoek2006dec}, and apply \textcite{harsanyi1988general}'s solution to such a problem. We can then compare it to OP with tie-breaking as an optimal solution to the LFC problem.

Below, we give an example in which OP with tie-breaking is equivalent to any OP algorithm in theory, as there is only one OP-optimal policy (ignoring differences between policies that do not matter under OP; see Appendix~\ref{appendix-characterization-of-other-play}). We also consider that policy as good solution to ZSC \emph{in spirit}. However, applying \textcite{harsanyi1988general}'s procedure to a corresponding standard-form game leads to a policy in which agents cannot coordinate and which thus leads to a lower payoff. We restrict ourselves to an informal exposition and leave a more rigorous analysis to future work.

Consider a version of the two-stage lever game with 10 instead of 2 levers. As in the two-stage lever game, pulling the same lever gives a reward of \(1\), non-coordination gives a reward of \(-1\), and the game is fully observable. Note that, like in the game with two levers, an OP-optimal policy uniformly randomizes between all levers in the first round. If no coordination was achieved in the first round, then in the second round, an optimal policy randomizes between the two levers that have been played in the first round by both players, similarly to the two-lever variant. There is a difference, however, if players coordinated on one lever in the first round. Clearly, in one optimal policy, players repeat their action from the first round, as was the case in the two-stage lever game. However, unlike in the two-lever case, here, there is no second optimal policy. It is not possible for the players to consistently switch to a different lever, as there are now not one but 9 other levers to choose from. Hence, the only optimal policy is one that chooses the unique lever that was chosen in the first round. This appears to us as a good solution to ZSC in this case.

Now consider a corresponding standard-form game \parencite[][ch.~2]{harsanyi1988general}. It is sufficient for us here to note that in this game, each player \(i=1,2\) is split into agents \(j_{\AOHistory_{i,t}}\) with distinct sets of actions \(\ActionSet_{\AOHistory_{i,t}}\) (corresponding to the 10 levers) for each possible action-observation history \(\AOHistory_{i,t}\in\AOHistorySet_{i,t}\) in the corresponding Dec-POMDP.
The payoff for a strategy for all agents of all players is then the expected return that the corresponding policy would receive. Importantly, symmetries as introduced by \textcite[][ch.~3.4]{harsanyi1988general} can permute each of the individual action sets \(\ActionSet_{\AOHistory_{i,t}}\) separately, as long as this does not change the payoffs (there are more rules for how symmetries can permute actions, players, and agents, but these do not matter for us here).
For instance, one symmetry may leave the actions of both players \(i=1,2\) in the first round unchanged, while it may apply one permutation to the action sets \(\ActionSet_{\AOHistory_{i,1}}\) corresponding to action-observation histories \(\AOHistory_{i,1}\in\AOHistorySet_{i,1}\) of all agents of both players in the second round. 
Since rewards in the second round do not depend on actions in the first round, and permuting all actions of all second-round agents in the same way does not change the rewards for actions, such a permutation is a symmetry of the game.

As a result, in the first and the second round, all individual actions are symmetric, and, unlike in the corresponding Dec-POMDP, the symmetries for both rounds can be applied independently of each other. Hence, a symmetry-invariant strategy needs to play all actions with equal probability in both rounds. Since \textcite{harsanyi1988general}'s solution always chooses a strategy that is invariant to symmetries \parencite[][ch.~3.4]{harsanyi1988general}, it follows that the strategy chosen by their procedure is a uniform distribution. Clearly, this strategy yields a lower return than the OP-optimal policy described above. In particular, since this applies independently of labelings, it follows from Theorem~\ref{thm-op-with-tie-breaking-informal} that the solution must be suboptimal in the associated LFC problem.

A similar argument could be made about cheap-talk: in a standard-form game, players using a symmetry-invariant policy would never be able to use cheap-talk, as they could not learn the meanings of each others' messages over time. Transforming a Dec-POMDP into a standard-form game thus yields too many symmetries, precluding players from coordinating based on the structure of the Dec-POMDP, even if it was possible to uniquely do so. In contrast, OP exhibits in a sense the opposite failure mode in the two-stage lever game, allowing players to coordinate arbitrarily due to too few symmetries between policies.

\section{Further experimental details}
\label{appendix-further-training-details}

Here, we provide additional details about the experiments outlined in Section~\ref{section-experiments}. We describe our implementation of OP (Appendix~\ref{appendix-op-implementation}), our implementation other OP with tie-breaking (Appendix~\ref{appendix-deep-tie-breaking}), and discuss our cross-play evaluation procedure as well as some further results (Appendix~\ref{appendix-XP-evaluation}).

\subsection{Other-play implementation}
\label{appendix-op-implementation}

\begin{algorithm}
\begin{algorithmic}
\newlength{\Input}
\settowidth{\Input}{{\bfseries Input: }}
\STATE \parbox{\Input}{\bfseries Input:}Dec-POMDP \(\DecP\)
\STATE \hskip \Input Number of training steps \(L\)
\STATE \hskip\Input Episode batch-size \(K\)
\STATE \hskip\Input Gradient-based optimizer
\STATE {\bfseries Output:} Joint policy \(\Policy\in\PolicySet^\DecP\)

\STATE Initialize \(\Param\)
 \FOR{\(l=1\) to \(L\)}
    \FOR{\(k=1\) to \(K\)}
        \STATE Sample profile of automorphisms \(\AutProfile^{(k)}\sim \U(\Aut(\DecP)^\mathcal{N})\)
        \STATE Sample history \(\History^{(k)}\sim \Prob^\DecP_{{\AutProfile^{(k)}}^*\Policy_\Param}\)
        using joint policy \({\AutProfile^{(k)}}^*\Policy_\Param\)
        \FOR{\(t=1\) to \(\Tmax\)}
            \STATE \(G^{(k)}_t\leftarrow\sum_{t'=t}^\Tmax \Reward^{(k)}_{t'}\)
        \ENDFOR
    \ENDFOR
    \STATE Compute loss
    \(\mathcal{L}(\Param)\leftarrow-\frac{1}{KTN}\sum_{k=1}^K\sum_{t=1}^\Tmax G^{(k)}_t\sum_{i=1}^N\log \Policy_{\theta,{\AutProfile^{(k)}_i}^{-1}i}({\AutProfile^{(k)}_i}^{-1}\Action^{(k)}_{i,t}\mid {\AutProfile_i^{(k)}}^{-1}\AOHistory_{i,t}^{(k)})\)
    \STATE Update \(\Param\) using \(\nabla_\Param\mathcal{L}(\Param)\) to minimize \(\mathcal{L}\)
\ENDFOR
 \STATE Return \(\Policy_\Param\)
 \end{algorithmic}
\caption{Other-play learning algorithm based on vanilla policy gradient}
\label{fig:pseudo-code-other-play-learning-algorithm}
\end{algorithm}

Our implementation of the OP learning algorithm is based on the PyMARL framework \parencite{samvelyan19smac}. We use recurrent neural networks to parameterize the policies of agents and a policy gradient algorithm to train agents' policies. Given that our toy problems are very small, they could also be solved by simple tabular methods. Nevertheless, we choose to employ this framework to demonstrate that our results transfer to state-of-the-art methods, even if the problems do not require them.

PyMARL is based on the PyTorch deep learning framework \parencite{NEURIPS2019-pytorch}. Neural network layers are implemented using the PyTorch module \texttt{nn.Linear} and the recurrent neural network uses a single \texttt{nn.GRUCell}, with the input encoding being one layer with ReLU activation functions. Hidden states are transformed into probabilities by a single \texttt{nn.Linear} layer followed by a softmax. The dimension of the hidden state is \(64\). Agent parameters are optimized using the \texttt{RMSProp} module, with a learning rate of \(0.0005\), an alpha of \(0.99\) and epsilon of \(0.00001\). These hyperparameters were all adopted as default values from the PyMARL framework.

Since our generalization of the OP objective requires policies that can randomize, and since it cannot be implemented as the SP objective in a modified Dec-POMDP (see Appendix~\ref{appendix-no-optimal-deterministic-policy}), it is not clear how to use multi-agent methods based on value functions \parencite[e.g.][]{sunehag2018value,foerster2017counterfactual}. For this reason, we use a vanilla multi-agent policy gradient algorithm without baseline (\citeauthor{nguyen2017policy}, \citeyear{nguyen2017policy}; \citeauthor{williams1992simple}, \citeyear{williams1992simple}; \citeauthor{sutton2018reinforcement}, \citeyear{sutton2018reinforcement}, ch.~13.1), which can easily be applied to our generalization of the OP objective (see Algorithm~\ref{fig:pseudo-code-other-play-learning-algorithm}).

We use \emph{weight sharing}, that is, all agents use the same neural network and receive an additional observation specifying their agent-ID. In the two-stage lever game, where agents are symmetric, we omit this agent-ID and thus force the resulting joint policy to be symmetric, \(\Policy_1=\Policy_2\). We can do this as a symmetric policy is optimal under OP in this case (see Theorem~\ref{thm-op-mixture}). As a benefit, we do not have to implement permutations of agents for OP. In the asymmetric lever game, since agents are not symmetric, agent-IDs are added to observations as one-hot vectors.

Lastly, we add a penalty for the negative entropy of a policy to the loss-function \parencite{mnih2016asynchronous,schulman2017equivalence}. The entropy of the probability distribution \(\Policy_{\theta,i}(\cdot\mid \AOHistory_{i,t})\) is defined as
\[H(\Policy_{\theta,i}(\cdot\mid \AOHistory_{i,t})):=-\sum_{\Action_i\in\ActionSet_i}\Policy_{\theta,i}(\Action_i\mid \AOHistory_{i,t})\log \Policy_{\theta,i}(\Action_i\mid \AOHistory_{i,t}).\]
The loss-function is then
\begin{equation}\label{eq:100}
\tilde{\mathcal{L}}(\Param)=-\frac{1}{KTN}\sum_{k=1}^K \sum_{t=0}^\Tmax
\left(G_{t}^{(k)}\sum_{i=1}^N\log\Policy_{\Param,i}(\Action_{i,t}^{(k)}\mid \AOHistory^{(k)}_{i,t}) + \alpha\sum_{i\in\PlayerSet}H(\Policy_{\Param,i}(\cdot\mid \AOHistory^{(k)}_{i,t}))\right),
\end{equation}
where \(\alpha\) is a hyperparameter.

We choose \(\alpha=0.5\), as the highest \(\alpha\) at which we still got fast convergence to an approximately optimal policy, after testing \(\alpha=1,0.5,0.1\), and \(0.05\). First, this encourages exploration and avoids a policy prematurely converging to a local minimum. Without this term, a small percentage of the learned policies in the asymmetric lever game converged to suboptimal equilibria. Second, we did this to make sure that agents learn to play a unique uniform distribution where actions do not matter for the OP value of a policy. Since it can be the case that actions do not matter but each choice of distribution still creates a different policy, this helps reduce the number of different policies that are learned, and thus facilitates tie-breaking between the remaining policies. Additionally, in the lever game with asymmetric players, it ensures that one can always infer from a distribution over histories the Dec-POMDP that these histories belong to. It hence suffices to let our hash function depend only on histories (see Appendix~\ref{appendix-random-tie-breaking-functions}).

Finally, we briefly outline how sampling from a policy \(\Isom^*\Policy\) is implemented, in the two-stage lever game where agents are symmetric and use the same policy network. Actions and observations are encoded as one-hot vectors, i.e., as elements of the canonical basis \(\{e_1,\dotsc,e_k\}\), where \(k\) is the cardinality of the respective set. For a given episode, one profile of automorphisms \(\AutProfile_1,\AutProfile_2\sim\U(\Aut(\DecP))\) is sampled. At time step \(t\), the observation input of the agent \(i\) is \(\AutProfile_i^{-1}(\ObservationRV_{i,t},\ActionRV_{i,t-1})\). Then an action \(\tilde{\ActionRV}_{i,t}\) is sampled from the agent policy, and that action is permuted by applying \(\AutProfile_i\), i.e., \(\ActionRV_{i,t}:=\AutProfile_i\tilde{\ActionRV}_{i,t}\). Otherwise, the Dec-POMDP model proceeds as normal. One can easily see that this results in a history \(\HistoryRV\sim\Prob_{\AutProfile^*\Policy}\).

\subsection{Deep tie-breaking}
\label{appendix-deep-tie-breaking}
We implement OP with tie-breaking as described in Section~\ref{section-other-play-with-tie-breaking} (see Algorithm~\ref{algorithm-other-play-with-tie-breaking}).
\begin{algorithm}
\begin{algorithmic}
\settowidth{\Input}{{\bfseries Input: }}
\STATE \parbox{\Input}{\bfseries Input:}Dec-POMDP \(\DecPSecond\)
\STATE \hskip \Input OP learning algorithm \(\LA^\OP\)
\STATE \hskip \Input Tie-breaking-function \(\chi\)
\STATE \hskip \Input Number of seeds \(K\)
\STATE {\bfseries Output:} Joint policy \(\Policy^*\in\PolicySet^\DecPSecond\)
\FOR{\(k=1\) to \(K\)}
    \STATE Train policy \(\Policy^{(k)}\sim\sigma^\OP(\DecPSecond)\)
    \STATE Calculate tie-breaking value \(x^{(k)}\leftarrow \chi(\DecPSecond,\Policy^{(k)})\)
\ENDFOR
 \STATE \(k_\mathrm{max}\leftarrow\argmax_{k=1}^Kx^{(k)}\)
 \STATE \(\Policy^*\leftarrow \Policy^{(k_\mathrm{max})}\)
 \STATE Return \(\Policy^*\)
\end{algorithmic}
    \caption{OP with tie-breaking}
    \label{algorithm-other-play-with-tie-breaking}
\end{algorithm}

Turning to the tie-breaking function, the function is only applied to actions of both agents and to rewards, but not to observations or states. This is because in our problems, states are always the same, and observations are completely determined by actions. The hash network has four hidden layers with ReLu activation functions, a hidden dimension of \(32\), and weights and biases are initialized uniformly in \([-1,1]\). We chose these hyperparameters mostly based on prior considerations, but we did compare neural network depths 2 to 5 and hidden-layer dimensions 8, 16, 32 and 64 to determine hyperparameters for which OP with tie-breaking performed well. To calculate the tie-breaking function, we use 2048 episode samples. See Algorithm~\ref{algorithm-tie-breaking-function-pseudo-code} for pseudo code.

\begin{algorithm}
\begin{algorithmic}
\settowidth{\Input}{{\bfseries Input: }}
\STATE \parbox{\Input}{\bfseries Input:}Dec-POMDP \(\DecPSecond\)
\STATE \hskip \Input Policy \(\Policy\in\PolicySet^\DecPSecond\)
\STATE \hskip \Input Neural network architecture \(\NN_\Hash\)
\STATE \hskip \Input Random seed \(n\)
\STATE \hskip \Input Number of episode samples \(K\)
\STATE \hskip \Input Tie-breaking value \(\chi(\DecPSecond,\Policy)\)
\STATE Initialize \(\NN_\Hash\) using random seed \(n\)
\FOR{\(k=1\) to \(K\)}
   \STATE Sample profile of automorphisms \(\AutProfile^{(k)}\sim\U(\Aut(\DecPSecond)^\PlayerSet)\)
   \STATE Sample history \(\History^{(k)}\sim\Prob^\DecPSecond_{{\AutProfile^{(k)}}^*\Policy}\)
\ENDFOR
\STATE \(\chi(\DecPSecond,\Policy)\leftarrow \frac{1}{N!K} \sum_{k=1}^K\sum_{\Isom_N\in\mathrm{Bij}(\PlayerSet)}\NN_\Hash(\Isom_N(\iota(\History^{(k)}))\)
\STATE Return \(\chi(\DecP,\Policy)\)
\end{algorithmic}
\caption{Tie-breaking function}
\label{algorithm-tie-breaking-function-pseudo-code}
\end{algorithm}

\subsection{Cross-play evaluation}
\label{appendix-XP-evaluation}
Recall that, to evaluate algorithms in the LFC problems for our environments, we simplify the objective to
\begin{equation}
\tilde{U}^\DecP(\LA):=
\E_{\Policy^{(i)}\sim\LA(\DecP),\,i=1,2}
    \Big[
    \E_{\AutProfile_i\sim \U(\Aut(\DecP)),\,i=1,2}\Big[
    \\
J^\DecP((\AutProfile_1^*\Policy^{(1)})_1,(\AutProfile_2^*\Policy^{(2)})_2)\Big]\Big].
\end{equation}
We discuss this simplification further in Appendix~\ref{appendix-equivariant-learning-algorithms}.

For each game, we train 320 joint policies in total, using OP with different seeds for network initialization and environmental randomness.  In the two-stage lever game, we train each policy for three million environment steps, and in the variant with asymmetric players for five million steps (see Figure~\ref{fig:training-curves} for learning curves).  The training for both environments took around three days on a MacBook Pro laptop from 2017. We split the 320 policies into 10 sets of 32 policies each.  The policies in each set are used for application of the tie-breaking method, while the 10 different sets represent independent runs which can be used to calculate cross-play values.

\begin{figure}
    \centering
    \begin{subfigure}{0.3\textwidth}
    \centering
    \includegraphics[width=1\textwidth]{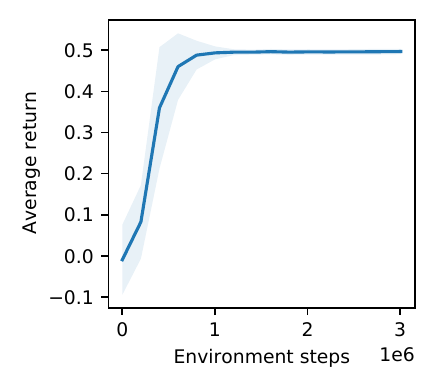}
    \caption{Two-stage lever game}
    \end{subfigure}
    \begin{subfigure}{0.3\textwidth}
    \centering
    \includegraphics[width=1\textwidth]{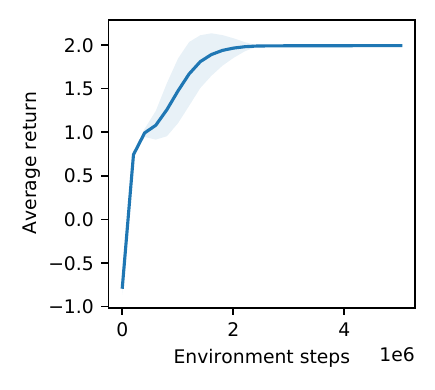}
    \caption{Asymmetric lever game}
    \end{subfigure}
        \caption{Learning curves for 320 independent training runs of the OP algorithm, with shaded standard deviation.}
            \label{fig:training-curves}
\end{figure}

Given a list of 10 policies \(\Policy^{(1)},\dotsc,\Policy^{(10)}\) produced by independent runs of an algorithm, we estimate the cross-play value \(G_{k,l}\approx\E_{\AutProfile_i\sim \U(\Aut(\DecP)),\,i=1,2}\left[
J((\AutProfile_1^*\Policy^{(k)})_1,(\AutProfile_2^*\Policy^{(l)})_2)\right]\) for any two indices of policies \(k, l\in\{1,\dotsc,10\}\). These values are used to print cross-play matrices (Figure~\ref{fig:21}) The average off-diagonal cross-play value is calculated as \(G:=\frac{1}{10(10-1)}\sum_{k\neq l}G_{k,l}\), where we leave out values on the diagonal because these do not represent cross-play between independent runs. Note that in the two-stage lever game, off-diagonal values can be higher than the optimal OP value. This is because in cross-play, agents can use different joint policies that accidentally work better in cross-play than a symmetric policy. It is counterbalanced in expectation by other entries of the matrix, in which unsuitable agents are matched. We provide in Table~\ref{tab:average-off-diag-XP} the average off-diagonal cross-play values used to create the graph in Figure~\ref{fig:avg-off-diagonal-XP}. 

Lastly, we categorize policies into classes of mutually compatible policies. To do so, we calculate a cross-play value for each combination of two out of the \(320\) trained joint policies, using \(256\) episodes each. We then dynamically build classes of policies by comparing a policy's expected return to the cross-play value with a policy from a given class, and assigning the policy to that class if the difference between the values is below a threshold of \(0.6\). If no class is compatible with a policy, the method creates a new class containing that policy. In that way, all 320 policies are assigned to a class. In Table~\ref{tab:3} we list the relative sizes of the different classes.

While in the two-stage lever game, both classes are represented approximately equally, a clear majority of policies belongs to one class in the asymmetric lever game. This class corresponds to the strategy in which player 2 switches to a different action upon non-coordination in the first round, and in which both players repeat their action given successful coordination in the first round. The least frequent class was the one in which player 2 switches to a different action upon non-coordination, but where both players switch to a different action if they were successful in the first round. The imbalance of classes in this case could be used to implement a tie-breaking rule that chooses the policy that is learned more often.

We think that a reasonable policy in the two-stage lever game is one in which a player repeats their action upon coordination.  Hence, a tie-breaking function that chooses the “repeat”-policy is preferable.  Interestingly, out of 20 seeds for the hash function, only 4 of the resulting tie-breaking functions gave higher values to that policy.

We used new random seeds for the final experiments that have not been used to improve the method. 

\begin{table}
    \caption{Average off-diagonal cross-play value and standard deviation for OP with tie-breaking with \(K=1\), 2, 4, 8, 16, and 32, in the two-stage lever game (TSLG) and asymmetric lever game (ALG).}
        \label{tab:average-off-diag-XP}
    \vskip 0.15in
    \small
    \centering
    \begin{tabular}{c|cccccc}
        Problem & 1&2&4&8&16&32 \\
        \hline
               TSLG & -0.03 (\textpm 0.00) & 0.14 (\textpm 0.13) & 0.48 (\textpm 0.06) & 0.50 (\textpm 0.00) & 0.50 (\textpm 0.00) & 0.50 (\textpm 0.00)  \\
      ALG &   0.90 (\textpm 0.00) & 1.23 (\textpm 0.08) & 1.75 (\textpm 0.22) & 1.90 (\textpm 0.16) & 1.96 (\textpm 0.11) & 1.97 (\textpm 0.10) 
    \end{tabular}
\end{table}

\begin{table}
    \small
    \centering
    \caption{Percentage of policies learned corresponding to different classes of mutually compatible policies, for two-stage lever game (TSLG) and asymmetric lever game (ALG), each time out of 320 seeds for training.}
    \vskip 0.15in
    \begin{tabular}{c|ccccc}
        Problem & Class 1 & Class 2 & Class 3 & Class 4 \\
        \hline
        TSLG& 50.94\%& 49.06\%& / & / \\
       ALG &43.75\% &20.94\%& 19.38\% &   15.94\% 
    \end{tabular}
    \label{tab:3}
\end{table}

\section{Formalization of a label-free coordination game and problem}
\label{appendix-formalization-of-the-zero-shot-coordination-game}
In this section, we formalize label-free coordination (LFC) games and the LFC problem, and we provide auxiliary results required to prove our main theorems.

In Appendix~\ref{appendix-recapitulation-of-dec-pomdps}, we recall the definition of Dec-POMDPs and introduce some additional notation. In Appendix~\ref{appendix-dec-pomdp-isomorphisms-and-automorphisms}, we recall the definition of isomorphisms and automorphisms, provide more comprehensive examples, and prove first lemmas about these concepts. Afterwards, in Appendix~\ref{appendix-pushforward-policies}, we discuss the pushforward and prove, among other things, that a pushforward policy has the same expected return as the original policy.  In Appendix~\ref{appendix-relation-to-group-theory}, we show that automorphisms define group actions on joint actions, policies, etc.\ and introduce the concept of an orbit. In Appendix~\ref{appendix-dec-pomdp-labelings-and-relabeled-dec-pomdps}, we then introduce Dec-POMDP labels to construct the set of relabeled Dec-POMDPs used to define an LFC game. In Appendix~\ref{appendix-the-zero-shot-coordination-game}, we recall the definition of LFC games and of the LFC problem, and we provide different expressions for the payoff in LFC games. We also prove that LFC games for isomorphic Dec-POMDPs are equivalent, up to a permutation of the principals in the game. In Appendix~\ref{appendix-optimal-symmetric-strategy-profiles}, we prove that any strategy profile that is optimal among those that respects symmetries between principals is a Nash equilibrium in the game, making use of group actions and orbits. This theorem is needed later to prove that all principals using OP with tie-breaking is a Nash equilibrium. Lastly, in Appendix~\ref{appendix-equivariant-learning-algorithms}, we briefly discuss a condition under which the objective in the LFC problem is equivalent to the formulation without relabeled Dec-POMDPs used in our experiments as outlined in Section~\ref{section-training-and-XP-evaluation}.

\subsection{Recapitulation of Dec-POMDPs}
\label{appendix-recapitulation-of-dec-pomdps}
Before we turn to isomorphisms and automorphisms, we briefly recapitulate Dec-POMDPs and introduce some additional notation that we will use throughout the following. That is, we also introduce a history of length \(t\), \(\History_t\in\HistorySet_t\), the set of deterministic policies \(\PolicySetDet\) and we define the measure space \((\Omega,\PowerSet(\Omega),\Prob_\Policy)\) corresponding to a Dec-POMDP in which agents follow the joint policy \(\Policy\). We also define the notation \(x_{-i}:=(x_1,\dotsc,x_{i-1},x_{i+1},\dotsc x_N)\) and the projection operator \(\Proj_i(x)=x_i\). Apart from this, the definitions in this section are a more elaborate version of those in Section~\ref{section-background} of the main text.

To begin, recall the definition of a Dec-POMDP.

\begin{defn}
A (finite-horizon) Dec-POMDP is a tuple
\[
\DecP=\Big(\PlayerSet,\StateSet,\ActionSet=\prod_{i\in \PlayerSet}\ActionSet_{i},P,\RewardFunction,\ObservationSet=\prod_{i\in \PlayerSet}\ObservationSet_{i},O,b_{0},\Tmax\Big)
\]
where
\begin{itemize}
\item $\mathcal{N}=\{1,\dots,N\}$, \(N\in\mathbb{N}\) is a finite set of agents.
\item $\StateSet$ is a finite set of states.
\item $\ActionSet_{i}$ is a finite set of actions for player $i\in\PlayerSet$.
\item $P\colon\StateSet\times\ActionSet\rightarrow\Delta(\StateSet)$
is the transition probability kernel (where $\Delta(\StateSet)$ denotes the
set of probability mass functions over $\StateSet$). 
\item $\RewardFunction\colon\StateSet\times\ActionSet\rightarrow\mathbb{R}$ is
the joint reward function.
\item $\ObservationSet_{i}$ is a finite set of observations for player $i\in\PlayerSet$.
\item $O\colon\StateSet\times\ActionSet\rightarrow\Delta(\ObservationSet)$ is
the observation probability kernel.
\item $b_{0}\in\Delta(\StateSet)$ is a distribution over the initial
state.
\item $\Tmax\in\mathbb{N}_0$ is the horizon of the problem.

\end{itemize}
\end{defn}

When considering different Dec-POMDPs \(\DecP,\DecPSecond\) at the same time, we write \(\ActionSet^\DecP,\ActionSet^\DecPSecond\), etc., to indicate to which Dec-POMDP the set belongs. Similarly, we do this for transition, observation, and reward functions. We omit the index \(\DecP\) if it is clear which Dec-POMDP is meant.


Given a Dec-POMDP \(\DecP\), define the set of \emph{local action-observation histories} of length \(t\in\{0,\dotsc,\Tmax\}\) for player
\(i\in\PlayerSet\) as
\[\AOHistorySet_{i,t}:=\left(\ActionSet_{i}\times\ObservationSet_{i}\right)^t,\]
(where \(\left(\ActionSet\times\ObservationSet\right)^0:=\{\emptyset\}\))
and the set of local action-observation histories for player \(i\) as
\[\AOHistorySet_{i}:=\bigcup_{0\leq t\leq \Tmax}\left(\ActionSet_{i}\times\ObservationSet_{i}\right)^t.\]
Moreover, we define the set of histories of length \(t\) as
\[\HistorySet_t:=\StateSet\times\ActionSet\times\RewardFunction(\StateSet\times\ActionSet)\times\left(\StateSet\times\ObservationSet\times\ActionSet\times\RewardFunction(\StateSet\times\ActionSet)\right)^t.\]
and \(\HistorySet:=\HistorySet_\Tmax\) as the set of histories. Histories \(\History\in\HistorySet\) are also called episodes, when one is talking about a particular sample of the stochastic process induced by agents following a policy in the Dec-POMDP, as defined below.

At step \(0\leq t\leq \Tmax\), agent \(i\in\PlayerSet\) chooses a distribution over actions, conditional on a past action-observation history \(\AOHistory_{i,t}\in \AOHistorySet_{i,t}\).
This choice is described by a \emph{local (stochastic) policy} for player $i$, which is a mapping $\Policy_{i}\colon\AOHistorySet_{i}\rightarrow\Delta(\ActionSet_{i})$. Since actions can be chosen stochastically, a policy and past observations do not imply which past actions have been taken, so policies are also able to condition on past actions (i.e., agents remember their past actions). We write \(\Policy_i(\Action_i\mid\AOHistory_{i,t})\) for the probability of action \(\Action_i\in\ActionSet_i\) given the action-observation history \(\AOHistory_{i,t}\in\AOHistorySet_{i,t}\). A \emph{(joint stochastic) policy} is a tuple $\Policy=(\Policy_{1},\dots,\Policy_{N})$
with a local policy for each player.
We denote $\PolicySet^\DecP$ as the set of joint stochastic policies for a Dec-POMDP $D$, and
$\PolicySet^{D}_{i}$ as the set of local policies for player $i$.
A \emph{local deterministic policy} for player \(i\in\PlayerSet\) is defined as a 
policy $\Policy_{i}$
such that $\Policy_{i}(\cdot \mid\AOHistory_{i,t})$ is concentrated on exactly one action (i.e., there exists \(\Action_i\in\ActionSet_i\) such that \(\Policy_{i}(\Action_i\mid\AOHistory_{i,t})=1\)) for
each $\AOHistory_{i,t}\in\AOHistorySet_{i,t}$.
A joint deterministic policy is defined analogously to joint stochastic policies, and \((\PolicySet^0)^\DecP\) is the set of joint deterministic policies for \(\DecP\).\footnote{Normally, a deterministic policy is defined as having values in \(\ActionSet_i\), but we can identify each element $\Action_i\in\ActionSet_{i}$
with a probability mass function in $\Delta(\ActionSet_{i})$ with support $\{\Action_i\}$, so these definitions are interchangeable. Since a deterministic policy outputs only one action for each action-observation history, it there is only one possible action for each observation history \((\Observation_{i,1},\dots,\Observation_{i,t})\), and we could hence write a deterministic policy as a map from observation histories to actions. We omit this here for simplicity.}

A Dec-POMDP \(\DecP\) together with a joint policy \(\Policy\in\PolicySet^\DecP\) specify the stochastic process according to which an episode evolves, yielding a distribution over histories. Formally, we define a discrete probability space \((\Omega,\PowerSet(\Omega),\Prob_\Policy)\) with random variables for states, actions, observations, and rewards at all time steps.

The distributions of these random variables are defined inductively in the following way. First, $S_{0}\sim b_{0}$, that is, the first state is an independent random variable with values in \(\StateSet\) and image distribution \(b_0\). Similarly, the first action of agent \(i\) is distributed according to $A_{i,0}\sim\Policy_{i}(\cdot\mid \emptyset)$ (note that \(\emptyset\in\AOHistorySet_i\) by definition), and the first reward is $R_{0}:=\RewardFunction(S_{0},A_{0})$, where we define \(\ActionRV_0:=(\ActionRV_{i,0})_{i\in\PlayerSet}\).

Now assume there are random variables for all states, actions, and observations until step \(0\leq t\leq\Tmax\) (in the case \(t=0\), there are no observations defined yet). We can summarize them into a random variable for the history of length \(t\) as
\[\HistoryRV_{t}:=(\StateRV_0,\ActionRV_0,(\StateRV_{t'},(\ObservationRV_{i,t'},\ActionRV_{i,t'})_{i\in\PlayerSet},\RewardRV_{t'})_{1\leq t'\leq t})\]
(and we let \(\HistoryRV:=\HistoryRV_\Tmax\)). 
At step \(t+1\), a new state \(S_{t+1}\sim P(\cdot\mid S_{t},A_{t})\) and a new joint observation \(O_{t+1}\sim O(\cdot\mid S_{t+1},A_{t})\) are sampled. Note that, in a slight abuse of notation, we use \(\ObservationRV\) for both observation probabilities and observation random variable. We can define a random variable for the action-observation history of agent \(i\) at time \(t+1\) as $\AOHistoryRV_{i,t+1}:=(A_{i,0},O_{i,1},A_{i,1},\dotsc,A_{i,t},O_{i,t+1})$. Conditioning on this action-observation history, agent \(i\) samples an action $A_{i,t+1}\sim\Policy_{i}(\cdot \mid \AOHistoryRV_{i,t+1})$. This yields a new joint action \(\ActionRV_{t+1}\). Finally, the new reward is
$R_{t+1}:=\RewardFunction(S_{t+1},A_{t+1})$, which concludes the definition.

We will sometimes make use of a simplified notation for a tuple excluding a particular player. Let \((x_i)_{i\in\PlayerSet}\) be a vector indexed by agents, with elements \(x_i\in \mathcal{X}_i\) where \(\mathcal{X}_i, i\in\PlayerSet\) are some sets. Then we define \(x_{-j}:=(x_1,\dotsc,x_{j-1},x_{j+1},\dotsc,x_N)\) for \(j\in\PlayerSet\). Moreover, for any \(x'_i\in X_i\), we then write \((x'_i,x_{-i}):=(x_1,\dotsc,x_{i-1},x_i',x_{i+1},\dotsc,x_N)\). We will also sometimes use the projection operator \(\Proj_i(x):=x_i\) to clearly refer to a particular element of \(x\) when \(x\) is a more complicated expression.

Now, given a problem \(\DecP\) and a joint policy \(\Policy\in\PolicySet^\DecP\), the measure \(\Prob_\Policy\) defined above specifies a distribution over rewards \(\RewardRV_0,\dotsc,\RewardRV_\Tmax\). We can use this fact to define the expected return of the policy. To that end, let $\E_{\Policy}:=\E_{\Prob_\Policy}$ denote the expectation with respect to \(\Prob_\Policy\).

\begin{defn}[Self-play objective]
\label{defn-self-play-objective}Let \(\DecP\) be a Dec-POMDP. Define the \emph{self-play (SP) objective} \(J^\DecP\colon \PolicySet^\DecP\rightarrow\mathbb{R}\) for \(\DecP\) via
\[J^{\DecP}(\Policy):=\E_{\Policy}\left[\sum_{t=0}^{\Tmax}R_{t}\right]\]
for \(\Policy\in\PolicySet^\DecP\). Here, \(J^\DecP(\Policy)\) is called the expected return of a joint policy $\Policy$.
\end{defn}

\subsection{Dec-POMDP isomorphisms and automorphisms}
\label{appendix-dec-pomdp-isomorphisms-and-automorphisms}

Here, we recall the definitions of isomorphisms and automorphisms from Sections~\ref{section-dec-pomdp-isomorphism} and \ref{section-other-play-generalization} and give more comprehensive examples than in the main text. Then we prove some elementary results.

Let \(\DecP,\DecPSecond\) be two Dec-POMDPs. Consider a tuple of bijective maps
\[\Isom:=(\Isom_{N},\Isom_{S},(\Isom_{A_i})_{i\in\PlayerSet},(\Isom_{O_i})_{i\in\PlayerSet}),\] where 
\begin{alignat}{2}\label{eq:7appendix}
&\Isom_{N}\colon &&\mathcal{N}^\DecP\rightarrow\PlayerSet^\DecPSecond\\
& \Isom_{S}\colon &&\StateSet^\DecP\rightarrow\StateSet^\DecPSecond\\
\forall i\in\PlayerSet\colon\quad&\Isom_{A_i}\colon &&\ActionSet_{i}^\DecP\rightarrow\ActionSet^\DecPSecond_{\Isom_N(i)}\\
\forall i \in\PlayerSet\colon\quad&\Isom_{O_i}\colon
&&\ObservationSet_{i}^\DecP\rightarrow\ObservationSet_{\Isom_N(i)}^\DecPSecond.\label{eq:8appendix}
\end{alignat}

Given a joint action \(\Action\in\ActionSet^\DecP\) and such tuple \(\Isom\), we define \begin{equation}\label{eq:13appendix}
\Isom_A(\Action):=\left(\Isom_{A_{\Isom^{-1}_N(j)}}\left(\Action_{\Isom^{-1}_N(j)}\right)\right)_{j\in\PlayerSet^\DecPSecond}\in\ActionSet^\DecPSecond,
\end{equation}
and analogously for \(\Observation\in\ObservationSet^\DecP\)
\begin{equation}\label{eq:14appendix}
\Isom_O(\Observation):=\left(\Isom_{O_{\Isom^{-1}_N(j)}}\left(\Observation_{\Isom^{-1}_N(j)}\right)\right)_{j\in\PlayerSet^\DecPSecond}\in\ObservationSet^\DecPSecond.\end{equation}
That is, in the joint action \(\Isom_A(\Action)\in\ActionSet^\DecPSecond\), the agent \(j=\Isom_N(i)\in\PlayerSet^\DecPSecond\) (where \(i\in\PlayerSet^\DecP\)) plays the action \(\Isom_{A_i}(a_i)\in\ActionSet_j^\DecPSecond\).
Note that it is really \(f_A(a)\in\ActionSet^\DecPSecond\), as for any $a\in\ActionSet^\DecP$ and \(i:=f_{N}^{-1}(j)\in \PlayerSet^\DecP\), by definition it is
\[f_{A_{i}}(a_{i})\in\ActionSet^\DecPSecond_{f_{N}(i)}=\ActionSet^\DecPSecond_{f_{N}(f_N^{-1}(j))}=\ActionSet^\DecPSecond_{j}.\]
The analogous holds for observations.

\begin{defn}[Dec-POMDP isomorphism]\label{dec-pomdp-isomorphism-appendix}
 Let $\DecP,\DecPSecond$ be Dec-POMDPs such that both have the same horizon \(\Tmax^\DecP=\Tmax^\DecPSecond\), and let \[\Isom:=(\Isom_{N},\Isom_{S},(\Isom_{A_i})_{i\in\PlayerSet},(\Isom_{O_i})_{i\in\PlayerSet})\] be a tuple of bijective maps as defined in equations (\ref{eq:7appendix})--(\ref{eq:8appendix}). Then \(\Isom\) is an isomorphism from \(\DecP\) to \(\DecPSecond\) if
for any \(\Action\in\ActionSet^\DecP\), \(s,s'\in\StateSet^\DecP\) and \(o\in\ObservationSet^\DecP\), it is
\begin{align}
P^\DecP(s'\mid s,a)= & P^\DecPSecond(f_{S}(s')\mid f_{S}(s),f_A(a))\label{eq:2appendix}\\
O^\DecP(o\mid s,a)= & O^\DecPSecond(f_{O}(o)\mid f_{S}(s),f_A(a))\label{eq:3appendix}\\
\RewardFunction^\DecP(s,a)= & \RewardFunction^\DecPSecond(f_{S}(s),f_A(a))\label{eq:4appendix}\\
b^\DecP_{0}(s)= & b_{0}^\DecPSecond(f_{S}(s)).
\end{align}
If such an isomorphism exists, $\DecP$ and $\DecPSecond$ are called isomorphic. We denote $\Iso(D,E)$ for the set of isomorphisms from \(D\) to \(E\).
\end{defn}

\begin{remark}
In a Dec-POMDP, distributions over histories and policies can be ranked by their associated expected return. If a reward function is multiplied with a positive constant \(\alpha\in\mathbb{R}_{>0}\) and shifted by a constant \(\beta\in\mathbb{R}\), then the new reward function \(\RewardFunction':= \alpha\RewardFunction+\beta\) still induces the same ranking. For this reason, one could consider such transformations as part of an isomorphism between two problems. For instance, \textcite[][p.~72]{harsanyi1988general} make it part of their definition of an isomorphism in the framework of standard-form games.
For simplicity, we ignore this complication here and consider only isomorphic problems with reward functions that have the same range.\end{remark} 

As in the main text, we write \(fa\) instead of \(f_A(a)\) and \(f\Action_i\) instead of \(f_{A_i}\Action_i\), and we do the same for observations, states, etc. We also write \(f\AOHistory_{i,t}\) for actions of isomorphisms on action-observation histories, which is defined as the element-wise application of \(f\), and letting \(\Isom \Reward := \Reward\) for rewards, we define \(\Isom\History\) for entire histories \(\History\in\HistorySet\).

Now we prove a first basic result about actions of isomorphisms. Note that an isomorphism \(\Isom\in\Iso(\DecP,\DecPSecond)\) is a bijective map
\begin{equation}
\label{eq:isomorphisms-is-bijective-map}
\Isom\colon \PlayerSet^\DecP\times\StateSet^\DecP\times\prod_{i\in\PlayerSet^\DecP}\ActionSet^\DecP_i\times\prod_{i\in\PlayerSet^\DecP}\ObservationSet^\DecP_i\rightarrow\PlayerSet^\DecPSecond\times\StateSet^\DecPSecond\times\prod_{i\in\PlayerSet^\DecPSecond}\ActionSet^\DecPSecond_{\Isom i}\times\prod_{i\in\PlayerSet^\DecPSecond}\ObservationSet^\DecPSecond_{\Isom i},
\end{equation}
which can be inverted and composed with other maps. The
following result shows that actions of isomorphisms on joint actions and joint observations are compatible with function composition and with taking inverses of functions. It follows that also the element-wise application to histories and action-observation histories is compatible in this way. In particular, this justifies omitting brackets when applying isomorphisms. We will use this lemma liberally in the following.

\begin{lem}\label{lemma-action-function-composition-compatible}
Let \(\DecP,\DecPSecond,\DecPThird\) be Dec-POMDPs, \(\Action\in\ActionSet^\DecP,\Observation\in\ObservationSet^\DecP\), let \(\Id\in\Iso(\DecP,\DecP)\) be the identity, and let \(\Isom\in\Iso(\DecP,\DecPSecond)\), \(\IsomSecond\in\Iso(\DecPSecond,\DecPThird)\). Then it is

\begin{enumerate}
    \item[(i)]
    \(\Id\Action=\Action\) and \(\Id\Observation=\Observation\).
    \item[(ii)] \(\IsomSecond(\Isom\Action)=(\IsomSecond\circ\Isom)\Action\) and \(\IsomSecond(\Isom\Observation)=(\IsomSecond\circ\Isom)\Observation\).
\end{enumerate}
In particular, actions of isomorphisms \(\Isom\) on joint actions and joint observations can be inverted using \(\Isom^{-1}\).
\end{lem}
\begin{proof}
First, it is $\Id a=(\Id a_{\Id ^{-1}i})_{i\in\mathcal{N}}=\Id $,
and analogously for $o$. Second,
\[
\IsomSecond(\Isom a)=\IsomSecond(\Isom a_{\Isom ^{-1}i})_{i\in\mathcal{N}}=(\IsomSecond(\Isom a_{\Isom ^{-1}\IsomSecond^{-1}i}))_{i\in\mathcal{N}}=(\IsomSecond \Isom a_{(\IsomSecond \Isom )^{-1}i})_{i\in\mathcal{N}}=(\IsomSecond\circ \Isom )a
\]
as function composition is associative and the inverse of $\IsomSecond\circ \Isom$ is
$\Isom^{-1}\circ \IsomSecond^{-1}$.

Next, it follows from the previous that \(\Isom^{-1}(\Isom \Action)= (\Isom^{-1}\circ\Isom)\Action = \Id \Action=\Action\). The analogous holds for \(\Observation\), which concludes the proof.
\end{proof}

The following corollary states that the action of \(\Isom\) on histories \(\History_t\in\HistorySet_t\) is bijective. An analogous corollary also hold for action-observation histories.

\begin{cor}\label{corollary-action-isomorphism-is-bijective-histories}
Let \(\DecP,\DecPSecond\) be isomorphic Dec-POMDPs with isomorphism \(\Isom\in\Iso(\DecP,\DecPSecond)\) and let \(t\in\{0,\dotsc,\Tmax\}\). Then \(\Isom_H\colon \HistorySet_t^\DecP\rightarrow\HistorySet_t^\DecPSecond,\History_t\mapsto\Isom\History_t\) is a bijective map.
\end{cor}
\begin{proof}
Let \(\History_t=(\State_0,\Action_0,\State_1,\Observation_1,\Action_1,\Reward_1,\dotsc,\State_t,\Observation_t,\Action_t,\Reward_t)\in\HistorySet^\DecP_t\).
Define \(\Isom_H\) as above and \(\Isom_H^{-1}\colon \History\mapsto\Isom^{-1}\History\). Then
\begin{multline}\Isom^{-1}(\Isom(\History_t))
=
\Isom^{-1}((\Isom\State_0,\Isom\Action_0,\Isom\State_1,\Isom\Observation_1,\Isom\Action_1,\Isom\Reward_1,\dotsc,\Isom\State_t,\Isom\Observation_t,\Isom\Action_t,\Isom\Reward_t))
\\
=(\Isom^{-1}\Isom\State_0,\Isom^{-1}\Isom\Action_0,\Isom^{-1}\Isom\State_1,\Isom^{-1}\Isom\Observation_1,\Isom^{-1}\Isom\Action_1,\Isom^{-1}\Isom\Reward_1,\dotsc,\Isom^{-1}\Isom\Reward_t))
\\
\overset{\text{Lemma~\ref{lemma-action-function-composition-compatible}}}{=}
(\State_0,\Action_0,\State_1,\Observation_1,\Action_1,\Reward_1,\dotsc,\State_t,\Observation_t,\Action_t,\Reward_t)
=\History_t.
\end{multline}
\end{proof}

Recall the definition of an \emph{automorphism}, which can be thought of as describing a symmetry of the Dec-POMDP.

\begin{defn}[Dec-POMDP automorphism] An isomorphism $f\in \Iso(D,D)$ from $D$ to itself
is called an automorphism. We define $\Aut(D):=\text{\ensuremath{\Iso(D,D)}}$ as the set
of all automorphisms of $D$.
\end{defn}

Now we give an example of an isomorphism and automorphism, using the lever coordination game.

\begin{example}[Lever coordination game]
Consider the lever coordination game introduced in Section~\ref{section-introduction}. This can be formalized as a Dec-POMDP \(\DecP\) with only one state, one observation for each agent, and in which \(\Tmax=0\). The agents are \(\PlayerSet=\{1,2\}\) and actions are \(\ActionSet_1=\ActionSet_2=\{1,\dotsc,10\}\). One possible reward function is \[\RewardFunction(\Action_1,\Action_2,\State)=\delta_{\Action_1,\Action_2}\left(1.0\cdot\mathds{1}_{\{1,\dotsc,9\}}(a_1)+0.9\cdot\mathds{1}_{\{10\}}(a_1)\right),\]
where \[\delta_{\Action_1,\Action_2}:=\begin{cases}1&\text{if }\Action_1=\Action_2\\0&\text{otherwise}\end{cases}\]
is the Kronecker delta.

As an isomorphic problem \(\DecPSecond\), consider a problem with the same sets of actions for both players, but where the reward function is defined as
\[\RewardFunction^\DecPSecond(\Action_1,\Action_2,\State)=\delta_{\Action_1,\Action_2}\left(1.0\cdot\mathds{1}_{\{2,\dotsc,10\}}(a_1)+0.9\cdot\mathds{1}_{\{1\}}(a_1)\right)\]
for \(\Action_1\in\ActionSet_1,\Action_2\in\ActionSet_2\) and the trivial state \(\State\).
Note that this is formally a different Dec-POMDP. Nevertheless, we can define an isomorphism \(\Isom\in\Iso(\DecP,\DecPSecond)\) in the following way. For \(i=1,2\), define \(\Isom_{A_i}\colon\ActionSet_i\rightarrow\ActionSet_i\) such that \(\Isom_{A_i}(10)=1\) and \(\Isom_{A_i}(1)=10\), and let \(\Isom_{A_1}=\Isom_{A_2}\) be arbitrary otherwise. Let the remaining components of \(\Isom\) be the identity map. Then for any joint action \(\Action\) and the trivial state \(\State\), it is
\begin{multline}
\RewardFunction^\DecPSecond(\Isom \State,\Isom\Action)=\delta_{\Isom\Action_1,\Isom\Action_2}\left(1.0\cdot\mathds{1}_{\{2,\dotsc,10\}}(\Isom a_1)+0.9\cdot\mathds{1}_{\{1\}}(\Isom a_1)\right)
\\
=\delta_{\Action_1,\Action_2}\left(1.0\cdot\mathds{1}_{\{1,\dotsc,9\}}(a_1)+0.9\cdot\mathds{1}_{\{10\}}(a_1)\right)=\RewardFunction^\DecP(\State,\Action)\end{multline}
and hence \(\DecP\) and \(\DecPSecond\) are isomorphic (as observation and transition probabilities as well as the initial state distribution are trivial here).

Note that we could have used any two (potentially different) permutations \(\hat{\Isom}_1,\hat{\Isom}_2\) of the two sets \(\ActionSet_1,\ActionSet_2\) and defined a new reward function \(\RewardFunction'(\State,\Action):=\RewardFunction^{\DecP}(\State,\hat{\Isom}_1^{-1}\Action_1,\hat{\Isom}_2^{-1}\Action_2)\). This reward function would then define a new Dec-POMDP \(\DecP'\), and the isomorphism from \(\DecP\) to \(\DecP'\) would be exactly \(\Isom\) defined by \(\Isom_{A_1}=\hat{\Isom}_1,\Isom_{A_2}=\hat{\Isom}_2\) and the identity in the other components, as \(\RewardFunction'(\Isom\State,\Isom\Action)=\RewardFunction^{\DecP}(\State,\Isom_{A_1}\hat{\Isom}_1^{-1}\Action_1,\Isom_{A_2}\hat{\Isom}_2^{-1}\Action_2)=\RewardFunction^{\DecP}(\State,\hat{\Isom}_1\hat{\Isom}_1^{-1}\Action_1,\hat{\Isom}_2\hat{\Isom}_2^{-1}\Action_2)=\RewardFunction^\DecP(\State,\Action)\). We will use this idea in Section~\ref{appendix-dec-pomdp-labelings-and-relabeled-dec-pomdps} to define relabeled Dec-POMDPs.

Next, consider the automorphisms of the lever coordination game. Note that the agents in this game are symmetric. For instance, we can define \(\Auto\) via \(\Auto_N(1)=2\) and \(\Auto_N(2)=1\), and such that \(\Auto_{A_1}=\Auto_{A_2}=\hat{\Auto}\), where \(\hat{\Auto}\) is any permutation of \(\{1,\dotsc,10\}\) such that \(\hat{\Auto}(10)=10\). Then one can easily check that \(\RewardFunction^\DecP(\Auto\State,\Auto\Action)=\RewardFunction^\DecP(\State,\Action)\) for any joint action \(\Action\) and the state \(\State\).
\end{example}

Next, we introduce the automorphisms of the two-stage lever game (Example~\ref{example-two-stage-lever-game}), which we will need later to prove that OP is suboptimal in the corresponding LFC problem.

\begin{example}
\label{example-two-stage-lever-game-automorphisms-isomorphisms}

Recall that in the two-stage lever game, there are two agents, \(\PlayerSet=\{1,2\}\), and the problem has two rounds, so \(\Tmax=1\). Each round, each agent has to pull one of two levers, \(\mathcal{A}_1=\mathcal{A}_2=\{1,2\}\). If both agents choose the same lever, they get a reward of \(1\). Otherwise, the reward is \(-1\). There is again only one state, but there are two observations, \(\ObservationSet_1=\ObservationSet_2=\{1,2\}\). In the second stage (\(t=1\)), each player observes the previous action of the other player, so \(\ObservationRV_{i,1}=\ActionRV_{-i,0}\) for \(i=1,2\). The reward and observation probabilities are given in Table~\ref{tab:4appendix}.

\begin{table}
\centering
    \caption{Reward function and observation probabilities in the two-stage lever game (Example~\ref{example-two-stage-lever-game}).}
    \label{tab:4appendix}
    \vskip 0.15in
    \small
    \begin{subtable}{.45\linewidth}
     \caption{Reward function \(\RewardFunction(\State,\Action)\) for each joint action \(\Action\)}
             \label{tab:4a}
    \centering
         \begin{tabular}{c|c}
        \(\Action\) & \(\RewardFunction(\State,\Action)\)\\
        \hline
       (1,1)  & 1 \\
        (1,2) & -1 \\
       (2,1)  & -1 \\
        (2,2) & 1 
    \end{tabular}
    \end{subtable}
    \centering
    \begin{subtable}{.45\linewidth}
      \centering
          \caption{Observation probabilities \(O(\Observation\mid\State, \Action)\) for joint observations \(\Observation\) and joint actions \(\Action\).}
    \label{tab:4b}
    \begin{tabular}{c|cccc}
     \( \Action\,\backslash\, \Observation\) & (1,1) & (1,2) & (2,1) & (2,2)\\
      \hline
       (1,1)  & 1 & 0 & 0 & 0\\
        (1,2) & 0 & 0 & 1 & 0\\
       (2,1)  & 0 & 1 & 0 & 0\\
        (2,2) & 0 & 0 & 0 &1
    \end{tabular}
    \end{subtable}
\end{table}

Using the table for reward function and observation probabilities, we can easily visualize isomorphisms and automorphisms. Consider any isomorphism \(\Isom\in\Iso(\DecP,\DecPSecond)\) where \(\DecPSecond\) is some other Dec-POMDP. Considering, for instance, the reward function, we know that \(\RewardFunction^\DecP(\State,\Action)=\RewardFunction^\DecPSecond(\Isom \State,\Isom\Action)\). This means that if we want to check the value of the reward function of \(\DecPSecond\) when agents choose action \(\Isom\Action\), we can look it up in the table corresponding to \(\RewardFunction^\DecP\) in the row for \(\Action\). So we can visualize an isomorphism by applying \(\Isom_A\) to the index (i.e., first) column of this table, but leaving the other cells unchanged. This creates a new table with the reward function for \(\DecPSecond\), or, equivalently, this new table corresponds to the reward function \(\RewardFunction^\DecP\) precomposed with \(\Isom^{-1}\), i.e., \(\RewardFunction^\DecP(\Isom^{-1}_S\cdot,\Isom^{-1}_A\cdot)=\RewardFunction^\DecPSecond\).
Analogously, we can apply \(\Isom_{O}\) to the header row and \(\Isom_A\) to the index column of the table with observation probabilities, yielding \(O^\DecP(\Isom^{-1}_O\cdot\mid \Isom^{-1}_S\cdot ,\Isom^{-1}_A\cdot)=O^\DecPSecond\).
An automorphism is then simply an isomorphism such that applying it to index column and header of the tables does not change the table, other than permuting rows and columns.

Now let \(\Auto\) be an automorphism. \(\Auto_N\) can either be the identity or it can switch both agents. In either case, it must be \(\Auto_{A_1}=\Auto_{A_2}\), which can also either be the identity or the map that switches the actions. The observation permutation then has to be equal to that of the actions, \(\Auto_{O_1}=\Auto_{O_2}=\Auto_{A_1}\). There is trivially only one option for the state permutation.

\begin{table}
\centering
    \caption{Visualization of precomposition of both reward function and observation probability kernel with \(\Auto^{-1}\), by applying \(\Auto\) to the index column and header row of the tables from Table~\ref{tab:4appendix}. Note that apart from a permutation of rows and columns, the tables are identical to the ones in Table~\ref{tab:4appendix}, showing that \(\Auto\) is an automorphism.}
        \label{tab:5appendix}
    \vskip 0.15in
    \small
    \begin{subtable}{.45\linewidth}
         \caption{Reward function \(\RewardFunction(\Auto^{-1}\State,\Auto^{-1}\Action)\)}
             \label{tab:5a}
    \centering
         \begin{tabular}{c|c}
        \(\Action\) & \(\RewardFunction(\State,\Action)\)\\
        \hline
       (2,2)  & 1 \\
        (1,2) & -1 \\
       (2,1)  & -1 \\
        (1,1) & 1 
    \end{tabular}
    \end{subtable}
    \centering
    \begin{subtable}{.45\linewidth}
      \centering
          \caption{Observation probabilities \(O(\Auto^{-1}\Observation\mid\Auto^{-1}\State,\Auto^{-1}\Action)\)}
              \label{tab:5b}
    \begin{tabular}{c|cccc}
     \( \Action\,\backslash\, \Observation\) & (2,2) & (1,2) & (2,1) & (1,1)\\
      \hline
       (2,2)  & 1 & 0 & 0 & 0\\
        (1,2) & 0 & 0 & 1 & 0\\
       (2,1)  & 0 & 1 & 0 & 0\\
        (1,1) & 0 & 0 & 0 &1
    \end{tabular}
    \end{subtable}
\end{table}
We visualize the application of the automorphism that switches agents as well as actions and thus also observations in Table~\ref{tab:5appendix}. For instance, using the definition of actions of automorphisms on joint actions in Equations~(\ref{eq:13appendix}) and (\ref{eq:14appendix}), we have \(\Isom_A(1,1)=(\Isom_{A_{\Isom_N 1}}1,\Isom_{A_{\Isom_N 2}}1)=(\Isom_{A_2}1,\Isom_{A_1} 1) = (2,2)\), and \(\Isom_A(1,2)=(\Isom_{A_2}2,\Isom_{A_1}1)=(1,2)\).
\end{example}

Before we turn to applying isomorphisms and automorphisms to policies, we briefly provide two basic results about isomorphisms and the relationship between isomorphisms and automorphisms.

First, we prove that function composition and inversion preserve isomorphisms.

\begin{lem}\label{lemma-inverse-composition-isomorphism}
Let \(\DecP,\DecPSecond,\DecPThird\) be Dec-POMDPs and let \(\Isom\in\Iso(\DecP,\DecPSecond)\) and \(\IsomSecond\in\Iso(\DecPSecond,\DecPThird)\). Then \(\Isom^{-1}\in \Iso(\DecPSecond,\DecP)\) and \(\IsomSecond\circ\Isom\in\Iso(\DecP,\DecPThird)\).
\end{lem}
\begin{proof}
Let \(\State,\State'\in\StateSet^\DecP\), \(\Action\in\ActionSet^\DecP\). First, using the definition of an isomorphism and associativity of function composition, it is
\[
P^\DecP(s'\mid s,a)=P^\DecPSecond(fs'\mid fs,fa)=P^\DecPThird(\IsomSecond \Isom s'\mid \IsomSecond\Isom s,\IsomSecond\Isom a)
=P^\DecPThird((\IsomSecond\circ\Isom) s'\mid (\IsomSecond\circ\Isom) s,(\IsomSecond\circ\Isom) a)
.\]
An analogous calculation applies to observation probabilities, reward functions and initial state distribution. This shows
that \(\IsomSecond\circ\Isom\in\Iso(\DecP,\DecPThird)\).

Next, let \(\State',\State\in\StateSet^\DecPSecond\). Using Lemma~\ref{lemma-action-function-composition-compatible} and the definition of an isomorphism, it is
\[
P^\DecPSecond(s'\mid s,a)=P^\DecPSecond(ff^{-1}s'\mid ff^{-1}s,ff^{-1}a)=P^\DecP(f^{-1}s'\mid f^{-1}s,f^{-1}a)
.\]
Again, an analogous calculation applies to observation probabilities and the other relevant functions. This shows
that $f^{-1}\in\Iso(\DecPSecond,\DecP)$.
\end{proof}

The next lemma shows that one can decompose isomorphisms into any isomorphism composed with an automorphism. Essentially, an isomorphism is a map from one Dec-POMDP to another, composed with a symmetry of that Dec-POMDP. This will be important later when we apply these concepts to define the LFC problem and show how OP relates to it. 


\begin{lem}\label{decompose-isomorphisms}
Let \(\DecP,\DecPSecond\) be Dec-POMDPs and \(\Isom,\IsomSecond\in \Iso(\DecP,\DecPSecond)\). Then there exists exactly one \(\Auto\in\Aut(\DecPSecond)\) such that \(\Auto\circ\Isom=\IsomSecond\). Analogously, there exists exactly one \(\Auto\in\Aut(\DecP)\) such that \(\Isom\circ\Auto=\IsomSecond\). In particular, it is
\[\Iso(\DecP,\DecPSecond)=\Isom\circ\Aut(\DecPSecond)=\Aut(\DecP)\circ \Isom,
\]
where \(\Isom\circ\Aut(\DecPSecond):=\{\Isom\circ\Auto\mid\Auto\in\Aut(\DecPSecond)\}\) and \(\Aut(\DecP)\circ \Isom\) is defined analogously.
\end{lem}
\begin{proof}
Existence: By Lemma~\ref{lemma-inverse-composition-isomorphism}, \(\Auto:=\IsomSecond\circ\Isom^{-1}\) is an isomorphism in \(\Iso(\DecPSecond,\DecPSecond)=\Aut(\DecPSecond)\), and it is \(\Auto\circ\Isom=(\IsomSecond\circ\Isom^{-1})\circ\Isom=\IsomSecond\).

Uniqueness: Assume \(\Auto\circ\Isom=\AutoSecond\circ\Isom=\IsomSecond\) for automorphisms \(\Auto,\AutoSecond\in\Aut(\DecPSecond)\). Then it follows that
\[\Auto=\Auto\circ(\Isom\circ\Isom^{-1})=(\Auto\circ\Isom)\circ\Isom^{-1}=(\AutoSecond\circ\Isom)\circ\Isom^{-1}=\AutoSecond.\]

The proof for the second part of the lemma is exactly analogous, but using \(\Auto:=\Isom^{-1}\IsomSecond\).

Turning to the ``in particular'' statement, the above shows that  \[\Iso(\DecP,\DecPSecond)\subseteq\Isom\circ\Aut(\DecPSecond)\]
and
\[\Iso(\DecP,\DecPSecond)\subseteq\Aut(\DecP)\circ\Isom.\]
The two inclusions in the other direction follow directly from Lemma~\ref{lemma-inverse-composition-isomorphism}.
\end{proof}

\subsection{Pushforward policies}
\label{appendix-pushforward-policies}
Recall the definition of pushforward policies.

\begin{defn}[Pushforward policy]\label{definition-pushforward-appendix}
Let $\DecP,\DecPSecond$ be isomorphic Dec-POMDPs, let \(\Isom\in\Iso(\DecP,\DecPSecond)\), and let \(\Policy\in\PolicySet^\DecP\). 
Then we define the pushforward $\Isom^*\Policy\in\PolicySet^\DecPSecond$ of \(\Policy\) by \(\Isom\) via
\[
(\Isom^{*}\Policy)_{i}(\Action_{i}\mid\AOHistory_{i,t}):=\Policy_{\Isom^{-1} i}(f^{-1}\Action_{i}\mid f^{-1}\AOHistory_{i,t})
\]
for all \(i\in\mathcal{N}^\DecPSecond,a_i\in\ActionSet_i^\DecPSecond,t\in\{0,\dotsc,\Tmax\},\) and \(\AOHistory_{i,t}\in\AOHistorySet_{i,t}^\DecPSecond\). That is, in the joint policy \(\Isom^*\Policy\), agent \(j\in\PlayerSet^\DecPSecond\) gets assigned the local policy \(\Policy_i\) of agent \(i:=\Isom^{-1}j\in\PlayerSet^\DecP\), precomposed with \(\Isom^{-1}\).
\end{defn}

One can easily see that $f^{*}\Policy$ is a policy for the Dec-POMDP $\DecPSecond$.
Hence, when \(\Isom\) is an automorphism, \(\Isom^*\Policy\) is a policy for the same Dec-POMDP as \(\Policy\).

Like actions of isomorphisms on joint actions and observations, the pushforward is compatible with function composition, and it can be inverted using the inverse map \(\Isom^{-1}\).

\begin{lem}\label{lemma-pull-back-function-composition-compatible}Let \(\DecP,\DecPSecond\) and \(\DecPThird\) be Dec-POMDPs. Let \(\Policy\in\PolicySet^\DecP\), \(\Id\in\Aut(\DecP,\DecP)\) be the identity and \(\Isom\in\Iso(\DecP,\DecPSecond)\), \(\IsomSecond\in\Iso(\DecPSecond,\DecPThird)\). Then it is
\begin{enumerate}
    \item[(i)] \(\Id^*\Policy=\Policy\).
    \item[(ii)] \(\IsomSecond^*(\Isom^*\Policy)=(\IsomSecond\circ\Isom)^*\Policy\).
\end{enumerate}

In particular, the pushforward can be inverted using the inverse map \(\Isom^{-1}\).

\end{lem}
\begin{proof}
First, let \(i\in\PlayerSet^\DecP,\Action_i\in\ActionSet_i^\DecP,t\in\{0,\dotsc,\Tmax\}\), and \(\AOHistory_{i,t}\in\AOHistorySet_{i,t}^\DecP\). Then
\[(\Id^*\Policy)_i(\Action_i\mid \AOHistory_{i,t})=\Policy_{\Id i}(\Id \Action_i\mid \Id \AOHistory_{i,t})=\Policy_i(\Action_i\mid\AOHistory_{i,t}),\]
which shows that \(\Id^*\Policy=\Policy\).

Second, let \(i\in\PlayerSet^{\DecPThird},\Action_i\in\ActionSet^{\DecPThird}_i\) and \(\AOHistory_{i,t}\in\AOHistorySet^{\DecPThird}_{i,t}\).
Using Lemma~\ref{lemma-action-function-composition-compatible}, it is
\begin{multline}
(\IsomSecond^*(\Isom^*\Policy))_i(\Action_i\mid \AOHistory_{i,t})
=
(\Isom^*\Policy)_{\IsomSecond^{-1} i}(\IsomSecond^{-1} \Action_i\mid\IsomSecond^{-1}\AOHistory_{i,t})
=\Policy_{\Isom^{-1}(\IsomSecond^{-1} i)}(\Isom^{-1}(\IsomSecond^{-1}\Action_i)\mid\Isom^{-1}(\IsomSecond^{-1}\AOHistory_{i,t}))
\\\overset{\text{Lemma~\ref{lemma-action-function-composition-compatible}}}{=}\Policy_{(\Isom\circ\IsomSecond)^{-1} i}((\IsomSecond\circ\Isom)^{-1}\Action_i\mid(\IsomSecond\circ\Isom)^{-1}\AOHistory_{i,t})
=((\IsomSecond\circ\Isom)^*\Policy)_i(\Action_i\mid\AOHistory_{i,t}),
\end{multline}
which proves that \(\IsomSecond^*(\Isom^*\Policy)=(\IsomSecond\circ\Isom)^*\Policy\).

Regarding the invertibility of the pushforward, note that by (ii), it is
\[(f^{-1})^*(f^*\Policy)\overset{\text{(ii)}}{=}(f^{-1}\circ f)^*\Policy=\Policy.\]
This concludes the proof.
\end{proof}

One may wonder how pushforward policies and isomorphisms are related. In particular, how is the distribution over histories induced by a pushforward policy related to the distribution induced by the original policy? This question is answered by the following theorem, which we will use throughout this paper. The theorem demonstrates how the isomorphism from \(\DecP\) to \(\DecPSecond\) preserves the structure of the problem. It is similar to \textcite[][]{kang2012exploiting}'s Theorem 3, which applies to automorphisms in POSGs and makes a statement about the expected returns of all agents under the pushforward policy. We will provide a result about expected returns as a corollary.

\begin{thm}\label{lem-pull-back-isomorphism-compatibility}
Let \(\DecP,\DecPSecond\) be isomorphic Dec-POMDPs, let \(\Isom\in\Iso(\DecP,\DecPSecond)\), and let \(\Policy\in\PolicySet^\DecP\). Then for any \(\History\in\HistorySet^\DecP\), it is
\[\Prob_{\Policy}(\HistoryRV^\DecP=\History)=\Prob_{\Isom^*\Policy}(\HistoryRV^\DecPSecond =\Isom \History).\]
In particular, for any \(i\in\PlayerSet,t\in\{0,\dotsc,\Tmax\}\) and \(\AOHistory_{i,t}\in\AOHistorySet_i\), it is \[\Prob_{\Policy}(\AOHistoryRV_{ i,  t}= \AOHistory_{i,t})=\Prob_{\Isom^*\Policy}(\AOHistoryRV_{\Isom i,t}=\Isom\AOHistory_{i,t})
.\]
\end{thm}

\begin{proof}
Proof by induction over \(t\in\{0,\dotsc,\Tmax\}\).

To start the induction, let \(\State\in\StateSet^\DecP,\Action\in\ActionSet,\Reward\in\mathbb{R}\) and \(\History_0:=(\State,\Action,\Reward)\in\HistorySet_0^\DecP\). Then it is
\begin{align}
    \Prob_\Policy(\HistoryRV_0^\DecP=\History_0)
    &=b_0^\DecP(\State)\prod_{i\in\PlayerSet^\DecP}\Policy_i(\Action_i\mid\emptyset)\delta_{\RewardFunction^\DecP(\State,\Action),\Reward}
    \\
    &=
    b_0^\DecPSecond(\Isom\State)\prod_{i\in\PlayerSet^\DecP}\Policy_{\Isom^{-1}\Isom i}(\Isom^{-1}\Isom\Action_i\mid\emptyset)
    \delta_{\RewardFunction^\DecPSecond(\Isom\State,\Isom\Action),\Reward}\label{eq:18}
    \\
    &=
    b_0^\DecPSecond(\Isom\State)\prod_{i\in\PlayerSet^\DecP}(\Isom^*\Policy)_{\Isom i}(\Isom\Action_i\mid\emptyset)
    \delta_{\RewardFunction^\DecPSecond(\Isom\State,\Isom\Action),\Reward}\label{eq:19}
    \\
    &=b_0^\DecPSecond(\Isom\State)\prod_{i\in\PlayerSet^\DecPSecond}(\Isom^*\Policy)_{i}(\Isom\Action_{\Isom^{-1}i}\mid\emptyset)
    \delta_{\RewardFunction^\DecPSecond(\Isom\State,\Isom\Action),\Reward}
    \\
    &=
    \Prob_{\Isom^*\Policy}(\HistoryRV_0^\DecPSecond=\Isom\History).\label{eq:20}
\end{align}
Here, in (\ref{eq:18}), we use the definition of an isomorphism, in (\ref{eq:19}) we use the definition of the pushforward policy, and in (\ref{eq:20}) we use the definition of \(\Isom\Action\) from Equation~(\ref{eq:13appendix}).

Next, let \(t>0\) and assume that it is \(\Prob_\Policy(\History^\DecP_{t-1}=\History_{t-1})=\Prob_{\Isom^*\Policy}(\History^\DecPSecond=\Isom\History_{t-1})\) for any \(\History_{t-1}\in\HistorySet_{t-1}\). Let \[\History_t=(\dotsc,\State_{t-1},\Observation_{t-1},\Action_{t-1},\Reward_{t-1},\State_t,\Observation_t,\Action_t,\Reward_t)\in\HistorySet_t^\DecP\]
arbitrary, define \(\Action_{i,t}\) for \(i\in\PlayerSet\) such that \((\Action_{i,t})_{i\in\PlayerSet^\DecP}=\Action_t\), and define \(\AOHistory_{i,t}\) as the action-observation history  for player \(i\in\PlayerSet^\DecP\) corresponding to \(\History_{t}\). 

If \(\Prob_\Policy(\HistoryRV^\DecP_{t-1}=\History_{t-1})=0\), then also \(\Prob_{\Isom^*\Policy}(\HistoryRV_{t-1}^\DecPSecond=\Isom\History_{t-1})=0\), and thus also
\[
\Prob_\Policy(\HistoryRV^\DecP_{t}=\History_{t})=0=\Prob_{\Isom^*\Policy}(\HistoryRV_{t}^\DecPSecond=\Isom\History_{t}).
\]
Thus, assume now that \(\Prob_\Policy(\HistoryRV^\DecP_{t-1}=\History_{t-1})>0\).
Then it is
\begin{align}
    &\Prob_\Policy(\StateRV_t^\DecP=\State_t,\ObservationRV^\DecP_t=\Observation_t,\ActionRV^\DecP_t=\Action_T,\RewardRV^\DecP_t=\Reward_t\mid\HistoryRV^\DecP_{t-1}=\History_{t-1})
    \\
    &=P^\DecP(\State_t\mid \State_{t-1},\Action_{t-1})O^\DecP(\Observation_{t}\mid \State_{t},\Action_{t-1})
    \prod_{i\in\PlayerSet^\DecP}\Policy_i(\Action_{i,t}\mid\AOHistory_{t})\delta_{\RewardFunction^\DecP(\State_t,\Action_t),\Reward}
    \\
    &=
    P^\DecPSecond(\Isom\State_t\mid \Isom\State_{t-1},\Isom\Action_{t-1})O^\DecPSecond(\Isom\Observation_{t}\mid \Isom\State_{t},\Isom\Action_{t-1})\label{eq:15}\\
    &\quad
    \prod_{i\in\PlayerSet^\DecP}\Policy_{\Isom^{-1}\Isom i}(\Isom^{-1}\Isom\Action_{i,t}\mid\Isom^{-1}\Isom\AOHistory_{t})\delta_{\RewardFunction^\DecPSecond(\Isom\State_t,\Isom\Action_t),\Reward}
    \\
   &=
    P^\DecPSecond(\Isom\State_t\mid \Isom\State_{t-1},\Isom\Action_{t-1})O^\DecPSecond(\Isom\Observation_{t}\mid \Isom\State_{t},\Isom\Action_{t-1})\label{eq:16}\\
    &\quad
    \prod_{i\in\PlayerSet^\DecP}(\Isom^*\Policy)_{\Isom i}(\Isom\Action_{i,t}\mid\Isom\AOHistory_{t})\delta_{\RewardFunction^\DecPSecond(\Isom\State_t,\Isom\Action_t),\Reward}
    \\
    &=
    P^\DecPSecond(\Isom\State_t\mid \Isom\State_{t-1},\Isom\Action_{t-1})O^\DecPSecond(\Isom\Observation_{t}\mid \Isom\State_{t},\Isom\Action_{t-1})
    \\
    &\quad\prod_{i\in\PlayerSet^\DecPSecond}(\Isom^*\Policy)_{i}(\Isom\Action_{\Isom^{-1}i,t}\mid\Isom\AOHistory_{\Isom^{-1}i,t})
    \delta_{\RewardFunction^\DecPSecond(\Isom\State_t,\Isom\Action_t),\Reward}
    \\
    &=
    \Prob_{\Isom^*\Policy}(\StateRV_t^\DecPSecond=\Isom\State_t,\ObservationRV^\DecPSecond_t=\Isom\Observation_t,\ActionRV^\DecPSecond_t=\Isom\Action_t,\RewardRV^\DecPSecond_t=\Reward_t
    \mid\HistoryRV^\DecPSecond_{t-1}=\Isom\History_{t-1}).\label{eq:17}
\end{align}
Again, we have used the definitions of isomorphism, pushforward policy, and Equations~(\ref{eq:13appendix}) and (\ref{eq:14appendix}) in lines (\ref{eq:15}), (\ref{eq:16}), and (\ref{eq:17}), respectively.
Using the inductive hypothesis, it follows that
\begin{align}\label{eq:21}&\Prob_\Policy(\HistoryRV^\DecP_t=\History_t)
    \\
   & =
    \Prob_\Policy(\StateRV_t^\DecP=\State_t,
    \ObservationRV^\DecP_t=\Observation_t,
    \ActionRV^\DecP_t=\Action_T,
    \RewardRV^\DecP_t=\Reward_t
    \mid\HistoryRV^\DecP_{t-1}=\History_{t-1})
    \Prob_\Policy(\HistoryRV^\DecP_{t-1}=\History_{t-1})
    \\
   & =\Prob_{\Isom^*\Policy}(\StateRV_t^\DecPSecond=\Isom\State_t,\ObservationRV^\DecPSecond_t=\Isom\Observation_t,\ActionRV^\DecPSecond_t=\Isom\Action_t,\RewardRV^\DecPSecond_t=\Reward_t\mid\HistoryRV^\DecPSecond_{t-1}=\Isom\History_{t-1})
     \Prob_\Policy(\HistoryRV^\DecP_{t-1}=\History_{t-1})
     \\
    & \overset{\text{I.H.}}{=}
     \Prob_{\Isom^*\Policy}(\StateRV_t^\DecPSecond=\Isom\State_t,\ObservationRV^\DecPSecond_t=\Isom\Observation_t,\ActionRV^\DecPSecond_t=\Isom\Action_t,\RewardRV^\DecPSecond_t=\Reward_t\mid\HistoryRV^\DecPSecond_{t-1}=\Isom\History_{t-1})
     \Prob_{\Isom^*\Policy}(\HistoryRV^\DecPSecond_{t-1}=\Isom\History_{t-1})
     \\
    & =\Prob_{\Isom^*\Policy}(\HistoryRV^\DecPSecond_t=\Isom\History_t).
\end{align}
This concludes the induction and thus proves that \(\Prob_\Policy(\HistoryRV^\DecP=\History)=\Prob_{\Isom^*\Policy}(\HistoryRV^\DecPSecond=\Isom\History)\) for any history \(\History:=\History_\Tmax\in\HistorySet_\Tmax^\DecP\).

Turning to the ``in particular'' part of the proposition, let \(i\in\PlayerSet,t\in\{0,\dotsc,\Tmax\}\), and \(\AOHistory_{i,t}\in\AOHistorySet_{i,t}\). Let \(\History_t\in\HistorySet_t^\DecP\) such that \(\Prob_\Policy(\HistoryRV^\DecP_t=\History_t)>0\) and thus also \(\Prob_{\Isom^*\Policy}(\HistoryRV^\DecPSecond_t=\Isom\History_t)>0\).

First, assume that actions and observations of agent \(i\) in this history equal those in \(\AOHistory_{i,t}\). Then it is \(\{\HistoryRV^\DecP_t=\History_t\}\subseteq \{\AOHistoryRV^\DecP_{i,t}=\AOHistory_{i,t}\}\) and hence \(\Prob_\Policy(\AOHistoryRV_{i,t}=\AOHistory_{i,t}\mid \HistoryRV_{t}=\History_t)=1\). Moreover, by Equations~(\ref{eq:13appendix}) and (\ref{eq:14appendix}), it follows for any \(t'\leq t,\Action_{t'}\in\ActionSet^\DecP\) and \(\Observation_{t'}\in\ObservationSet^\DecP\) that
\(\Proj_{\Isom i}(\Isom \Action_{t'})=\Isom \Action_{i,t'}\) and analogously \(\Proj_{\Isom i}(\Isom \Observation_{t'})=\Isom \Observation_{i,t'}\). Hence, it follows that actions of observations of agent \(\Isom i\) in \(\Isom \History_t\) are equal to those in \(\Isom\AOHistory_{i,t}\), and thus it is \(\{\HistoryRV^\DecPSecond_t=\Isom\History_t\}\subseteq \{\AOHistoryRV^\DecPSecond_{\Isom i,t}=\Isom\AOHistory_{i,t}\}\), which implies \(\Prob_{\Isom^*\Policy}(\AOHistoryRV^\DecPSecond_{\Isom i,t}=\Isom\AOHistory_{i,t}\mid \HistoryRV_{t}=\Isom\History_t)=1\).

If, on the other hand, a history \(\History_t\) disagrees with \(\AOHistory_{i,t}\) in any way, then trivially \(\Prob_\Policy(\AOHistoryRV_{i,t}=\AOHistory_{i,t}\mid \HistoryRV_{t}=\History_t)=0\) and thus by same argument as before also \(\Prob_{\Isom^*\Policy}(\AOHistoryRV_{\Isom i,t}=\Isom\AOHistory_{i,t}\mid \HistoryRV_{t}=\Isom\History_t)=0\).

It follows that for any \(\History_t\in\HistorySet^\DecP_t\) such that \(\Prob_\Policy(\HistoryRV^\DecP_t=\History)>0\), it is
\begin{equation}
   \Prob_\Policy(\AOHistoryRV_{i,t}=\AOHistory_{i,t}\mid \HistoryRV_{t}=\History_t)= \Prob_{\Isom^*\Policy}(\AOHistoryRV^\DecPSecond_{\Isom i,t}=\Isom\AOHistory_{i,t}\mid \HistoryRV_{t}=\Isom\History_t)\label{eq:22},
\end{equation}
and hence
\begin{align}\Prob_{\Policy}(\AOHistoryRV_{ i,  t}= \AOHistory_{i,t})
&=
\sum_{\History_t\in\HistorySet^\DecP_t}
\Prob_\Policy(\AOHistoryRV_{i,t}=\AOHistory_{i,t}\mid \HistoryRV_{t}=\History_t)\Prob_\Policy(\HistoryRV^\DecP_t=\History_t)
\\&\overset{(\ref{eq:21})}{=}
\sum_{\History_t\in\HistorySet^\DecP_t}
\Prob_\Policy(\AOHistoryRV_{i,t}=\AOHistory_{i,t}\mid \HistoryRV_{t}=\History_t)\Prob_{\Isom^*\Policy}(\HistoryRV^\DecP_t=\Isom\History_t)
\\&\overset{(\ref{eq:22})}{=}
\sum_{\History_t\in\HistorySet^\DecP_t}
\Prob_{\Isom^*\Policy}(\AOHistoryRV_{\Isom i,t}=\Isom\AOHistory_{i,t}\mid \HistoryRV_{t}=\Isom\History_t)
\Prob_{\Isom^*\Policy}(\HistoryRV^\DecP_t=\Isom\History_t)
\\&=
\sum_{\History_t\in\Isom^{-1}\left(\HistorySet^\DecP_t\right)}
\Prob_{\Isom^*\Policy}(\AOHistoryRV_{\Isom i,t}=\Isom\AOHistory_{i,t}\mid \HistoryRV_{t}=\History_t)
\Prob_{\Isom^*\Policy}(\HistoryRV^\DecP_t=\History_t)
\\&=\Prob_{\Isom^*\Policy}(\AOHistoryRV_{\Isom i,t}=\Isom\AOHistory_{i,t})\label{eq:23}.
\end{align}
In line (\ref{eq:23}), we have used that, by Corollary~\ref{corollary-action-isomorphism-is-bijective-histories}, \(\Isom\) is a bijective map when applied to histories, and thus \(\Isom^{-1}\left(\HistorySet^\DecP_t\right)=\HistorySet^\DecPSecond_t\).
This concludes the proof.
\end{proof} 

It is an immediate corollary that the expected return of a policy is not changed by the pushforward.

\begin{cor}\label{cor-sp-invariant-to-pullback}
Let \(\DecP,\DecPSecond\) be Dec-POMDPs, let \(\Isom\in\Iso(\DecP,\DecPSecond)\), and let \(\Policy\in\PolicySet^\DecP\). Then
\[J^{\DecP}(\Policy)=J^\DecPSecond(\Isom^*\Policy).\]
\end{cor}
\begin{proof} By
Corollary~\ref{corollary-action-isomorphism-is-bijective-histories} and Theorem~\ref{lem-pull-back-isomorphism-compatibility}, it is
\begin{multline}
\Prob_\Policy(\Isom\left(\HistoryRV^\DecP\right)=\History^\DecPSecond)
\overset{\text{Corollary~\ref{corollary-action-isomorphism-is-bijective-histories}}}{=}
\Prob_\Policy(\HistoryRV^\DecP=\Isom^{-1}\History^\DecPSecond)
\\
\overset{\text{Theorem~\ref{lem-pull-back-isomorphism-compatibility}}}{=}
\Prob_{\Isom^*\Policy}(\HistoryRV^\DecPSecond = \Isom(\Isom^{-1}\History^\DecPSecond))
=\Prob_{\Isom^*\Policy}(\HistoryRV^\DecPSecond = \History^\DecPSecond)
\end{multline}
for any history \(\History^\DecPSecond\in\HistorySet^\DecPSecond\).
This shows that the random variable \(\Isom\left(\HistoryRV^\DecP\right)\) has the same image distribution under \(\Prob_\Policy\) as the variable \(\HistoryRV^\DecPSecond\) under \(\Prob_{\Isom^*\Policy}\). In particular, this means that for any \(t=0,\dotsc,\Tmax\), the variables \(\Isom(\RewardRV_t^\DecP)=\RewardRV^\DecP_t\) and \(\RewardRV^\DecPSecond_t\) have the same distribution in the respective probability spaces (*).
Using the definition of the expected return, it follows that
\begin{equation}J^{\DecP}(\Policy)=\E_{\Policy}\left[\sum_{t=0}^\Tmax \RewardRV^\DecP_t\right]
\overset{\text{(*)}}{=}
\E_{\Isom^*\Policy}\left[\sum_{t=0}^\Tmax 
\RewardRV^\DecPSecond_t\right]=J^\DecPSecond(\Isom^*\Policy).\end{equation}
\end{proof}

\subsection{Relation to group theory}
\label{appendix-relation-to-group-theory}

The study of symmetries is a focus of group theory, and the concepts introduced above hence correspond to group-theoretic notions. For instance, as we show below, \(\Aut(\DecP)\) is a group, and its elements do act on the elements of a Dec-POMDP in the sense of group actions. We discuss this here as we will need these results later, in the discussion of symmetric profiles of learning algorithms in Appendix~\ref{appendix-optimal-symmetric-strategy-profiles}, as well as in the discussion of random tie-breaking functions in Appendix~\ref{appendix-random-tie-breaking-functions}. For a reference on the group-theoretic concepts discussed here, see \textcite[][ch.~3]{rotman2012introduction}.

We begin by showing that \(\Aut(\DecP)\) is a group.

\begin{prop}Let \(\DecP\) be a Dec-POMDP. Then
\((\Aut(\DecP),\circ)\) is a group, where \(\circ\) is the function composition.
\end{prop}
\begin{proof}
First, we show that the binary operation \[
\circ\colon \Aut(\DecP)\times\Aut(\DecP)\rightarrow\Aut(\DecP),(\Auto,\AutoSecond)\mapsto\Auto\circ\AutoSecond\]
is well-defined.
By Equation~\ref{eq:isomorphisms-is-bijective-map}, an automorphism \(\Auto\in\Aut(\DecP)\) is a bijective self-map, so we can compose any two automorphisms \(\Auto,\AutoSecond\in\Aut(\DecP)\). Moreover, by Lemma~\ref{lemma-inverse-composition-isomorphism}, for any \(\Auto\in\Aut(\DecP),\AutoSecond\in\Aut(\DecP)\), we also have \(\AutoSecond\circ\Auto\in\Iso(\DecP,\DecP)=\Aut(\DecP)\). This shows that \(\Aut(\DecP)\) is closed under function composition.

Second, note that \(\circ\) is an associative operation as function composition is associative. Moreover, for the identity map \(\Id\), it is \(\Id\circ \Auto=\Auto\) for any \(\Auto\in\Aut(\DecP)\), so \(\Aut(\DecP)\) has a neutral element. Lastly, by Lemma~\ref{lemma-inverse-composition-isomorphism}, it is also \(\Auto^{-1}\in \Aut(\DecP)\), and since \(\Auto^{-1}\circ\Auto=\Id\), this implies that \(\Auto\) has an inverse in \(\Aut(\DecP)\). This concludes the proof.
\end{proof}

Next, we turn to group actions, which formalize the idea that elements of groups can be applied to sets. In the case of symmetry groups, this connects the abstract group elements with their role as transformation of an underlying set. For instance, consider the set \(X\) of the vertices of an equilateral triangle in \(\mathbb{R}^2\) and the cyclic group \(\faktor{\mathbb{Z}}{3\mathbb{Z}}\). Each element of \(\faktor{\mathbb{Z}}{3\mathbb{Z}}\) can be regarded as a rotation of the vertices of the triangle, mapping one vertex to another.

\begin{defn}[Group action]
Let \((G,\cdot)\) be a group with identity \(\Id\) and let \(X\) be any set.
A group action is defined as a map \(\alpha\colon G\times X\rightarrow X\) such that
\begin{enumerate}
\item[(i)] Identity: \(\alpha(\Id,x)=x\) for any \(x\in X\)
\item[(ii)] Compatibility: \(\alpha(f, \alpha(g,x)) = \alpha(f\cdot g,x)\) for any \(f,g\in G\), \(x\in X\).
\end{enumerate}
\end{defn}
It is common to write \(\Auto x:=\alpha(\Auto,x)\) for \(\Auto\in G\), \(x\in X\), if it is clear which group action is referred to.

We have already proven these two properties for isomorphisms and both their actions on joint actions and observations, as well as the pushforward of policies, in Lemma~\ref{lemma-action-function-composition-compatible} and Lemma~\ref{lemma-pull-back-function-composition-compatible}, respectively. Hence, it follows that also \(\Aut(\DecP)\) acts on these sets in the sense of group actions.

\begin{cor}
Let \(\DecP\) be a Dec-POMDP.
The actions of \(\Aut(\DecP)\) on \(\ActionSet\) and \(\ObservationSet\), defined as \(\alpha_A\colon (\Auto,\Action)\mapsto\Auto_A(\Action)\) respectively \(\alpha_O\colon (\Auto,\Observation)\mapsto\Auto_O(\Observation)\) as in Equations~\ref{eq:13appendix} and \ref{eq:14appendix} are group actions. Similarly, the pushforward of policies by automorphisms \((\Auto,\Policy)\mapsto\Auto^*\Policy\) as defined in Definition~\ref{definition-pushforward-appendix} is an action of \(\Aut(\DecP)\) on \(\PolicySet^\DecP\).
\end{cor}
\begin{proof}
This follows directly from Lemma~\ref{lemma-action-function-composition-compatible} and Lemma~\ref{lemma-pull-back-function-composition-compatible}.
\end{proof}

Of course, automorphisms also act on states and agents, and one can also easily see that they act on histories and action-observation histories.

Some further results immediately follow from this, such as the fact that \(\PlayerSet\) decomposes into equivalence classes of \emph{orbits} under \(\Aut(\DecP)\). 
The orbit of agent \(i\) is defined as the set of all agents \(j\) that can be obtained from \(i\) by applying automorphisms.

\begin{defn}[Orbit] Let \(\DecP\) be a Dec-POMDP and assume that \(\Aut(\DecP)\) acts on the set \(X\). Then for \(x\in X\), the set
\[\Aut(\DecP)x:=\{\Auto x\mid\Auto\in\Aut(\DecP)\}\]
is called the orbit of \(x\) under \(\Aut(\DecP)\). 
\end{defn}

For instance, for \(i\in\PlayerSet\), the set
\(\Aut(\DecP)i:=\{\Auto i\mid\Auto\in\Aut(\DecP)\}\)
is the orbit of agent \(i\) under \(\Aut(\DecP)\). It is a standard result from group theory that orbits form a partition \(\{\Aut(\DecP)i\mid i\in\PlayerSet\}\subseteq\PowerSet(\PlayerSet)\) of the set. This follows from the fact that since group actions have an identity and are invertible, belonging to the same orbit is an equivalence relation. For example, in the case of the triangle in \(\mathbb{R}^2\), all the vertices in \(V\) can be reached from any other vertex by rotations, so they all belong to the same orbit. In general, though, the orbits may form any other partition of the set. 

\subsection{Dec-POMDP labelings and relabeled Dec-POMDPs}
\label{appendix-dec-pomdp-labelings-and-relabeled-dec-pomdps}

In the following, assume that a Dec-POMDP \(\DecP\) is given. Here, we want to define a set of isomorphic Dec-POMDPs \(\DecPSet\) as described in Section \ref{section-the-zero-shot-coordination-game}, in which the sets of states, actions, etc.\ are of the form \(\{1,2,\dotsc,k-1,k\}\subseteq\mathbb{N},k\in\mathbb{N}\). This set can then be used to define the LFC game for \(\DecP\) in a way that does not depend on labels.

We begin by defining a \emph{labeling} of \(\DecP\). A labeling \(\Isom\) is a special Dec-POMDP isomorphism from \(\DecP\) to another, relabeled Dec-POMDP, that can be constructed using \(\Isom\).

\begin{defn}[Dec-POMDP labeling]
A Dec-POMDP labeling is a tuple of bijective maps
\[\Isom:=(\Isom_{N},\Isom_{S},(\Isom_{A_i})_{i\in\PlayerSet},(\Isom_{O_i})_{i\in\PlayerSet}),\] where 
\begin{alignat}{2}
&\Isom_{N}\colon &&\PlayerSet\rightarrow\{1,\dotsc,|\PlayerSet|\}\\
& \Isom_{S}\colon &&\StateSet\rightarrow\{1,\dotsc,|\StateSet|\}\\
\forall i\in\PlayerSet\colon\quad&\Isom_{A_i}\colon &&\ActionSet_{i}\rightarrow\{1,\dotsc,|\ActionSet_i|\}\\
\forall i \in\PlayerSet\colon\quad&\Isom_{O_i}\colon
&&\ObservationSet_{i}\rightarrow\{1,\dotsc,|\ObservationSet_i|\}.
\end{alignat}
We denote \(\Sym(\DecP)\) for the set of labelings of \(\DecP\).
\end{defn}

Note that if \(X\) is some set, then \(\Sym(X)\) usually denotes the \emph{symmetric group} of \(X\). The symmetric group is the set of permutations of \(X\), together with the operation of function composition. We use the same notation, as \(\Sym(\DecP)\) can be understood of as containing all the permutations of the different sets that \(\DecP\) consists of, with the caveat that we first map those sets to subsets of the first \(k\) natural numbers. This is done for simplification, especially regarding the treatment of the individual action and observation sets of different agents.

Next, we introduce the pushforward Dec-POMDP \(\Isom^*\DecP\) for a labeling \(\Isom\in\Sym(\DecP)\). This is a Dec-POMDP that is isomorphic to \(\DecP\), with isomorphism \(\Isom\). In the following, we let \(\Isom\in\Sym(\DecP)\) act on joint actions, observations, etc., in the same way as before for isomorphisms. For instance, for \(\Action\in\ActionSet\), it is \(\Isom\Action:=(\Isom_{A_{\Isom_N^{-1}(i)}}(\Action_{\Isom_N^{-1}(i)}))_{i\in \{1,\dotsc,|\PlayerSet|\}}\). The compatibility of these actions with function composition and function inversion trivially still hold.

\begin{defn}[Relabeled Dec-POMDP]
Let $f\in\Sym(D)$. Let \(N\in\mathbb{N}\) such that \(\{1,\dotsc,N\}=\PlayerSet\). The pushforward of \(\DecP\) by \(\Isom\), called a relabeled Dec-POMDP, is the Dec-POMDP
\[
f^*D:=(\hat{\mathcal{N}},\hat{\StateSet},\hat{\ActionSet},\hat{P},\hat{\RewardFunction},\hat{\ObservationSet},\hat{O},\hat{b}_{0},\hat{\Tmax}),
\]
where
\begin{itemize}
\item \(\hat{\PlayerSet}:=\{1,\dotsc,N\}=\PlayerSet\).
\item \(\hat{\StateSet}:=\{1,\dotsc,|\StateSet|\}\).
\item $\hat{\ActionSet}_{i}:=\{1,\dotsc,|\ActionSet_{\Isom^{-1}i}|\}$ for \(i\in \hat{\PlayerSet}\).
\item $\hat{P}(s'\mid s,a):=P(\Isom^{-1}s'\mid \Isom^{-1}s,\Isom^{-1}a)$ for \(s',s\in\hat{\StateSet},\Action\in\hat{\ActionSet}\).
\item $\hat{\RewardFunction}(s,a):=\RewardFunction(f^{-1}s,f^{-1}a)$ for  \(\State\in\hat{\StateSet},\Action\in\hat{\ActionSet}\).
\item $\hat{\ObservationSet}_{i}:=\{1,\dotsc,|\ObservationSet_{\Isom^{-1}i}|\}$ for \(i\in \hat{\PlayerSet}\).
\item $\hat{O}(o\mid s,a):=O(\Isom^{-1}o\mid \Isom^{-1}s,\Isom^{-1}a)$ for \(\Observation\in\hat{\ObservationSet},\State\in\hat{\StateSet},\Action\in\hat{\ActionSet}\).
\item $\hat{b}_{0}(s):=b_{0}(f^{-1}s)$ for \(\State\in\hat{\StateSet}\).
\item \(\hat{\Tmax}:=\Tmax\).
\end{itemize}
\end{defn}

First, we have to check that this is well-defined, e.g., that \(\Isom^{-1}\Action\in\ActionSet\) for any \(\Action\in\hat{\ActionSet}\). For states, it is clear from the definition of \(\Sym(\DecP)\) that \(\Isom^{-1}(\hat{\StateSet})=\Isom^{-1}(\{1,\dotsc,|\StateSet|\})=\StateSet\). Moreover, the same applies to agents, i.e., \(\Isom^{-1}i\in \PlayerSet\) for any \(i\in\{1,\dotsc,|\PlayerSet|\}\). This leaves joint actions and observations.

\begin{prop}
For $\hat{\ActionSet},\hat{\ObservationSet}$ as defined above, it is $\Isom^{-1}(\hat{\ActionSet})=\ActionSet$
and $\Isom^{-1}(\hat{\ObservationSet})=\ObservationSet$.
\end{prop}
\begin{proof}
Let $\hat{a}\in\hat{\ActionSet}$. Then for any \(i\in \hat{\PlayerSet}\), we can define \(j\in\PlayerSet\) and \(\Action_{j}\in\ActionSet_{j}\) such that \(\Isom j = i\) and \(\Isom \Action_j = \hat{\Action}_i\). 
Then
\[\Isom^{-1}\hat{a}=
(\Isom^{-1} \hat{\Action}_{\Isom j}))_{j\in \PlayerSet}
=
( \Isom^{-1}\Isom\Action_{j}))_{j\in\PlayerSet}
=(\Action_j)_{j\in\PlayerSet}=\Action\in\ActionSet
.\]
The same argument works for $\hat{o}\in\hat{\ObservationSet}$.
\end{proof}

Importantly, it can be \(\Isom^*\DecP\neq\DecP\) for a labeling \(\Isom\in\Sym(\DecP)\). 
Nevertheless, it is easy to see from the definitions that \(\Isom^*\DecP\) is actually isomorphic to \(\DecP\), with isomorphism \(\Isom\). This also implies that it is \(\Isom^*\DecP=\DecP\) if and only if \(\Isom\) is an automorphism.

\begin{lem}\label{prop-pushforward-creates-isomorphic-dec-pomdp}
For any \(\Isom\in \Sym(\DecP)\), it is \(\Isom\in\Iso(\DecP,\Isom^*\DecP)\).
\end{lem}
\begin{proof}
This follows directly from the definition of an isomorphism, together with Lemma~\ref{lemma-action-function-composition-compatible}. For instance, considering transition probabilities, it is
\[P(s'\mid s,a)=P(\Isom^{-1}\Isom s'\mid \Isom^{-1}\Isom s, \Isom^{-1}\Isom a)=\hat{P}(\Isom s'\mid \Isom s \mid \Isom a)
\]
for any \(s',s\in\StateSet,\Action\in\ActionSet\). Similar calculations apply to all the other relevant functions.
\end{proof}

It follows as a corollary that \(\Isom^*\Policy\) is a policy for the Dec-POMDP \(\Isom^*\DecP\), where \(\Policy\in\PolicySet^\DecP,\Isom\in\Sym(\DecP)\). Note also that the results in Lemma~\ref{lemma-pull-back-function-composition-compatible} still apply to the pushforward by labelings.

\begin{cor}
Let \(\Isom\in\Sym(\DecP)\) and \(\Policy\in\PolicySet^\DecP\). Then \(\Isom^*\Policy\in\PolicySet^{\Isom^*\DecP}\).
\end{cor}
\begin{proof}
Follows from the definition of \(\Isom^*\Policy\) and Lemma~\ref{prop-pushforward-creates-isomorphic-dec-pomdp}.
\end{proof}

Lastly, we provide some further useful results about labelings. First, the set \(\Sym(\DecP)\) already contains all the isomorphisms in \(\Iso(\DecP,\Isom^*\DecP)\).
We will need this result later to relate results about isomorphisms and automorphisms to the relabeled Dec-POMDPs used in an LFC game.

\begin{lem}\label{lem-sym-is-superset-of-iso-pushforward}
Let \(\Isom\in\Sym(\DecP)\). Then
\[\Iso(\DecP,\Isom^*\DecP)=\{\IsomSecond\in\Sym(\DecP)\mid \IsomSecond^*\DecP=\Isom^*\DecP\}\]
\end{lem}
\begin{proof}
``\(\supseteq\)'': for any \(\IsomSecond\in\Sym(\DecP)\) such that \(\IsomSecond^*\DecP=\Isom^*\DecP\), it is also \(\Iso(\DecP,\IsomSecond^*\DecP)=\Iso(\DecP,\Isom^*\DecP)\), and thus it follows from Lemma~\ref{prop-pushforward-creates-isomorphic-dec-pomdp} that \(\IsomSecond\in \Iso(\DecP,\IsomSecond^*\DecP)=\Iso(\DecP,\Isom^*\DecP)\).

``\(\subseteq\)'': Let \(\IsomHat\in \Iso(\DecP,\Isom^*\DecP)\). Note that the set of agents, states, and the individual action and observation sets in \(\Isom^*\DecP\) are all of the form \(\{1,\dotsc,k\}\) where \(k\in\mathbb{N}\) depends on the respective set. Now consider, for instance, the map \(\IsomHat_{A_i}\) for \(i\in\PlayerSet\). Then by the definition of an isomorphism, \(\IsomHat_{A_i}\) is a bijective map, and its domain and codomain are \(\ActionSet_i\) and \(\{1,\dotsc,k\}\) for some \(k\in\mathbb{N}\). Moreover, since \(\IsomHat_{A_i}\) is bijective, it must be \(k=|\ActionSet_i|\). \emph{Mutatis mutandis}, the same applies to all of the other maps that are part of the tuple \(\IsomHat\). Hence, \(\IsomHat\) satisfies the definition of a Dec-POMDP labeling, so \(\IsomHat\in\Sym(\DecP)\).

Next, it follows that \(\IsomHat\in\Iso(\DecP,\IsomHat^*\DecP)\) by Lemma~\ref{prop-pushforward-creates-isomorphic-dec-pomdp}, and thus \(\Id=\IsomHat\circ\IsomHat^{-1}\in \Iso(\Isom^*\DecP,\IsomHat^*\DecP)\) by the assumption and Lemma~\ref{lemma-inverse-composition-isomorphism}. Hence, using the definition of an isomorphism, it follows that also \(\Isom^*\DecP=\IsomHat^*\DecP\). This shows that
\[\IsomHat\in\{\IsomSecond\in\Sym(\DecP)\mid \IsomSecond^*\DecP=\Isom^*\DecP\},\]
which concludes the proof.
\end{proof}

Second, we show that labelings and pushforward are compatible with composition with isomorphisms.

\begin{lem}
\label{d-sym-closed-under-labeling}
Let \(\DecP,\DecPSecond\) be isomorphic Dec-POMDPs with \(\Isom\in\Iso(\DecP,\DecPSecond)\). Then
\[\Sym(\DecP)=\Sym(\DecPSecond)\circ \Isom.\]
Moreover, it is \(\IsomSecond^*\DecPSecond=(\IsomSecond\circ\Isom)^*\DecP\) for any \(\IsomSecond\in\Sym(\DecPSecond)\).
\end{lem}
\begin{proof}
First, let \(\IsomSecond\in\Sym(\DecP)\) and define \(\IsomHat:=\IsomSecond\circ\Isom^{-1}\). Note that \(\IsomHat\) has as components bijective maps with a domain and codomain that satisfies the definition of a labeling of \(\DecPSecond\). Hence, \(\IsomHat\in\Sym(\DecPSecond)\). Next, let \(\IsomSecond\in\Sym(\DecPSecond)\). Then similarly, \(\IsomSecond\circ\Isom\) fulfills the requirements for a labeling in \(\Sym(\DecP)\).

To prove the second statement, let again \(\IsomSecond\in\Sym(\DecPSecond)\). Note that since \(\DecP\) and \(\DecPSecond\) are isomorphic, they must have the same set of players and sets of states with the same cardinalities. Now let \(i\in\PlayerSet^\DecP\). Using the definition of an isomorphism and of a labeling, it is then
\[\ActionSet^{\IsomSecond^*\DecPSecond}_i=\{1,\dotsc,|\ActionSet_{\IsomSecond^{-1}i}^\DecPSecond|\}
=\{1,\dotsc,|\ActionSet_{\Isom^{-1}(\IsomSecond^{-1}i)}^\DecP|\}=\{1,\dotsc,|\ActionSet_{(\IsomSecond\circ\Isom)^{-1}i}^\DecP|\}= \ActionSet^{(\IsomSecond\circ\Isom)^*\DecP}_i.\]
A similar argument applies to the sets \(\ObservationSet_i\) for \(i\in\PlayerSet\). Finally, let \(\State,\State'\in\StateSet^{\IsomSecond^*\DecPSecond},\Action\in\ActionSet^{\IsomSecond^*\DecPSecond}\). Using again the definition of an isomorphism and a labeling, it follows that
\begin{multline}P^{\IsomSecond^*\DecPSecond}(\State'\mid\State,\Action)
= P^\DecPSecond(\IsomSecond^{-1}\State'\mid\IsomSecond^{-1}\State,\IsomSecond^{-1}\Action)
\\
= P^\DecP(\Isom^{-1}\IsomSecond^{-1}\State'\mid\Isom^{-1}\IsomSecond^{-1}\State,\Isom^{-1}\IsomSecond^{-1}\Action)
=P^{(\IsomSecond\circ\Isom)^*\DecP}(\State'\mid\State,\Action).\end{multline}
Again, an analogous argument applies to the observation probability kernel and reward function, as well as the initial state distribution.
This concludes the proof.
\end{proof}

\subsection{The label-free coordination game and problem}
\label{appendix-the-zero-shot-coordination-game}

Here, we recall the definitions of an LFC game and of the LFC problem. To begin, we define a measure space of policies and recall the definition of a learning algorithm. For any Dec-POMDP \(\DecP\), let \(\Delta(\PolicySet^\DecP)\) be the set of measures on the space \((\PolicySet^\DecP,\mathcal{F}^\DecP)\) where \(\mathcal{F}^\DecP:=\otimes_{i\in\PlayerSet}\mathcal{F}^\DecP_i\) is a product \(\sigma\)-Algebra and \(\mathcal{F}_i^\DecP\subseteq\PowerSet(\PolicySet_i^\DecP)\) are \(\sigma\)-Algebras that make the random variables and sets discussed in this paper measurable. For instance, for \(i\in\PlayerSet\), this could be the Borel \(\sigma\)-Algebra with respect to the standard topology on \(\PolicySet_i^\DecP\) that comes from regarding \(\PolicySet_i^\DecP\) as a subset of \([0,1]^{\ActionSet^\DecP_i\times\AOHistorySet^\DecP_i}\). Although we do not investigate this here, all the relevant functions and sets should be measurable in that sense. 

\begin{defn}[Learning algorithm] Let \(\DecPSet\) be a finite set of Dec-POMDPs. A learning algorithm for \(\DecPSet\) is a map \[\LA\colon \DecPSet\rightarrow \bigcup_{D\in\mathcal{D}}\Delta(\PolicySet^\DecP)\]
such that \(\LA(D)\in\Delta(\PolicySet^D)\) for all \(D\in\mathcal{D}\). We write \(\LASet^\mathcal{D}\) for the set of learning algorithms for \(\mathcal{D}\).
\end{defn}

Note that this definition is general enough so as to include planning algorithms that construct a policy directly from the environment dynamics, instead of incrementally updating a policy from experience. Nevertheless, here, we imagine that \(\LA(\DecP)\) is a policy that was trained by an RL algorithm, using a simulator of \(\DecP\). Note also that a learning algorithm can learn different joint policies in different training runs, which we formalize as outputting a measure over joint policies.

Similarly to the case of policies, for a distribution \(\Distr\in\Delta(\PolicySet^\DecP)\) and an isomorphism \(\Isom\in\Iso(\DecP,\DecPSecond)\), we can define a pushforward distribution \(\Isom^*\Distr:=\Distr \circ (\Isom^*)^{-1}\in\Delta(\PolicySet^\DecPSecond)\), which is the image measure of \(\Distr\) under \(\Isom^*\). It is apparent that for two isomorphisms \(\Isom\in\Iso(\DecP,\DecPSecond),\IsomSecond\in\Iso(\DecPSecond,\DecPThird)\), it is \(\IsomSecond^*(\Isom^*\Distr)=(\IsomSecond\circ\Isom)^*\Distr\).

In the following, for some distributions \(\Distr^{(i)}\in\Delta(\PolicySet^\DecP)\) for \(i\in\PlayerSet\) and bounded measurable function \(\eta\colon \PolicySet^\DecP\times\dotsb\times\PolicySet^\DecP\rightarrow\mathbb{R}\), we will use the notational shorthands
\[\E_{\Policy^{(i)}\sim\Distr^{(i)},\,i\in\PlayerSet}
    \left[\eta(\Policy^{(1)},\dotsc,\Policy^{(N)})\right]
:=\E_{\Policy^{(1)}\sim\Distr^{(1)}}\left[\dots\left[\E_{\Policy^{(N)}\sim\Distr^{(N)}}
    \left[\eta(\Policy^{(1)},\dotsc,\Policy^{(N)})\right]\right]\dots\right]
\]
 and
\[\E_{\Policy^{(i)}\sim\Distr^{(i)}}
    \left[
\eta(\Policy^{(1)},\dotsc,\Policy^{(N)})\right]:=\int_{\PolicySet^\DecP}\eta(\Policy^{(1)},\dotsc,\Policy^{(N)})\mathrm{d}\Distr^{(i)}(\Policy^{(i)}).\]
Note that by Fubini's theorem \parencite[see][ch.~8]{williams1991probability}, it is
\[\E_{\Policy^{(i)}\sim\Distr^{(i)},\,i\in\PlayerSet}
    \left[\eta(\Policy^{(1)},\dotsc,\Policy^{(N)})\right]
    =\int_{\PolicySet^\DecP\times\dotsb\times\PolicySet^\DecP}\eta(\Policy^{(1)},\dots,\Policy^{(N)})\mathrm{d}\otimes_{i\in\PlayerSet}\Distr^{(i)}.
    \]
    
Now we define the LFC game for a Dec-POMDP.
    
\begin{defn}[Label-free coordination game]
Let \(\DecP\) be a Dec-POMDP and define \(\DecPSet:=\{\Isom^*\DecP\mid\Isom\in\Sym(\DecP)\}\). The \emph{label-free coordination (LFC) game} for \(\DecP\) is defined as a tuple \(\Gamma^\DecP:=(\PlayerSet^\DecP,(\LASet^\DecPSet)_{i\in\PlayerSet},(U^{\DecP})_{i\in\PlayerSet})\) where
\begin{itemize}
    \item \(\PlayerSet^\DecP\) is the set of players, called principals.
    \item \(\LASet^\DecPSet\) is the set of strategies for all principals \(i\in\PlayerSet^\DecP\).
    \item the common payoff for the strategy profile \(\LAProfile_1,\dotsc,\LAProfile_N\in\LASet^\DecP\) is 
    \begin{equation}U^\DecP(\LAProfile):=
    \E_{\DecP_i\sim \U(\DecPSet),\,i\in\PlayerSet}\Big[\E_{\Isom_j\sim\U(\Iso(\DecP_j,\DecP)),\,j\in\PlayerSet}\Big[\\
    \E_{\Policy^{(k)}\sim\Isom_k^*\LAProfile_k(\DecP_k),\,k\in\mathcal{N}}
    \Big[
J^\DecP((\Policy^{(l)}_l)_{l\in\PlayerSet})\Big]\Big]\Big],\end{equation}
where \(\mathcal{U}(\DecPSet)\) is a uniform distribution over \(\DecPSet\) and \(\U(\Iso(\DecP_j,\DecP))\) a uniform distribution over \(\Iso(\DecP_j,\DecP)\).
\end{itemize}
\end{defn}
\begin{remark}
Note that set of strategies \(\LASet^\DecP\) is continuous. The game \(\Gamma^\DecP\) could hence be considered a continuous game, which is a generalization of the concept of a normal-form game to continuous strategy spaces \parencite[see][]{glicksberg1952generalization}. In continuous games, it is usually assumed that the set of strategies is compact and that the payoffs are continuous functions, which we believe does apply in our case.

Moreover, we believe that there is some other normal-form game with finite strategy space such that the \emph{mixed strategies} in that game correspond to the set of strategies \(\LASet^\DecPSet\) in \(\Gamma^\DecP\) \parencite[for a reference on these concepts from game theory, see][]{osborne1994course,gibbons1992game}. In particular, one can see that the set of strategies \(\LASet^\DecPSet\) is already convex.

We do not need any further characterization of an LFC game in the following, so we do not investigate issues such as compactness or convexity of the set of strategies. The formalism at hand was chosen primarily to work well with an intuitive formulation of the LFC problem and to suit our discussion of the OP algorithm.
\end{remark}

Next, to recall the definition of the LFC problem, let any set \(\DecPSetSecond\) of Dec-POMDPs be given, and denote \(\overline{\DecPSetSecond}:=\bigcup_{\DecP\in\DecPSetSecond}\DecPSet^\DecP\) where \(\DecPSet^\DecP:=\{\Isom^*\DecP\mid\Isom\in\Sym(\DecP)\}\) is the set of all relabeled problems of \(\DecP\). The LFC problem for \(\DecPSetSecond\) is then defined as the problem of finding one learning algorithm \(\LA\in\LASet^{\overline{\DecPSetSecond}}\) to be used by principals in a randomly drawn game \(\Gamma^\DecP\) for \(\DecP\sim\U(\DecPSetSecond)\).

\begin{defn}[Label-free coordination problem]
Let \(\DecPSetSecond\) be any set of Dec-POMDPs.
Define the objective \(U^\DecPSetSecond\colon \LASet^{\overline{\DecPSetSecond}}\rightarrow\mathbb{R}\) via
\begin{equation}
U^\DecPSetSecond(\LA):=\E_{\DecPSecond\sim\U(\DecPSetSecond)}\left[U^\DecPSecond(\LA,\dotsc,\LA)\right]
\end{equation}
for
\(\LA\in\LASet^{\overline{\DecPSetSecond}}\). Then we define the \emph{Label-free coordination (LFC) problem} for \(\DecPSetSecond\) as the optimization problem
\begin{equation}
\max_{\LA\in\Sigma^{\overline{\DecPSetSecond}}}U^\DecPSetSecond(\LA)\end{equation}
and we call \(U^\DecPSetSecond(\LA)\) the value of \(\LA\) in the LFC problem for \(\DecPSetSecond\).
If \(\DecPSetSecond=\{\DecP\}\), we write \(U^\DecP:=U^{\{\DecP\}}\) in a slight abuse of notation and refer to this as the LFC problem for \(\DecPSecond\).
\end{defn}

\begin{remark}
The aim of the LFC problem is to find a general learning algorithm to recommended to principals in any LFC game. For this reason, we defined the problem here for a distribution over LFC games. However, a learning algorithm is optimal in the problem for a set of Dec-POMDPs if and only if it is optimal in the problem for each Dec-POMDP in that set. That is because, as one can easily see, the sets \(\DecPSet^\DecP,\DecPSet^\DecPSecond\) do never overlap for two non-isomorphic Dec-POMDPs \(\DecP,\DecPSecond\), and we will show in Corollary~\ref{lem-lfc-problems-isomorphic-decpompds-identical} that the LFC problems for two isomorphic Dec-POMDPs are identical. So to evaluate a learning algorithm in the LFC problem for a set of Dec-POMDPs, we can decompose the set into equivalence classes of isomorphic Dec-POMDPs and evaluate the learning algorithm separately for each of these classes. In the following, we will thus simplify our analysis and restrict ourselves entirely to problems defined for single Dec-POMDPs. Note that the objective in the LFC problem is then simply
\begin{equation}U^\DecP(\LA)=U^{\{\DecP\}}(\LA)=\E_{\DecPSecond\sim\U(\{\DecP\})}\left[U^\DecPSecond(\LA,\dotsc,\LA)\right]=U^\DecP(\LA,\dotsc,\LA).
\end{equation}
If we then prove, e.g., that a learning algorithm is optimal in any such problem, it follows that it is also optimal for the problem defined for any set of Dec-POMDPs.
\end{remark}

Now we will provide two different expressions of the payoff in an LFC game. To that end, let \(\DecP\) be a Dec-POMDP and define \(\DecPSet:=\{\Isom^*\DecP\mid\Isom\in\Sym(\DecP)\}\). First, we provide an expression in terms of labelings.

\begin{lem}
\label{lem-payoff-lfc-game-old-form}
Let \(\LAProfile_1,\dotsc,\LAProfile_N\in\LASet^\DecPSet\).
Then
\begin{equation}
    U^\DecPSet(\LAProfile_1,\dotsc,\LAProfile_N)
    =
    \E_{\SymProfile\sim \U(\Sym(\DecP)^\PlayerSet)}\left[
    \E_{\Policy^{(i)}\sim(\SymProfile_i^{-1})^*\LAProfile_i(\SymProfile_i^*\DecP),\,i\in\mathcal{N}}
    \left[
J^\DecP((\Policy^{(j)}_j)_{j\in\PlayerSet})\right]\right]
\end{equation}
\end{lem}
\begin{proof}
It follows from Lemma~\ref{lem-sym-is-superset-of-iso-pushforward} that it is \(\Sym(\DecP)=\bigcup_{\DecPSecond\in\DecPSet}\Iso(\DecP,\DecPSecond)\), where one can easily see that the union is disjoint (i). Moreover, it follows from Lemma~\ref{decompose-isomorphisms} that \(|\Iso(\DecP,\DecPSecond)|=|\Iso(\DecP,\DecPThird)|\) for any Dec-POMDPs \(\DecPSecond,\DecPThird\in\DecPSet\), so there exists \(M\in\mathbb{N}\) such that \(M=|\Iso(\DecP,\DecPSecond)|\) for any \(\DecPSecond\in\DecPSet\), and from (i) it follows that \(|\Sym(\DecPSecond)|=|\DecPSet||M|\) (ii).

Next, by Lemma~\ref{lemma-inverse-composition-isomorphism}, for any \(\Isom\in\Iso(\DecP,\DecP_j)\), it is \(\Isom^{-1}\in \Iso(\DecP_j,\DecP)\) for any \(\Isom\in\Iso(\DecP_j,\DecP)\). In addition, by the same Lemma, for \(\Isom\in\Iso(\DecP_j,\DecP)\), it is \(\Isom=(\Isom^{-1})^{-1}\) and \(\Isom^{-1}\in \Iso(\DecP,\DecP_j)\) and thus \(\Isom\in\{\IsomSecond^{-1}\mid\IsomSecond\in\Iso(\DecP,\DecP_j)\}\). Hence, it is
\begin{equation}
\label{eq:507}
\{\IsomSecond\mid\IsomSecond\in\Iso(\DecP_j,\DecP)\}=\{\IsomSecond^{-1}\mid\IsomSecond\in\Iso(\DecP,\DecP_j)\}.\end{equation}

Using the above, it follows that
\begin{align}&U^{\DecPSet}(\LAProfile_1,\dotsc,\LAProfile_N)\label{eq:36c}
\\
&=
\E_{\DecP_i\sim \U(\DecPSet),\,i\in\PlayerSet}\left[
    \E_{\IsoProfile_j\sim\Iso(\DecP_j,\DecP), \,j\in\PlayerSet}\left[
    \E_{\Policy^{(k)}\sim\SymProfile_k^*\LAProfile_k(\DecP_k),\,k\in\mathcal{N}}
    \left[
J^\DecP((\Policy^{(l)}_l)_{l\in\PlayerSet})\right]\right]\right]\label{eq:36d}
\\
&\overset{(\ref{eq:507})}{=}\label{eq:506a}
\E_{\DecP_i\sim \U(\DecPSet),\,i\in\PlayerSet}\left[
    \E_{\IsoProfile_j\in\Iso(\DecP,\DecP_j),\,j\in\PlayerSet}\left[
    \E_{\Policy^{(k)}\sim(\IsoProfile_k^{-1})^*\LAProfile_k(\DecP_k),\,k\in\mathcal{N}}
    \left[
J^\DecP((\Policy^{(l)}_l)_{l\in\PlayerSet})\right]\right]\right]
\\
&\overset{\text{(i),\,(ii)}}{=}
    \E_{\SymProfile_i\sim \U(\Sym(\DecPSecond)),\,i\in\PlayerSet}\left[
    \E_{\Policy^{(j)}\sim(\SymProfile_j^{-1})^*\LAProfile_j(\SymProfile_j^*\DecP),\,j\in\mathcal{N}}
    \left[
J^\DecP((\Policy^{(k)}_k)_{k\in\PlayerSet})\right]\right]
\\
&=\E_{\SymProfile\sim \U(\Sym(\DecPSecond)^\PlayerSet)}\left[
    \E_{\Policy^{(i)}\sim(\SymProfile_i^{-1})^*\LAProfile_i(\SymProfile_i^*\DecP),\,i\in\mathcal{N}}
    \left[
J^\DecP((\Policy^{(j)}_j)_{j\in\PlayerSet})\right]\right].
\end{align}
This concludes the proof.
\end{proof}

Second, we can prove a useful decomposition of the payoff in an LFC game into isomorphisms and automorphisms. We can already see here the connection to the OP objective (we will recall the OP objective in Appendix~\ref{appendix-generalization-of-other-play}). Recall the projection operator, \(\Proj_i(x):=x_i\) for \(x=(x_i)_{i\in\PlayerSet}\).

\begin{lem}
\label{lemma-decomposition-payoff-zero-shot-coordination-game}
Let \(\LAProfile_1,\dotsc,\LAProfile_N\in\LASet^\DecPSet\).
For any \(\DecPSecond,\DecPThird\in\DecPSet\), choose \(\Isom_{\DecPSecond,\DecPThird}\in\Iso(\DecPSecond,\DecPThird)\) arbitrarily. Then
\begin{align}
    &U^\DecPSet(\LAProfile_1,\dotsc,\LAProfile_N)
   \\ &=\label{eq:36b}
    \E_{\DecP_i\sim \U(\DecPSet),\,i\in\PlayerSet}\left[
    \E_{\Policy^{(j)}\sim\Isom_{\DecP_j,\DecP}^*\LAProfile_j(\DecP_j),\,j\in\mathcal{N}}
    \left[
        \E_{\AutProfile\in\Aut(\DecP)^\PlayerSet}\left[
    J^\DecP\left(\left(\Proj_k(\AutProfile_k^*\Policy^{(k)})\right)_{k\in\PlayerSet}\right)\right]\right]\right].
\end{align}
\end{lem}

\begin{proof}
Let \(i\in\PlayerSet,\DecP_i\in\DecPSet\). We know from Lemma~\ref{decompose-isomorphisms} that for any \(\Isom\in \Iso(\DecP_i,\DecP)\) there is a unique \(\Auto\in\Aut(\DecP)\) such that \(\Isom=\Auto\circ\Isom_{\DecP_i,\DecP}\). Also, for the pushforward measure, it is \((\Auto\circ\Isom_{\DecP_i,\DecP})^*\LAProfile_i(\DecP_i)=\Auto^*\Isom_{\DecP_i,\DecP}^*\LAProfile_i(\DecP_i)\) for any \(\Auto\in\Aut(\DecP)\). Using this, it follows that
\begin{align}&\label{eq:36e}
\E_{\DecP_i\sim \U(\DecPSet),\,i\in\PlayerSet}\left[
    \E_{\IsoProfile_j\sim\U(\Iso(\DecP_j,\DecP)), \,j\in\PlayerSet}\left[
    \E_{\Policy^{(k)}\sim\SymProfile_k^*\LAProfile_k(\DecP_k),\,k\in\mathcal{N}}
    \left[
J^\DecP((\Policy^{(l)}_l)_{l\in\PlayerSet})\right]\right]\right]
\\
&=
\E_{\DecP_i\sim \U(\DecPSet),\,i\in\PlayerSet}\left[
    \E_{\AutProfile\sim\U(\Aut(\DecP)^\PlayerSet)}\left[
    \E_{\Policy^{(k)}\sim\AutProfile_k^*\Isom_{\DecP_k,\DecP}^*\LAProfile_k(\DecP_k),\,k\in\mathcal{N}}
    \left[
J^\DecP((\Policy^{(j)}_j)_{j\in\PlayerSet})\right]\right]\right]
\\
&=
\E_{\DecP_i\sim \U(\DecPSet),\,i\in\PlayerSet}\left[
    \E_{\Policy^{(k)}\sim\Isom_{\DecP_k,\DecP}^*\LAProfile_k(\DecP_k),\,k\in\mathcal{N}}
    \left[
        \E_{\AutProfile\sim\U(\Aut(\DecP)^\PlayerSet)}\left[
J^\DecP\left(\left(\Proj_j(\AutProfile_j^*\Policy^{(j)})\right)_{j\in\PlayerSet}\right)\right]\right]\right],\label{eq:36f}
\end{align}
where we have used a change of variables for pushforward measures in the last line.
\end{proof}

Finally, we can show that the payoff in an LFC game is equal for isomorphic Dec-POMDPs, up to a possible permutation of principals. Intuitively, this means that the game does not depend on labels for the problem.

\begin{thm}
\label{zero-shot-coordination-game-invariant-to-isomorphism}
Let \(\DecP,\DecPSecond\) be isomorphic and \(\Isom\in\Iso(\DecP,\DecPSecond)\) arbitrary. Define
\(\DecPSet:=\{\Isom^*\DecP\mid\Isom\in\Sym(\DecP)\}\) and \(\DecPSetSecond:=\{\Isom^*\DecPSecond\mid\Isom\in\Sym(\DecPSecond)\}\). 
Then \(\DecPSet=\DecPSetSecond\),
and for any profile of algorithms \(\LAProfile=(\LAProfile_1,\dotsc,\LAProfile_N)\in\Sigma^\DecPSet\), it is
\[U^{\DecP}(\LAProfile_1,\dotsc,\LAProfile_N)
=U^{\DecPSecond}(\LAProfile_{\Isom^{-1}1},\dotsc,\LAProfile_{\Isom^{-1}N}).
\]
\end{thm}
\begin{proof}
First, note that by Lemma~\ref{d-sym-closed-under-labeling}, it is
\begin{multline}\DecPSetSecond=\{\IsomSecond^*\DecPSecond\mid\IsomSecond\in\Sym(\DecPSecond)\}\overset{\text{Lemma~\ref{d-sym-closed-under-labeling}}}{=}
\{(\IsomSecond\circ\Isom)^*\DecP\mid\IsomSecond\in\Sym(\DecPSecond)\}
\\=
\{\IsomHat^*\DecP\mid\IsomHat\in\Sym(\DecPSecond)\circ\Isom\}\overset{\text{Lemma~\ref{d-sym-closed-under-labeling}}}{=}\{\IsomHat^*\DecP\mid\IsomHat\in\Sym(\DecP)\}=\DecPSet.\end{multline}
Now let \(\LAProfile_1,\dotsc,\LAProfile_N\in\LASet^\DecPSet\) arbitrary. 
Then, using the expression of \(U^\DecP\) from Lemma~\ref{lem-payoff-lfc-game-old-form}, it is
\begin{align}&U^{\DecP}(\LAProfile_{1},\dotsc,\LAProfile_N)
\\&=
\E_{\SymProfile\sim \U(\Sym(\DecP)^\PlayerSet)}\left[
    \E_{\Policy^{(i)}\sim(\SymProfile_i^{-1})^*\LAProfile_i(\SymProfile_i^*\DecP),\,i\in\mathcal{N}}
    \left[
J^\DecP((\Policy^{(j)}_j)_{j\in\PlayerSet})\right]\right]
\\
&=
\E_{\SymProfile\sim \U(\Sym(\DecP)^\PlayerSet)}\left[
    \E_{\Policy^{(i)}\sim(\SymProfile_i^{-1})^*\LAProfile_i(\SymProfile_i^*\DecP),\,i\in\mathcal{N}}
    \left[
J^\DecPSecond(\Isom^*(\Policy^{(j)}_j)_{j\in\PlayerSet})\right]\right]
\label{eq:25-1}
\\
&=
\E_{\SymProfile\sim \U(\Sym(\DecP)^\PlayerSet)}\left[
    \E_{\Policy^{(i)}\sim(\SymProfile_i^{-1})^*\LAProfile_i(\SymProfile_i^*\DecP),\,i\in\mathcal{N}}
    \left[
J^\DecPSecond((\Policy^{(\Isom^{-1}j)}_{\Isom^{-1}j}(\Isom^{-1}\cdot\mid\Isom^{-1}\cdot))_{j\in\PlayerSet})\right]\right]
\label{eq:25-2}
\\
&=
\E_{\SymProfile\sim \U(\Sym(\DecP)^\PlayerSet)}\left[
    \E_{\Policy^{(i)}\sim(\SymProfile_i^{-1})^*\LAProfile_i(\SymProfile_i^*\DecP),\,i\in\mathcal{N}}
    \left[
J^\DecPSecond((\Proj_j(\Isom^*\Policy^{(\Isom^{-1}j)}))_{j\in\PlayerSet})\right]\right]
\label{eq:25-3}
\\
&=
\E_{\SymProfile\sim \U(\Sym(\DecP)^\PlayerSet)}\left[
    \E_{\Policy^{(i)}\sim\Isom^*(\SymProfile_i^{-1})^*\LAProfile_i(\SymProfile_i^*\DecP),\,i\in\mathcal{N}}
    \left[
J^\DecPSecond((\Proj_j(\Policy^{(\Isom^{-1}j)}))_{j\in\PlayerSet})\right]\right]
\label{eq:25-4}
\\
&=
\E_{\SymProfile\sim \U(\Sym(\DecP)^\PlayerSet)}\left[
    \E_{\Policy^{(i)}\sim((\SymProfile_i\circ \Isom^{-1})^{-1})^*\LAProfile_i(\SymProfile_i^*\DecP),\,i\in\mathcal{N}}
    \left[
J^\DecPSecond((\Proj_j(\Policy^{(\Isom^{-1}j)}))_{j\in\PlayerSet})\right]\right]
\label{eq:25-41}
\\
&=
\E_{\SymProfile\sim \U(\Sym(\DecP)^\PlayerSet)}\left[
    \E_{\Policy^{(i)}\sim((\SymProfile_i\circ \Isom^{-1})^{-1})^*\LAProfile_i((\SymProfile_i\circ\Isom^{-1})^*\DecPSecond),\,i\in\mathcal{N}}
    \left[
J^\DecPSecond((\Proj_j(\Policy^{(\Isom^{-1}j)}))_{j\in\PlayerSet})\right]\right]
\label{eq:25-5}
\\
&=
\E_{\SymProfile\sim \U(\Sym(\DecPSecond)^\PlayerSet)}\left[
    \E_{\Policy^{(i)}\sim(\SymProfile_i^{-1})^*\LAProfile_i(\SymProfile_i^*\DecPSecond),\,i\in\mathcal{N}}
    \left[
J^\DecPSecond((\Proj_j(\Policy^{(\Isom^{-1}j)}))_{j\in\PlayerSet})\right]\right]
\label{eq:25-6}
\\
&=
\E_{\SymProfile\sim \U(\Sym(\DecPSecond)^\PlayerSet)}\left[
    \E_{\Policy^{(i)}\sim(\SymProfile_{\Isom^{-1}i}^{-1})^*\LAProfile_{\Isom^{-1}i}(\SymProfile_{\Isom^{-1}i}^*\DecPSecond),\,i\in\mathcal{N}}
    \left[
J^\DecPSecond((\Proj_j(\Policy^{(j)}))_{j\in\PlayerSet})\right]\right]
\label{eq:25-7}
\\
&=
\E_{\SymProfile\sim \U(\Sym(\DecPSecond)^\PlayerSet)}\left[
    \E_{\Policy^{(i)}\sim(\SymProfile_{i}^{-1})^*\LAProfile_{ \Isom^{-1}i}(\SymProfile_{i}^*\DecPSecond),\,i\in\PlayerSet}
    \left[
J^\DecPSecond((\Proj_i(\Policy^{(i)}))_{i\in\PlayerSet})\right]\right]
\label{eq:25-8}
\\&=U^{\DecPSecond}(\LAProfile_{\Isom^{-1}1},\dotsc,\LAProfile_{\Isom^{-1}N}).
\end{align}
Here, in (\ref{eq:25-1}), we use Theorem~\ref{lem-pull-back-isomorphism-compatibility}; in (\ref{eq:25-2}) and (\ref{eq:25-3}), we use the definition of the pushforward policy; in (\ref{eq:25-4}), we apply a change of variables for pushforward measures, applied to each of the measures \((\SymProfile_i^{-1})^*\LAProfile_i(\SymProfile_i^*\DecP)\) separately; in (\ref{eq:25-41}), we use the associativity of the pushforward measure as well as the fact that \((\SymProfile_i\circ \Isom^{-1})^{-1}=\Isom\circ\IsoProfile_i^{-1}\); in (\ref{eq:25-5}) and in (\ref{eq:25-6}), we use the second respectively first part of Lemma~\ref{d-sym-closed-under-labeling}; in (\ref{eq:25-7}), we again use a change of variables for pushforward measures, this time applied to the joint measure \(\otimes_{i\in\PlayerSet}(\SymProfile_i^{-1})^*\LAProfile_i(\SymProfile_i^*\DecP)\); and in (\ref{eq:25-8}) we use the symmetry of the set \(\Sym(\DecPSecond)^\PlayerSet\) with respect to player permutations, concluding the proof.
\end{proof}

First, this result implies that the LFC problems for two isomorphic problems are identical.

\begin{cor}
\label{lem-lfc-problems-isomorphic-decpompds-identical}
Let \(\DecP,\DecPSecond\) be two isomorphic Dec-POMDPs and let \(\DecPSet:=\{\Isom^*\DecP\mid\Isom\in\Sym(\DecP)\}\) and \(\DecPSetSecond:=\{\Isom^*\DecPSecond\mid\Isom\in\Sym(\DecPSecond)\}\). Then \(\DecPSet=\DecPSetSecond\) and \(U^\DecP(\LA)= U^\DecPSecond(\LA)\) for any \(\LA\in\LASet^{\DecPSet}=\LASet^\DecPSetSecond\).
\end{cor}
\begin{proof}
By Theorem~\ref{zero-shot-coordination-game-invariant-to-isomorphism}, we have \(\DecPSet=\DecPSetSecond\).
Now let \(\LA\in\LASet^\DecPSet=\LASet^\DecPSetSecond\). 
Then again by Theorem~\ref{zero-shot-coordination-game-invariant-to-isomorphism}, it is
\[U^\DecP(\LA)
=U^\DecP(\LA,\dotsc,\LA) \overset{\text{Theorem~\ref{zero-shot-coordination-game-invariant-to-isomorphism}}}{=} U^\DecPSecond(\LA,\dotsc,\LA)=U^\DecPSecond(\LA).\]
\end{proof}

Second, the theorem shows that symmetries between the agents in a Dec-POMDP are also symmetries between principals in \(\Gamma^\DecP\).

\begin{cor}\label{cor-zero-shot-coordination-game-player-symmetries}
Let \(\DecP\) be a Dec-POMDP. Then it is
\[U^{\DecP}(\LA_1,\dotsc,\LA_N)
=U^{\DecP}(\LA_{\Auto^{-1}1},\dotsc,\LA_{\Auto^{-1}N}).
\]
for any \(\Auto\in\Aut(\DecP)\).
\end{cor}
\begin{proof}
This follows from Theorem~\ref{zero-shot-coordination-game-invariant-to-isomorphism}, using that \(\Aut(\DecP)=\Iso(\DecP,\DecP)\).
\end{proof}

\subsection{Optimal symmetric strategy profiles}
\label{appendix-optimal-symmetric-strategy-profiles}

Above, we have shown that the payoff in an LFC game is invariant with respect to symmetries of the agents in \(\DecP\). Similarly to the case of Dec-POMDPs and their symmetries, we can also apply the concept of symmetry to profiles of learning algorithms. We can then ask whether a profile of learning algorithms is invariant to symmetries of the principals, in which case we say that the profile is symmetric. In the following, we will show that optimal profiles among the ones that are symmetric are Nash equilibria of an LFC game. Since we will show in Appendix~\ref{appendix-proof-of-theorem-2} that a profile in which all principals choose OP with tie-breaking is an optimal symmetric profile, it will result as a corollary that all principals using OP with tie-breaking is a Nash equilibrium of the game.

To begin, we define symmetric principals and profiles of learning algorithms. In the following, let again \(\DecP\) be a Dec-POMDP and \(\DecPSet:=\{\Isom^*\DecP\mid\Isom\in\Sym(\DecP)\}\).

\begin{defn}[Symmetric principals and strategy profiles]
We say that two principals \(i,j\in\PlayerSet\) are symmetric if there exists an automorphism \(\Auto\in\Aut(\DecP)\) such that \(i=\Auto j\). A profile of learning algorithms \(\LAProfile_1,\dotsc,\LAProfile_N\in\LASet^\DecP\) is called symmetric if it is \(\LA_i=\LA_{\Auto^{-1}i}\) for any automorphism \(\Auto\in\Aut(\DecP)\) and principal \(i\in\PlayerSet\).
\end{defn}

An optimal symmetric profile is then defined as a symmetric profile \(\LAProfile\) such that for all other symmetric profiles \(\tilde{\LAProfile}\), it is \(U^\DecP(\LAProfile)\geq U^\DecP(\tilde{\LAProfile})\). Note that if there are non-symmetric principals in \(\PlayerSet\), then for a single learning algorithm \(\LA\in\LASet^\DecPSet\), the property that \(\LA,\dotsc,\LA\) is an optimal symmetric profile in the LFC \emph{game} for \(\DecP\) is stronger than the property that the algorithm \(\LA\) is optimal in the LFC \emph{problem} for \(\DecP\). The latter only requires that \(U^{\DecP}(\LA,\dotsc,\LA)\geq U^\DecP(\LA',\dotsc,\LA')\) for all \(\LA'\in\LASet^\DecPSet\), while \(\LA,\dotsc,\LA\) being an optimal symmetric profile means that \(U^\DecP(\LA,\dotsc,\LA)\geq U^\DecP(\LAProfile_1,\dotsc,\LAProfile_N)\) for all symmetric profiles \(\LAProfile_1,\dotsc,\LAProfile_N\in\LASet^\DecPSet\), where a profile could potentially include different learning algorithms for non-symmetric principals.

For a simple characterization of symmetric profiles, consider the orbit \(\Aut(\DecP)i\) of a principal \(i\in\PlayerSet\), as defined in Appendix~\ref{appendix-relation-to-group-theory}. Clearly, saying that a profile is symmetric can equivalently be expressed as saying that all principals from the same orbit are assigned the same learning algorithm.

\begin{lem}\label{prop-symmetric-orbit}
A profile \(\LAProfile_1,\dotsc,\LAProfile_N\in\LASet^\DecP\) is symmetric if and only if it is \(\LAProfile_i=\LAProfile_j\) for any two symmetric principals \(i,j\in\PlayerSet\).
\end{lem}
\begin{proof}
This can be easily seen from the definition of the orbit and the properties of actions of automorphisms on agents and principals.
\end{proof}

Lastly, we define a Nash equilibrium of an LFC game.

\begin{defn}[Nash equilibrium of an LFC game] A profile of learning algorithms \(\LAProfile_1,\dotsc,\LAProfile_N\in\LASet^\DecPSet\) is a Nash equilibrium of the LFC game for \(D\) if, for any principal \(i\in\PlayerSet\) and learning algorithm \(\LAProfile_i'\in\LASet^\DecPSet\), it is
\[
U^\DecP(\LAProfile)\geq U^\DecP(\LAProfile_i,\LAProfile_{-i}).
\]
\end{defn}

Now we show that any optimal symmetric profile is a Nash equilibrium. An analogous result for normal-form games was proven in \textcite{emmons2021symmetry}. Our proof closely follows that proof, adapted to our setting.

\begin{thm}
\label{prop-symmetry-invariant-nash-equilibrium}
Any optimal symmetric strategy profile in an LFC game is a Nash equilibrium.
\end{thm}
\begin{proof}
In the following, fix a Dec-POMDP \(\DecP\) and let \(\DecPSet:=\{\Isom^*\DecP\mid\Isom\in\Sym(\DecP)\}\). Let \(U:=U^\DecP\).
Let \(\LAProfile_1,\dots,\LAProfile_N\) be an optimal symmetric strategy profile. Towards a contradiction, assume that \(\LAProfile\) is not a Nash equilibrium, i.e., that there is \(i\in\PlayerSet\) and \(\tilde{\LAProfile}_i\in\LASet^\DecPSet\) such that \(U(\tilde{\LAProfile}_i,\LAProfile_{-i})>U(\LAProfile)\). We show that then there is another symmetric strategy profile \(\hat{\LAProfile}\) that achieves a higher payoff than \(\LAProfile\), \(U(\hat{\LAProfile})>U(\LAProfile)\), contradicting the assumption that \(\LAProfile\) was optimal among the symmetric profiles.

To that end, for arbitrary \(p\in(0,1]\) define the profile \(\hat{\LAProfile}_j:=p \tilde{\LAProfile}_i + (1-p)\LAProfile_i\) for any \(j\in\Aut(\DecP)i\) and \(\hat{\LAProfile}_j:=\LAProfile_j\) for \(j\in\PlayerSet\setminus\Aut(\DecP)i\). Note that, since we jointly change all learning algorithms in one orbit \(\Aut(\DecP)i\), the remaining profile \(\hat{\LAProfile}\) is symmetric by Lemma~\ref{prop-symmetric-orbit}. Next, let \(K:=|\Aut(\DecP)i|\), choose Dec-POMDPs \(\DecP_1,\dotsc,\DecP_N\in\DecPSet\), and measurable sets \(\mathcal{Z}_1,\dotsc,\mathcal{Z}_N\subseteq\PolicySet^\DecP\). Then it is
\begin{align}
&(\otimes_{j\in\PlayerSet}\hat{\LAProfile}(\DecP_j))(\prod_{l\in\PlayerSet}\mathcal{Z}_l)
\\
&=\prod_{j\in\PlayerSet\setminus\Aut(\DecP)i}\LAProfile(\DecP_j)(\mathcal{Z}_j)
\prod_{l\in\Aut(\DecP)i}\left(p\tilde{\LAProfile}_i(\DecP_l)(\mathcal{Z}_l) + (1-p)\LAProfile_i(\DecP_l)(\mathcal{Z}_l)\right)
\\\nonumber
&= (1-p)^K\prod_{j\in\PlayerSet}\LAProfile_j(\DecP_j)(\mathcal{Z}_j)
\\
\nonumber
&\quad\quad+\sum_{j\in\Aut(\DecP)i}p(1-p)^{K-1}\tilde{\LAProfile}_i(\DecP_j)(\mathcal{Z}_j)\prod_{l\in\PlayerSet\setminus\{ j\}}\LAProfile_l(\DecP_l)(\mathcal{Z}_l)
\\&\quad\quad+
\sum_{k=2,\dotsc,K}p^k(1-p)^{K-k}\Mixture^{j,k,\DecP_j}\left(\prod_{j\in\PlayerSet}\mathcal{Z}_j\right)
\\
\nonumber
&= (1-p)^K\otimes_{j\in\PlayerSet}\LAProfile_j(\DecP_j)\left(\prod_{j\in\PlayerSet}\mathcal{Z}_j\right)
\\
\nonumber
&\quad\quad+\sum_{j\in\Aut(\DecP)i}p(1-p)^{K-1}\left(\tilde{\LAProfile}_i(\DecP_j)\otimes_{l\in\PlayerSet\setminus\{ j\}}\LAProfile_l(\DecP_l)\right)\left(\prod_{j\in\PlayerSet}\mathcal{Z}_j\right)
\\&\quad\quad+\label{eq:84}
\sum_{k=2,\dotsc,K}p^k(1-p)^{K-k}\Mixture^{j,k,\DecP_j}\left(\prod_{j\in\PlayerSet}\mathcal{Z}_j\right)
\end{align}
where \(\Mixture^{k,\DecP_1,\dotsc,\DecP_N}\) is some measure (not necessarily a probability measure) on the space \(\PolicySet^\DecP\times\dotsb\times\PolicySet^\DecP\) that depends on \(k\) and \(\DecP_1,\dotsc,\DecP_N\), but not on \(p\).
This tells us that we can also decompose the integral with respect to the measure \(\otimes_{j\in\PlayerSet}\hat{\LAProfile}(\DecP_j)\) as in (\ref{eq:84}). It follows that
\begin{align}&U(\hat{\LAProfile}_1,\dotsc,\hat{\LAProfile}_N)\label{eq:85aa}
\\&=\E_{\SymProfile\sim \U(\Sym(\DecP)^\PlayerSet)}\left[
    \E_{\Policy^{(j)}\sim(\SymProfile_j^{-1})^*\hat{\LAProfile}_j(\SymProfile_j^*\DecP),\,j\in\mathcal{N}}
    \left[
J^\DecP((\Policy^{(l)}_l)_{l\in\PlayerSet})\right]\right]
\\
&
=\E_{\SymProfile\sim \U(\Sym(\DecP)^\PlayerSet)}\left[
    \E_{\Policy^{(j)}\sim\hat{\LAProfile}_j(\SymProfile_j^*\DecP),\,j\in\mathcal{N}}
    \left[
J^\DecP\left(\left(\Proj_l((\SymProfile_l^{-1})^*\Policy^{(l)})\right)_{l\in\PlayerSet}\right)\right]\right]\label{eq:85c}
\\
&\overset{(\ref{eq:84})}{=}\nonumber
\E_{\SymProfile\sim \U(\Sym(\DecP)^\PlayerSet)}\Big[
(1-p)^{K}
  \E_{\Policy^{(j)}\sim\LAProfile_j(\SymProfile_j^*\DecP),\,j\in\mathcal{N}}
    \left[
J^\DecP\left(\left(\Proj_l((\SymProfile_l^{-1})^*\Policy^{(l)})\right)_{l\in\PlayerSet}\right)
\right]
\\\nonumber
&\quad+\sum_{j\in\Aut(\DecP)i}
p(1-p)^{K-1}
 \E_{\Policy^{(m)}\sim\LAProfile_m(\SymProfile_m^*\DecP),\,m\in\mathcal{N}\setminus j}
    \left[
     \E_{\Policy^{(j)}\sim\tilde{\LAProfile}_i(\SymProfile_j^*\DecP)}
    \left[
J^\DecP\left(\left(\Proj_l((\SymProfile_l^{-1})^*\Policy^{(l)})\right)_{l\in\PlayerSet}\right)
\right]\right]
\\
&\quad+
 \sum_{k=2,\dotsc,K}p^k(1-p)^{K-k}
 \int
J^\DecP\left(\left(\Proj_l((\SymProfile_l^{-1})^*\Policy^{(l)})\right)_{l\in\PlayerSet}\right)\mathrm{d}\Mixture^{k,\IsoProfile_1^*\DecP,\dotsc,\IsoProfile_N^*\DecP}
\left(\Policy^{(1)},\dotsc,\Policy^{(N)}\right)
\Big]
\\
&
=\nonumber
(1-p)^{K}U(\LAProfile)
+
\sum_{j\in\Aut(\DecP)i}p(1-p)^{K-1}U(\LAProfile_1,\dotsc,\LAProfile_{j-1},\tilde{\LAProfile}_i,\LAProfile_{j+1},\dotsc,\LAProfile_N)
\\
&\quad+\label{eq:85a}
 \sum_{k=2,\dotsc,K}p^k(1-p)^{K-k}C_k
\\
&=\label{eq:85b}
B(m=0,p)U(\LAProfile)
+
B(m=1,p)U(\tilde{\LAProfile}_i,\LAProfile_{-i})
+
 \sum_{k=2,\dotsc,K}p^k(1-p)^{K-k}C_k
\end{align}
for some constants \(C_k\in\mathbb{R}\) for \(k=2,\dotsc,K\), and where we write \(B(k=0,p)\) to denote the probability that a binomial distribution with \(K\) trials and success chance \(p\) has \(0\) successful trials (and \(B(k=1,p)\) analogously). Here, in (\ref{eq:85c}), we use a change of variables for pushforward measures, separately for each measure \(\hat{\LAProfile}_j(\SymProfile_j^*\DecP)\); in (\ref{eq:85a}), we use the linearity of the expectation; and in (\ref{eq:85b}) we use that \(U\) and \(\LAProfile\) are both invariant to symmetries between principals, and thus for \(\Auto\in\Aut(\DecP)\) with \(\Auto^{-1}j=i\), it is
\begin{multline}
U(\tilde{\LAProfile}_i,\LAProfile_{-i})
=U(\tilde{\LAProfile}_{\Auto^{-1}1},\dotsc,\LAProfile_{\Auto^{-1}(j-1)},\tilde{\LAProfile}_i,\LAProfile_{\Auto^{-1}(j+1)},\dots,\LAProfile_{\Auto^{-1}N})
\\
=U(\LAProfile_1,\dotsc,\LAProfile_{j-1},\tilde{\LAProfile}_i,\LAProfile_{j+1},\dotsc,\LAProfile_N).\end{multline}

Now note that if we can show that
\[(1-B(m=0,p))U(\LAProfile)<B(m=1,p)U(\tilde{\LAProfile}_i,\LAProfile_{-i})+\sum_{k=2,\dotsc,K}p^k(1-p)^{K-k}C_k,\]
then it would also follow that
\begin{multline}
U(\hat{\LAProfile})
\overset{\text{(\ref{eq:85aa})--(\ref{eq:85b})}}{=}
B(m=0,p)U(\LAProfile)
+
B(m=1,p)U(\tilde{\LAProfile}_i,\LAProfile_{-i})
+\sum_{k=2,\dotsc,K}p^k(1-p)^{K-k}C_k
\\
>
B(m=0,p)U(\LAProfile)
+
(1-B(m=0,p))U(\LAProfile)
=
U(\LAProfile).
\end{multline}
Thus, this would show that \(\hat{\LAProfile}\) is a symmetric profile with higher payoff than \(\LAProfile\), proving the required contradiction.

In the following, we show the equivalent condition
\[U(\LAProfile)<\frac{B(1,p)}{B(m>0,p)}U(\tilde{\LAProfile}_i,\LAProfile_{-i})+\frac{\sum_{k=2,\dotsc,K}p^k(1-p)^{K-k}C_k}
{B(m>0,p)}
,\]
where \(B(m>0,p)=\sum_{k=1}^KB(m=k,p)=1-B(m=0,p)\).
To that end, note that since \(U(\LAProfile)<U(\tilde{\LAProfile}_i,\LAProfile_{-i})\) by assumption, we can choose some small \(\epsilon>0\) such that still
\[U(\LAProfile)<U(\tilde{\LAProfile}_i,\LAProfile_{-i}) - \epsilon.\]

Moreover, note that \(B(m>0,p)\)
is a polynomial in \(p\), and the degree of its nonzero term with lowest degree is \(1\). Similarly, for \(\sum_{k=2,\dotsc,K}p^{K-k}(1-p)^2C_k\), that lowest degree is \(2\).
Hence, it is
\[\frac{\sum_{k=2,\dotsc,K}p^k(1-p)^{K-k}C_k}{B(m>0,p)}=
\frac{\sum_{k=2,\dotsc,K}p^{k-1}(1-p)^{K-k}C_k}{C+Q(p)}\]
for some constant \(C\neq 0\) and polynomial \(Q\) in \(p\), and it follows that
\[\lim_{p\rightarrow 0}\frac{\sum_{k=2,\dotsc,K}p^k(1-p)^{K-k}C_k}{B(m>0,p)}
=
\lim_{p\rightarrow 0}\frac{\sum_{k=2,\dotsc,K}p^{k-1}(1-p)^{K-k}C_k}{C+Q(p)}=0.\]
With the same argument, it is also \(\lim_{p\rightarrow0}\frac{B(m>1,p)}{B(m>0,p)}=0\) and thus
\[\lim_{p\rightarrow0}\frac{B(m=1,p)}{B(m>0,p)}
=
1 - \lim_{p\rightarrow0}\frac{B(m>1,p)}{B(m>0,p)}=1.\]

Hence, we can find some \(p>0\) that is small enough such that both \[\frac{\sum_{k=2,\dotsc,K}p^k(1-p)^{K-k}C_k}{B(m>0,p)}<\frac{\epsilon}{2}\]
and
\[1+\frac{\epsilon}{2|U(\tilde{\LAProfile}_i,\LAProfile_{-i})|}>\frac{B(m=1,p)}{B(m>0,p)}>1-\frac{\epsilon}{2|U(\tilde{\LAProfile}_i,\LAProfile_{-i})|}.\]
It follows that
\begin{multline}
U(\LAProfile)<U(\tilde{\LAProfile}_i,\LAProfile_{-i}) - \epsilon
=\left(
1-\frac{\epsilon}{2U(\tilde{\LAProfile}_i,\LAProfile_{-i})}
\right)
U(\tilde{\LAProfile}_i,\LAProfile_{-i})-\frac{\epsilon}{2}
\\
<\frac{B(m=1,p)}{B(m>0,p)}U(\tilde{\LAProfile}_i,\LAProfile_{-i})+\frac{\sum_{k=2,\dotsc,K}p^k(1-p)^{K-k}C_k}{B(m>0,p)},
\end{multline}
which is what we wanted to show. This concludes the proof.
\end{proof}

\subsection{Evaluating equivariant learning algorithms}
\label{appendix-equivariant-learning-algorithms}
In the following, let \(\DecP\) be a Dec-POMDP and define \(\DecPSet:=\{\Isom^*\DecP\mid\Isom\in\Sym(\DecP)\}\).
In the special case in which a learning algorithm is in a sense independent from the used labels, we can find a simpler form of the algorithm's value in the LFC problem for \(\DecP\). To do so, we define the notion of an equivariant learning algorithm.

\begin{defn}[Equivariant learning algorithms]\label{definition-equivariant-algorithm} Let \(\LA\in\LASet^\DecPSet\). Then \(\LA\) is called \emph{equivariant} if for any two labelings \(\Isom,\IsomSecond\in\Sym(\DecP)\), it is
\[(\Isom^{-1})^*\LA(\Isom^*\DecP)=(\IsomSecond^{-1})^*\LA(\IsomSecond^*\DecP).\]
\end{defn}

\begin{remark}
We believe that a learning algorithm implemented via neural networks and one-hot encodings, as used in our experiments, should be equivariant. To see this, note that by a symmetry argument, the distribution over functions corresponding to a randomly initialized neural network is invariant with respect to coordinate permutations. Assume that actions, observations, and agents of a given problem are implemented as one-hot vectors, i.e., elements of a canonical basis \(\{e_1,\dotsc,e_k\}\in\mathbb{R}^k\) where \(k\in\mathbb{N}\) is the cardinality of the respective set. Then the distribution over randomly initialized neural network policies will also not depend on particular assignments of actions, etc., to one-hot vectors. We conjecture that, if the used optimizer is equivariant with respect to coordinate permutations (i.e., if the parameter dimensions are permuted, then the prescribed updates to the parameters are equally permuted), then the resulting learning algorithm is equivariant. We leave a rigorous exploration of this issue to future work.
\end{remark}

An equivariant algorithm can be evaluated in the LFC problem for \(\DecP\) by evaluating its cross-play value in any Dec-POMDP
\(\DecPSecond\in\DecPSet\). The resulting policies can be permuted by random automorphisms or they can be evaluated as they are.
\begin{prop}\label{prop-equivariant-algorithm-zero-shot-coordination-problem}
Let \(\LA\in\LASet^\DecPSet\) be equivariant. Then for any \(\Isom\in\Sym(\DecP)\) and \(\DecPSecond=\Isom^*\DecP\in \DecPSet\), it is
\begin{align}U^\DecP(\LA)&=
\E_{\Policy^{(i)}\sim\LA(\DecPSecond),\,i\in\mathcal{N}}\left[\E_{\AutProfile\in\U(\Aut(\DecPSecond)^\PlayerSet)}
    \left[
J^\DecPSecond\left(\left(\Proj_j(\AutProfile_j^*\Policy^{(j)})\right)_{j\in\PlayerSet}\right)\right]\right]
\\
&=
\E_{\Policy^{(i)}\sim\LA(\DecPSecond),\,i\in\mathcal{N}}
    \left[
J^\DecPSecond\left(\left(\Policy^{(j)}_j\right)_{j\in\PlayerSet}\right)\right].
\end{align}
\end{prop}
\begin{proof}
First, let \(\Isom\in\Sym(\DecP)\) and \(\DecPSecond:=\Isom^*\DecP\). Using the expression of the payoff in the LFC game for \(\DecP\) from Lemma~\ref{lem-payoff-lfc-game-old-form}, it is
\begin{align}
U^\DecP(\LA)=U^\DecP(\LA,\dotsc,\LA)&=
\E_{\SymProfile\sim \U(\Sym(\DecP)^\PlayerSet)}\left[
    \E_{\Policy^{(i)}\sim(\SymProfile_i^{-1})^*\LA(\SymProfile_i^*\DecP),\,i\in\mathcal{N}}
    \left[
J^\DecP((\Policy^{(j)}_j)_{j\in\PlayerSet})\right]\right]
\\&=\label{eq:102a}
    \E_{\Policy^{(i)}\sim(\Isom^{-1})^*\LA(\DecPSecond),\,i\in\mathcal{N}}
    \left[
J^\DecP((\Policy^{(j)}_j)_{j\in\PlayerSet})\right]
\\&=\label{eq:102c}
    \E_{\Policy^{(i)}\sim\LA(\DecPSecond),\,i\in\mathcal{N}}
    \left[
J^\DecP((\Isom^{-1})^*(\Policy^{(j)}_j)_{j\in\PlayerSet})\right]
\\&=\label{eq:102b}
    \E_{\Policy^{(i)}\sim\LA(\DecPSecond),\,i\in\mathcal{N}}
    \left[
J^\DecPSecond((\Policy^{(j)}_j)_{j\in\PlayerSet})\right],
\end{align}
where we use equivariance in (\ref{eq:102a}), a change of variables for pushforward measures in (\ref{eq:102c}), and Theorem~\ref{lem-pull-back-isomorphism-compatibility} in (\ref{eq:102b}).

Second, by Lemma~\ref{lem-sym-is-superset-of-iso-pushforward}, it is \(\Iso(\DecP,\DecPSecond)\subseteq\Sym(\DecP)\). Thus, using
Lemma~\ref{decompose-isomorphisms}, it is \(\Aut(\DecPSecond)\circ\Isom=\Iso(\DecP,\DecPSecond)\subseteq\Sym(\DecP)\).
Hence, it follows that
\begin{align}
U^\DecP(\LA)=U^\DecP(\LA,\dotsc,\LA)&=
    \E_{\Policy^{(i)}\sim(\Isom^{-1})^*\LA(\DecPSecond),\,i\in\mathcal{N}}
    \left[
J^\DecP((\Policy^{(j)}_j)_{j\in\PlayerSet})\right]
\\&=\label{eq:101a}
\E_{\AutProfile\sim \U(\Aut(\DecPSecond)^\PlayerSet}\left[
    \E_{\Policy^{(i)}\sim(\Isom^{-1}\circ\AutProfile_i)^*\LA(\DecPSecond),\,i\in\mathcal{N}}
    \left[
J^\DecP((\Policy^{(j)}_j)_{j\in\PlayerSet})\right]\right]
\\&=\label{eq:101b}
\E_{\AutProfile\sim \U(\Aut(\DecPSecond)^\PlayerSet}\left[
    \E_{\Policy^{(i)}\sim\AutProfile_i^*\LA(\DecPSecond),\,i\in\mathcal{N}}
    \left[
J^\DecPSecond((\Policy^{(j)}_j)_{j\in\PlayerSet})\right]\right]
\\&=\label{eq:101c}
\E_{\AutProfile\sim \U(\Aut(\DecPSecond)^\PlayerSet}\left[
    \E_{\Policy^{(i)}\sim\LA(\DecPSecond),\,i\in\mathcal{N}}
    \left[
J^\DecPSecond((\Proj_j(\AutProfile_i^*\Policy^{(j)}))_{j\in\PlayerSet})\right]\right],
\end{align}
where we again use equivariance in (\ref{eq:101a}), a change of variables and Theorem~\ref{lem-pull-back-isomorphism-compatibility} in (\ref{eq:101b}), and another change of variables in (\ref{eq:101c}).
This concludes the proof.
\end{proof}

If a learning algorithm is equivariant, we can use either of the expressions above to evaluate it in the LFC problem. One may choose the first expression, i.e., apply random automorphisms to policies, since this transformation maps different policies that are equivalent under OP to a unique automorphism-invariant policy and thus reduces the variance of the cross-play values across different samples \(\Policy\sim\LA(\DecP)\). We will introduce these notions of equivalence and automorphism-invariant policies in Appendix~\ref{appendix-characterization-of-other-play}.

Note that an equivariant learning algorithm is not necessarily one that performs well in the LFC problem. For instance, a SP algorithm may be equivariant. However, if an algorithm is equivariant and it does well in cross-play, then the preceding shows that it will also do well in the LFC problem.

\section{Characterization of other-play and of label-free coordination games}
\label{appendix-characterization-of-other-play}

In this section, we define for a given policy \(\Policy\) a policy \(\Psi(\Policy)\) that corresponds to agents choosing local policies that are randomly permuted by automorphisms, and that is itself invariant to pushforward by automorphisms. We then use this self-map on policies \(\Psi\), called the \emph{symmetrizer}, to characterize both the OP objective and the payoff in an LFC game. This characterization helps us to analyze the OP-optimal policies in the two-stage lever game in Appendix~\ref{appendix-proof-of-theorem-1}, and it allows us to prove a stronger result, showing that any OP algorithm that is not concentrated on only one equivalence class in the two-stage lever game is suboptimal in the corresponding LFC problem. We also use this notion of equivalence classes of policies to define OP with tie-breaking in Appendix~\ref{appendix-proof-of-theorem-2}, which allows us to then show that random tie-breaking functions exist.

In the following, in Appendix~\ref{appendix-generalization-of-other-play}, we recall the definition of our generalization of OP, and we define policies that are invariant to automorphism. Afterwards, in Appendix~\ref{appendix-policies-corresponding-to-distributions}, we introduce the concept of a policy corresponding to a distribution over policies. In Appendix~\ref{appendix-other-play-distribution}, we introduce the other-play distribution, in which each agent's local policy is chosen as pushforward by a random automorphism, and we define the symmetrizer. Using these concepts, in Appendix~\ref{appendix-main-characterizations}, we give a new expression for both the OP objective and the payoff in an LFC game. The OP objective can be understood of as transforming a policy into one that is invariant to automorphisms, and evaluating that policy in SP. Finally, in Appendix~\ref{appendix-no-optimal-deterministic-policy}, we show that our generalized OP objective can in general not be understood of as the SP objective in a modified Dec-POMDP.

\subsection{Generalization of other-play}
\label{appendix-generalization-of-other-play}
In the following, fix a Dec-POMDP \(\DecP\). Recall that for a profile of automorphisms \(\AutProfile\in\Aut(\DecP)^\PlayerSet\) and a joint policy \(\Policy\in\PolicySet^\DecP\), we define the joint policy \(\AutProfile^*\Policy:=\hat{\Policy}\), where the local policy \(\hat{\Policy}_i\) of agent \(i\in\PlayerSet\) is given by the local policy of agent \(i\) in the pushforward policy \(\AutProfile^*_i\Policy\). That is, for \(i\in\PlayerSet\), we define
\[\hat{\Policy}_i:= \Proj_i(\AutProfile^*_i\Policy)=
\Policy_{\AutProfile_i^{-1}i}(\AutProfile_i^{-1}\cdot\mid \AutProfile_i^{-1}\cdot).\]

Using this, we define the OP objective as the expected return of a policy that is randomly permuted by such profiles of automorphisms.

\begin{defn}[Other-play objective]\label{defn-other-play-objective-general-appendix}
Define \(J_\OP^\DecP\colon\PolicySet^\DecP\rightarrow\mathbb{R}\) via 
\begin{equation}
J^{D}_\OP(\Policy):=\E_{\AutProfile\sim \U(\Aut(D)^\mathcal{N})}\left[J^{D}(\AutProfile^*\Policy)
\right]\label{eq:1appendix}
\end{equation}
for \(\Policy\in\PolicySet^\DecP\),
where \(\U(\Aut(D)^\mathcal{N})\) is a uniform distribution over \(\Aut(D)^\mathcal{N}\). We say that \(J^\DecP_\OP\) is the \emph{other-play (OP) objective} of \(\DecP\), and \(J^\DecP_\OP(\Policy)\) is the \emph{OP value} of \(\Policy\in\PolicySet^\DecP\).
\end{defn}

\begin{remark}\label{remark-other-play-always-admits-a-maximum}
It is clear that this objective always admits a maximum. For instance, we can consider \(\PolicySet^\DecP\) as a subset of \(\prod_{i\in\PlayerSet}[0,1]^{\ActionSet_i\times\AOHistorySet_i}\) with its standard topology. As \(\PolicySet^\DecP\) is a Cartesian product of simplices, it is a compact subset of this space. Moreover, one can check that the objective \(J^\DecP(\Policy)\) is continuous in the policy \(\Policy\), and that for \(\AutProfile\in\Aut(\DecP)^\PlayerSet\), the map \(\Policy\mapsto\AutProfile^*\Policy\) is continuous as well. Thus, also \(J_\OP^\DecP(\Policy)\) is continuous in \(\Policy\), as it is a finite linear combination of continuous functions. By the extreme value theorem, it follows that the function always attains a maximum.
\end{remark}

Next, recall our formal definition of an OP learning algorithm as any algorithm that achieves an optimal OP value in expectation.

\begin{defn}[Other-play learning algorithm]\label{defn-other-play-learning-algorithm-appendix}
Let \(\DecPSet\) be a finite set of Dec-POMDPs. A learning algorithm \(\sigma\in\Sigma^\DecPSet\) is called an \emph{OP learning algorithm} if for any \(\DecP\in\DecPSet\), it is
\[\E_{\Policy\sim\sigma^\OP(\DecP)}[J^\DecP_\OP(\Policy)]=\max_{\Policy\in\PolicySet^\DecP}J^\DecP_\OP(\Policy).\]
\end{defn}

Now fix again a Dec-POMDP \(\DecP\) and consider the notion of invariance to automorphism.

\begin{defn}[Invariance to automorphism] A policy \(\Policy\in\PolicySet\) is called \emph{invariant to automorphism} if \(f^*\Policy=\Policy\) for any \(f\in\Aut(\DecP)\).
\end{defn}

Clearly, if a policy is invariant to automorphism, then it has the same OP value and expected return.

\begin{prop}\label{prop-invariant-policy-self-play-value-equals-other-play-value}
Let \(\Policy\in\PolicySet\) be invariant to automorphism. Then
\[J_\OP(\Policy)=J(\Policy).\]
\end{prop}
\begin{proof}
It is
\begin{multline}J_\OP(\Policy)=
\E_{\AutProfile\sim \U(\Aut(D)^\mathcal{N})}\left[J(\AutProfile^*\Policy)
\right]
=\E_{\AutProfile\sim \U(\Aut(D)^\mathcal{N})}\left[J((\Proj_i(\AutProfile_i^*\Policy))_{i\in\PlayerSet})
\right]
\\\overset{\text{(*)}}{=}\E_{\AutProfile\sim \U(\Aut(D)^\mathcal{N})}\left[J((\Proj_i(\Policy))_{i\in\PlayerSet})
\right]
=J(\Policy),\end{multline}
where we have used invariance to automorphism in (*).
\end{proof}

In the following, we will show that if a policy \(\Policy\) is not already invariant to automorphism, then one can understand the OP objective as first transforming the policy into a policy \(\Psi(\Policy)\) that is invariant to automorphism by applying the symmetrizer \(\Psi\), and then evaluating the expected return of that policy. In that way, the OP objective ensures that policies cannot make use of arbitrary symmetry-breaking.

\subsection{Policies corresponding to distributions over policies}
\label{appendix-policies-corresponding-to-distributions}

In this section, let some Dec-POMDP \(\DecP\) be fixed. As a first step towards defining the symmetrizer \(\Psi\), we will define policies corresponding to distributions over policies for general distributions. Afterwards, we will turn to the particular policy \(\Psi(\Policy)\) that corresponds to the \emph{OP distribution} of \(\Policy\).

Recall that we introduced the set of distributions over policies as \(\Delta(\PolicySet)\), the set of measures on the space \((\PolicySet,\mathcal{F})\). Let \(\Distr\in\Delta(\PolicySet)\). For a given distribution \(\Distr\in\Delta(\PolicySet)\) and agent \(i\in\PlayerSet\), the marginal distribution \(\Distr_i\) is defined as \(\Distr_i(\mathcal{Z}_i):=\Distr(\Proj_i^{-1}(\mathcal{Z}_i))\) for any measurable set of local policies \(\mathcal{Z}_i\subseteq\PolicySet_i\).
We say that \(\Distr\) has \emph{independent local policies}, if \(\Distr=\otimes_i\Distr_i\), i.e., \(\Distr\) decomposes into independent marginal distributions over local policies for each agent. We denote such distributions by \(\Mixture\).

Now let \(\Mixture\in\Delta(\PolicySet)\) be a distribution with independent local policies. We want to construct a policy \(\Policy^\Mixture\) that represents each agent sampling a local policy \(\Policy_i\sim\Mixture_i\) in the beginning of an episode, and then choosing actions according to that policy until the end of the episode. This policy should be equivalent to \(\Mixture\) in the sense that it should yield the same expected return as \(\Mixture\), where we define the expected return of \(\Mixture\) as
\begin{equation}
    \label{expected-return-mixture}
    J(\Mixture):=\E_{\Policy\sim\Mixture}\left[J(\Policy)\right]=\E_{\Policy_i\sim\Mixture_i,\,i\in\PlayerSet}\left[J(\Policy)\right].
\end{equation}

The statement that such a policy exists is analogous and more general than a result by \textcite{kuhn1953contributions}, which says that in an extensive-form game, given some conditions on the game, for every mixed strategy there is an equivalent behavior strategy. Kuhn's theorem is relevant to Dec-POMDPs since there is a correspondence between Dec-POMDPs and extensive-form games \parencite{oliehoek2006dec}. For Dec-POMDPs, the analogous result states that for every distribution over deterministic policies, there is an equivalent stochastic policy.

We cannot directly apply Kuhn's theorem here, as we require a result for distributions over stochastic policies instead of deterministic policies. This is because such a result fits better with our remaining setup---for instance, the domain of the OP objective is the set of stochastic policies, and learning algorithms are defined as distributions over stochastic policies. Nevertheless, our proof is based on similar ideas as, for instance, the proof of Kuhn's theorem in \textcite{maschler_solan_zamir_2013}, after translating between the different formal frameworks of extensive-form games and Dec-POMDPs.

\co{Maybe in the following you can say again that more specifically $\Omega$ is the set of complete histories. (Right?)}
\jt{\(\Omega\) is just any set on which all of the random variables are defined, but indeed one could just define \(\Omega:=\HistorySet\)  and that would work. If you think this is better, I could explicitly define it that way.}

To define a policy \(\Policy^\Mixture\) that is equivalent to \(\Mixture\), we begin by defining a new measure space \((\PolicySet\times\Omega,\mathcal{F}\otimes\PowerSet(\Omega))\), which is the product space of the space of policies \((\PolicySet,\mathcal{F})\) with the Dec-POMDP environment \((\Omega,\PowerSet(\Omega))\). On this space, for a given distribution \(\Mixture\), we define a probability measure \(\Prob_{\Mixture}\) which represents the procedure outlined above, i.e., in which a policy \(\Policy\) is chosen according to \(\Mixture\), and then samples in \(\Omega\) are distributed according to \(\Prob_\Policy\). Formally, we define \(\Prob_\Mixture\) as the unique measure such that
\begin{equation}\label{defn-semidirect-product}
\Prob_\Mixture(\mathcal{Z}\times \mathcal{Q}):=\E_{\Policy\sim\Mixture}\left[\mathds{1}_{\mathcal{Z}}\Prob_\Policy(\mathcal{Q})\right]
\footnote{\(\Prob_\Mixture\) is the semidirect product of \(\Mixture\) with the Markov kernel \(\kappa(\cdot\mid \Policy):=\Prob_\Policy(\cdot)\).}
\end{equation}
for any measurable sets \(\mathcal{Z}\subseteq\PolicySet\) and \(\mathcal{Q}\subseteq\Omega\).
Note that the product sets \(\mathcal{Z}\times\mathcal{Q}\) for measurable \(\mathcal{Z}\subseteq\PolicySet\) and \(\mathcal{Q}\subseteq\Omega\) are a \(\pi\)-system and generate the product \(\sigma\)-Algebra \(\mathcal{F}\otimes\PowerSet(\Omega)\). Hence, by Carathéodory's extension theorem, there is a unique measure satisfying this definition \parencite[see][ch. 1]{williams1991probability}.

On this new product space \(\PolicySet\times\Omega\), define the random variable \(\PolicyLatent\) as the projection onto \(\PolicySet\). We can define histories \(\HistoryRV\) and all other random variables defined on the space \((\Omega,\PowerSet(\Omega))\) by composing them with the projection onto \(\Omega\) (for notational convenience, we denote these random variables using the same symbols in both spaces). 

\begin{remark}
\label{remark-conditional-probability}
Note that the conditional probability of a particular trajectory \(\History\) given the policy \(\PolicyLatent\) is just the probability of that history under \(\Prob_\PolicyLatent\), that is,
\begin{equation}\label{conditional-probability-mixture}
\Prob_\Mixture(\HistoryRV=\History\mid \PolicyLatent)=\Prob_\PolicyLatent(\HistoryRV=\History)
\end{equation}
for any \(\History\in\HistorySet\).
Here, the conditional probability is a random variable, defined via the conditional expectation
\[\Prob_\Mixture(\HistoryRV=\History\mid\PolicyLatent):=\E_\Mixture\left[\mathds{1}_{\HistoryRV=\History}\mid  \PolicyLatent\right].\]
Intuitively, for a given sample \((\Policy,\omega)\), the value of that random variable is the best estimate of the probability of \(\{\HistoryRV=\History\}\) given \(\PolicyLatent=\Policy\), ignoring \(\omega\). As, in general, \(\{\PolicyLatent=\Policy\}\) may have zero probability, it is impossible to define the conditional probability \(\Prob(\HistoryRV=\History\mid\PolicyLatent=\Policy)\) via \(\Prob(\HistoryRV=\History\mid\PolicyLatent=\PolicyLatent):=\frac{\Prob(\HistoryRV=\History,\PolicyLatent=\Policy)}{\Prob(\PolicyLatent=\Policy)}\). It is still possible to define the \emph{conditional expectation} \(\E_\Mixture\left[\mathds{1}_{\HistoryRV=\History}\mid  \PolicyLatent\right]\), though.

Here, we briefly give the definition of the conditional expectation and show that (\ref{conditional-probability-mixture}) is correct. For a reference on conditional expectations, refer to \parencite[][ch.~9]{williams1991probability}. Applied to our setup, the conditional expectation of \(\mathds{1}_{\HistoryRV=\History}\) given \(\PolicyLatent\) is any random variable (which can be shown to be almost surely unique), denoted by \(\E_\Mixture\left[\mathds{1}_{\HistoryRV=\History}\mid  \PolicyLatent\right]\), that is measurable with respect to \[\sigma(\PolicyLatent):=\{\PolicyLatent^{-1}(\mathcal{Z})\mid \mathcal{Z}\in\mathcal{F}\}=\{\mathcal{Z}\times\Omega\mid\mathcal{Z}\in\mathcal{F}\}\]
such that
\begin{equation}\label{eq:46}\E_{\Mixture}\left[\mathds{1}_{\mathcal{X}}\E_\Mixture\left[\mathds{1}_{\HistoryRV=\History} \mid \PolicyLatent\right]\right]
=\E_\Mixture\left[\mathds{1}_{\mathcal{X}}\mathds{1}_{\HistoryRV=\History}\right]
.\end{equation}
for all \(\mathcal{X}\in\sigma(\PolicyLatent)\). The fact that the random variable is measurable with respect to \(\sigma(\PolicyLatent)\) can be equivalently expressed as saying that it can be written as a function of \(\PolicyLatent\). Moreover, Equation~\ref{eq:46} says that the conditional expectation should represent correct averages of the random variable \(\mathds{1}_{\HistoryRV=\History}\) over the level-sets in \(\sigma(\PolicyLatent)\).

To show that (\ref{conditional-probability-mixture}) is correct, let \(\mathcal{X}\in\sigma(\PolicyLatent)\) arbitrary. 
Since \(\mathcal{X}\) is of the form \(\mathcal{Z}\times\Omega\), and \(\{\HistoryRV=\History\}=\PolicySet\times\mathcal{Q}\) for some \(\mathcal{Q}\subseteq\Omega\), it is \((\mathcal{Z}\times\Omega)\cap\{\HistoryRV=\History\}=\mathcal{Z}\times\mathcal{Q}\).
Using Equation~\ref{defn-semidirect-product}, it follows that
\begin{multline}
\E_\Mixture\left[\mathds{1}_{\mathcal{X}}\Prob_\PolicyLatent(\HistoryRV=\History)\right]
=
\E_\Mixture\left[\mathds{1}_{\mathcal{Z}\times\Omega}\Prob_\PolicyLatent(\mathcal{Q})\right]
=\E_{\Policy\sim \Mixture}\left[\mathds{1}_{\mathcal{Z}}(\Policy)\Prob_\Policy(\mathcal{Q})\right]
\\
\overset{(\ref{defn-semidirect-product})}{=}
\Prob_\Mixture(\mathcal{Z}\times \mathcal{Q})
=\E_\Mixture\left[\mathds{1}_{\PolicySet\times\Omega}\mathds{1}_{\mathcal{Z}\times\mathcal{Q}}\right]
=\E_\Mixture\left[\mathds{1}_{\mathcal{X}}\mathds{1}_{\HistoryRV=\History}\right].
\end{multline}
Hence, letting \(\E_\Mixture[\mathds{1}_{\HistoryRV=\History}\mid \PolicyLatent]:=\Prob_\PolicyLatent(\HistoryRV=\History)\) satisfies condition (\ref{eq:46}). \(\Prob_\PolicyLatent(\HistoryRV=\History)\) is also \(\sigma(\PolicyLatent)\)-measurable, since it is just a function of \(\PolicyLatent\).

We will repeatedly make use of (\ref{conditional-probability-mixture}) in the following, together with the tower property, which, applied to our case, says that
\begin{equation}\label{eq:equation-tower-prop-z}\Prob_\Mixture(\HistoryRV=\History)=\E_\Mixture\left[\Prob_\Mixture\left(\HistoryRV=\History\mid\PolicyLatent\right)\right]
\overset{(\ref{conditional-probability-mixture})}{=}
\E_\Mixture\left[\Prob_\PolicyLatent\left(\HistoryRV=\History\right)\right]
\end{equation}
for any history \(\History\in\HistorySet\) \parencite[see][ch.~9.7]{williams1991probability}.
\end{remark}

Using the measure space defined above, we now define a local policy \(\Policy_i^{\Mixture_i}\) for an agent \(i\in\PlayerSet\), corresponding to a distribution \(\Mixture_i\in\Delta(\PolicySet_i)\). For an action \(\Action_i\in\ActionSet_i\) and an action-observation history \( \AOHistory_{i,t}\in\AOHistorySet_{i,t}\), we define \(\Policy_i^{\Mixture_i}(a_i\mid \AOHistory_{i,t})\) as the probability that agent \(i\), who follows a policy that is sampled from \(\Mixture_i\), plays action \(\Action_i\), conditional on $\{\AOHistoryRV_{i,t}=\AOHistory_{i,t}\}$.

\begin{defn}\label{policy-mixture}Let \(i\in\PlayerSet\) and let \(\Mixture_i\in\Delta(\PolicySet_i)\) be a distribution over local policies of agent \(i\). We define the local policy \(\Policy_i^{\Mixture_i}\) corresponding to \(\Mixture_i\) in the following way. For \(\Action_i\in\ActionSet_i\), \(t\in\{0,\dotsc,\Tmax\}\) and \(\AOHistory_{i,t}\in\AOHistorySet_{i,t}\), let 
\begin{equation}
\Policy^{\Mixture_i}_i(a_i\mid \AOHistory_{i,t})
:= \Prob_{\Mixture_i\otimes\Mixture_{-i}}(\ActionRV_{i,t}=a_i\mid \AOHistoryRV_{i,t}=\AOHistory_{i,t}),\end{equation}
where \(\Mixture_{-i}\) is any distribution over \(\PolicySet_{-i}\) with independent local policies such that \[\Prob_{\Mixture_i\otimes\Mixture_{-i}}( \AOHistoryRV_{i,t}=\AOHistory_{i,t})>0.\] If no such distribution exists, we let \(\Policy^\Mixture_i(a_i\mid \AOHistory_{i,t}):=\frac{1}{|\ActionSet_i|}\).
\end{defn}

Note that if \(\Prob_{\Mixture_i\otimes\Mixture_{-i}}( \AOHistoryRV_{i,t}=\AOHistory_{i,t})=0\) for all the distributions \(\Mixture_{-i}\in\Delta(\PolicySet_{-i})\) over opponent policies, then agent \(i\)'s action-observation history \(\AOHistory_{i,t}\) is almost never reached, independent of the other agents' policies. In that case, we can define the policy arbitrarily and this will never matter for the distribution over histories under \(\Prob_{\Mixture_i\otimes\Mixture_{-i}}\), and as we will see, neither for the distribution under \(\Prob_{\Policy_i^{\Mixture_i},\Policy_{-i}}\) for arbitrary \(\Policy_{-i}\).
%

First, we need to make sure that \(\Policy_i^{\Mixture_i}\) is well-defined, i.e., it does not depend on the chosen distribution \(\Mixture_{-i}\). Unfortunately, the proof for the following Lemma is somewhat technical.

\begin{lem}\label{lemma-mixtures-factorize}
Let \(i\in\PlayerSet\), \(\Mixture_i\in\Delta(\PolicySet_i)\), \(t\in\{0,\dotsc,\Tmax\}\), \(\AOHistory_{i,t}\in\AOHistorySet_i\), and \(\Action_i\in\ActionSet_i\). Let \(\Mixture_{-i},\Mixture'_{-i}\in\Delta(\PolicySet_{-i})\) be any two distributions with independent local policies such that \(\Prob_{\Mixture_i\otimes\Mixture_{-i}}( \AOHistoryRV_{i,t}=\AOHistory_{i,t})>0\) and \(\Prob_{\Mixture_i\otimes\Mixture'_{-i}}( \AOHistoryRV_{i,t}=\AOHistory_{i,t})>0\). Then it is
\[\Prob_{\Mixture_i\otimes\Mixture_{-i}}(\ActionRV_{i,t}=a_i\mid \AOHistoryRV_{i,t}=\AOHistory_{i,t})
=\Prob_{\Mixture_i\otimes\Mixture'_{-i}}(\ActionRV_{i,t}=a_i\mid \AOHistoryRV_{i,t}=\AOHistory_{i,t}).\]
\end{lem}
\begin{proof}
In the following, let \(i\in\PlayerSet, t\in\{0,\dotsc,\Tmax\}\) be fixed.
Our goal is to find an expression for the distribution of \(\PolicyLatent_i\) given \(\AOHistoryRV_{i,t}\) that only depends on \(\Mixture_i\). If we can do that, we can also show that the probability of a particular action chosen by an agent is independent of the distributions over other agents' policies. To begin, we analyze the joint distribution of \(\PolicyLatent_i\) and \(\AOHistoryRV_{i,t}\) for an arbitrary distribution \(\Mixture\in\Delta(\PolicySet)\) with independent local policies. Note that by assumption, the \(Z_j\) are independent for \(j\in\PlayerSet\) under \(\Mixture\).
We now show that conditioning on \(\AOHistoryRV_{i,t}\) still leaves \(Z_i\) independent from \(Z_{-i}\).

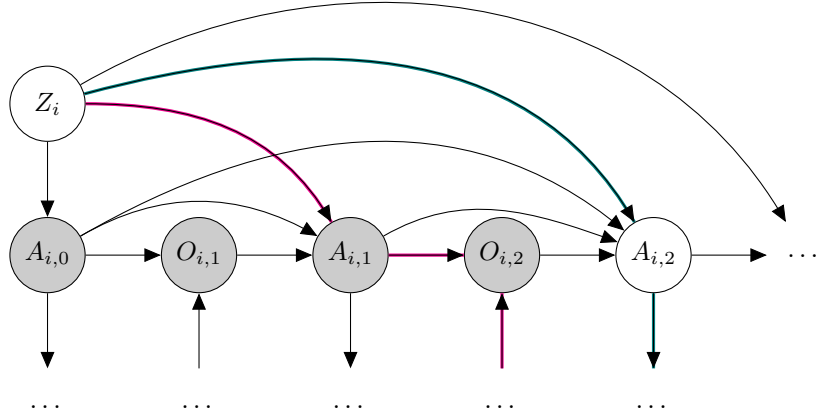
\begin{figure}
    \centering
            \begin{tikzpicture}
            \node[draw, circle, text centered, inner sep=0pt,minimum size=1cm] (z) {$Z_i$};
            
            \node[draw, circle, below=of z, text centered,fill=black!20, inner sep=0pt,minimum size=1cm] (a10) {$A_{i,0}$};
            \node[below=of a10, text centered, inner sep=0pt,minimum size=1cm] (dots1) {$\dots$};
            
            \node[draw, circle, right=of a10, text centered,fill=black!20, inner sep=0pt,minimum size=1cm] (o11) {$O_{i,1}$};
            \node[below=of o11, text centered, inner sep=0pt,minimum size=1cm] (dots2) {$\dots$};
            
            \node[draw, circle, right=of o11, text centered,fill=black!20, inner sep=0pt,minimum size=1cm] (a11) {$A_{i,1}$};
            \node[below=of a11, text centered, inner sep=0pt,minimum size=1cm] (dots3) {$\dots$};
            
            \node[draw, circle, right=of a11, text centered,fill=black!20, inner sep=0pt,minimum size=1cm] (o12) {$O_{i,2}$};
            \node[circle, below=of o12, text centered, inner sep=0pt,minimum size=1cm] (dots4) {$\dots$};
            
            \node[draw, circle, right=of o12, text centered, inner sep=0pt,minimum size=1cm] (a12) {$A_{i,2}$};
            \node[circle, below=of a12, text centered, inner sep=0pt,minimum size=1cm] (dots5) {$\dots$};
            \node[circle, right=of a12, text centered, inner sep=0pt,minimum size=1cm] (dots6) {$\dots$};

            
            \draw[magenta, very thick](z) to [out=0, in=120] (a11);
            \draw[magenta, very thick](a11)--(o12);
            \draw[magenta, very thick](o12)--(dots4);
            \draw[teal, very thick] (z) to [out=15,in=120]  (a12);
            \draw[teal, very thick] (a12) -- (dots5);
            
            \edge {z} {a10};
            \draw[-> ] (z) to [out=0,in=120]  (a11);
            \draw[-> ] (z) to [out=15,in=120]  (a12);
            \draw[-> ] (z) to [out=30,in=120]  (dots6);      
            
            \draw[-> ] (a10) --  (o11); 
            \draw[-> ] (a10) --  (dots1);    
            
            \draw[-> ] (o11) --  (a11);    
            \draw[-> ] (dots2) --  (o11);    
            
            \draw[-> ] (a11) --  (o12);    
            \draw[-> ] (a11) --  (dots3);    
            
            \draw[-> ] (dots4) --  (o12);  
            \draw[-> ] (o12) --  (a12);  
            \draw[-> ] (a12) -- (dots5);
            \draw[-> ] (a12) -- (dots6);
            
            \draw[-> ] (a10) to [out=30,in=150] (a11);
            \draw[-> ] (a10) to [out=30,in=140] (a12);
            \draw[-> ] (a11) to [out=30,in=160] (a12);

        \end{tikzpicture}
    \caption{Part of a Bayesian graph of the random variables on the space \(\PolicySet\times\Omega\). Here, the gray-marked nodes are part of the action-observation history \(\AOHistoryRV_{i,2}\). Conditional on \(\AOHistoryRV_{i,2}\), there is an unblocked path, marked in teal, from \(Z_i\) to \(A_{i,2}\) and to nodes below \(A_{i,2}\) not displayed here, making them \(d\)-connected and thus dependent. The path marked in a lighter magenta, on the other hand, is blocked, illustrating that the part of the graph below \(O_{i,2}\) may be independent of \(Z_i\).
    }
    \label{fig:bayesian_graph}
\end{figure}

To see this, one can consider a Bayesian graph of the random variables on the space \(\PolicySet\times\Omega\). For a reference on Bayesian graphs and the concepts discussed below, refer to \textcite[][ch.~1.2]{Pearl2009-bb}. In Figure~\ref{fig:bayesian_graph}, we have displayed a part of this Bayesian graph with nodes for the agent \(i\), indicating left-out parts of the graph with dots. The arrows in the graph illustrate the dependence relationships between the different variables: if there is an arrow from one node to another, this means that the other node depends on that node. Nodes belonging to the action-observation history \(\AOHistoryRV_{i,2}\) of agent \(i\) are marked in gray.

Given such a Bayesian graph, the \(d\)-separation criterion tells us which variables are dependent after conditioning on a set \(\mathcal{V}\) of variables. The criterion specifies valid, ``unblocked'' paths in the graph, depending on the graph structure and the nodes that are being conditioned on. The \(d\)-separation criterion says that two variables are dependent conditional on the variables in \(\mathcal{V}\) if and only if there is an unblocked path in the graph between them; the variables are then said to be \(d\)-connected. If there is no path, then the variables are \(d\)-separated.

An unblocked path can contain a \emph{chain} \(X\rightarrow W\rightarrow Y\) or a \emph{fork} \(X\leftarrow W\rightarrow Y\) if \(W\) is not being conditioned on. If it contains a \emph{collider} \(X\rightarrow W\leftarrow Y\), on the other hand, i.e., the two incident edges to \(W\) in the path are both directed towards the node, then the path is blocked, unless \(W\) or a descendant of \(W\) in the graph is in \(\mathcal{V}\). For instance, the path that is marked in a lighter magenta in Figure~\ref{fig:bayesian_graph} is blocked---the path cannot go from \(Z_i\) over \(\ActionRV_{i,1}\) to \(\ObservationRV_{i,2}\). An unblocked path is marked in teal.

\jt{changed the following a bit}
Using the \(d\)-separation criterion, one can tell that \(Z_i\) may be \(d\)-separated from left-out parts of the graph in Figure~\ref{fig:bayesian_graph} indicated by the dots below \(\ActionRV_{i,0}\), for instance, but that it is \(d\)-connected to some variables in the part below \(\ActionRV_{i,2}\). We do not work this out here completely, but once considering the entire graph, one can see that there is no unblocked path from \(Z_i\) to the \(Z_{-i}\), because such a path would inevitably have to traverse a collider that is not being conditioned on and that has no descendants that are being conditioned on (for instance, the observations of all other agents are colliders that connect the \(Z_{-i}\) with \(Z_i\)). Generalizing from the example of \(\AOHistory_{i,2}\), we can conclude that \(Z_i\) and \(Z_{-i}\) are independent given \(\AOHistoryRV_{i,t}\).

Now let \(\mathcal{Z}_j\subseteq \PolicySet_j\) be measurable sets for \(j\in\PlayerSet\), and let \(\AOHistory_{i,t}\in\AOHistorySet_{i,t}\) arbitrary. Then, using the above, it follows that
\begin{align}
&\E_{\Policy_{-i}\sim\Mixture_{-i}}[\mathds{1}_{\Policy_{-i}\in \mathcal{Z}_{-i}}\E_{\PolicyLatent_i\sim\mu_i}[\mathds{1}_{\Policy_i\in \mathcal{Z}_i}\Prob_\Policy(\AOHistoryRV_{i,t}=\AOHistory_{i,t})]]
\\\label{eq:27b}
&=
\Prob_{\Mixture}(Z_i\in \mathcal{Z}_i,\AOHistoryRV_{i,t}=\AOHistory_{i,t},\PolicyLatent_{-i}\in \mathcal{Z}_{-i})
\\
&=
\Prob_{\Mixture}(Z_i\in \mathcal{Z}_i\mid\AOHistoryRV_{i,t}=\AOHistory_{i,t},\PolicyLatent_{-i}\in C_{-i})
\Prob_{\Mixture}(\AOHistoryRV_{i,t}=\AOHistory_{i,t},\PolicyLatent_{-i}\in \mathcal{Z}_{-i})
\\
&=\label{eq:27a}
\Prob_{\Mixture}(Z_i\in \mathcal{Z}_i\mid\AOHistoryRV_{i,t}=\AOHistory_{i,t})
\Prob_{\Mixture}(\AOHistoryRV_{i,t}=\AOHistory_{i,t},\PolicyLatent_{-i}\in \mathcal{Z}_{-i})
\\\label{eq:27c}
&=
\E_{\Policy_{-i}\sim\Mixture_{-i}}[\mathds{1}_{\Policy_{-i}\in \mathcal{Z}_{-i}}\Prob_{\Mixture}(Z_i\in \mathcal{Z}_i\mid\AOHistoryRV_{i,t}=\AOHistory_{i,t})\E_{\Policy_i\sim\mu_i}[\Prob_\Policy(\AOHistoryRV_{i,t}=\AOHistory_{i,t})]],
\end{align}
where we use the argument about conditional independence in (\ref{eq:27a}) and the definition of \(\Prob_\Mixture\) from Equation~\ref{defn-semidirect-product} in (\ref{eq:27b}) and (\ref{eq:27c}).

Since the sets \(\mathcal{Z}_{j}\) for \(j\in\PlayerSet\setminus\{i\}\) were arbitrary, it follows that \(\Mixture_{-i}\)-almost surely, it is \begin{multline}\label{equation-marginals-are-equal-1}\E_{\Policy_i\sim\mu_i}[\mathds{1}_{\Policy_i\in \mathcal{Z}_i}\Prob_{\Policy_i,Z_{-i}}(\AOHistoryRV_{i,t}=\AOHistory_{i,t})]
\\=
\Prob_{\Mixture}(Z_i\in \mathcal{Z}_i\mid\AOHistoryRV_{i,t}=\AOHistory_{i,t})\E_{\Policy_i\sim\mu_i}[\Prob_{\Policy_i,Z_{-i}}(\AOHistoryRV_{i,t}=\AOHistory_{i,t})].\end{multline}
Moreover, since \(\Mixture\) was arbitrary, Equation~\ref{equation-marginals-are-equal-1} holds for any distribution with independent local policies. If we divide by the term \(\E_{\Policy_i\sim\mu_i}[\Prob_{\Policy_i,Z_{-i}}(\AOHistoryRV_{i,t}=\AOHistory_{i,t})]\), this becomes
\begin{equation}\label{eq:200}
\Prob_{\Mixture}(Z_i\in \mathcal{Z}_i\mid\AOHistoryRV_{i,t}=\AOHistory_{i,t})
=\frac{\E_{\Policy_i\sim\mu_i}[\mathds{1}_{\Policy_i\in \mathcal{Z}_i}\Prob_{\Policy_i,Z_{-i}}(\AOHistoryRV_{i,t}=\AOHistory_{i,t})]}{
\E_{\Policy_i\sim\mu_i}[\Prob_{\Policy_i,Z_{-i}}(\AOHistoryRV_{i,t}=\AOHistory_{i,t})]},\end{equation}
which gives us a formula for the distribution of \(\PolicyLatent_i\) given \(\AOHistoryRV_{i,t}\) under the measure \(\Prob_\Mixture\) that is independent of the distributions \(\Mixture_{-i}\). But to be able to do so, we have to find a value for \(\PolicyLatent_{-i}\) such that \(\E_{\Policy_i\sim\mu_i}[\Prob_{\Policy_i,Z_{-i}}(\AOHistoryRV_{i,t}=\AOHistory_{i,t})]\) is nonzero under \(\Mixture\).

Now let \(\Mixture_i\in\Delta(\PolicySet_i)\) and let \(\Mixture_{-i},\Mixture'_{-i}\in\Delta(\PolicySet_{-i})\) be any two distributions with independent local policies such that \(\Prob_{\Mixture_i\otimes\Mixture_{-i}}( \AOHistoryRV_{i,t}=\AOHistory_{i,t})>0\) and \(\Prob_{\Mixture_i\otimes\Mixture'_{-i}}( \AOHistoryRV_{i,t}=\AOHistory_{i,t})>0\). Define \(\Mixture_i':=\Mixture_i\), \(\Mixture:=\Mixture_i\otimes\Mixture_{-i}\), and \(\Mixture':=\Mixture_i\otimes\Mixture'_{-i}\). 
Our goal is to show that
\[\Prob_{\Mixture}(\ActionRV_{i,t}=a_i\mid \AOHistoryRV_{i,t}=\AOHistory_{i,t})
=\Prob_{\Mixture'}(\ActionRV_{i,t}=a_i\mid \AOHistoryRV_{i,t}=\AOHistory_{i,t}).\]
To that end, we define a third distribution \(\MixtureHat:=\otimes_{j\in\PlayerSet}(\frac{1}{2}\Mixture_j+\frac{1}{2}\Mixture_j')\). Apparently, it is then \(\Mixture_i=\MixtureHat_i=\Mixture'_i\) and also \(\MixtureHat\) has independent local policies.
We now prove separately for \(\Mixture\) and \(\Mixture'\) that the distribution of \(\PolicyLatent_i\) given \(\{\AOHistoryRV_{i,t}=\AOHistory_{i,t}\}\) under \(\Mixture\) respectively \(\Mixture'\) is equal to the one under \(\MixtureHat\). To do so, we use that \(\Mixture\) and \(\Mixture'\) are absolutely continuous with respect to \(\MixtureHat\) to find desired values for \(\PolicyLatent_{-i}\) such that we can apply Equation~\ref{eq:200}.

First, let \(\mathcal{Z}_i\in\PolicySet_i\) be an arbitrary measurable set. Using Equation~\ref{equation-marginals-are-equal-1}, we can find measurable sets \(\mathcal{Z}_{-i},\hat{\mathcal{Z}}_{-i}\subseteq\PolicySet_{-i}\) such that \(\Mixture_{-i}(\mathcal{Z}_{-i})=1=\MixtureHat_{-i}(\hat{\mathcal{Z}}_{-i})\), and such that
\begin{multline}\label{eq:28a}\E_{\Policy_i\sim\mu_i}[\mathds{1}_{\Policy_i\in \mathcal{Z}_i}\Prob_{\Policy_i,\Policy_{-i}}(\AOHistoryRV_{i,t}=\AOHistory_{i,t})]
\\=
\Prob_{\Mixture}(Z_i\in \mathcal{Z}_i\mid\AOHistoryRV_{i,t}=\AOHistory_{i,t})\E_{\Policy_i\sim\mu_i}[\Prob_{\Policy_i,\Policy_{-i}}(\AOHistoryRV_{i,t}=\AOHistory_{i,t})].\end{multline}
for any \(\Policy_{-i}\in \mathcal{Z}_{-i}\)
and
\begin{multline}\label{eq:28b}\E_{\Policy_i\sim\MixtureHat_i}[\mathds{1}_{\Policy_i\in \mathcal{Z}_i}\Prob_{\Policy_i,\PolicyHat_{-i}}(\AOHistoryRV_{i,t}=\AOHistory_{i,t})]
\\=
\Prob_{\hat{\Mixture}}(Z_i\in \mathcal{Z}_i\mid\AOHistoryRV_{i,t}=\AOHistory_{i,t})\E_{\Policy_i\sim\mu_i}[\Prob_{\Policy_i,\PolicyHat_{-i}}(\AOHistoryRV_{i,t}=\AOHistory_{i,t})].\end{multline}
for any \(\PolicyHat_{-i}\in \hat{\mathcal{Z}}_{-i}\).

Next, it follows that \(\Prob_\Mixture(\{Z_{-i}\in \mathcal{Z}_{-i}\}\cap\{\AOHistoryRV_{i,t}=\AOHistory_{i,t}\})>0\),
which by definition of \(\MixtureHat\) implies
\(\Prob_{\MixtureHat}(\{Z_{-i}\in \mathcal{Z}_{-i}\}\cap\{\AOHistoryRV_{i,t}=\AOHistory_{i,t}\})>0\)
and thus by definition of \(\hat{\mathcal{Z}}_{-i}\) also \(\Prob_{\MixtureHat}(\{Z_{-i}\in \mathcal{Z}_{-i}\cap\hat{\mathcal{Z}}_{-i}\}\cap\{\AOHistoryRV_{i,t}=\AOHistory_{i,t}\})>0\). Since 
\begin{multline}\Prob_{\MixtureHat}(\{Z_{-i}\in \mathcal{Z}_{-i}\cap\hat{\mathcal{Z}}_{-i}\}\cap\{\AOHistoryRV_{i,t}=\AOHistory_{i,t}\})
\\
=\E_{\Policy_{-i}\sim\MixtureHat_{-i}}\left[\mathds{1}_{\mathcal{Z}_{-i}\cap\hat{\mathcal{Z}}_{-i}}(\Policy_{-i})\E_{\Policy_i\sim\MixtureHat_i}[\Prob_{\Policy_i,\Policy_{-i}}(\AOHistoryRV_{i,t}=\AOHistory_{i,t})]\right],\end{multline}
there must be
\(\Policy_{-i}\in \Proj_{\PolicySet_{-i}}(\{Z_{-i}\in \mathcal{Z}_{-i}\cap\hat{\mathcal{Z}}_{-i}\}\cap\{\AOHistoryRV_{i,t}=\AOHistory_{i,t}\})\)
such that
\[\E_{\Policy_i\sim\MixtureHat_i}[\Prob_{\Policy_i,\Policy_{-i}}(\AOHistoryRV_{i,t}=\AOHistory_{i,t})]>0.\]

Lastly, using that \(\MixtureHat_i=\Mixture_i\), it follows that
\begin{equation}\E_{\Policy_i\sim\Mixture_i}[\Prob_{\Policy_i,\Policy_{-i}}(\AOHistoryRV_{i,t}=\AOHistory_{i,t})]
=\E_{\Policy_i\sim\MixtureHat_i}[\Prob_{\Policy_i,\Policy_{-i}}(\AOHistoryRV_{i,t}=\AOHistory_{i,t})]>0.\end{equation}
Hence, we can use Equations~\ref{eq:28a} and \ref{eq:28b} to conclude that
\begin{align}
    \Prob_{\Mixture}(Z_i\in \mathcal{Z}_i\mid\AOHistoryRV_{i,t}=\AOHistory_{i,t})
&\overset{\text{(\ref{eq:28a})}}{=}\frac{
    \E_{\Policy_i\sim\mu_i}[\mathds{1}_{\Policy_i\in \mathcal{Z}_i}\Prob_{\Policy_i,\Policy_{-i}}(\AOHistoryRV_{i,t}=\AOHistory_{i,t})]
}{
    \E_{\Policy_i\sim\mu_i}[\Prob_{\Policy_i,\Policy_{-i}}(\AOHistoryRV_{i,t}=\AOHistory_{i,t})]
}
\\
&\overset{\Mixture_i=\MixtureHat_i}{=}\frac{
    \E_{\Policy_i\sim\MixtureHat_i}[\mathds{1}_{\Policy_i\in \mathcal{Z}_i}\Prob_{\Policy_i,\Policy_{-i}}(\AOHistoryRV_{i,t}=\AOHistory_{i,t})]
}{
    \E_{\Policy_i\sim\MixtureHat_i}[\Prob_{\Policy_i,\Policy_{-i}}(\AOHistoryRV_{i,t}=\AOHistory_{i,t})]
}
\\
&\overset{\text{(\ref{eq:28b})}}{=}
\Prob_{\MixtureHat}(Z_i\in \mathcal{Z}_i\mid\AOHistoryRV_{i,t}=\AOHistory_{i,t}).
\end{align}

Now note that one can make an exactly analogous argument for \(\Mixture'\) and \(\MixtureHat\), potentially using a different \(\Policy_{-i}\). Hence, it follows that
\begin{equation}\label{equation-marginals-are-equal-2}
\Prob_{\Mixture}(Z_i\in \mathcal{Z}_i\mid\AOHistoryRV_{i,t}=\AOHistory_{i,t})=\Prob_{\MixtureHat}(Z_i\in \mathcal{Z}_i\mid\AOHistoryRV_{i,t}=\AOHistory_{i,t})=\Prob_{\Mixture'}(Z_i\in \mathcal{Z}_i\mid\AOHistoryRV_{i,t}=\AOHistory_{i,t}).
\end{equation}
Since \(\mathcal{Z}_i\in\PolicySet_i\) was arbitrary, it follows that the distribution of \(Z_i\) given \(\{\AOHistoryRV_{i,t}=\AOHistory_{i,t}\}\) is equal under \(\Prob_\Mixture\) and \(\Prob_{\Mixture'}\).

To conclude the proof, we can use this to show that also the distribution of actions under \(\Prob_\Mixture\) and \(\Prob_{\Mixture'}\) are equal, conditional on \(\{\AOHistoryRV_{i,t}=\AOHistory_{i,t}\}\). For \(\Action_i\in\ActionSet_i\), it is
\begin{align}
    \label{eq:32a}
    \Prob_{\Mixture}(\ActionRV_{i,t}=\Action_i\mid \AOHistoryRV_{i,t}=\AOHistory_{i,t})
    &=\E_{\Mixture}\left[\Prob_\Mixture(\ActionRV_{i,t}=\Action_i\mid\PolicyLatent, \AOHistoryRV_{i,t})\mid \AOHistoryRV_{i,t}=\AOHistory_{i,t}\right]
    \\\label{eq:32c}
    &=\E_{\Mixture}\left[\Prob_\PolicyLatent(\ActionRV_{i,t}=\Action_i\mid \AOHistoryRV_{i,t})\mid \AOHistoryRV_{i,t}=\AOHistory_{i,t}\right]
    \\&=\E_{\Mixture}\left[\PolicyLatent_i(\Action_i\mid \AOHistory_{i,t})\mid \AOHistoryRV_{i,t}=\AOHistory_{i,t}\right]
    \\
    &\overset{(\ref{equation-marginals-are-equal-2})}{=}
    \E_{\Mixture'}\left[\PolicyLatent_i(\Action_i\mid \AOHistory_{i,t})\mid \AOHistoryRV_{i,t}=\AOHistory_{i,t}\right]
    \\&=\E_{\Mixture'}\left[\Prob_\PolicyLatent(\ActionRV_{i,t}=\Action_i\mid \AOHistoryRV_{i,t})\mid \AOHistoryRV_{i,t}=\AOHistory_{i,t}\right]
    \\
    \label{eq:32d}
    &=\E_{\Mixture'}\left[\Prob_{\Mixture'}(\ActionRV_{i,t}=\Action_i\mid \PolicyLatent,\AOHistoryRV_{i,t})\mid \AOHistoryRV_{i,t}=\AOHistory_{i,t}\right]
    \\
    \label{eq:32b}
    &=\Prob_{\Mixture'}(\ActionRV_{i,t}=\Action_i\mid \AOHistoryRV_{i,t}=\AOHistory_{i,t}),
\end{align}
where we have used the tower property in (\ref{eq:32a}) and (\ref{eq:32b}), and Equation~\ref{conditional-probability-mixture} in (\ref{eq:32c}) and (\ref{eq:32d}). This is what we wanted to show.
\end{proof}


In the following, we write \(\Policy^\Mixture:=(\Policy_i^{\Mixture_i})_{i\in\PlayerSet}\) for the joint policy corresponding to a distribution \(\Mixture\) with independent local policies.
Our next goal is to prove that the distribution over histories is the same under \(\Prob_{\Policy^\Mixture}\) as under \(\Prob_{\Mixture}\).

\begin{prop}\label{mixture-lemma}
Consider any distribution \(\Mixture\in\Delta(\PolicySet)\) with independent local policies. Let \(\Policy^\Mixture\) be the joint policy corresponding to \(\Mixture\), as defined above. Then the history \(\HistoryRV\) has the same distribution under \(\Prob_\Mixture\) as under \(\Prob_{\Policy^\Mixture}\). In particular,
 it is \(J^D(\Mixture)=J^D(\Policy^\Mixture)\).
\end{prop}
\begin{proof}
Fix a distribution \(\Mixture\in\Delta(\PolicySet)\) with independent local policies. We show by induction that for all \(t\in\{0,\dotsc,\Tmax\}\), it is \(\Prob_\Mixture(\HistoryRV_t=\History_t)=\Prob_{\Policy^\Mixture}(\HistoryRV_t=\History_t)\) for any \(\History_t\in\HistorySet_t\).

First, note that by Definition~\ref{policy-mixture}, for any \(i\in\PlayerSet,\Action_i\in\ActionSet_i, t\in\{0,\dotsc,h\}\), and \(\AOHistory_{i,t}\in\AOHistorySet_i\) such that \(\Prob_\Mixture( \AOHistoryRV_{i,t}=\AOHistory_{i,t})>0\), it is
\begin{equation}\label{eq:5}\Prob_{\Policy^\Mixture}(\ActionRV_{i,t}=\Action_i\mid \AOHistoryRV_{i,t}=\AOHistory_{i,t})=\Policy^{\Mixture_i}_i(a_i\mid \AOHistoryRV_{i,t}=\AOHistory_{i,t}) = 
 \Prob_\Mixture(\ActionRV_{i,t}=\Action_i\mid \AOHistoryRV_{i,t}=\AOHistory_{i,t}).
\end{equation}
In particular, this holds for \(t=0\), in which case it is \(\AOHistoryRV_{i,0}=\emptyset\) for \(i\in\PlayerSet\). Hence, for \(\Action\in\ActionSet, \State\in\StateSet\), it is
\begin{multline}\Prob_\Mixture(\StateRV_0=\State,\ActionRV_0=\Action)
\overset{(\ref{eq:equation-tower-prop-z})}{=}
\E_\Mixture[\Prob_Z(\StateRV_0=\State,\ActionRV_{0}=\Action)]
=\E_\Mixture[\Prob_Z(\StateRV_0=\State)\Prob_Z(\ActionRV_0=\Action)]
\\
\overset{\text{(i)}}{=}b_0(\State)\E_\Mixture[\Prob_Z(\ActionRV_0=\Action)]
=
b_0(\State)\E_\Mixture[\prod_{i\in\PlayerSet}Z_i(\ActionRV_{i,0}\mid \emptyset)]
\overset{\text{(ii)}}{=}b_0(\State)\prod_{i\in\PlayerSet}\E_\Mixture[Z_i(\Action_{i}\mid \emptyset)]
\\
\overset{\text{(\ref{eq:equation-tower-prop-z})}}{=}
b_0(\State)\prod_{i\in\PlayerSet}\Prob_\Mixture(\ActionRV_{i,0}=\Action_{i})
=b_0(\State)\prod_{i\in\PlayerSet}\Prob_\Mixture(\ActionRV_{i,0}=\Action_{i})
\overset{(\ref{eq:5})}{=}
b_0(\State)\prod_{i\in\PlayerSet}\Prob_{\Policy^\Mixture}(\ActionRV_{i,0}=\Action_i)
\\
=b_0(\State)\Prob_{\Policy^\Mixture}(\ActionRV_{0}=\Action)
=\Prob_{\Policy^\Mixture}(\StateRV_0=\State,\ActionRV_0=\Action).
\end{multline}
Here, we also use 
(i) the fact that that the initial state distribution does not depend on the policy, and (ii) the fact that \(\Mixture\) has independent local policies and thus also the \(Z_i\) are independent in the probability space \(\Prob_\Mixture\).

Next, assume that \(0\leq t-1\leq \Tmax\) and \(\Prob_\Mixture(\HistoryRV_{t-1}=\History_{t-1})=\Prob_{\Policy^\Mixture}(\HistoryRV_{t-1}=\History_{t-1})\) for any \(\History_{t-1}\in\HistorySet_{t-1}\). Let \(\History_{t-1}=(\State_0,\Action_0,\Reward_0,\State_1,\dotsc,\Reward_{t-1})\in\HistorySet_{t-1}\) arbitrary such that
\(\Prob_\Mixture(\HistoryRV_{t-1}=\History_{t-1})=\Prob_{\Policy^\Mixture}(\HistoryRV_{t-1}=\History_{t-1})>0\). As in the proof of Lemma~\ref{lemma-mixtures-factorize}, it follows from considering the \(d\)-separation criterion on a Bayesian graph of the random variables defined on \(\PolicySet\times\Omega\) that conditioning on \(\HistoryRV_{t-1}\) does not make the \(Z_i\) dependent \parencite[see][ch. 1.2]{Pearl2009-bb}. The same criterion also says that, after conditioning on \(A_{i,t-1}\) and \(\AOHistoryRV_{i,t-1}\), one cannot gain additional information about \(Z_i\) from the other components of \(\HistoryRV_{t-1}\) or from \(\ObservationRV_{i,t}\) (that is, \(A_{i,t-1}\) and \(\AOHistoryRV_{i,t-1}\) \(d\)-separate \(Z_i\) from the rest of the history).
Using this in (\ref{eq:26a}), it follows for arbitrary \(\State_t,\Observation_t,\Action_t,\Reward_t\) that
\begin{align}\label{eq:26c}
&\Prob_\Mixture(\StateRV_t=\State_t,\ObservationRV_t=\Observation_t,\ActionRV_t=\Action_t,\RewardRV_t=\Reward_t\mid \HistoryRV_{t-1}=\History_{t-1})
\\
&=\E_\Mixture[\Prob_\Mixture(\StateRV_t=\State_t,\ObservationRV_t=\Observation_t,\ActionRV_t=\Action_t,\RewardRV_t=\Reward_t\mid \HistoryRV_{t-1},Z)\mid \HistoryRV_{t-1}=\History_{t-1}]
\\
&
=\E_\Mixture[P(\State_t\mid \State_{t-1},\Action_{t-1})O(\Observation_t\mid \State_t,\Action_{t-1})\prod_{i\in\PlayerSet}Z_i(\Action_{i,t}\mid\AOHistory_{i,t})\mathds{1}_{\RewardFunction(\State,\Action)=\Reward}\mid \HistoryRV_{t-1}=\History_{t-1}]
\\
\nonumber
&=
P(\State_t\mid \State_{t-1},\Action_{t-1})O(\Observation_t\mid \State,\Action_{t-1})
\\\label{eq:26a}
&\quad\prod_{i\in\PlayerSet}\E_\Mixture[Z_i(\Action_{i,t}\mid \AOHistory_{i,t})\mid \ActionRV_{i,t-1}=\Action_{i,t-1}, \AOHistoryRV_{i,t-1}=\AOHistory_{i,t-1}]\mathds{1}_{\RewardFunction(\State_t,\Action_t)=\Reward_t}
\\
&=P(\State_t\mid \State_{t-1},\Action_{t-1})O(\Observation_t\mid \State,\Action_{t-1})\prod_{i\in\PlayerSet}\Prob_\Mixture(\ActionRV_{i,t}=\Action_{i,t}\mid\AOHistoryRV_{i,t}=\AOHistory_{i,t})\mathds{1}_{\RewardFunction(\State_t,\Action_t)=\Reward_t}
\\\label{eq:26b}
&=P(\State_t\mid \State_{t-1},\Action_{t-1})O(\Observation_t\mid \State,\Action_{t-1})\prod_{i\in\PlayerSet}\Prob_{\Policy^\Mixture}(\ActionRV_{i,t}=\Action_{i,t}\mid\AOHistoryRV_{i,t}=\AOHistory_{i,t})\mathds{1}_{\RewardFunction(\State_t,\Action_t)=\Reward_t}
\\
&= \Prob_{\Policy^\Mixture}(\StateRV_t=\State_t,\ObservationRV_t=\Observation_t,\ActionRV_t=\Action_t,\RewardRV_t=\Reward_t\mid \HistoryRV_{t-1}=\History_{t-1}).\label{eq:26d}
\end{align}
In (\ref{eq:26b}), we again use Equation~\ref{eq:5}, which is possible since  \(\Prob_\Mixture(\HistoryRV_{t-1}=\History_{t-1})>0\) implies that also \(\Prob_\Mixture( \AOHistoryRV_{i,t-1}=\AOHistory_{i,t-1})>0\), where \(\AOHistory_{i,t-1}\) is defined as the projection of \(\History_{t-1}\) onto \(\AOHistorySet_{i,t-1}\).

To conclude the inductive step, let \(\History_t\in\HistorySet_t\) and choose \(\History_{t-1}\) as the projection of \(\History_t\) onto \(\HistorySet_{t-1}\). If \[\Prob_\Mixture(\HistoryRV_{t-1}=\History_{t-1})=\Prob_{\Policy^\Mixture}(\HistoryRV_{t-1}=\History_{t-1})=0,\]
necessarily also \(\Prob_\Mixture(\HistoryRV_{t}=\History_{t})=\Prob_{\Policy^\Mixture}(\HistoryRV_{t}=\History_{t})=0\) and there is nothing more to show. Assume now that this is not the case. Using the inductive hypothesis and Equations (\ref{eq:26c})--(\ref{eq:26d}) in (*), it then follows that
\begin{multline}\Prob_\Mixture( \HistoryRV_{t}=\History_t)
=\Prob_\Mixture( \HistoryRV_{t}=\History_t\mid \HistoryRV_{t-1}=\History_{t-1})\Prob_\Mixture(\HistoryRV_{t-1}=\History_{t-1})
\\=\Prob_{\Mixture}(\StateRV_t=\State_t,\ObservationRV_t=\Observation_t,\ActionRV_t=\Action_t,\RewardRV_t=\Reward_t\mid \HistoryRV_{t-1}=\History_{t-1})\Prob_\Mixture(\HistoryRV_{t-1}=\History_{t-1})
\\\overset{\text{(*)}}{=}\Prob_{\Policy^\Mixture}(\StateRV_t=\State_t,\ObservationRV_t=\Observation_t,\ActionRV_t=\Action_t,\RewardRV_t=\Reward_t\mid \HistoryRV_{t-1}=\History_{t-1})\Prob_{\Policy^\Mixture}(\HistoryRV_{t-1}=\History_{t-1})
\\=\Prob_{\Policy^\Mixture}( \HistoryRV_{t}=\History_t).
\end{multline}

This concludes the induction. In particular, it follows that \[\Prob_\Mixture(\HistoryRV_\Tmax=\History_\Tmax)=\Prob_{\Policy^\Mixture}(\HistoryRV_\Tmax=\History_\Tmax).\] Hence, \(\HistoryRV=\HistoryRV_\Tmax\) has the same distribution under both \(\Prob_\Mixture\) and \(\Prob_{\Policy^\Mixture}\).

Turning to the ``in particular'' statement, we can use the tower property (i) and Equation~\ref{conditional-probability-mixture} (ii) to follow that
\begin{multline}J(\Policy^\Mixture)=\E_{\Policy^\Mixture}[\sum_{t=1}^hR_t]=\E_\Mixture[\sum_{t=1}^hR_t]
\\
\overset{\text{(i)}}{=}\E_\Mixture\left[\E_\Mixture\left[\sum_{t=1}^hR_t\,\middle\vert\, Z\right]\right]
\overset{\text{(ii)}}{=}\E_\Mixture\left[\E_Z\left[\sum_{t=1}^hR_t\right]\right]
\\
=\E_{\Policy\sim \Mixture}[J(\Policy)]
=J(\Mixture).
\end{multline}
\end{proof}



Before we turn to the OP distribution, we prove another useful Lemma, stating that the mapping from distributions to corresponding joint policies and the pushforward by isomorphisms commute. That is, the policy corresponding to the pushforward of a distribution and the pushforward of the policy corresponding to that distribution are the same.

\begin{lem}\label{lemma-isomorphism-mixture-commute}
Let \(\DecP,\DecPSecond\) be isomorphic Dec-POMDPs with isomorphism \(\Isom\in\Iso(\DecP,\DecPSecond)\). Let \(\Policy\in\PolicySet^\DecP\) and let \(\Mixture\in \Delta(\PolicySet^\DecP)\) be any distribution with independent local policies. Then
\[\Isom^*\Policy^\Mixture=\Policy^{\Isom^*\Mixture}.\]
\end{lem}
\begin{proof}
First, we have to find an expression for the marginal distributions \((\Isom^*\Mixture)_i\) and prove that \(\Isom^*\Mixture\) is a distribution with independent local policies. To that end, consider measurable sets \(\mathcal{Z}_i\subseteq\PolicySet^\DecPSecond_i\) for \(i\in\PlayerSet^\DecPSecond\) and define \(\mathcal{Z}:=\prod_{i\in\PlayerSet^\DecPSecond}\mathcal{Z}_i\). In the following, we adopt the notation \(\Policy_i\circ\Isom:= \Policy_i(\Isom\cdot\mid\Isom\cdot)\) and \(\mathcal{Z}_i\circ \Isom:= \{\Policy_i\circ\Isom\mid\Policy_i\in\mathcal{Z}_i\}\). Note that
\begin{multline}\label{eq:29}(\Isom^*)^{-1}(\mathcal{Z})
=\{\Policy\mid \Isom^*\Policy\in\mathcal{Z}\}
=\{\Policy\mid \forall i\in\PlayerSet^\DecPSecond\colon \Policy_{\Isom^{-1}i}\circ\Isom^{-1}\in\mathcal{Z}_i)\}
\\
=\{\Policy\mid \forall j\in\PlayerSet^\DecP\colon \Policy_{j}\in\mathcal{Z}_{\Isom j}\circ\Isom\}
=\prod_{j\in\PlayerSet^\DecP}\mathcal{Z}_{\Isom j}\circ\Isom.
\end{multline}
Hence, using in (i) that \(\Mixture\) has independent local policies, it follows that
\begin{multline}\label{eq:30}(\Isom^*\Mixture)(\mathcal{Z})
=\Mixture((\Isom^*)^{-1}(\mathcal{Z}))
\overset{(\ref{eq:29})}{=}\Mixture(\prod_{j\in\PlayerSet^\DecP}\mathcal{Z}_{\Isom j}\circ \Isom)
\overset{\text{(i)}}{=}\prod_{j\in\PlayerSet^\DecP}\Mixture_j(\mathcal{Z}_{\Isom j}\circ \Isom)
=\prod_{i\in\PlayerSet^\DecPSecond}\Mixture_{\Isom^{-1}i}(\mathcal{Z}_{i}\circ \Isom).
\end{multline}

Now let \(i\in\PlayerSet^\DecPSecond\) arbitrary. With the choice of \(\hat{\mathcal{Z}}_i:=\mathcal{Z}_i\) and \(\hat{\mathcal{Z}}_k:=\PolicySet_k^\DecPSecond\) for all \(k\in\PlayerSet^\DecPSecond\setminus\{i\}\), it is
\begin{multline}\label{eq:70}(\Isom^*\Mixture)_i(\mathcal{Z}_i)
=\Isom^*\Mixture(\Proj_i^{-1}(\mathcal{Z}_i))
=\Isom^*\Mixture(\hat{\mathcal{Z}}_1\times\dotsb\times\hat{\mathcal{Z}}_N))
\\
\overset{(\ref{eq:30})}{=}
\prod_{k\in\PlayerSet^\DecPSecond}\Mixture_{\Isom^{-1}k}(\hat{\mathcal{Z}}_{k}\circ \Isom)
=\Mixture_{\Isom^{-1}i}(\mathcal{Z}_{i}\circ \Isom)
,\end{multline}
as
\(\Mixture_{\Isom^{-1}k}(\hat{\mathcal{Z}}_{k}\circ \Isom)=\Mixture_{\Isom^{-1}k}(\PolicySet_{ k}\circ\Isom)=\Mixture_{\Isom^{-1}k}(\PolicySet_{\Isom^{-1} k})=1\)
for any \(k\in\PlayerSet^\DecPSecond\setminus\{i\}\).

Since \(i\) was arbitrary, this shows that \(\Isom^*\Mixture(\mathcal{Z})=\prod_{i\in\PlayerSet^\DecPSecond}(\Isom^*\Mixture)_i(\mathcal{Z}_i)\).
Since the sets \(\mathcal{Z}_i\) were arbitrary and the Cartesian products of these sets are a \(\pi\)-system and generate the product \(\sigma\)-Algebra \(\mathcal{F}\), this shows that \(\Isom^*\Mixture\) has independent local policies. 

Next, let \(i\in\PlayerSet^\DecPSecond\), \(t\in\{0,\dotsc,\Tmax\}\) and \(\AOHistory_{i,t}\in\AOHistorySet_{i,t}^\DecPSecond\). Note that \(\Proj_i(\Isom^*\Policy^\Mixture)=\Policy_{\Isom^{-1}i}^{\Mixture_{\Isom^{-1}i}}\circ\Isom^{-1}\) and 
\(
\Proj_i(\Policy^{\Isom^*\Mixture})
=\Policy_i^{(\Isom^*\Mixture)_i}\).
Hence, it remains to prove that \(\Policy_{\Isom^{-1}i}^{\Mixture_{\Isom^{-1}i}}\circ \Isom^{-1}=\Policy_i^{(\Isom^*\Mixture)_i}\).

To that end, let \(j\in\PlayerSet^\DecP\) such that \(\Isom j= i\). By Theorem~\ref{lem-pull-back-isomorphism-compatibility}, it is
\begin{equation}
    \label{equation-histories-pullbacks}
    \Prob_{\Policy}(\AOHistoryRV_{j,  t} =\Isom^{-1} \AOHistory_{i,t})
    =
    \Prob_{\Isom^*\Policy}(\AOHistoryRV_{ i,t}=\AOHistory_{i,t})
\end{equation}
for any \(\Policy\in \PolicySet^\DecP\). Letting, \(\Mixture_{-j}\in\Delta(\PolicySet_{-j}^\DecP)\) with independent local policies arbitrary and defining \(\Mixture:=\Mixture_{i}\otimes\Mixture_{-i}\), this means that also
\begin{multline}
\label{eq:31}
\Prob_\Mixture(\AOHistoryRV_{ j,  t} =\Isom^{-1} \AOHistory_{i,t})
\overset{\text{(ii)}}{=}\E_\Mixture\left[
\Prob_\PolicyLatent(\AOHistoryRV_{ j,  t} = \Isom^{-1}\AOHistory_{i,t})
\right]
=\E_{\Mixture}\left[
\Prob_{\Isom^*\PolicyLatent}(\AOHistoryRV_{ i,  t} = \AOHistory_{i,t})
\right]
\\
=\E_{\Isom^*\Mixture}\left[
\Prob_{\PolicyLatent}(\AOHistoryRV_{i,  t} = \AOHistory_{i,t})
\right]
\overset{\text{(ii)}}{=}\Prob_{\Isom^*\Mixture}(\AOHistoryRV_{ i,  t} = \AOHistory_{i,t}),
\end{multline}
where we use Equation~\ref{eq:equation-tower-prop-z} in (ii).
Since \(\Isom^*\) is a bijection on the space of policies, it is
\[\{\Isom^*(\Mixture_j\otimes\Mixture_{-j})\mid \Mixture_{-j}\in\Delta(\PolicySet_{-j}^\DecP)\}=\{(\Isom^*\Mixture)_i\otimes\MixtureHat_{-i}\mid\MixtureHat_{-i}\in\Delta(\PolicySet_{-i}^\DecPSecond)\}.\]
Hence, (\ref{eq:31}) implies that there is some \(\Mixture_{-j}\in\Delta(\PolicySet_{-j}^\DecP)\) with independent local policies such that \(\Prob_{\Mixture_j\otimes\Mixture_{-j}}(\AOHistoryRV_{j,  t} =\Isom^{-1} \AOHistory_{i,t})>0\) if and only if there is \(\MixtureHat_{-i}\in\Delta(\PolicySet_{-i}^\DecPSecond)\) with independent local policies such that \(\Prob_{(\Isom^*\Mixture)_i\otimes \MixtureHat_{-i}}(\AOHistoryRV_{i,  t} =\AOHistory_{i,t})>0\).

Using this fact, it suffices to distinguish the two cases where such a distribution does exist and where it does not exist. First, assume that it does not exist. Then by Definition~\ref{policy-mixture}, both \(\Policy_{j}^{\Mixture_j}(\cdot\mid \Isom^{-1} \AOHistory_{i,t})\) and \(\Policy_i^{(\Isom^*\Mixture)_i}(\cdot\mid\AOHistory_{i,t})\) are uniform distributions. Since \(\Isom^{-1}\) is a bijection on \(\ActionSet_i^\DecPSecond\), also \(\Policy_{j}^{\Mixture_{j}}(\Isom^{-1}\cdot \mid \Isom^{-1}\AOHistory_{i,t})\) is a uniform distribution, and hence
\[\Policy_{j}^{\Mixture_{j}}(\Isom^{-1}\cdot \mid \Isom^{-1}\AOHistory_{i,t})=\Policy_i^{(\Isom^*\Mixture)_i}(\cdot\mid \AOHistory_{i,t}).
\]

Second, consider the case in which a distribution \(\Mixture_{-j}\in\Delta(\PolicySet_{-j}^\DecP)\) with independent local policies exists such that
\[\Prob_{\Mixture_j\otimes\Mixture_{-j}}(\AOHistoryRV_{j,  t} =\Isom^{-1} \AOHistory_{i,t})>0,\]
and define \(\Mixture:=\Mixture_j\otimes\Mixture_{-j}\).
Then for \(\Action_i\in\ActionSet^\DecPSecond_i\), it is
\begin{align}
    \label{eq:33a}
    \Policy^{\Mixture_j}_{j}(\Isom^{-1} \Action_i\mid \Isom^{-1}\AOHistory_{i,t})
    &= \Prob_\Mixture(\ActionRV_{j,t}=\Isom^{-1} \Action_i\mid \AOHistoryRV_{j,t}=\Isom^{-1}\AOHistory_{i,t})
    \\
    &= \E_\Mixture[\Prob_{\Mixture}(\ActionRV_{j,t}=\Isom^{-1} a_i\mid\AOHistoryRV_{j, t},Z)\mid \AOHistoryRV_{ j,t}=\Isom^{-1}\AOHistory_{i,t}]
    \\&= \E_\Mixture[Z_{j}(\Isom^{-1} \Action_i\mid\Isom^{-1}\AOHistory_{i,t})\mid 
    \AOHistoryRV_{j,t}=\Isom^{-1}\AOHistory_{i,t}]
    \\
    &=\frac{\int_{\PolicySet^\DecP}\int_{\Omega}(\Isom^*Z)_{i}( \Action_i\mid\AOHistory_{i,t})\mathds{1}_{\AOHistoryRV_{ j,t}={\Isom^{-1}}\AOHistory_{i,t}}\mathrm{d}\Prob_{Z}\mathrm{d}\Mixture
    }
    {
    \int_{\PolicySet^\DecP}\int_{\Omega}\mathds{1}_{\AOHistoryRV_{j,t}={\Isom^{-1}}\AOHistory_{i,t}}
    \mathrm{d}\Prob_{Z}\mathrm{d}\Mixture
    }
    \\
    &=
    \frac{\int_{\PolicySet^\DecP}(\Isom^*Z)_{ i}( \Action_i\mid\AOHistory_{i,t})\Prob_Z({\AOHistoryRV_{ j,t}={\Isom^{-1}}\AOHistory_{i,t}})\mathrm{d}\Mixture
    }
    {
    \int_{\PolicySet^\DecP}\Prob_Z(\AOHistoryRV_{j,t}={\Isom^{-1}}\AOHistory_{i,t})\mathrm{d}\Mixture
    }
    \\
    &
    =\label{eq:33b}
    \frac{\int_{\PolicySet^\DecP}(\Isom^*Z)_{ i}( \Action_i\mid\AOHistory_{i,t})\Prob_{\Isom^*Z}({\AOHistoryRV_{i,t}=\AOHistory_{i,t}})\mathrm{d}\Mixture
    }
    {
    \int_{\PolicySet^\DecP}\Prob_{\Isom^*Z}(\AOHistoryRV_{ i,t}=\AOHistory_{i,t})\mathrm{d}\Mixture
    }
    \\
    &
    =
    \frac{\int_{\PolicySet^\DecPSecond}Z_{i}( \Action_i\mid\AOHistory_{i,t})\Prob_{Z}({\AOHistoryRV_{i,t}=\AOHistory_{i,t}})\mathrm{d}\Mixture\circ(\Isom^*)^{-1}
    }
    {
    \int_{\PolicySet^\DecPSecond}\Prob_{Z}(\AOHistoryRV_{ i,t}=\AOHistory_{i,t})\mathrm{d}\Mixture\circ(\Isom^*)^{-1}
    }
    \\
    &
    = \E_{\Isom^*\Mixture}[Z_{i}( \Action_i\mid\AOHistory_{i,t})\mid \AOHistoryRV_{ i,t}=\AOHistory_{i,t}]
    \\
    &= \Prob_{\Isom^*\Mixture}(\ActionRV_{i,t}= \Action_i\mid \AOHistoryRV_{ i,t}=\AOHistory_{i,t})
    \\\label{eq:33c}
    &=\Policy_i^{(\Isom^*\Mixture)_i}(a_i\mid \AOHistory_{i,t}),
\end{align}
where we have used Definition~\ref{policy-mixture} in (\ref{eq:33a}) and (\ref{eq:33c}), and Theorem~\ref{lem-pull-back-isomorphism-compatibility} in (\ref{eq:33b}).
This concludes the second case and thus the proof.

\jt{note to myself: I could include a bit more explanation here, e.g. using "Equation~\ref{conditional-probability-mixture} and the tower property" etc. But not super important. low prio}
\end{proof}

\subsection{The other-play distribution and the symmetrizer}
\label{appendix-other-play-distribution}
Using the idea of a policy corresponding to a distributions over policies introduced above, we can now define a policy corresponding to the distribution over policies used in the OP objective. In the following, fix again a Dec-POMDP \(\DecP\).
\begin{defn}[Other-play distribution]
Let \(\Policy\in\PolicySet\). We define the \emph{OP distribution} of \(\Policy\)
as the distribution \begin{equation}
\label{equation-other-play-mixture}
\Mixture:=|\Aut(\DecP)|^{-N}\sum_{\AutProfile\in\Aut(\DecP)^\PlayerSet}\delta_{\AutProfile^*\Policy},
\end{equation}
where \(\delta\) is the Dirac measure, i.e., for any measurable set \(\mathcal{Z}\subseteq\PolicySet\), it is
\[\delta_{\AutProfile^*\Policy}(\mathcal{Z})=
\begin{cases}1&\text{if }\AutProfile^*\Policy\in \mathcal{Z}
\\
0&\text{ otherwise.}
\end{cases}
\]
\end{defn}
Intuitively, agent \(i\) chooses one of the automorphisms \(\AutProfile_i\in\Aut(\DecP)\)
uniformly at random in the beginning of an episode and then follows the local policy \(\Proj_i(\AutProfile_i^*\Policy)\).
It can easily be shown that this distribution has independent local policies.
\begin{lem}
Let \(\Policy\in\PolicySet\) and let \(\Mixture\) be the OP distribution of \(\Policy\). Then \(\Mixture\) has independent local policies.
\end{lem}
\begin{proof}
Let \(\mathcal{Z}_i\subseteq\PolicySet_i\) measurable for \(i\in\PlayerSet\) and let \(\mathcal{Z}:=\prod_{i\in\PlayerSet}\mathcal{Z}_i\). Note that for any \(i\in\PlayerSet\) and \(\AutProfile\in\Aut(\DecP)^\PlayerSet\), it is
\begin{equation}\label{eq:300}
\delta_{\AutProfile^*\Policy}(\Proj_i^{-1}(\mathcal{Z}_i))=\prod_{j\in\PlayerSet\setminus\{i\}}\delta_{\Proj_j(\AutProfile^*_j\Policy)}(\PolicySet_j)\delta_{\Proj_i(\AutProfile^*_i\Policy)}(\mathcal{Z}_i)=\delta_{\Proj_i(\AutProfile^*_i\Policy)}(\mathcal{Z}_i).\end{equation}
Hence, it follows that
\begin{multline}
\Mixture(\mathcal{Z})
=|\Aut(\DecP)|^{-N}\sum_{\AutProfile_i\in\Aut(\DecP)^\PlayerSet}\delta_{\AutProfile^*\Policy}(\mathcal{Z})
=\prod_{i\in\PlayerSet}\sum_{\AutProfile_i\in\Aut(\DecP)}|\Aut(\DecP)|^{-1}\delta_{\Proj_i(\AutProfile_i^*\Policy)}(\mathcal{Z}_i))
\\
=\prod_{i\in\PlayerSet}\sum_{\AutProfile\in\Aut(\DecP)^\PlayerSet}|\Aut(\DecP)|^{-(N-1)}|\Aut(\DecP)|^{-1}\delta_{\Proj_i(\AutProfile_i^*\Policy)}(\mathcal{Z}_i))
\\
\overset{(\ref{eq:300})}{=}
\prod_{i\in\PlayerSet}|\Aut(\DecP)|^{-N}\sum_{\AutProfile\in\Aut(\DecP)^\PlayerSet}\delta_{\AutProfile^*\Policy}(\Proj_i^{-1}(\mathcal{Z}_i))
=\prod_{i\in\PlayerSet}\Mixture_i(\mathcal{Z}_i)
.\end{multline}
This shows that \(\Mixture=\otimes_{i\in\PlayerSet}\Mixture_i\).
\end{proof}

Using the OP distribution of a policy, we can define the symmetrizer \(\Psi^\DecP\) for \(\DecP\), which maps a policy \(\Policy\) to a policy \(\Psi^\DecP(\Policy)\) that corresponds to the OP distribution of \(\Policy\). If it is clear which Dec-POMDP is considered, we also write \(\Psi(\Policy)\).

\begin{defn}[Symmetrizer]\label{policy-corresponding-to-other-play-distribution}
We define the \emph{symmetrizer} for the Dec-POMDP \(\DecP\) as the map \(\Psi^\DecP\colon\PolicySet^\DecP\rightarrow\PolicySet^\DecP\) such that for any policy \(\Policy\in\PolicySet\) and OP distribution
\(\Mixture\) of \(\Policy\), it is
\[\Psi^\DecP(\Policy):=\Policy^\Mixture.\]
\end{defn}

It is clear that if a policy is already invariant to automorphism, then \(\Policy_i=\Psi_i(\Policy)\) for \(i\in\PlayerSet\), excluding action-observation histories that can never be reached under \(\Policy_i\). We formulate a slightly weaker proposition below, which is easier to prove.

\begin{prop}
\label{lemma-invariant-policy-psi}
Let \(\Policy\in\PolicySet\) be invariant to automorphism, and assume that, for all \(t\in\{0,\dotsc,\Tmax\}\) and \(\AOHistory_{i,t}\in\AOHistorySet_{i,t}\), it is \(\Prob_\Policy(\AOHistoryRV_{i,t}=\AOHistory_{i,t})>0\). Then it is \(\Psi(\Policy)=\Policy\).
\end{prop}
\begin{proof}
Let \(\Mixture\) be the OP distribution of \(\Policy\). Since \(\Policy\) is invariant to automorphism, it is
\[\Mixture=|\Aut(\DecP)|^{-N}\sum_{\AutProfile\in\Aut(\DecP)^\PlayerSet}\delta_{\AutProfile^*\Policy}=\delta_{\Policy}.\]
Now let \(i\in\PlayerSet\), \(t\in\{0,\dotsc,\Tmax\}\), \(\Action_i\in\ActionSet_i\) and \(\AOHistory_{i,t}\in\AOHistorySet_{i,t}\) arbitrary. 
Using Equation~\ref{eq:equation-tower-prop-z}, it follows that 
\[\Prob_{\delta_\Policy}(\AOHistoryRV_{i,t}=\AOHistory_{i,t})
\overset{(\ref{eq:equation-tower-prop-z})}{=}
\E_{\delta_\Policy}\left[\Prob_Z(\AOHistoryRV_{i,t}=\AOHistory_{i,t})\right]
=\Prob_\Policy(\AOHistoryRV_{i,t}=\AOHistory_{i,t})>0\]
and
\begin{multline}
\label{eq:40}
\Prob_{\delta_\Policy}(\ActionRV_{i,t}=\Action_i\mid\AOHistoryRV_{i,t}=\AOHistory_{i,t})
\overset{(\ref{eq:equation-tower-prop-z})}{=}
\E_{\delta_\Policy}\left[
\Prob_Z(\ActionRV_{i,t}=\Action_i\mid\AOHistoryRV_{i,t})\mid \AOHistoryRV_{i,t}=\AOHistory_{i,t}
\right]
\\
=\Prob_{\Policy}(\ActionRV_{i,t}=\Action_i\mid\AOHistoryRV_{i,t}=\AOHistory_{i,t}).
\end{multline}
Hence, we can apply Definition~\ref{policy-mixture} and conclude that
\begin{multline}\Psi_i(\Policy)(\Action_i\mid\AOHistory_{i,t})=\Policy^{\Mixture_i}_i(\Action_i\mid\AOHistory_{i,t})
\overset{\text{Definition~\ref{policy-mixture}}}{=}\Prob_{\delta_\Policy}(\ActionRV_{i,t}=\Action_i\mid\AOHistoryRV_{i,t}=\AOHistory_{i,t})
\\\overset{(\ref{eq:40})}{=}\Prob_{\Policy}(\ActionRV_{i,t}=\Action_i\mid\AOHistoryRV_{i,t}=\AOHistory_{i,t})=\Policy_i(\Action_i\mid\AOHistory_{i,t}).\end{multline}
\end{proof}

It will be helpful to refer to policies as equivalent if they have the same image under \(\Psi\).
\begin{defn}\label{defn-other-play-equivalent}
Let \(\Policy,\Policy'\in\PolicySet^\DecP\). We say that \(\Policy\) and \(\Policy'\) are equivalent, denoted as \(\Policy\equiv_\DecP\Policy\), if \(\Psi^\DecP(\Policy)=\Psi^\DecP(\Policy')\). Moreover, we write \([\Policy]:=\{\Policy'\mid \Policy'\equiv_D\Policy\}\) for the equivalence class of \(\Policy\).
\end{defn}
It is clear that \(\equiv_\DecP\) is an equivalence relation, since it is induced by the function \(\Psi^\DecP\). It follows that under \(\equiv_\DecP\), \(\PolicySet^\DecP\) decomposes into a partition of equivalence classes, denoted by \(\faktor{\PolicySet^\DecP}{\equiv_\DecP}\).

Applying Lemma~\ref{lemma-isomorphism-mixture-commute} to the symmetrizer in particular, we can show that it commutes with isomorphisms, and that the policy \(\Psi(\Policy)\) is invariant to automorphism.

\begin{cor}\label{corollary-psi-isom-commute}
Let \(\Isom\in\Iso(\DecP,\DecPSecond)\) and \(\Policy\in\PolicySet^\DecP\). Then it is
\[\Isom^*\Psi^\DecP(\Policy)=\Psi^{\DecPSecond}(\Isom^*\Policy).\]
If \(\DecPSecond=\DecP\), then
\[\Isom^*\Psi^\DecP(\Policy)=\Psi^{\DecP}(\Policy),\]
i.e., \(\Psi^\DecP(\Policy)\) is invariant to automorphism.
\end{cor}
\begin{proof}
Let \(\Mixture\) be the other-play distribution of \(\Policy\) and \(\MixtureHat\) the other-play distribution corresponding to \(\Isom^*\Policy\).
Then, using the associativity of function composition and pushforward proven in Lemma~\ref{lemma-pull-back-function-composition-compatible}, and using the ``in particular'' part of Lemma~\ref{decompose-isomorphisms}, it follows that
\begin{multline}\label{equation-image-measure-pull-back}
\MixtureHat
=
|\Aut(\DecPSecond)|^{-N}\sum_{\AutProfile\in\Aut(\DecPSecond)^\PlayerSet}\delta_{\AutProfile^*(\Isom^*\Policy)}
=
|\Aut(\DecPSecond)|^{-N}\sum_{\AutProfile\in\Aut(\DecPSecond)^\PlayerSet}\delta_{(\Proj_i(\AutProfile_i^*\Isom^*\Policy))_{i\in\PlayerSet}}
\\
\overset{\text{Lemma~\ref{decompose-isomorphisms}}}{=}
|\Aut(\DecP)|^{-N}\sum_{\AutProfile\in\Aut(\DecP)^\PlayerSet}\delta_{(\Proj_i(\Isom^*\AutProfile_i^*\Policy))_{i\in\PlayerSet}}
=
|\Aut(\DecP)|^{-N}\sum_{\AutProfile\in\Aut(\DecP)^\PlayerSet}\delta_{\Isom^*(\AutProfile^*\Policy)}
\\
=
|\Aut(\DecP)|^{-N}\sum_{\AutProfile\in\Aut(\DecP)^\PlayerSet}\delta_{\AutProfile^*\Policy}\circ (\Isom^*)^{-1}
=\Mixture\circ (\Isom^*)^{-1}.
\end{multline}
Thus, using Lemma~\ref{lemma-isomorphism-mixture-commute}, it is \begin{equation}\label{equation-psi-isom-commute}
\Isom^*\Psi^\DecP(\Policy)
=\Isom^*\Policy^\Mixture
\overset{\text{Lemma~\ref{lemma-isomorphism-mixture-commute}}}{=}
\Policy^{\Isom^*\Mixture}
=\Policy^{\Mixture\circ(\Isom^*)^{-1}}
\overset{(\ref{equation-image-measure-pull-back})}{=}
\Policy^{\MixtureHat}
=\Psi^\DecPSecond(\Isom^*\Policy)
\end{equation}

Finally, assume \(\DecPSecond=\DecP\). Then \(\Isom\) is an automorphism and \(\Aut(\DecPSecond)=\Aut(\DecP)=\Aut(\DecP)\circ\Isom\) by Lemma~\ref{decompose-isomorphisms}. Hence,
\begin{multline}\label{equation-image-measure-pull-back-auto}
\Mixture\circ (\Isom^*)^{-1}
=
|\Aut(\DecP)|^{-N}\sum_{\AutProfile\in\Aut(\DecP)^\PlayerSet}\delta_{\AutProfile^*\Policy}\circ (\Isom^*)^{-1}
\\
=
|\Aut(\DecP)|^{-N}\sum_{\AutProfile\in\Aut(\DecP)^\PlayerSet}\delta_{(\Proj_i(\Isom^*\AutProfile_i^*\Policy))_{i\in\PlayerSet}}
\\
\overset{\text{Lemma~\ref{decompose-isomorphisms}}}{=}
|\Aut(\DecP)|^{-N}\sum_{\AutProfile\in\Aut(\DecP)^\PlayerSet}\delta_{\AutProfile^*\Policy}
=\Mixture.
\end{multline}
By (\ref{equation-psi-isom-commute}), it follows that
\[\Isom^*\Psi^\DecP(\Policy)
\overset{(\ref{equation-psi-isom-commute})}{=}
\Policy^{\Mixture\circ(\Isom^*)^{-1}}
\overset{(\ref{equation-image-measure-pull-back-auto})}{=}
\Policy^\Mixture=\Psi^\DecP(\Policy),
\]
which concludes the proof.
\end{proof}

A direct corollary is that we can define the pushforward purely in terms of equivalence classes of policies. This will also be useful later.

\begin{defn}\label{defn-pushforward-of-equivalence-classes}
Let \(\DecP,\DecPSecond\) be isomorphic Dec-POMDPs with \(\Isom\in\Iso(\DecP,\DecPSecond)\). Let \([\Policy]\in\faktor{\PolicySet^\DecP}{\equiv_\DecP}\). We define the pushforward equivalence class \(\Isom^*[\Policy]\in\faktor{\PolicySet_{\DecPSecond}}{\equiv_{\DecPSecond}}\) via
\[f^*[\Policy]:=[f^*\Policy].\]
\end{defn}

The following corollary show that this is well-defined, that the pushforward of an equivalence class does not depend on the particular chosen isomorphism, and that it is compatible with function composition.

\begin{cor}\label{cor-pull-back-well-defined} 
\begin{enumerate}
    \item[(i)] The pushforward of an equivalence class is well-defined, i.e., for any \(\Policy,\Policy'\in\PolicySet^\DecP\) such that \(\Policy\equiv_\DecP\Policy'\), it is \([\Isom^*\Policy]=[\Isom^*\Policy']\).
    \item[(ii)]Any two isomorphisms \(\Isom,\Isom'\in\Iso(\DecP,\DecPSecond)\) induce the same pushforward.
    \item[(iii)]Analogous results to those in Lemma~\ref{lemma-pull-back-function-composition-compatible} apply to the pushforward of equivalence classes.
\end{enumerate} 
\end{cor}
\begin{proof}
First, let \(\Policy\equiv_\DecP\Policy'\in\PolicySet^\DecP\) and \(\Isom\in\Iso(\DecP,\DecPSecond)\) arbitrary. Then, using Corollary~\ref{corollary-psi-isom-commute} and the definition of \(\equiv_\DecP\), it is
\[\Psi^{\DecPSecond}(\Isom^*\Policy)=\Isom^*\Psi^\DecP(\Policy)=\Isom^*\Psi^\DecP(\Policy')=\Psi^{\DecPSecond}(\Isom^*\Policy').\]
Thus, \([\Isom^*\Policy]=[\Isom^*\Policy']\), which proves the first part.

Second, let \(\Isom,\IsomSecond\in \Iso(\DecP,\DecPSecond)\) and \(\Policy\in\PolicySet^\DecP\) arbitrary. By Lemma~\ref{decompose-isomorphisms}, there then exists \(\Auto\in\Aut(\DecPSecond)\) such that \(\IsomSecond=\Auto\circ\Isom\). Hence, using the second and first part of Corollary~\ref{corollary-psi-isom-commute} and Lemma~\ref{lemma-pull-back-function-composition-compatible}, it is
\begin{equation}
\label{eq:340}\Psi^{\DecPSecond}(\Isom^*\Policy)
\overset{\text{Corollary~\ref{corollary-psi-isom-commute}}}{=}
\Auto^*\Psi^{\DecPSecond}(\Isom^*\Policy)
\overset{\text{Corollary~\ref{corollary-psi-isom-commute}}}{=}
\Psi^{\DecPSecond}(\Auto^*(\Isom^*\Policy))
\overset{\text{Lemma~\ref{lemma-pull-back-function-composition-compatible}}}{=}
\Psi^{\DecPSecond}((\Auto\circ\Isom)^*\Policy)
=\Psi^{\DecPSecond}(\IsomHat^*\Policy).\end{equation}
Finally, it follows that
\[\Isom^*[\Policy]=[\Isom^*\Policy]\overset{\ref{eq:340}}{=}[\IsomSecond^*\Policy]=\IsomSecond^*[\Policy],
\]
which concludes the second part.

The third part follows directly from Lemma~\ref{lemma-pull-back-function-composition-compatible} by using the definition of the pushforward of equivalence classes.
\end{proof}

In the following, we say that two equivalence classes \([\Policy],[\Policy']\) for \(\Policy\in\PolicySet^\DecP,\Policy'\in\PolicySet^\DecPSecond\) correspond to each other if there exists an isomorphism \(\Isom\in\Iso(\DecP,\DecPSecond)\) such that \(\Isom^*[\Policy]=[\Policy']\). In that case, in a slight abuse of the terms, we also say that \(\Policy\) and \(\Policy'\) are equivalent, extending the equivalence between policies defined above to policies for different Dec-POMDPs. Using Corollary~\ref{cor-pull-back-well-defined}, one can see that two policies \(\Policy,\Policy'\in\PolicySet^\DecP\) are equivalent in the sense that \([\Policy]=[\Policy']\) if and only if there exists an isomorphism \(\Isom\in\Iso(\DecP,\DecP)\) such that \(\Isom^*[\Policy]=[\Policy']\), so this extended notion is equivalent to the old one for two policies  \(\Policy,\Policy'\in\PolicySet^\DecP\). We continue to reserve the notation \([\Policy]\) and \(\equiv\) for policies from the same Dec-POMDP.

\subsection{Main characterizations}
\label{appendix-main-characterizations}

Having defined the symmetrizer \(\Psi\), we can now characterize the OP objective as transforming a policy \(\Policy\) into an invariant policy \(\Psi(\Policy)\) and evaluating the expected return of that policy.
This means that we can ``pass to the quotient'' and consider the OP objective as a map \(\tilde{J}_\OP\) of equivalence classes, \(\tilde{J}_\OP([\Policy]):=J_\OP(\Policy)\) for \([\Policy]\in\faktor{\PolicySet}{\equiv}\), using the equivalence relation on policies introduced above. The result is essentially a rigorous version of \textcite{hu2020other}'s Proposition~1 in our setup. It will help us later to analyze the OP-optimal policies in a given example, as it implies that we can restrict ourselves to considering representatives \(\Psi(\Policy)\) of equivalence classes. 

\begin{thm}\label{thm-op-mixture} Let \(\DecP\) be a Dec-POMDP, let \(\Policy\in\PolicySet^\DecP\), and let \(\Psi\) be the symmetrizer for \(\DecP\). Then \(\Psi(\Policy)\) is invariant to automorphism, and it is
\begin{equation}\label{eq:350}
J^\DecP_\OP(\Policy)=J^\DecP(\Psi(\Policy)).
\end{equation}
In particular, we can consider the OP objective as a function of equivalence classes \([\Policy]\in\faktor{\PolicySet^\DecP}{\equiv_\DecP}\), and if there exists an optimal policy for the OP objective, then there also exists an optimal policy that is invariant to automorphism.
\end{thm}

\begin{proof}
In the following, fix a Dec-POMDP \(\DecP\). 
Let \(\Policy\in\PolicySet\) and let \(\Mixture\) be the OP distribution of \(\Policy\), such that \(\Psi(\Policy)=\Policy^\Mixture\). Then by the second part of Corollary~\ref{corollary-psi-isom-commute}, \(\Psi(\Policy)\) is invariant to automorphism.
Moreover, using Proposition~\ref{mixture-lemma}, it is
\begin{multline}\label{eq:34}
J_\OP(\Policy)=\E_{\AutProfile\sim \U(\Aut(\DecP)^\PlayerSet)}\left[J(\AutProfile^*\Policy)
\right]
=\sum_{\AutProfile\in\Aut(\DecP)^\PlayerSet}|\Aut(\DecP)|^{-N}J(\AutProfile^*\Policy)
\\=
\sum_{\AutProfile\in\Aut(\DecP)^\PlayerSet}|\Aut(\DecP)|^{-N}\int_{\Policy'\in\PolicySet}J(\Policy')
\mathrm{d}\left(\delta_{\AutProfile^*\Policy}\right)
\\
=
\int_{\Policy'\in\PolicySet}J(\Policy')\mathrm{d}\left(\sum_{\AutProfile\in\Aut(\DecP)^\PlayerSet}|\Aut(\DecP)|^{-N}\delta_{\AutProfile^*\Policy}\right)
\\
\overset{(\ref{equation-other-play-mixture})}{=}
\int_{\Policy'\in\PolicySet}J(\Policy')\mathrm{d}\Mixture
\overset{(\ref{expected-return-mixture})}{=}J(
\Mixture)
\overset{\text{Proposition~\ref{mixture-lemma}}}{=}
J(\Psi(\Policy)),
\end{multline}
which proves Equation~\ref{eq:350}.

Turning to the ``in particular'' statement, let \(\Policy'\equiv\Policy\) for a second policy \(\Policy'\in\PolicySet\).
By Definition~\ref{defn-other-play-equivalent}, this means that \(\Psi(\Policy')=\Psi(\Policy)\).
Hence, by Equation~\ref{eq:34}, it follows that \(J_\OP(\Policy)=J_\OP(\Policy')\), which shows that the function \(\tilde{J}_{\OP}\colon\faktor{\PolicySet}{\equiv}\rightarrow\mathbb{R},[\Policy]\mapsto J_\OP(\Policy)\) is well-defined.

Lastly, assume that there is \(\Policy\in\argmax_{\Policy'\in\PolicySet}J_\OP(\Policy')\) (see Remark~\ref{remark-other-play-always-admits-a-maximum} regarding the existence of such a policy). Then by Equation~\ref{eq:34}, it is also \(\Psi(\Policy)\in\argmax_{\Policy'\in\PolicySet}J_\OP(\Policy')\), so \(\Psi(\Policy)\) is an OP-optimal policy that is invariant to automorphism.
\end{proof}

As a corollary, we can show that isomorphisms do not affect the OP value of a policy (we already know this about SP from Corollary~\ref{cor-sp-invariant-to-pullback}). In the following, we define \(\PolicySet_\OP^\DecP:=\argmax_{\Policy\in\PolicySet^\DecP}J^\DecP_\OP(\Policy)\) for a Dec-POMDP \(\DecP\).

\begin{cor}\label{op-invariant-to-pullback}
Let \(\DecP\), \(\DecPSecond\) be isomorphic Dec-POMDPs with \(\Isom\in\Iso(\DecP,\DecPSecond)\), and let \(\Policy\in\PolicySet^\DecP\). Then it is
\[J_\OP^{\DecP}(\Policy)=J_\OP^\DecPSecond(\Isom^*\Policy).\]
In particular, if \(\Policy\in\PolicySet_\OP^\DecP\), then also \(\Isom^*\Policy\in\PolicySet_\OP^{\DecPSecond}\).
\end{cor}
\begin{proof}
Using Theorem~\ref{thm-op-mixture}, Corollary~\ref{corollary-psi-isom-commute}, and Corollary~\ref{cor-sp-invariant-to-pullback}, it is
\begin{multline}\label{eq:35}
J_\OP^{\DecPSecond}(\Isom^*\Policy)
\overset{\text{Theorem~\ref{thm-op-mixture}}}{=}
J^{\DecPSecond}(\Psi^{\DecPSecond}(\Isom^*\Policy))
\overset{\text{Corollary~\ref{corollary-psi-isom-commute}}}{=}
J^{\DecPSecond}(\Isom^*\Psi^{\DecP}(\Policy))
\\
\overset{\text{Corollary~\ref{cor-sp-invariant-to-pullback}}}{=}
J^{\DecP}(\Psi^{\DecP}(\Policy))
\overset{\text{Theorem~\ref{thm-op-mixture}}}{=}
J_\OP^{\DecP}(\Policy).
\end{multline}

Turning to the ``in particular'' statement, assume that \(\Policy\in\PolicySet^\DecP_\OP\). By Lemma~\ref{lemma-inverse-composition-isomorphism}, it is \(\Isom^{-1}\in\Iso(\DecPSecond,\DecP)\). Hence, for any \(\PolicyTilde\in\PolicySet^{\DecPSecond}\), it follows from the preceding that
\begin{equation}
J^{\DecPSecond}_\OP(\PolicyTilde)=J^{\DecP}_\OP((\Isom^{-1})^*\PolicyTilde)
\leq J^{\DecP}_\OP(\Policy)=J^{\DecPSecond}_\OP(\Isom^*\Policy).
\end{equation}
This shows that \(\Isom^*\Policy\in \PolicySet^{\DecPSecond}_\OP\).
\end{proof}

Finally, we turn to the connection between OP and the payoff in an LFC game.
The following result will be helpful in both showing the inadequacy of OP and in proving that OP with tie-breaking is optimal. It shows that equivalent policies in \([\Policy]\in\faktor{\PolicySet}{\equiv}\) are all compatible when played against each other by different principals in the LFC game. The proof is based on Lemma~\ref{lemma-decomposition-payoff-zero-shot-coordination-game} and Proposition~\ref{mixture-lemma}.

\begin{thm}\label{thm-zero-shot-invariant}
Let \(\DecP\) be a Dec-POMDP, let \(\Psi\) be the symmetrizer for \(\DecP\), and define \({\DecPSet:=\{\Isom^*\DecP\mid \Isom\in\Sym(\DecP)\}}\).
Let \(\LAProfile_1,\dotsc,\LAProfile_N\in\LASet^\DecPSet\).
For any \(\DecPSecond\in\DecPSet\), choose \(\Isom_{\DecP,\DecPSecond}\in\Iso(\DecP,\DecPSecond)\) arbitrarily. Then it is
\begin{equation}
 U^\DecP(\LAProfile)=\E_{\DecP_i\sim U(\DecPSet),\,i\in\PlayerSet}\left[
 \E_{\Policy^{(j)}\sim\Isom_{\DecP_j,\DecP}^*\LAProfile_j(\DecP_j),\,j\in\mathcal{N}}
 \left[
 J^\DecP \left(\left(\Psi_k(\Policy^{(k)})\right)_{k\in\PlayerSet}\right)
 \right]
 \right].
\end{equation}
\end{thm}

\begin{proof}
First, consider arbitrary joint policies \(\Policy^{(1)},\dots,\Policy^{(N)}\in\PolicySet^\DecP\) and let \(\Mixture^{(i)}\) be the OP distribution of \(\Policy^{(i)}\), so that \(\Policy^{\Mixture^{(i)}}=\Psi(\Policy^{(i)})\) for \(i\in\PlayerSet\). Define the distribution
\begin{equation}\label{eq:37}
\MixtureHat(\Policy^{(1)},\dotsc,\Policy^{(N)}):=|\Aut(\DecP)|^{-N}\sum_{\AutProfile\in\Aut(\DecP)^\PlayerSet}\otimes_{i\in\PlayerSet}\delta_{\Proj_i(\AutProfile_i^*\Policy^{(i)})}
\end{equation}
as a function of \(\Policy^{(1)},\dotsc,\Policy^{(N)}\).
It can easily be seen that \(\MixtureHat_(\Policy^{(1)},\dotsc,\Policy^{(N)})\in \Delta(\PolicySet^\DecP)\) and that it has independent local policies. Moreover, \(\MixtureHat(\Policy^{(1)},\dotsc,\Policy^{(N)})_i=\Mixture^{(i)}_i\), i.e., the marginal distribution for agent \(i\in\PlayerSet\) is equal in \(\MixtureHat(\Policy^{(1)},\dotsc,\Policy^{(N)})\) and \(\Mixture^{(i)}\). Hence, also the corresponding local policies are identical, that is,
\begin{equation}\label{equation-mixture-hat-equal-psi}
\Policy_i^{\MixtureHat(\Policy^{(1)},\dotsc,\Policy^{(N)})_i}
=\Policy^{\Mixture^{(i)}_i}_i=\Psi_i(\Policy^{(i)})
\end{equation}
for \(i\in\PlayerSet\).

It follows that
\begin{align}\label{eq:36}
&U^\DecP(\LAProfile_1,\dotsc,\LAProfile_N)
\\\label{eq:38a}
&=
\E_{\DecP_i\sim U(\DecPSet),\,i\in\PlayerSet}\left[
    \E_{\Policy^{(j)}\sim\Isom_{\DecP_j,\DecP}^*\LAProfile_j(\DecP_j),\,j\in\mathcal{N}}
    \left[
        \E_{\AutProfile\in\Aut(\DecP)^\PlayerSet}\left[
    J^\DecP((\Proj_k(\AutProfile_k^*\Policy^{(k)}))_{k\in\PlayerSet})\right]\right]\right]
    \\
    &=
    \E_{\DecP_i\sim U(\DecPSet),\,i\in\PlayerSet}\Bigg[
    \E_{\Policy^{(j)}\sim\Isom_{\DecP_j,\DecP}^*\LAProfile_j(\DecP_j),\,j\in\mathcal{N}}
    \Bigg[
    \\&\quad\quad
   |\Aut(\DecP)|^{-N}\sum_{\AutProfile\in\Aut(\DecP)^\PlayerSet}
       \int_{\Policy\in\PolicySet^\DecP}
    J^\DecP(\Policy)\mathrm{d}\left(\otimes_{k\in\PlayerSet}\delta_{\Proj_k(\AutProfile_k^*\Policy^{(k)})}
    \right)\Bigg]\Bigg]
\\
    &=
    \E_{\DecP_i\sim U(\DecPSet),\,i\in\PlayerSet}\Bigg[
    \E_{\Policy^{(j)}\sim\Isom_{\DecP_j,\DecP}^*\LAProfile_j(\DecP_j),\,j\in\mathcal{N}}
    \Bigg[
    \\&\quad\quad\int_{\Policy\in\PolicySet^\DecP}
    J^\DecP(\Policy)\mathrm{d}\left(|\Aut(\DecP)|^{-N}\sum_{\AutProfile\in\Aut(\DecP)^\PlayerSet}\otimes_{k\in\PlayerSet}\delta_{\Proj_k(\AutProfile_k^*\Policy^{(k)})}
    \right)\Bigg]\Bigg]
\\
&\overset{(\ref{eq:37})}{=}
\E_{\DecP_i\sim U(\DecPSet),\,i\in\PlayerSet}\left[
    \E_{\Policy^{(j)}\sim\Isom_{\DecP_j,\DecP}^*\LAProfile_j(\DecP_j),\,j\in\mathcal{N}}
    \left[\int_{\Policy\in\PolicySet^\DecP}
    J^\DecP(\Policy)\mathrm{d}\MixtureHat(\Policy^{(1)},\dotsc,\Policy^{(N)})\right]\right]
\\
&\overset{(\ref{expected-return-mixture})}{=}
\E_{\DecP_i\sim U(\DecPSet),\,i\in\PlayerSet}\left[
    \E_{\Policy^{(j)}\sim\Isom_{\DecP_j,\DecP}^*\LAProfile_j(\DecP_j),\,j\in\mathcal{N}}
    \left[
    J^\DecP(\MixtureHat(\Policy^{(1)},\dotsc,\Policy^{(N)}))\right]\right]
\\\label{eq:38b}
&=
\E_{\DecP_i\sim U(\DecPSet),\,i\in\PlayerSet}\left[
    \E_{\Policy^{(j)}\sim\Isom_{\DecP_j,\DecP}^*\LAProfile_j(\DecP_j),\,j\in\mathcal{N}}
    \left[
    J^\DecP((\Policy_k^{\MixtureHat(\Policy^{(1)},\dotsc,\Policy^{(N)})_k})_{k\in\PlayerSet})\right]\right]
\\
&\overset{(\ref{equation-mixture-hat-equal-psi})}{=}
\E_{\DecP_i\sim U(\DecPSet),\,i\in\PlayerSet}\left[
    \E_{\Policy^{(j)}\sim\Isom_{\DecP_j,\DecP}^*\LAProfile_j(\DecP_j),\,j\in\mathcal{N}}
    \left[
    J^\DecP\left(\left(\Psi_k(\Policy^{(k)})\right)_{k\in\PlayerSet}\right)\right]\right],
\end{align}
where we have used
Lemma~\ref{lemma-decomposition-payoff-zero-shot-coordination-game} in (\ref{eq:38a}) and Proposition~\ref{mixture-lemma} in (\ref{eq:38b}).
This concludes the proof.
\jt{note to self: I could change the way I explain this equation a bit, maybe add lemma etc to the equation signs. not super important}
\end{proof}

Based on this result, the LFC game for \(\DecP\) can be understood in the following way. Principal \(i\in\PlayerSet\) observes a randomly relabeled problem \(\DecP_i\in\DecPSet\) and trains a joint policy  \(\Policy^{(i)}\sim\LAProfile_i(\DecP_i)\) on this problem. The resulting policy \(\Policy^{(i)}\) is then translated back into a policy \(\Isom_{\DecP_i,\DecP}^*\Policy^{(i)}\) for the original problem, using any isomorphism \(\Isom_{\DecP_i,\DecP}\in\Iso(\DecP_i,\DecP)\). Finally, this joint policy is made invariant to automorphism by applying the symmetrizer, and agent \(i\) in the original problem \(\DecP\) is assigned the local policy \(\Psi_i(\Isom_{\DecP_i,\DecP}^*\Policy^{(i)})\).

\subsection{Other-play is not self-play in a different Dec-POMDP}
\label{appendix-no-optimal-deterministic-policy}
\textcite{hu2020other} show that the OP objective \(\tilde{J}^\DecP_\OP\) introduced by them can be understood of as the SP objective in a special Dec-POMDP. That is, for every Dec-POMDP \(\DecP\), there is a second Dec-POMDP \(\DecPSecond\) with \(\PolicySet^\DecPSecond=\PolicySet^\DecP\) such that for any \(\Policy\in\PolicySet^\DecP\), it is \(\tilde{J}^\DecP_\OP(\Policy)=\max_{\Policy'\in\PolicySet^\DecP}\tilde{J}^\DecP_\OP(\Policy')\) if and only if \(J^\DecPSecond(\Policy)=\max_{\Policy'\in\PolicySet^\DecPSecond}J^\DecPSecond(\Policy')\).
Interestingly, when including player permutations, this is not the case anymore. We will prove this here, using the characterization of OP from the last section.

Intuitively, if agents are symmetric, then under OP, they will always act according to the same local policy in the environment. In some Dec-POMDPs, this means that it is optimal for the agents to randomize their actions, to end up with different actions some of the time. For instance, consider the following game:

\begin{example}\label{ex-matching-pennies-2}
There are two players with two actions \(\ActionSet_i:=\{\Action_{i,1},\Action_{i,2}\}\) for \(i=1,2\) each, and an episode lasts only one step, making this a simple normal-form game. Rewards for each joint action are displayed in Table~\ref{tab:rewards-example-2}.

\begin{table}
    \caption{Rewards for each joint action in Example~\ref{ex-matching-pennies-2}.}
    \label{tab:rewards-example-2}
    \vskip 0.15in
    \small
    \centering
    \begin{tabular}{c|c|c}
         &  \(a_{2,1}\) & \(a_{2,2}\)\\
     \hline
      \(a_{1,1}\)&  \(-\frac{1}{2}\) & \(1\) \\
     \hline
      \(a_{1,2}\) &  \(1\) & \(-1\)
    \end{tabular}
\end{table}
\end{example}

This example demonstrates that sometimes there does not exist a deterministic policy that is optimal under the OP objective.

\begin{lem}\label{lem-ex-does-not-have-optimal-deterministic-policy}
In Example~\ref{ex-matching-pennies-2}, for any deterministic policy \(\Policy\), it is \(J_\OP(\Policy)<\max_{\Policy'\in\PolicySet^\DecP}J_\OP(\Policy')\).
\end{lem}
\begin{proof}
Let \(R\in\mathbb{R}^{2,2}\) denote a matrix containing rewards as in Table~\ref{tab:rewards-example-2}.

First, note that in this game, players are symmetric, but actions are not. Moreover, there are no observations and only one state. Hence, \(\Aut(\DecP)=\{\Auto,\Id\}\) where \(\Auto_N1=2,\Auto_N2=1\) and \(\Id\) is the identity.

Now consider any deterministic policy \(\Policy=(\Policy_1,\Policy_2)\), corresponding to two vectors of action-probabilities
\[x,y\in \left\{\begin{bmatrix}1\\0\end{bmatrix},\begin{bmatrix}0\\1\end{bmatrix}\right\}\]
for the two players. Due to the symmetry of both players, the OP distribution \(\Mixture\) of \(\Policy\) assigns each policy \(\Policy_1,\Policy_2\) to either player with probability \(\frac{1}{2}\), so in \(\Psi(\Policy)=\Policy^\Mixture\), both players play the distribution \(z:=\frac{1}{2}x + \frac{1}{2}y\) and receive a reward of
\[J^\OP(\Policy)\overset{\text{Theorem~\ref{thm-op-mixture}}}{=}J(\Psi(\Policy))=z^\top R z.\]
It follows by the definition of \(x,y\) that
\[z\in \left\{\begin{bmatrix}1\\0\end{bmatrix},\begin{bmatrix}0\\1\end{bmatrix},\begin{bmatrix}\frac{1}{2}\\\frac{1}{2}\end{bmatrix}\right\}.\]
Clearly, then \(z^\top Rz\) is maximized at \(z=[\frac{1}{2},\frac{1}{2}]^\top\), yielding a reward of
\begin{equation}\label{eq:510}
J^\OP(\Policy)=z^\top R z=\frac{1}{8}.
\end{equation}

Next, define \(x=[\frac{4}{7},\frac{3}{7}]^\top\) and let \(\Policy^*\) be a policy such that \(\Policy^*_1=\Policy^*_2\) and the two action-probabilities of both local policies are given by the vector \(x\). Note that since \(\Auto^*\Policy^*=(\Policy^*_{\Auto i})_{i=1,2}=\Policy^*\), it follows from the above that \(\Policy^*\) is invariant to automorphism.

It follows by Proposition~\ref{prop-invariant-policy-self-play-value-equals-other-play-value} that we can evaluate the OP value of \(\Policy^*\) by evaluating its expected return. That is,
\begin{equation}J^\OP(\Policy^*)\overset{\text{Proposition~\ref{prop-invariant-policy-self-play-value-equals-other-play-value}}}{=}J(\Policy^*)=x^\top R  x=-\frac{1}{2}\left(\frac{4}{7}\right)^2+2\frac{4}{7}\frac{3}{7} -1\left(\frac{3}{7}\right)^2=\frac{1}{7}.\end{equation}
It follows that
\begin{equation}
    J_\OP(\Policy)\overset{(\ref{eq:510})}{=}\frac{1}{8}<\frac{1}{7}=J^\OP(\Policy^*)\leq \max_{\Policy'\in\PolicySet^\DecP}J_\OP(\Policy'),
\end{equation}
which concludes the proof.
\end{proof}

The fact that there is no deterministic optimal policy in Example~\ref{ex-matching-pennies-2} is in conflict with a canonical result about Dec-POMDPs.

\begin{thm}[{\citeauthor{oliehoek2008optimal}, \citeyear{oliehoek2008optimal}, sec.~2.4.4}]
\label{thm-every-decpomdp-has-optimal-deterministic-policy}
In every Dec-POMDP \(\DecP\), there is a deterministic policy \(\Policy\in(\PolicySet^0)^\DecP\) such that \(J^\DecP(\Policy)=\max_{\Policy'\in\PolicySet^\DecP}J^\DecP(\Policy')\).
\end{thm}

As a result, we can prove the following.

\begin{prop}
\label{cor-optimal-deterministic}
There exists a Dec-POMDP \(\DecP\) such that for any other Dec-POMDP \(\DecPSecond\) with \(\PolicySet^\DecPSecond=\PolicySet^\DecP\), there exists a policy \(\Policy\in\PolicySet^\DecPSecond\) that is optimal for the SP objective of \(\DecPSecond\), but not optimal for the OP objective of \(\DecP\).
\end{prop}
\begin{proof}
Let \(\DecP\) be the Dec-POMDP as described in Example~\ref{ex-matching-pennies-2}. Assume, towards a contradiction, that there exists a Dec-POMDP \(\DecPSecond\) with \(\PolicySet^\DecPSecond=\PolicySet^\DecP\) such that any optimal policy in that Dec-POMDP is optimal under the OP objective of \(\DecP\). Then by Theorem~\ref{thm-every-decpomdp-has-optimal-deterministic-policy}, there exists a deterministic policy \(\Policy\in(\PolicySet^0)^\DecPSecond\) such that \(J^\DecPSecond(\Policy)\) is maximal. Hence, by the assumption, also \(J^\DecP_\OP(\Policy)\) is maximized. But by Lemma~\ref{lem-ex-does-not-have-optimal-deterministic-policy}, it must be \(\max_{\PolicyTilde\in\PolicySet^\DecP}J^\DecP_\OP(\PolicyTilde)>J^\DecP_\OP(\Policy)\). This is a contradiction, which means that \(\DecP\) is an example of a Dec-POMDP that has the desired properties.
\end{proof}

This shows that to optimize the OP objective, we have to directly consider that objective and we cannot simply apply an RL algorithm to a different Dec-POMDP. Also, the fact that we need stochastic policies means that it is not immediately clear how to apply a Bellman equation to the objective.


\section{Other-play is not optimal in the label-free coordination problem}
\label{appendix-proof-of-theorem-1}
In this section, our goal is to prove a rigorous version of Theorem~\ref{thm-op-not-optimal-informal} from the main text.

Recall that in Appendix~\ref{appendix-characterization-of-other-play}, we introduced the symmetrizer \(\Psi\colon\PolicySet\rightarrow\PolicySet\), which maps a joint policy \(\Policy\) to the policy corresponding to agents following randomly permuted local policies \((\AutProfile_i^*\Policy)_i\) where \(\AutProfile_i\sim\U(\Aut(\DecP))\). This represents the random permutations employed in the OP objective, and hence by Theorem~\ref{thm-op-mixture}, it is \(J_\OP(\Policy)=J(\Psi(\Policy))\), i.e., the OP value of \(\Policy\) is equal to the SP value of \(\Psi(\Policy)\). Moreover, we defined the equivalence classes \([\Policy]=\Psi^{-1}(\{\Psi(\Policy)\})\) of policies that get mapped to the same policy under \(\Psi\).

As defined in Appendix~\ref{appendix-generalization-of-other-play}, an OP learning algorithm is any learning algorithm such that the policies that it learns achieve optimal OP value in expectation. In particular, it can be a learning algorithm that learns different policies in different training runs, as long as it chooses OP-optimal policies with probability \(1\).
As we have seen in Theorem~\ref{thm-zero-shot-invariant}, in an LFC game, it does not matter which policy from an equivalence class \([\Policy]\) is chosen. Unfortunately, though, there can also be different OP-optimal policies that are not equivalent. In this case, if an OP learning algorithm is not concentrated on only compatible policies, it is not optimal in the corresponding LFC problem. 

In the remainder of this section, we will prove this statement. Concretely, in Appendix~\ref{appendix-two-incompatible}, we will recall the two policies \(\Policy^R,\Policy^S\) in the two-stage lever game that we introduced in Section~\ref{section-other-play-not-optimal}. We will show that both are optimal under OP, but that their cross-play value is inferior to the optimal OP value. In Appendix~\ref{appendix-proof-op-suboptimal}, we will then formally state and prove the result that, if an OP algorithm is not concentrated on only one of the two incompatible equivalence classes of policies \([\Policy^R],[\Policy^S]\) in the two-stage lever game, then the algorithm is suboptimal in the LFC problem for that game.

\subsection{Two incompatible optimal policies in the two-stage lever game}
\label{appendix-two-incompatible}

We begin by mapping out the space of OP-optimal policies in the two-stage lever game. Recall that this was a game with two agents, which proceeds in two rounds. Both agents have two actions, and their goal in both rounds is to choose the same action, for a reward of \(1\). Failure of coordination leads to a reward of \(-1\). Moreover, in the second round, agents observe the actions of the other agent from the first round.
In the following, let \(\DecP\) stand for the Dec-POMDP associated to this game as described in Example~\ref{example-two-stage-lever-game}.

Recall the two policies \(\Policy^R\) and \(\Policy^S\) introduced in Section~\ref{section-other-play-not-optimal}. In both policies, agents randomize uniformly between both levers in the first round. They also both randomize in the second round if coordination was unsuccessful in the first one. If coordination in the first round was successful, there are two different strategies: in \(\Policy^R\), both agents repeat their respective actions from round one. In \(\Policy^S\), both agents switch to the action they did not play in round one.

The following lemma shows that both policies are optimal under OP.

\begin{lem}\label{lemma-two-classes-in-coop}
Both \(\Policy^R\) and \(\Policy^S\) as described above are invariant to automorphism, and they maximize the OP objective, where
\[J_\OP(\Policy^R)=J_\OP(\Policy^S)=\max_{\Policy\in\PolicySet^\DecP}J_\OP(\Policy)= \frac{1}{2}.\]
\end{lem}
\begin{proof}
Let \(\PolicyTilde\in\PolicySet^\DecP\) arbitrary and define \(\Policy:=\Psi(\PolicyTilde)\). By 
Theorem~\ref{thm-op-mixture}, \(\Policy\) must be invariant to automorphism, and it must be
\[J_\OP(\PolicyTilde)=J(\Policy).\]
Moreover, if we show that \(\Policy^R\) and \(\Policy^S\) are invariant to automorphism, then by Proposition~\ref{prop-invariant-policy-self-play-value-equals-other-play-value}, their OP value equals their expected return. Hence, to show that \(\Policy^R\) and \(\Policy^S\) are optimal, it suffices to show that they are invariant to automorphism, and then compare their expected returns to the expected return of \(\Policy\).

To begin, consider the set of automorphisms in the game, as described in Example~\ref{example-two-stage-lever-game-automorphisms-isomorphisms}. Let \(\Auto\in\Aut(\DecP)\). Recall that the state permutation is the trivial identity map, and the permutations for actions and observations of both players have to be equal, \(\Auto_{A_1}=\Auto_{A_2}=\Auto_{O_1}=\Auto_{O_2}\). Hence, we can represent \(\Auto\) as a tuple \((\Auto_N,\Auto_A)\). There are four possible combinations of choices for permutations \(\Auto_A,\Auto_N\colon \{1,2\}\rightarrow\{1,2\}\). Either permutation can either be the identity map \(\begin{pmatrix}
    1 & 2\\
    1 & 2
    \end{pmatrix}\) or the inversion \(\begin{pmatrix}
    1 & 2\\
    2 & 1
\end{pmatrix}\).

Now consider \(\Policy^R\) and \(\Policy^S\). Note that both policies are symmetric in the agents, so \(\Policy_1^R=\Policy_2^R\) and \(\Policy_1^S=\Policy_2^S\). Invariance to automorphism is trivially fulfilled in the first stage, i.e., \(\Policy_i^R(\Action_i\mid\emptyset)=\Policy^R_{\Auto^{-1}_N i}(\Auto^{-1}_A\Action_i\mid\emptyset)\) and \(\Policy_i^S(\Action_i\mid\emptyset)=\Policy^S_{\Auto^{-1}_N i}(\Auto^{-1}_A\Action_i\mid\emptyset)\) for any \(i\in\{1,2\},\Action_i\in\{1,2\}\), since both agents randomize uniformly between actions. Now consider the second stage. Note that since automorphisms for the actions and observations of both players have to be identical, an automorphism can never map an action-observation history with \(\ObservationRV_{i,1}\neq\ActionRV_{i,0}\), i.e., such that one's own action and the observed action from the other agent differ, onto one in which they are the same, and vice versa. So we can consider the condition of invariance to automorphism separately for the case in which players achieved coordination in the first stage and for the case where they did not.

In the latter case, players randomize their actions uniformly, so in this case the policy is also trivially invariant. Now consider the former case, i.e., for \(i=1,2\), it is either \(\AOHistoryRV_{i,1}=(1,1)\) or \(\AOHistoryRV_{i,1}=(2,2)\). If both players repeat their action, then for \(i\in\{1,2\}\), \(\Action_i=\Observation_i\in\{1,2\}\), and any automorphism \(\Auto=(\Auto_N,\Auto_A)\), it is
\[\Policy^R_i(\Action_i\mid \Action_i,\Observation_i)=1=\Policy^R_{\Auto^{-1}_Ni}(\Auto_A^{-1}\Action_i\mid \Auto^{-1}_A\Action_i,\Auto_A^{-1}\Observation_i).\]
In the case where they change their action, for \(\Action_i'\in\{1,2\}\setminus\{\Action_i\}\), it is \(\Auto_A^{-1}\Action_i'\neq \Auto_A^{-1}\Action_i\), and thus
\[\Policy^S_i(\Action_i'\mid \Action_i,\Observation_i)=1=\Policy^S_{\Auto^{-1}_Ni}(\Auto_A^{-1}\Action'_i\mid \Auto_A^{-1}\Action_i,\Auto_A^{-1}\Observation_i).\]
In conclusion, this shows that \(\Auto^*\Policy^S=\Policy^S\) and \(\Auto^*\Policy^R=\Policy^R\).
Both \(\Policy^S\) and \(\Policy^R\) have an expected return of
\[J(\Policy^S)=J(\Policy^R)=\E[R_1 + R_2] = \frac{1}{2}(1+1) + \frac{1}{2}(-1+\frac{1}{2}\cdot 1 + \frac{1}{2}\cdot (-1))=\frac{1}{2}\]
i.e., in the first round, they coordinate in half of the cases, in which case they coordinate again, and if they do not coordinate in the first round, they have an equal chance of achieving coordinating or not in the second round.

Now consider \(\Policy\). First, choosing \(\Auto_N\) as the inversion and \(\Auto_A\) as the identity, it follows that
\[\Policy_1\overset{\text{(!)}}{=}\Policy_{\Auto_N^{-1}1}(\Auto_A^{-1}\cdot\mid\Auto_A^{-1}\cdot)=\Policy_{2}(\cdot\mid\cdot)=\Policy_{2},\] 
where in (!) we use that \(\Policy\) is invariant to automorphism. This shows that \(\Policy\) must by symmetric in players, so \(\Policy_1=\Policy_2\). Moreover, choosing \(\Auto_N\) as the identity and \(\Auto_A\) as the inversion, it is
\[\Policy_i(1\mid \emptyset)\overset{\text{(!)}}{=}\Policy_{\Auto_N^{-1}i}(\Auto_A^{-1}1\mid\emptyset)=\Policy_{i}(2\mid\emptyset),\] 
so under \(\Policy\), too, both agents must choose both actions with equal probability in the first stage.

Turning to the second stage, we assume that \(\Policy\) receives a maximal reward of \(1\) in the second stage if both agents coordinated in the first stage. Now consider the case in which coordination failed in the first step. Note that since \(\Policy\) is symmetric in the agents, it is not possible for agents to consistently choose to play either the action of agent \(1\) or of agent \(2\) in the second round if those actions did not coincide in the first round. Moreover, due to invariance to action and observation permutations, the probability \(p\) that an agent repeats their action from the first round must be the same, no matter whether that action was a \(1\) or \(2\), as for \(\Action\neq\Action'\in\{1,2\}\), it must be \(\Policy_i(\Action\mid\Action,\Action')=\Policy_i(\Action'\mid\Action',\Action)\) for \(i=1,2\) due to invariance to automorphism. Due to the symmetry of agents, that probability must also be the same for both agents. Hence, we can define
\[p:=\Policy_i(\Action\mid \Action,\Action')=1-\Policy_i(\Action'\mid\Action,\Action')\]
for \(i=1,2\) and \(\Action\neq\Action'\in\{1,2\}\). The return in the second stage in this case is then
\begin{multline}\E[R_2\mid A_{2,1}\neq A_{1,1}]
\\= 1\cdot \bigg(\Policy_1(\Action_{1,1}\mid \Action_{1,1},\Action_{2,1})\Policy_2(\Action_{1,1}\mid \Action_{2,1}, \Action_{1,1})
+\Policy_1(\Action_{2,1}\mid \Action_{1,1},\Action_{2,1})\Policy_2(\Action_{2,1}\mid \Action_{2,1}, \Action_{1,1})\bigg)
\\
-1\cdot
\bigg(\Policy_1(\Action_{1,1}\mid \Action_{1,1},\Action_{2,1})\Policy_2(\Action_{2,1}\mid \Action_{2,1}, \Action_{1,1})
+\Policy_1(\Action_{2,1}\mid \Action_{1,1},\Action_{2,1})\Policy_2(\Action_{1,1}\mid \Action_{2,1}, \Action_{1,1})\bigg)
\\
=1 \cdot\bigg(p(1-p) + (1-p)p\bigg) - 1\cdot \bigg(pp + (1-p)(1-p)\bigg)=0,\end{multline}
where \(\Action_{1,1}\neq\Action_{2,1}\in\{1,2\}\) are arbitrary.
This shows that also \[J(\Policy)\leq \frac{1}{2}(1+1) + \frac{1}{2}(-1+0)=\frac{1}{2}.\]
Thus, it follows that \(J(\Policy)\leq \frac{1}{2}=J(\Policy^R)=J(\Policy^S)\), which shows that \(\Policy^R,\Policy^S\) are both optimal and thus concludes the proof.
\end{proof}

Now consider the case in which one agent chooses a local policy from \(\Policy^R\) and another agent chooses a local policy from \(\Policy^S\). It is clear that this will yield a suboptimal expected return compared with \(\Policy^R\) or \(\Policy^S\), as agents will always fail to coordinate in the second round, if they coordinated in the first round.

\begin{lem}\label{lemma-XP-suboptimal}
\[J(\Policy^R_1,\Policy^S_2)=J(\Policy^S_1,\Policy^R_2)=-\frac{1}{2}<\frac{1}{2}.\]
\end{lem}
\begin{proof}
Both policies are equal in the first round, in which there is a 50\% chance that agents coordinate. If agents do not coordinate in the first round, then they both randomize uniformly in the second stage. If agents do coordinate in the first round, then agent \(1\) using \(\Policy^R_1\) repeats their action from the first round, while agent \(2\), using \(\Policy^S_2\), will switch to a different action. In this case, the reward is thus always \(-1\). As a result, it is
\[Jwe(\Policy_1^R,\Policy_2^S)=\frac{1}{2}(1-1)+\frac{1}{2}(-1+\frac{1}{2}\cdot1+ \frac{1}{2}\cdot(-1))=-\frac{1}{2}<\frac{1}{2}.\]
The same argument applies if agent \(1\) uses \(\Policy^S_1\) and agent \(2\) uses \(\Policy^R_2\).
\end{proof}

\subsection{Proof that other-play is suboptimal}
\label{appendix-proof-op-suboptimal}

Using the two lemmas from the last section, we can now show that if an OP learning algorithm is not concentrated on only one of \([\Policy^S]\) or \([\Policy^R]\), it is not optimal in the LFC problem for \(\DecP\). Note that we could just choose a particular learning algorithm, for instance, one that randomizes uniformly between \(\Policy^R\) and \(\Policy^S\), and show that that algorithm is suboptimal. This would then prove that a suboptimal OP learning algorithm exists. However, we show a more general result. Specifically, take any learning algorithm that learns with positive probability a policy equivalent to $\pi^R$ in one relabeling of \(\DecP\) and, also with positive probability, a policy equivalent to $\pi^L$ in some potentially different relabeling. Then this learning algorithm is suboptimal in the LFC problem for \(\DecP\).

In the following, let \(\DecPSet:=\{\Isom^*\DecP\mid\Isom\in\Sym(\DecP)\}\) be the set of different relabeled Dec-POMDPs of the two-stage lever game \(\DecP\). For any Dec-POMDP \(\DecPSecond\in\DecPSet\), choose \(\Isom_{\DecP,\DecPSecond}\in\Iso(\DecP,\DecPSecond)\) arbitrarily. Below, we define the class of algorithms relevant to our theorem.

\begin{defn}\label{definition-LA-mixture-two-classes}
Let \(\LA\in\Sigma^\DecPSet\) be a learning algorithm such that for any \({\DecPSecond}\in\DecPSet\), \({\DecPSecond}=\Isom^*\DecP\), it is
\[\LA({\DecPSecond})=\alpha^{\DecPSecond}\Mixture^{{\DecPSecond}} +(1-\alpha^{\DecPSecond}){\Mixture'}^{{\DecPSecond}},\]
where \(\alpha_{\DecPSecond}\in[0,1]\) and \(\Mixture^{{\DecPSecond}},{\Mixture'}^{{\DecPSecond}}\in\Delta(\PolicySet_{\DecPSecond})\) are chosen such that 
\[{\Mixture}^{\DecPSecond}(\Isom_{\DecP,\DecPSecond}^*[\Policy^R])=1={\Mixture'}^{\DecPSecond}(\Isom_{\DecP,\DecPSecond}^*[\Policy^S]).\]
That is, \(\sigma(\DecPSecond)\) is a mixture of a distribution \(\Mixture^\DecPSecond\) with weight only on the equivalence class of policies corresponding to \(\Policy^R\), and a distribution \({\Mixture'}^\DecPSecond\) that only puts weight on policies corresponding to \(\Policy^S\). We say that \(\LA\) learns both \([\Policy^R]\) and \([\Policy^S]\) if there exist \(\DecPTilde,\DecPTilde'\in\DecPSet\) such that \(\alpha_{\DecPTilde}>0\) and \(\alpha_{\DecPTilde'}<1\).
\end{defn}
Note that Corollary~\ref{cor-pull-back-well-defined} (ii) ensures that this definition does not depend on the chosen isomorphisms \(\Isom_{\DecP,\DecPSecond}\).

The above condition is very weak: we only require there to be some relabeled Dec-POMDP on which the learning algorithm chooses an equivalent policy of \(\Policy^R\) some of the time, and some potentially different relabeled Dec-POMDP where the algorithm chooses a policy equivalent to \(\Policy^S\) some of the time. We now show that a learning algorithm with this property is an OP learning algorithm per Definition~\ref{defn-other-play-learning-algorithm-appendix}, but that it is not optimal in the LFC problem for \(\DecP\).

\begin{thm}\label{thm-op-not-optimal}
In the two-stage lever game, there are two classes of OP-optimal policies, denoted by \([\Policy^R]\) and \([\Policy^S]\). Any learning algorithm that learns both \([\Policy^R]\) and \([\Policy^S]\) in the sense of Definition~\ref{definition-LA-mixture-two-classes} is an OP learning algorithm, but it is not optimal in the LFC problem for that game.
\end{thm}
\begin{proof}
First, note that under \(\Policy^R,\Policy^S\), all action-observation histories are reached with positive probability, and by Lemma~\ref{lemma-two-classes-in-coop}, they are both invariant to automorphism. Hence, by Proposition~\ref{lemma-invariant-policy-psi}, it is \(\Psi(\Policy^R)=\Policy^R\) and \(\Psi(\Policy^S)=\Policy^S\) (*). This also shows that \([\Policy^R]\) and \([\Policy^S]\) are distinct. Let \(\LA\) be a learning algorithm as specified in Definition~\ref{definition-LA-mixture-two-classes}, and define \(\Mixture^\DecPSecond,{\Mixture'}^\DecPSecond\) and \(\alpha^\DecPSecond\) for \(\DecPSecond\in\DecPSet\) as in that definition. Also, let \(\DecPTilde,\DecPTilde'\in\DecPSet\) such that \(\alpha_\DecPTilde>0\) and \(\alpha_{\DecPTilde'}<1\).

We begin by showing that \(\LA\) is an OP learning algorithm.
For any \({\DecPSecond}\in\DecPSet\), it is
\begin{align}
\E_{\Policy\sim\LA({\DecPSecond})}[J^{\DecPSecond}_\OP(\Policy)]
\label{eq:39a}
&=\alpha_\DecPSecond\E_{\Policy\sim\Mixture^{{\DecPSecond}}}[J^{\DecPSecond}_\OP(\Policy)]
+(1-\alpha_\DecPSecond)\E_{\Policy\sim{\Mixture'}^{{\DecPSecond}}}[J^{\DecPSecond}_\OP(\Policy)]
\\
\label{eq:39b}
&=\alpha_\DecPSecond\E_{\Policy\sim\Mixture^{{\DecPSecond}}}[
\mathds{1}_{\Isom_{\DecP,\DecPSecond}^*[\Policy^R]}J_\OP^{\DecPSecond}(\Policy))] + (1-\alpha_\DecPSecond)\E_{\Policy\sim{\Mixture'}^{{\DecPSecond}}}[\mathds{1}_{\Isom_{\DecP,\DecPSecond}^*[\Policy^S]}J_\OP^{\DecPSecond}(\Policy)]
\\
\nonumber
&=\alpha_\DecPSecond\E_{\Policy\sim\Mixture^{{\DecPSecond}}}[
\mathds{1}_{\Isom_{\DecP,\DecPSecond}^*[\Policy^R]}J_\OP^{\DecP}(\Isom_{\DecPSecond,\DecP}^*\Policy))]
\\\label{eq:39c}
&\quad\quad + (1-\alpha_\DecPSecond)\E_{\Policy\sim{\Mixture'}^{{\DecPSecond}}}[\mathds{1}_{\Isom_{\DecP,\DecPSecond}^*[\Policy^S]}J_\OP^{\DecP}(\Isom_{\DecPSecond,\DecP}^*\Policy)]
\\\nonumber
&=
\alpha_\DecPSecond\E_{\Policy\sim\Isom_{\DecPSecond,\DecP}^*\Mixture^{{\DecPSecond}}}[
\mathds{1}_{\Isom_{\DecPSecond,\DecP}^*\left(\Isom_{\DecP,\DecPSecond}^*[\Policy^R]\right)}J_\OP^{\DecP}(\Policy))]
\\
\label{eq:39d}
&\quad \quad+ (1-\alpha_\DecPSecond)\E_{\Policy\sim\Isom_{\DecPSecond,\DecP}^*{\Mixture'}^{{\DecPSecond}}}[\mathds{1}_{\Isom_{\DecPSecond,\DecP}^*\left(\Isom_{\DecP,\DecPSecond}^*[\Policy^S]\right)}J_\OP^{\DecP}(\Policy)]
\\
\label{eq:39e}
&=
\alpha_\DecPSecond\E_{\Policy\sim\Isom_{\DecPSecond,\DecP}^*\Mixture^{{\DecPSecond}}}[
\mathds{1}_{[\Policy^R]}J_\OP^{\DecP}(\Policy))] + (1-\alpha_\DecPSecond)\E_{\Policy\sim\Isom_{\DecPSecond,\DecP}^*{\Mixture'}^{{\DecPSecond}}}[\mathds{1}_{[\Policy^S]}J_\OP^{\DecP}(\Policy)]
\\ 
\label{eq:39f}
&=
\alpha_\DecPSecond J_\OP^{\DecP}(\Policy^R) + (1-\alpha_\DecPSecond)J_\OP^{\DecP}(\Policy^S)
\\
\label{eq:39g}
&=
\alpha_\DecPSecond \max_{\Policy\in\PolicySet^\DecP}J^\DecP_\OP(\Policy) + (1-\alpha_\DecPSecond)\max_{\Policy\in\PolicySet^\DecP}J^\DecP_\OP(\Policy)
\\
&=
\max_{\Policy\in\PolicySet^\DecP}J^\DecP_\OP(\Policy),\end{align}
using Definition~\ref{definition-LA-mixture-two-classes} in (\ref{eq:39a}) and (\ref{eq:39b}); Corollary~\ref{op-invariant-to-pullback} in (\ref{eq:39c}); a change of variables for pushforward measures in (\ref{eq:39d});
Lemma~\ref{lemma-inverse-composition-isomorphism} and parts (ii) and (iii) of Corollary~\ref{cor-pull-back-well-defined} in (\ref{eq:39e}); the ``in particular'' part of Theorem~\ref{thm-op-mixture} in (\ref{eq:39f}); and Lemma~\ref{lemma-two-classes-in-coop} in (\ref{eq:39g}).
This shows that \(\LA\) is an OP learning algorithm.

Next, we turn to proving that \(\LA\) is suboptimal in the LFC problem.
To that end, define the auxiliary function \[X(\DecP_1,\DecP_2,\Policy^{(1)},\Policy^{(2)}):=\delta_{\DecP_1,\DecPTilde}\delta_{\DecP_2,\DecPTilde'}\mathds{1}_{[\Policy^R]}(\Policy^{(1)})\mathds{1}_{[\Policy^S]}(\Policy^{(2)})\]
for any \(\DecP_1,\DecP_2\in\DecPSet\) and \(\Policy^{(1)},\Policy^{(2)}\in\PolicySet^\DecP\). This function is \(1\) whenever the first Dec-POMDP is \(\DecPTilde\) and the policy \(\Policy^{(1)}\) is such that its pushforward policy is in \([\Policy^R]\), and when the second Dec-POMDP is \(\DecPTilde'\) and the pushforward of \(\Policy^{(2)}\) is in \([\Policy^S]\). Otherwise, it is \(0\).

Recall that by Lemma~\ref{lemma-inverse-composition-isomorphism}, it is \((\Isom_{\DecPSecond,\DecP}^*)^{-1}\in\Iso(\DecP,\DecPSecond)\) for \(\DecPSecond\in\DecPSet\), and by Corollary~\ref{cor-pull-back-well-defined} (ii), the pushforward of an equivalence class does not depend on the particular chosen isomorphism. Define
\[\beta
:=
\E_{D_1,D_2\sim \U(\DecPSet)}[
\E_{\Policy^{(i)}\sim f_{\DecP,D_i}^*\sigma(D_i),\, i=1,2}[
X(\DecP_1,\DecP_2,\Policy^{(1)},\Policy^{(2)})
]]\]
and note that \(\beta\leq 1\).
This number is relevant to us as it is a lower bound on the probability that agents 1 and 2 will have incompatible joint policies in the objective of the LFC problem. By assumption about \(\LA\), we know that
\[(\Isom_{\DecPTilde,\DecP}^*\LA(\DecPTilde))([\Policy^R])=\sigma(\DecPTilde)((\Isom_{\DecPTilde,\DecP}^*)^{-1}[\Policy^R])=\alpha_\DecPTilde>0\]
and similarly
\[(\Isom_{\DecPTilde',\DecP}^*\LA(\DecPTilde'))([\Policy^S])=
\sigma(\DecPTilde')((\Isom_{\DecPTilde',\DecP}^*)^{-1}[\Policy^S])
=1-\alpha_{\DecPTilde'}>0,\]
and it is \(\U(\DecPSet)(\{\DecPTilde\})=\U(\DecPSet)(\{\DecPTilde'\})=\frac{1}{\DecPSet}\).
Hence, it follows that \(0<\beta\leq 1\) (**). That is, with nonzero probability, using \(\LA\) leads to incompatible policies.

Our goal is now to prove that \(U^\DecP(\LA)<\frac{1}{2}\). To that end, we also have to show that the value in the cases where \(X(\DecP_1,\DecP_2,\Policy^{(1)},\Policy^{(2)})=0\) is bounded by \(\frac{1}{2}\).
Note that for any \(\DecPSecond\in\DecPSet\), it is
\[(\Isom_{\DecPSecond,\DecP}^*\LA(\DecPSecond))([\Policy^R]\cup[\Policy^S])=
\LA(\DecPSecond)(((\Isom_{\DecPSecond,\DecP}^*)^{-1}[\Policy^R])\cup((\Isom_{\DecPSecond,\DecP}^*)^{-1}[\Policy^S]))=1,\]
and for \(\Policy^{(1)},\Policy^{(2)}\in [\Policy^R]\cup[\Policy^S]\), using the definition of equivalence, it is \begin{equation}\label{eq:45}
J^D(\Psi(\Policy^{(1)})_1,\Psi(\Policy^{(2)})_2)=J^D(\Psi(\Policy^{C})_1,\Psi(\Policy^{C'})_2)\overset{(*)}{=}J^D(\Policy^{C}_1,\Policy^{C'}_2)\overset{\text{Lemma~\ref{lemma-two-classes-in-coop}}}{=}\frac{1}{2}
\end{equation}
if both policies are from the same class, i.e., \(C=C'\in\{R,S\}\), and
\begin{equation}
\label{eq:42}
J^D(\Psi(\Policy^{(1)})_1,\Psi(\Policy^{(2)})_2)=J^D(\Psi(\Policy^{C})_1,\Psi(\Policy^{C'})_2)\overset{(*)}{=}J^D(\Policy^{C}_1,\Policy^{C'}_2)\overset{\text{Lemma~\ref{lemma-XP-suboptimal}}}{=}-\frac{1}{2}
\end{equation}
if they come from different classes, i.e., \(C\neq C'\in\{R,S\}\).
It follows that for any \(\DecP_1,\DecP_2\in\DecPSet\), it is
\begin{equation}
\label{eq:43}
J^D(\Psi(\Policy^{(1)})_1,\Psi(\Policy^{(2)})_2)\leq \frac{1}{2}
\end{equation}
almost surely if \(\Policy^{(1)}\sim \Isom_{D_1,\DecP}^*\LA(D_1)\) and \(\Policy^{(2)}\sim \Isom_{D_2,\DecP}^*\LA(D_2)\). 

Using this bound together with (**), it follows that
\begin{align}
U^\DecP(\LA)&=U^\DecP(\LA,\dotsc,\LA)\\
\label{eq:41a}
&=\E_{D_1,D_2\sim \U(\DecPSet)}[\E_{\Policy^{(i)}\sim \Isom_{D_i,\DecP}^*\LA(D_i),\, i=1,2}[J^D(\Psi_1(\Policy^{(1)}),\Psi_2(\Policy^{(2)}))]]
\\\nonumber
&=\E_{D_1,D_2\sim \U(\DecPSet)}\big[\E_{\Policy^{(i)}\sim \Isom_{D_i,\DecP}^*\LA(D_i),\, i=1,2}\big[
J^\DecP(\Psi(\Policy_1^{(1)}),\Psi_2(\Policy^{(2)}))X(\DecP_1,\DecP_2,\Policy^{(1)},\Policy^{(2)})
\\&\quad\quad\label{eq:41b}
+
J^D(\Psi_1(\Policy^{(1)}),\Psi_2(\Policy^{(2)}))
(1-X(\DecP_1,\DecP_2,\Policy^{(1)},\Policy^{(2)}))
\big]\big]
\\\nonumber
&\leq
\E_{D_1,D_2\sim \U(\DecPSet)}\Bigg[\E_{\Policy^{(i)}\sim \Isom_{\DecP_i,\DecP}^*\sigma(\DecP_i),\, i=1,2}\Bigg[
\left(-\frac{1}{2}\right)X\left(\DecP_1,\DecP_2,\Policy^{(1)},\Policy^{(2)}\right)
\\&\quad\quad+\label{eq:41c}
\frac{1}{2}\left(1-X\left(\DecP_1,\DecP_2,\Policy^{(1)},\Policy^{(2)}\right)\right)
\Bigg]\Bigg]
\\&=\label{eq:41d}
  \left(-\frac{1}{2}\right)\beta
+ \frac{1}{2}(1-\beta)
\\
&\overset{\text{(**)}}{<}\frac{1}{2},
\end{align}
where we use Theorem~\ref{thm-zero-shot-invariant} in (\ref{eq:41a}) and Equations~\ref{eq:42} and \ref{eq:43} in (\ref{eq:41c}).

To see that this is suboptimal in the LFC problem, consider \(\LA^*\) defined via \(\LA^*(\DecPSecond):=\delta_{\Isom_{\DecP,\DecPSecond}^*\Policy^R}\) for any \(\DecPSecond\in\DecPSet\). 
Since
\begin{equation}
\label{eq:44}
\Isom_{\DecPSecond,\DecP}^*\LA(\DecPSecond)=\delta_{\Isom_{\DecP,\DecPSecond}^*\Policy^R}\circ(\Isom_{\DecPSecond,\DecP}^*)^{-1}=\delta_{\Isom_{\DecPSecond,\DecP}^*(\Isom_{\DecP,\DecPSecond}^*\Policy^R)}
=\delta_{\Policy^R}
\end{equation}
for any \(\DecPSecond\in\DecPSet\), it follows that
\begin{multline}U^\DecP(\LA^*)=U^\DecP(\LA^*,\dotsc,\LA^*)\\
\overset{\text{Theorem~\ref{thm-zero-shot-invariant}}}{=}\E_{D_1,D_2\sim \U(\DecPSet)}[\E_{\Policy^{(i)}\sim \Isom_{D_i,\DecP}^*\LA^*(D_i),\, i=1,2}[
J^\DecP(\Psi_1(\Policy^{(1)}),\Psi_2(\Policy^{(2)}))]]
\\
\overset{(\ref{eq:44})}{=}
J^D(\Psi_1(\Policy^R),\Psi_2(\Policy^R))\overset{(\ref{eq:45})}{=}\frac{1}{2}\overset{\text{(\ref{eq:41a})--(\ref{eq:41d})}}{>}U^\DecP(\LA),
\end{multline}
which concludes the proof.
\end{proof}


%

\section{Other-play with tie-breaking}
\label{appendix-proof-of-theorem-2}

In this section, we formally define OP with tie-breaking as introduced in Section~\ref{section-other-play-with-tie-breaking}, state and prove a rigorous version of Theorem~\ref{thm-op-with-tie-breaking-informal}, and provide an additional result about random tie-breaking functions. In Appendix~\ref{appendix-definition-op-tie-breaking}, we define OP with tie-breaking and discuss to what degree the formal definition is satisfied by our method. In Appendix~\ref{appendix-optimality-zero-shot-coordination-problem}, we show that OP with tie-breaking is optimal in the LFC problem and that all principals using OP with tie-breaking is an optimal symmetric Nash equilibrium of any LFC game. In Appendix~\ref{appendix-random-tie-breaking-functions} we prove that a modification of the tie-breaking function introduced in Section~\ref{section-other-play-with-tie-breaking} satisfies our definition of a tie-breaking function.

\subsection{Definition of other-play with tie-breaking}
\label{appendix-definition-op-tie-breaking}


Recall that OP with tie-breaking was introduced as an extension of OP, to fix the failure of OP in the LFC problem. A tie-breaking function ranks the different OP-optimal equivalence classes of policies in a given problem. For instance, a tie-breaking function could compare the two incompatible policies in the two-stage lever game, \(\Policy^R\) and \(\Policy^S\), and choose the policy under which actions are more highly correlated, which is \(\Policy^R\). Another tie-breaking function would be one that samples from an OP learning algorithm and ranks policies in terms of their relative frequencies. If one policy is learned more often than another, that tie-breaking function would be able to distinguish between them. OP with tie-breaking is defined as an algorithm that chooses an OP-optimal policy that maximizes a tie-breaking function.

Note that an obvious alternative approach, making a learning algorithm deterministic by coordinating on a random seed, would not work in the LFC problem. This is because the learned policies have to be compatible even across different, relabeled Dec-POMDPs. A labeling modifies the representation of the data used to train the algorithm, which likely leads to the same effect as a resampling of the random seed. Hence, fixing a random seed fails to ensure that the chosen policies are compatible when principals do not coordinate on labels for the problem. (Note that a learning algorithm that outputs consistent policies on a given problem, but incompatible policies on different, relabeled problems is explicitly included as a non-optimal algorithm in Theorem~\ref{thm-op-not-optimal}.) Additionally, this would be an unprincipled way to deal with the choice between different maximizers of the OP objective, leaving no possibility, e.g., for an explicit bias towards some policies over others.

Turning to the definitions, recall again from Appendix~\ref{appendix-characterization-of-other-play} that we say that for a Dec-POMDP \(\DecP\), \(\Policy,\Policy'\in\PolicySet^\DecP\) are equivalent, \(\Policy\equiv\Policy'\), if \(\Psi(\Policy)=\Psi(\Policy')\), where \(\Psi\) is the symmetrizer for \(\DecP\) that maps a joint policy \(\Policy\) to the joint policy \(\Psi(\Policy)\) corresponding to agents following randomly permuted local policies \((\AutProfile_i^*\Policy)\) for \(\AutProfile_i\sim\U(\Aut(\DecP))\). Moreover, we defined \([\Policy]\) as the equivalence class of \(\Policy\), and \(\Isom^*[\Policy]:=[\Isom^*\Policy]\) for \(\Isom\in\Iso(\DecP,\DecPSecond)\).

In the following, define \(\PolicySet^\DecP_\OP:=\argmax_{\Policy\in\PolicySet^\DecP}J^\DecP_{\OP}(\Policy)\)  as the set of OP-optimal policies for any Dec-POMDP \(\DecP\). Let \(\DecPSet\) be any set of Dec-POMDPs. A tie-breaking function takes in Dec-POMDPs \(\DecP\in\DecPSet\) and policies \(\Policy\in \PolicySet^\DecP\) and outputs values in \([0,1]\) that can be used to consistently break ties between policies, across different isomorphic Dec-POMDPs.

\begin{defn}[Tie-breaking function]
\label{defn-tie-breaking-functions}
Let \(\chi\colon \{(\DecP,\Policy)\mid \DecP\in\DecPSet,\Policy\in\PolicySet^\DecP\}\rightarrow [0,1]\). Then
\begin{enumerate}
    \item[(a)] \(\chi\) is called a tie-breaking function for \(\DecPSet\) if
\begin{enumerate}
    \item[(i)]for any \(\DecP\in\DecPSet\), \(\chi\) attains a maximum on the set
    \[\{(\DecP,\Policy)\mid  \Policy\in \PolicySet^\DecP_\OP\}.\]
    \item[(ii)]for any \(\DecP\in\DecPSet\) and \(\Policy,\Policy'\in\PolicySet^\DecP\), it is
    \[\chi(\DecP,\Policy)=\chi(\DecP,\Policy')\Rightarrow \Policy\equiv_\DecP\Policy'.\]
\end{enumerate}
\item[(b)]
\(\chi\) is called invariant to isomorphism if for any \(\DecP,\DecPSecond\in\DecPSet\), \(\Isom\in \Iso(\DecP,\DecPSecond)\), and \(\Policy\in\PolicySet^\DecP,\Policy'\in\PolicySet^\DecPSecond\), it is
    \[\Isom^*[\Policy]=[\Policy']\Rightarrow\chi(\DecP,\Policy)=\chi(\DecPSecond,\Policy').\]
\end{enumerate}
\end{defn}

\(\chi\) being a tie-breaking function ensures that there is always a unique equivalence class of policies that maximizes the function for a given Dec-POMDP. It being invariant to isomorphism ensures that \(\chi\) chooses corresponding equivalence classes of policies on different, isomorphic Dec-POMDPs in \(\DecPSet\).

Using a tie-breaking function that is invariant to isomorphism, we can define OP with tie-breaking.

\begin{defn}[Other-play with tie-breaking]\label{definition-other-play-with-tie-breaking}
Let \(\chi\) be a tie-breaking function for \(\DecPSet\) that is invariant to isomorphism. Let \(\LA^\chi\in\LASet^\DecPSet\) be a learning algorithm such that for any \(\DecP\in\DecPSet\), there exists a measurable set \(\mathcal{Z}\subseteq\PolicySet_\OP^\DecP\) such that \(\LA^\chi(\DecP)(\mathcal{Z})=1\) and \(\mathcal{Z}\subseteq \argmax_{\Policy\in\PolicySet_\OP^\DecP}\chi(\DecP,\Policy)\). Then we say that \(\LA^\chi\) is an OP with tie-breaking learning algorithm for \(\DecPSet\).
\end{defn}

This definition implies that no matter the problem \(\DecP\in\DecPSet\), the algorithm always learns the OP-optimal policy that achieves the highest tie-breaking value. Since the tie-breaking function is invariant to isomorphism, this means that policies learned in different training runs and when trained on relabeled Dec-POMDPs are always compatible.

Finally, recall that the practical method introduced in Section~\ref{section-other-play-with-tie-breaking} consists of sampling \(K\in\mathbb{N}\) policies using an OP learning algorithm, applying a tie-breaking function to each policy, and then choosing the one with the highest value. Clearly, if all the OP-optimal equivalence classes of policies in any of the Dec-POMDPs in \(\DecPSet\) are among the first \(K\) learned policies, then this algorithm will satisfy our definition. This appears to be the case for the OP algorithm and toy examples used in our experiments, at least for large enough \(K\), and when we ignore differences between policies that matter little for the agents' expected returns. Moreover, it is easy to see that the algorithm will still always pick equivalent policies if for any two isomorphic Dec-POMDPs \(\DecP,\DecPSecond\in\DecPSet\), if \(\LA^\OP\) learns a policy \(\Policy\) with positive probability in \(\DecP\), it also learns an equivalent policy \(\Policy'\) in \(\DecPSecond\) with positive probability, in the sense that \(\Isom^*[\Policy]=[\Policy']\) for \(\Isom\in\Iso(\DecP,\DecPSecond)\). That is, it does not matter if some policies are never learned, if this happens consistently across isomorphic problems. One way for an algorithm \(\LA^\OP\) to have this property is by being equivariant (see Appendix~\ref{appendix-equivariant-learning-algorithms}).

\subsection{Other-play with tie-breaking is an optimal symmetric profile}
\label{appendix-optimality-zero-shot-coordination-problem}
In the following, fix a Dec-POMDP \(\DecP\) and the set of relabeled Dec-POMDPs \(\DecPSet:=\{\Isom^*\DecP\mid\Isom\in\Sym(\DecP\}\), and let \(\chi\) be a tie-breaking function for \(\DecPSet\) that is invariant to isomorphism. Let \(U(\LAProfile):=U^\DecP(\LAProfile)\) stand for the payoff in the LFC game for \(\DecP\) given the profile of learning algorithms \(\LAProfile_1,\dotsc,\LAProfile_N\in\LASet^\DecPSet\). Recall that the learning algorithm \(\LA\in\LASet^\DecPSet\) is optimal in the LFC problem for \(\DecP\) if \(U(\LA)\geq U(\LA')\) for all \(\LA'\in\LASet^\DecPSet\), where \(U(\LA):=U(\LA,\dotsc,\LA)\).

Recall from Appendix~\ref{appendix-optimal-symmetric-strategy-profiles} that we call a profile of learning algorithms \(\LAProfile_1,\dotsc,\LAProfile_N\in\LASet^\DecPSet\) symmetric if \(\LAProfile_i=\LAProfile_{\Auto^{-1}i}\) for any principal \(i\in\PlayerSet\) and \(\Auto\in\Aut(\DecP)\). \(\LAProfile\) is defined as an optimal symmetric profile if it is symmetric and for any other symmetric profile \(\LAProfile'_1,\dotsc,\LAProfile_N'\in\LASet^\DecPSet\), it is \(U(\LAProfile)\geq U(\LAProfile')\).

Turning to our main theorem, we show that a profile \(\LA^\chi,\dotsc,\LA^\chi\) in which all principals choose OP with tie-breaking is an optimal symmetric profile in the LFC game for \(\DecP\). That is, as long as symmetric principals choose the same learning algorithm, they cannot do better than all choosing OP with tie-breaking. In particular, this implies that \(\LA^\chi\) is optimal in the LFC problem. Note, though, that we prove a stronger statement, as optimality for the LFC problem only requires that \(U(\LA^\chi,\dotsc,\LA^\chi)\geq U(\LA',\dotsc,\LA')\) for all \(\LA'\in\LASet^\DecPSet\), while we show that \(U(\LA^\chi,\dotsc,\LA^\chi)\geq U^\DecP(\LAProfile_1,\dotsc,\LAProfile_N)\) for all symmetric profiles \(\LAProfile_1,\dotsc,\LAProfile_N\in\LASet^\DecPSet\). Afterwards, we will apply Theorem~\ref{prop-symmetry-invariant-nash-equilibrium} to conclude that all principals using OP with tie-breaking is also a Nash equilibrium, i.e., given that all principals use OP with tie-breaking, no individual principal can do better by switching to a different algorithm.

\begin{thm}\label{thm-op-with-tie-breaking}
Let \(\LA^\chi\) be an OP with tie-breaking learning algorithm for \(\DecPSet\), as defined in Definition~\ref{definition-other-play-with-tie-breaking}. Then all principals using \(\LA^\chi\) is an optimal symmetric strategy profile in the LFC game for \(\DecP\). 
In particular, \(\LA^\chi\) is optimal in the LFC problem for \(\DecP\), and it is
\(U(\LA^\chi)=\max_{\LA\in\LASet^\DecPSet}U(\LA)=\max_{\Policy\in\PolicySet^\DecP}J^\DecP_\OP(\Policy)\).
\end{thm}

\begin{proof}
We begin by showing that \(U(\LA^\chi,\dotsc,\LA^\chi)=\max_{\Policy\in\PolicySet^\DecP}J^\DecP_\OP(\Policy)\). Afterwards, we show that one cannot get a better payoff in the LFC game than that using a symmetric profile of learning algorithms, i.e., that \(\LA^\chi,\dotsc,\LA^\chi\) is an optimal symmetric profile. It then follows immediately that also \(U(\LA^\chi)=U(\LA^\chi,\dotsc,\LA^\chi)\geq U(\LA,\dotsc,\LA)=U(\LA)\) for any \(\LA\in\LASet^\DecP\), i.e., that \(\LA^\chi\) is optimal in the LFC problem for \(\DecP\).

Recall the definition \(\PolicySet_\OP^\DecPSecond:=\argmax_{\Policy'\in\PolicySet^\DecPSecond}J_\OP^\DecPSecond(\Policy')\) for \(\DecPSecond\in\DecPSet\).
Let \(\DecPSecond,\DecPThird\in\DecPSet\) arbitrary and \(\Isom,\IsomSecond\in\Sym(\DecP)\) such that \(\DecPSecond=\Isom^*\DecP\) and \(\DecPThird=\IsomSecond^*\DecP\). We want to show that \(\LA^\chi\) learns compatible OP-optimal policies in both Dec-POMDPs. To that end, let \(\mathcal{Z}\subseteq\PolicySet_\OP^\DecPSecond\) measurable such that \(\mathcal{Z}\subseteq\argmax_{\Policy\in\PolicySet^\DecPSecond}\chi(\DecPSecond,\Policy)\) and such that \(\LA(\DecPSecond)(\mathcal{Z})=1\). Then since \(\chi\) is a tie-breaking function, there must exist a policy \(\Policy'\in\PolicySet_\OP^\DecPSecond\) such that \(\Policy'\in\mathcal{Z}\subseteq [\Policy']\) and thus \(\LA(\DecPSecond)([\Policy'])=1\) and \(\Policy'\in \argmax_{\Policy\in\PolicySet_\OP^\DecPSecond}\chi(\DecPSecond,\Policy)\) (i). 
Letting \(\Policy:=(\Isom^{-1})^*\Policy'\), it follows from Corollary~\ref{op-invariant-to-pullback} that \(\Policy\in \PolicySet_\OP^\DecP\) (ii).

Now we want to show that \(\LA^\chi\) learns policies in \([\IsomSecond^*\Policy]\) in \(\DecPThird\).
Let \(\PolicyTilde:=(\IsomSecond\circ\Isom^{-1})^*\Policy'\) and note that \(\IsomSecond\circ\Isom^{-1}\in \Iso(\DecPSecond,\DecPThird)\), so \(\PolicyTilde\in\PolicySet^\DecPThird\). Since \(\chi\) is invariant to isomorphism, it is \[\chi(\DecPThird,\PolicyHat)=\chi(\DecPSecond,(\Isom\circ\IsomSecond^{-1})^*\PolicyHat)\overset{\text{(i)}}{\leq}\chi(\DecPSecond,\Policy')
,\]
for any \(\PolicyHat\in\PolicySet^\DecPThird\). Here, equality holds for \(\PolicyHat:=\PolicyTilde\), as it is \[(\Isom\circ\IsomSecond^{-1})^*\PolicyTilde=(\Isom\circ\IsomSecond^{-1})^*(\IsomSecond\circ\Isom^{-1})^*\Policy'=\Policy'\] by Lemma~\ref{lemma-pull-back-function-composition-compatible}. Moreover, due to Corollary~\ref{op-invariant-to-pullback} it is again \(\PolicyTilde\in \PolicySet_\OP^\DecPThird\). This shows that \(\PolicyTilde\in\argmax_{\PolicyHat\in\PolicySet_\OP^\DecPThird}\chi(\DecPThird,\PolicyHat)\), and since \(\chi\) is a tie-breaking function, it is \(\argmax_{\PolicyHat\in\PolicySet_\OP^\DecPThird}\chi(\DecPThird,\PolicyHat)\subseteq[\PolicyTilde]\). Hence, by the definition of \(\LA^\chi\), it follows that \(\LA^\chi(\DecPThird)([\PolicyTilde])=1\) (iii).

Now choose any arbitrary isomorphism \(\Isom_{\DecPThird,\DecP}\in\Iso(\DecPThird,\DecP)\). By Lemma~\ref{lemma-pull-back-function-composition-compatible}, it is \((\Isom_{\DecPThird,\DecP}^*)^{-1}=(\Isom_{\DecPThird,\DecP}^{-1})^*\). Hence, by the second part of Lemma~\ref{cor-pull-back-well-defined}, it is \((\Isom_{\DecPThird,\DecP}^*)^{-1}[\Policy]=\IsomSecond^*[\Policy]\) (iv). Moreover, using again Lemma~\ref{lemma-pull-back-function-composition-compatible}, it is
\begin{equation}
\label{eq:701}
(\Isom\circ\IsomSecond^{-1})^*\PolicyTilde=(\Isom\circ\IsomSecond^{-1})^*(\IsomSecond\circ\Isom^{-1})^*\Policy'=\Policy'.
\end{equation}
It follows that
\begin{multline}\label{eq:702}
(\Isom_{\DecPThird,\DecP}^*)^{-1}([\Policy])
\overset{\text{(iv)}}{=}
\IsomSecond^*[\Policy]
=\IsomSecond^*[(\Isom^{-1})^*\Policy']
\overset{\text{(\ref{eq:701})}}{=}\IsomSecond^*[(\Isom^{-1})^*(\Isom\circ\IsomSecond^{-1})^*\PolicyTilde]
\\\overset{\text{Definition~\ref{defn-pushforward-of-equivalence-classes}}}{=}
\IsomSecond^*(\Isom^{-1})^*(\Isom\circ\IsomSecond^{-1})^*[\PolicyTilde]
\overset{\text{Corollary~\ref{cor-pull-back-well-defined}~(iii)}}{=}
[\PolicyTilde]
\end{multline}
and thus
\begin{equation}\label{eq:700}
    \Isom_{\DecPThird,\DecP}^*\LA^\chi(\DecPThird)([\Policy])=\LA^\chi(\DecPThird)((\Isom_{\DecPThird,\DecP}^*)^{-1}([\Policy]))
    \overset{(\ref{eq:702})}{=}\LA^\chi(\DecPThird)([\tilde{\Policy}])\overset{\text{(iii)}}{=}1.
\end{equation}
Since \(\DecPThird\) was arbitrary, we can conclude that the above equation holds for any \(\DecPThird\in\DecPSet\) and isomorphism \(\Isom_{\DecPThird,\DecP}\in\Iso(\DecPThird,\DecP)\).

Using this fact together with the definition of equivalence of policies in Definition~\ref{defn-other-play-equivalent}, as well as the expression for the payoff in the LFC game from Theorem~\ref{thm-zero-shot-invariant}, it follows that
\begin{eqnarray}\label{eq:55a}
U(\LA^\chi,\dotsc,\LA^\chi)
&\overset{\text{Theorem~\ref{thm-zero-shot-invariant}}}{=}&
\E_{\DecP_i\sim U(\DecPSet),\,i\in\PlayerSet}\left[
 \E_{\Policy^{(j)}\sim\Isom_{\DecP_j,\DecP}^*\LAProfile_j(\DecP_j),\,j\in\mathcal{N}}
 \left[
 J^\DecP \left(\left(\Psi_k(\Policy^{(k)})\right)_{k\in\PlayerSet}\right)
 \right]
 \right]
\\\label{eq:55b}
&\overset{\text{(\ref{eq:700})}}{=}&
\E_{\DecP_i\sim U(\DecPSet),\,i\in\PlayerSet}\left[
 \E_{\Policy^{(j)}\sim\Isom_{\DecP_j,\DecP}^*\LAProfile_j(\DecP_j),\,j\in\mathcal{N}}
 \left[
 J^\DecP \left(\left(\Psi_k(\Policy)\right)_{k\in\PlayerSet}\right)
 \right]
 \right]
\\
&=&J^\DecP(\Psi(\Policy))\\
&\overset{\text{Theorem~\ref{thm-op-mixture}}}{=}&
J^D_\OP(\Policy)\\
&\overset{\text{(ii)}}{=}&
\max_{\PolicyHat\in\PolicySet^\DecP} J_D^\OP(\PolicyHat),
\label{eq:55c}
\end{eqnarray}

Next, we show that we cannot do better than that with a symmetry-invariant profile of learning algorithms. To that end, let \(\LAProfile_1,\dotsc,\LAProfile_N\in\LASet^{\DecPSet}\) arbitrary such that \(\LAProfile_i=\LAProfile_{\Auto^{-1}i}\) for any \(\Auto\in\Aut(\DecP)\), \(i\in\PlayerSet\). For \(i\in\PlayerSet\), define the distribution
\[
\Distr^{(i)}:=
|\Sym(\DecP)|^{-1}\sum_{\Isom\in \Sym(\DecP)}
(\Isom^{-1})^*\LAProfile_i(\Isom^*\DecP)
\]
and let \(\Mixture:=\otimes_{i\in\PlayerSet}\Distr^{(i)}_i\).
Then it is
\begin{align}\label{eq:705a}
U(\LAProfile)&=
\E_{\SymProfile\sim \U(\Sym(D)^{\mathcal{N}})}[\E_{\Policy^{(i)}\sim(\SymProfile_i^{-1})^*\LAProfile_i(\SymProfile_i^*\DecP),\,i\in\mathcal{N}}[
J^\DecP((\Policy^{(i)}_i)_{i\in\PlayerSet})]]
\\
&=
\sum_{\SymProfile\in\Sym(\DecP)^\PlayerSet}|\Sym(\DecP)|^{-N}\int_{(\Policy^{(i)})_{i\in\PlayerSet}\in(\PolicySet^\DecP)^\PlayerSet}J^\DecP((\Policy^{(i)}_i)_{i\in\PlayerSet})\mathrm{d}\otimes_{i\in\PlayerSet}(\SymProfile_i^{-1})^*\LAProfile_i(\SymProfile_i^*\DecP)
\\
&=
\sum_{\SymProfile\in\Sym(\DecP)^\PlayerSet}|\Sym(\DecP)|^{-N}
\int_{\Policy\in\PolicySet^\DecP}J^\DecP(\Policy)\mathrm{d}\otimes_{i\in\PlayerSet}(\SymProfile_i^{-1})^*\LAProfile_i(\SymProfile_i^*\DecP)\circ \Proj_i^{-1}
\\
&=
\int_{\Policy\in\PolicySet^\DecP}
J^\DecP(\Policy)\mathrm{d}\otimes_{i\in\PlayerSet}\left(|\Sym(\DecP)|^{-1}\sum_{\SymProfile_i\in\Sym(\DecP)}(\SymProfile_i^{-1})^*\LAProfile_i(\SymProfile_i^*\DecP)\circ \Proj_i^{-1}\right)
\\
&=
\int_{\Policy\in\PolicySet^\DecP}
J^\DecP(\Policy)\mathrm{d}\Mixture
\\
&\overset{(\ref{expected-return-mixture})}{=}J^\DecP(\Mixture).\label{eq:705b}
\end{align}

Now we want to show that it is \(J^\DecP(\Mixture)\leq \max_{\Policy\in\PolicySet^\DecP}J^\DecP(\Policy)\), by proving that \(\Mixture\) and thus also the corresponding policy \(\Policy^\Mixture\) is invariant to pushforward by automorphism.
Let \(\Auto\in\Aut(\DecP)\) and let \(\mathcal{Z}_i\subseteq\PolicySet_i^\DecP\) measurable for \(i\in\PlayerSet\). Recall that by Lemmas~\ref{lemma-action-function-composition-compatible} and \ref{lemma-pull-back-function-composition-compatible}, \(\Auto\) is a bijective self-map on action-observation histories and \(\Auto^*\) a bijective self-map on \(\PolicySet^\DecP\), and that isomorphisms can be inverted and composed by Lemma~\ref{lemma-inverse-composition-isomorphism}. In the following, we use the notations \(\Policy_i\circ \Auto:=\Policy_i(\Auto\cdot\mid \Auto\cdot)\) and \(\mathcal{Z}_i\circ \Auto:=\{\Policy_i(\Auto\cdot\mid\Auto\cdot)\mid \Policy_i\in \mathcal{Z}_i\}\) for \(i\in\PlayerSet\). 

By definition, \(\Mixture\) has independent local policies. Hence, we can apply Equation~\ref{eq:30} from the proof of Lemma~\ref{lemma-isomorphism-mixture-commute}, which says that
\begin{equation}\label{eq:60}(\Auto^*\Mixture)(\mathcal{Z})
=\prod_{i\in\PlayerSet^\DecP}\Mixture_{\Auto^{-1}i}(\mathcal{Z}_{i}\circ \Auto).\end{equation}
Moreover, it is \((\Auto^*)^{-1}(\PolicySet^\DecP)=\PolicySet^\DecP\) and by definition, \((\Auto^*\Policy)_i=\Policy_{\Auto^{-1}i}\circ\Auto^{-1}\), which implies that
\begin{align}\label{eq:71}\Proj^{-1}_{\Auto^{-1}i}(\mathcal{Z}_i\circ \Auto)
&=\{\Policy\mid\Policy\in\PolicySet^\DecP,\Policy_{\Auto^{-1}i}\in\mathcal{Z}_i\circ\Auto\}
\\&=\{\Policy\mid\Policy\in(\Auto^*)^{-1}(\PolicySet^\DecP),\Policy_{\Auto^{-1}i}\circ\Auto^{-1}\in\mathcal{Z}_i\}
\\&=\{\Policy\mid\Auto^*\Policy\in\PolicySet^\DecP,(\Auto^*\Policy)_i\in\mathcal{Z}_i\}
\\&=\{(\Auto^*)^{-1}\Policy\mid\Policy\in\PolicySet^\DecP,\Policy_i\in\mathcal{Z}_i\}
\\&=(\Auto^*)^{-1}(\Proj^{-1}_i(\mathcal{Z}_i)).
\end{align}
Hence, it follows for \(i\in\PlayerSet\) that
\begin{align}\label{eq:72}
    \Distr^{(i)}(\Proj^{-1}_{\Auto^{-1}i}(\mathcal{Z}_i\circ \Auto))
    &=
    \Distr^{(i)}((\Auto^*)^{-1}(\Proj^{-1}_i(\mathcal{Z}_i)))
   \\&=
|\Sym(\DecP)|^{-1}\sum_{\Isom\in \Sym(\DecP)}
(\Isom^{-1})^*\LAProfile_i(\Isom^*\DecP)((\Auto^*)^{-1}(\Proj_i^{-1}(\mathcal{Z}_i)))
   \\&=
|\Sym(\DecP)|^{-1}\sum_{\Isom\in \Sym(\DecP)}
\Auto^*(\Isom^{-1})^*\LAProfile_i(\Isom^*\DecP)(\Proj_i^{-1}(\mathcal{Z}_i))
   \\& =
|\Sym(\DecP)|^{-1}\sum_{\Isom\in \Sym(\DecP)}
((\Isom\circ\Auto^{-1})^{-1})^*\LAProfile_i(\Isom^*\DecP)(\Proj_i^{-1}(\mathcal{Z}_i))
   \\& =\label{eq:72a}
|\Sym(\DecP)|^{-1}\sum_{\Isom\in \Sym(\DecP)}
((\Isom\circ\Auto^{-1})^{-1})^*\LAProfile_i((\Isom\circ\Auto^{-1})^*\DecP)(\Proj_i^{-1}(\mathcal{Z}_i))
   \\ &=\label{eq:72b}
|\Sym(\DecP)|^{-1}\sum_{\Isom\in \Sym(\DecP)}
(\Isom^{-1})^*\LAProfile_i(\Isom^*\DecP)(\Proj_i^{-1}(\mathcal{Z}_i))
\\&=
\Distr^{(i)}(\Proj_i^{-1}(\mathcal{Z}_i)).\label{eq:72c}
\end{align}
Here, we use in (\ref{eq:72a}) and (\ref{eq:72b}) that \(\Auto^{-1}\in\Iso(\DecP,\DecP)\) and thus by Lemma~\ref{d-sym-closed-under-labeling}, it is \((\Isom\circ\Auto^{-1})^*\DecP=\Isom^*\DecP\) and  \(\Sym(\DecP)=\Sym(\DecP)\circ\Auto^{-1}\), respectively.

Next, note that by assumption, it is \(\Distr^{(i)}=\Distr^{(\Auto^{-1}i)}\) for \(i\in\PlayerSet\) (v). It follows that
\begin{multline}
(\Auto^*\Mixture)(\mathcal{Z}) 
\overset{(\ref{eq:60})}{=}\prod_{i\in\PlayerSet^\DecPSecond}\Mixture_{\Auto^{-1}i}(\mathcal{Z}_{i}\circ \Auto)
=
\prod_{i\in\PlayerSet^\DecP}\Distr^{(\Auto^{-1}i)}(\Proj_{\Auto^{-1}i}^{-1}(\mathcal{Z}_{i}\circ \Auto))
\\
\overset{\text{(v)}}{=}
\prod_{i\in\PlayerSet^\DecP}\Distr^{(i)}(\Proj_{\Auto^{-1}i}^{-1}(\mathcal{Z}_{i}\circ \Auto))
\overset{\text{(\ref{eq:72})--(\ref{eq:72c})}}{=}
\prod_{i\in\PlayerSet^\DecP}\Distr^{(i)}(\Proj_{i}^{-1}(\mathcal{Z}_{i}))
\\=\prod_{i\in\PlayerSet^\DecP}\Mixture_{i}(\mathcal{Z}_{i})
=\Mixture(\mathcal{Z}).
\end{multline}
Since the sets \(\prod_{i\in\PlayerSet}\mathcal{Z}_i\), \(\mathcal{Z}_i\in\mathcal{F}_i\) are a \(\pi\)-system and generate \(\mathcal{F}\), this shows that \(\Mixture=\Auto^*\Mixture\). By Lemma~\ref{lemma-isomorphism-mixture-commute}, it is thus
\begin{equation}\label{eq:73}
  \Auto^*\Policy^\Mixture=\Policy^{\Auto^*\Mixture}=\Policy^\Mixture
\end{equation}
for any \(\Auto\in\Aut(\DecP)\), where \(\Policy^\Mixture\) is the policy corresponding to \(\Mixture\) as defined in Section~\ref{appendix-policies-corresponding-to-distributions}.

Using Proposition~\ref{mixture-lemma}, it follows that
\begin{multline}
U(\LAProfile_1,\dotsc,\LAProfile_N)
\overset{\text{(\ref{eq:705a})--(\ref{eq:705b})}}{=}J^\DecP(\Mixture)
\overset{\text{Proposition~\ref{mixture-lemma}}}{=}J(\Policy^\Mixture)
=\E_{\AutProfile\sim\U(\Aut(\DecP)^\PlayerSet)}
[J^\DecP(\Policy^{\Mixture})]
\\=\E_{\AutProfile\sim\U(\Aut(\DecP)^\PlayerSet)}
\left [ J^\DecP \left ( (\Proj_i(\Policy^{\Mixture}))_{i\in\PlayerSet}\right)\right]
\overset{\text{(\ref{eq:73})}}{=}
\E_{\AutProfile\sim\U(\Aut(\DecP)^\PlayerSet)}
\left[J^\DecP\left((\Proj_i(\AutProfile_i^*\Policy^{\Mixture}))_{i\in\PlayerSet}\right)\right]
\\=J^\DecP_\OP(\Policy^\Mixture)
\leq \max_{\Policy\in\PolicySet^\DecP}J^\DecP_\OP(\Policy)
\overset{\text{(\ref{eq:55a})--(\ref{eq:55c})}}{=}U(\LA^\chi,\dotsc,\LA^\chi).
\end{multline}
This concludes the proof.
\end{proof}

Finally, it follows as a corollary of Theorem~\ref{prop-symmetry-invariant-nash-equilibrium} that the profile \(\LA^\chi,\dotsc,\LA^\chi\) is a Nash equilibrium of the LFC game.

\begin{cor}
All principals using OP with tie-breaking is a Nash equilibrium of the LFC game for \(\DecP\).
\end{cor}
\begin{proof}
By Theorem~\ref{thm-op-with-tie-breaking}, a profile in which all principals use OP with tie-breaking is an optimal symmetric strategy profile in the LFC game for \(\DecP\). Hence, by Theorem~\ref{prop-symmetry-invariant-nash-equilibrium}, it is a Nash equilibrium of the game.
\end{proof}

\subsection{Random tie-breaking functions}
\label{appendix-random-tie-breaking-functions}

In this last section, we show that a certain tie-breaking function, based on random hashes of normal forms of histories, is a tie-breaking function that is invariant to isomorphisms, in the sense of Definition~\ref{defn-tie-breaking-functions} above. As we prove that a certain random function is a tie-breaking function with probability \(1\), one may call this approach a ``probabilistic method'' \parencite[cf.][]{alon2016probabilistic}.

In the following, let \(\DecPSet\) be any set of Dec-POMDPs and \(\DecP\in\DecPSet\). Recall that in Section~\ref{section-other-play-with-tie-breaking}, we defined a function \(\iota\) that maps histories to normal forms, where the normal form is a history in which the first occurrence of each state, action, or observation is set to a \(0\), the second occurrence is set to a \(1\), and so on, and where, if an element in \(\History\in\HistorySet^\DecP\) repeats itself, the number is repeated. For \(\iota(\History):=(\State_0,(\Action_{i,0})_{i\in\PlayerSet},\Reward_0,\dotsc,\State_\Tmax,(\Observation_{i,\Tmax})_{i\in\PlayerSet},(\Action_{i,\Tmax})_{i\in\PlayerSet},\Reward_\Tmax)\), we then define
\[\Isom_N(\iota(\History)):=(\State_0,(\Action_{\Isom_N^{-1}i,0})_{i\in\PlayerSet},\Reward_0,\dotsc,\State_\Tmax,(\Observation_{\Isom_N^{-1}i,\Tmax})_{i\in\PlayerSet},(\Action_{\Isom_N^{-1}i,\Tmax})_{i\in\PlayerSet},\Reward_\Tmax).\] for a permutation  \(\Isom_N\in\mathrm{Bij}(\PlayerSet)\) of \(\PlayerSet\). The tie-breaking function in Section~\ref{section-other-play-with-tie-breaking} was defined as
\begin{equation}\label{equation-old-tie-breaking-function}
\tilde{\chi}^\Hash(\DecP,\Policy):=\frac{1}{N!}\sum_{\Isom_N\in\mathrm{Bij}(\PlayerSet)}\E_{\AutProfile\sim\U(\Aut(\DecP)^\PlayerSet)}\left[\E_{\AutProfile^*\Policy}\left[\Hash(\Isom_N(\iota(\HistoryRV)))\right]\right],
\end{equation}
for some neural random neural network \(\Hash\).

To be able to prove a theoretical result, we have to modify this tie-breaking function to make it dependent on the Dec-POMDP that belongs to a normal form \(\iota(\History)\). The above function is simpler to implement, as we do not have to implement a representation of a Dec-POMDP, and it was sufficient for our experimental results. We leave it to future work to determine to what degree the above version works in general.

In order to define the modified function formally, let \((\Omega_\Hash,\mathcal{E},\Prob_\Hash)\) be some probability space, such that for any \(\DecP\in\DecPSet\), \(\Isom\in\Sym(\DecP)\) and \(\History\in\HistorySet^{\Isom^*\DecP}\), there is an independent, identically distributed real-valued random variable \(\Hash(\Isom^*\DecP,\History)\) on this space. 
Assume that
\begin{equation}\label{eq:condition-hash-function}
\Prob_\Hash(\Hash(\Isom^*\DecP, \History)=\lambda)=0\quad \forall \lambda\in\mathbb{R},\DecP\in\DecPSet,\Isom\in\Sym(\DecP),\History\in\HistorySet^{\Isom^*\DecP}.
\end{equation}
For instance, this would be satisfied by uniformly distributed hash values \(\Hash(\Isom^*\DecP,\History)\sim\U([0,1])\).

Now for a given Dec-POMDP \(\DecP\) and history \(\History\in\HistorySet^\DecP\), denote \(\SymNorm(\History)\subseteq\Sym(\DecP)\) for the set of labelings \(\Isom\in\Sym(\DecP)\) such that \(\Isom\History=\iota(\History)\). It is easy to see that for any history \(\History\in\HistorySet^\DecP\), there must exist some labelings \(\Isom\in\Sym(\DecP)\) such that \(\Isom\History=\iota(\History)\), so \(\SymNorm(\History)\) is non-empty. Then our tie-breaking function here is defined as
\begin{equation}\label{equation-hash-tie-breaking-function}
\chi^\Hash(\DecP,\Policy):=
\E_{\Psi^{\DecP}(\Policy)}\left[\E_{\Isom\sim\U( \SymNorm(\HistoryRV^\DecP))}\left[\Hash(\Isom^*\DecP,\Isom\HistoryRV^{\DecP})\right]\right]
\end{equation}
for \(\DecP\in\DecPSet,\Policy\in\PolicySet^\DecP\). 
We have to sample from the symmetrized policy \(\Psi(\Policy)\) instead of just \(\Policy\), since we need a distribution over histories under which equivalent histories have equal probabilities, and the tie-breaking function needs to be equal for equivalent policies in order to always have a maximizer. The expectation over labelings in \(\SymNorm(\History)\) makes the function invariant to isomorphisms. We need to make the hash function depend on the Dec-POMDP, since otherwise, we are unable to prove that it can distinguish all histories that it needs to be able to distinguish.

\begin{remark}
\label{remark-monte-carlo-estimate-tie-breaking}
By Proposition~\ref{mixture-lemma}, the distribution over histories under \(\Prob_{\Psi(\Policy)}\) is the same as the one under \(\Prob_\Mixture\) where \(\Mixture\) is the OP distribution of \(\Policy\in\PolicySet^{\DecP}\). Hence, using the definition of \(\Prob_\Mixture\) from Equation~\ref{defn-semidirect-product} and the definition of the OP distribution from Equation~\ref{equation-other-play-mixture}, it follows that
\begin{eqnarray}
\chi^\Hash(\DecP,\Policy)&=&
\E_{\Psi^{\DecP}(\Policy)}\left[\E_{\Isom\sim\U( \SymNorm(\HistoryRV^\DecP))}\left[\Hash(\Isom^*\DecP,\Isom\HistoryRV^{\DecP})\right]\right]
\\
&
\overset{\text{Proposition~\ref{mixture-lemma}}}{=}&
\E_\Mixture\left[\E_{\Isom\sim\U( \SymNorm(\HistoryRV^\DecP))}\left[\Hash(\Isom^*\DecP,\Isom\HistoryRV^{\DecP})\right]\right]
\\
&
\overset{(\ref{defn-semidirect-product})}{=}&
\E_{\PolicyTilde\sim\Mixture}\left[\E_{\PolicyTilde}\left[\E_{\Isom\sim\U( \SymNorm(\HistoryRV^\DecP))}\left[\Hash(\Isom^*\DecP,\Isom\HistoryRV^{\DecP})\right]\right]\right]
\\
&\overset{(\ref{equation-other-play-mixture})}{=}&
\label{eq:51}
\E_{\AutProfile\sim\U(\Aut(\DecP)^\PlayerSet)}\left[\E_{\AutProfile^*\Policy}\left[\E_{\Isom\sim\U( \SymNorm(\HistoryRV^\DecP))}\left[\Hash(\Isom^*\DecP,\Isom\HistoryRV^{\DecP})\right]\right]
\right]\end{eqnarray}
for any \(\DecP\in\DecPSet,\Policy\in\PolicySet^\DecP\).
Given a realization of the hash function \(\Hash\), one can hence compute an estimate of \(\chi^\Hash(\DecP,\Policy)\) by sampling a Monte Carlo estimate of the expectation in (\ref{eq:51}).
One can easily see that without the dependence on Dec-POMDPs, this formulation is equal to the one in Equation~\ref{equation-old-tie-breaking-function}.
\end{remark}

Now we show that \(\chi^\Hash\) is invariant to isomorphism. Note that this statement holds for \emph{any} sample \(\omega\in\Omega_\Hash\) and thus for any sample \(\chi^\Hash(\cdot)(\omega)\) of the tie-breaking function. Afterwards, we will show that \(\chi^\Hash\) is almost surely a tie-breaking function. We first need a small lemma.

\begin{lem}\label{lem-normal-form-invariant}
Let \(\DecP\), \(\DecPSecond\) isomorphic Dec-POMDPs with \(\Isom\in\Iso(\DecP,\DecPSecond)\) and let \(\History\in\HistorySet^\DecP\). Then it is
\[\{(\IsomSecond^*\DecPSecond,\IsomSecond(\Isom\History))\mid \IsomSecond\in\SymNorm(\Isom\History)\}
=\{(\IsomHat^*\DecP,\IsomHat\History)\mid \IsomHat\in\SymNorm(\History)\}.\]
\end{lem}
\begin{proof}
Note that for any labeling \(\IsomSecond\in \SymNorm(\Isom\History)\), \(\IsomSecond\Isom\History=(\IsomSecond\circ\Isom)\History\) is in a normal form, and \(\IsomSecond\circ\Isom\in\Sym(\DecP)\) by Lemma~\ref{d-sym-closed-under-labeling}. Hence, also \(\IsomSecond\circ\Isom\in\SymNorm(\History)\). An analogous argument from considering \(\Isom^{-1}\) shows that for any \(\IsomHat\in\SymNorm(\History)\), it is \((\IsomHat\circ\Isom^{-1})(\Isom\History)=\IsomHat\History\) in a normal form an thus \(\IsomHat\circ\Isom^{-1}\in\SymNorm(\Isom\History)\). It follows that \(\SymNorm(\Isom\History)\circ\Isom=\SymNorm(\History)\) (i). Moreover, by Lemma~\ref{d-sym-closed-under-labeling}, it is \(\IsomSecond^*\DecPSecond=(\IsomSecond\circ\Isom)^*\DecP\) (ii), and thus
\begin{align}\label{eq:86}
\{(\IsomSecond^*\DecPSecond,\IsomSecond(\Isom\History))\mid \IsomSecond\in\SymNorm(\Isom\History)\}
&\overset{\text{(ii)}}{=}
\{((\IsomSecond\circ\Isom)^*\DecP,(\IsomSecond\circ\Isom)\History)\mid \IsomSecond\in\SymNorm(\Isom\History)\}
\\
&\overset{\text{(i)}}{=}
\{(\IsomHat^*\DecP,\IsomHat\History)\mid \IsomHat\in\SymNorm(\History)\}.\end{align}

\end{proof}

Using this Lemma, we show the first result.

\begin{prop}\label{prop-chi-invariant-under-equivalent-policies}
\(\chi^\Hash\) as defined in Equation~\ref{equation-hash-tie-breaking-function} is invariant to isomorphism.
\end{prop}
\begin{proof}

Let \(\DecP,\DecPSecond\in\DecPSet\), \(\Isom\in\Iso(\DecP,\DecPSecond)\) and let \(\Policy\in\PolicySet^\DecP,\Policy'\in\PolicySet^\DecPSecond\) such that \(\Isom^*[\Policy]=[\Policy']\). We need to show that \(\chi^\Hash(\DecP,\Policy)=\chi^\Hash(\DecPSecond,\Policy')\).

Note that \(\Isom^*[\Policy]=[\Isom^*\Policy]\) by definition, and thus it follows from the assumption that \(\Isom^*\Policy\equiv\Policy'\), which means that \(\Psi^\DecPSecond(\Isom^*\Policy)=\Psi^\DecPSecond(\Policy')\) (i) by the definition of \(\equiv\).


It follows that
\begin{align}
    \chi^\Hash(\DecPSecond,\Policy')
    &=
    \E_{\Psi^{\DecPSecond}(\Policy')}\left[\E_{\IsomSecond\sim\U( \SymNorm(\HistoryRV^\DecPSecond))}\left[\Hash(\IsomSecond^*\DecPSecond,\IsomSecond\HistoryRV^{\DecPSecond})\right]\right]
    \\
    &\overset{\text{(i)}}{=}
    \E_{\Psi^{\DecPSecond}(\Isom^*\Policy)}\left[
    \E_{\IsomSecond\sim\U( \SymNorm(\HistoryRV^\DecPSecond))}\left[\Hash(\IsomSecond^*\DecPSecond,\IsomSecond\HistoryRV^{\DecPSecond})\right]
    \right]
    \\\label{eq:52a}
    &=
    \E_{\Isom^*\Psi^{\DecP}(\Policy)}\left[
    \E_{\IsomSecond\sim\U( \SymNorm(\HistoryRV^\DecPSecond))}\left[\Hash(\IsomSecond^*\DecPSecond,\IsomSecond\HistoryRV^{\DecPSecond})\right]
    \right]
 \\\label{eq:52b}
    &=
    \E_{\Psi^{\DecP}(\Policy)}\left[
    \E_{\IsomSecond\sim\U( \SymNorm(\Isom\HistoryRV^\DecP))}\left[\Hash(\IsomSecond^*\DecPSecond,\IsomSecond(\Isom\HistoryRV^{\DecP}))\right]
    \right]
    \\
    &=\label{eq:52c}
    \E_{\Psi^{\DecP}(\Policy)}\left[
    \E_{\IsomHat\sim\U( \SymNorm(\HistoryRV^\DecP))}\left[\Hash(\IsomHat^*\DecP,\IsomHat\HistoryRV^{\DecP})\right]
    \right]
\end{align}
where in (\ref{eq:52a}), we use that isomorphisms and symmetrizer \(\Psi\) commute by Corollary~\ref{corollary-psi-isom-commute}, in (\ref{eq:52b}), we use Theorem~\ref{lem-pull-back-isomorphism-compatibility}, and in (\ref{eq:52c}), we use Lemma~\ref{lem-normal-form-invariant}.

This shows that \(\chi^\Hash\) is invariant to isomorphism.
\end{proof}

To show that \(\chi^\Hash\) is almost surely a tie-breaking function, we need two technical assumptions. The first, substantial one, is that the set of OP-optimal equivalence classes of policies is finite, to make sure that there always exists a unique maximizer of \(\chi^\Hash\). If this was not the case, our method for finding a tie-breaking function could fail. We leave it to future work to investigate whether and under what conditions this may be the case.

The second, less substantial assumption is that, if two policies \(\Policy,\Policy'\) are not equivalent, then \(\Psi(\Policy)\) and \(\Psi(\Policy')\) also do not lead to the same distribution over histories. Without this assumption, \(\Psi(\Policy)\) and \(\Psi(\Policy')\) might differ on action-observation histories that are never reached, such that the policies induce the same distribution over histories but are not equivalent. Choosing policies which differ in that way should not matter for the objective of the LFC problem, but we wanted to avoid dealing with this complication in our proof that OP with tie-breaking is optimal. For this reason, we have stronger requirements for tie-breaking functions here, thus necessitating this technical assumption.\footnote{Note that in Definition~\ref{policy-mixture}, each local policy is defined separately, so that \(\Psi(\Policy)_i(\cdot\mid\AOHistory_{i,t})\) can in principle be defined as something other than the uniform distribution, even if the  action-observation history \(\AOHistory_{i,t}\) is never reached under the joint policy \(\Psi(\Policy)\). Nevertheless, if all agents choose local policies from joint policies that lead to the same distribution over histories, the resulting cross-play policy should also induce that distribution, even if there exist opponent distributions where the local policies might differ.}

\begin{prop}\label{prop-almost-surely-tie-breaking-function}
Let \(\chi^\Hash\) be defined as in Equation~\ref{equation-hash-tie-breaking-function}. Assume that
\begin{enumerate}
    \item[(i)] for any \(\DecP\in\DecPSet\), the set \(\faktor{\PolicySet^\DecP_\OP}{\equiv}\) is finite, where \(\PolicySet^\DecP_\OP:=\argmax_{\Policy\in\PolicySet^\DecP}J_\OP^\DecP(\Policy)\).
    \item[(ii)] for any \(\DecP\in\DecPSet\) and policies \(\Policy,\Policy'\in\PolicySet^\DecP\), if \(\Prob_{\Psi^\DecP(\Policy)}(\HistoryRV=\History)=\Prob_{\Psi^\DecP(\Policy')}(\HistoryRV=\History)\) for all \(\History\in\HistorySet^\DecP\), it follows that \(\Psi^\DecP(\Policy)=\Psi^\DecP(\Policy')\).
\end{enumerate}
Then \(\chi^\Hash\) is \(\Prob_\Hash\)-almost surely a tie-breaking function for \(\DecPSet\).
\end{prop}

For the proof, we first need a standard lemma in probability theory.

\begin{lem}\label{lem-sum-probability-zero}
Assume \(\mathcal{J}\) is a finite set and \(X_j,j\in\mathcal{J}\) is a collection of independent real-valued random variables such that for any \(j\in\mathcal{J}\), it is \(\Prob(X_j=\lambda)=0\) for any \(\lambda\in\mathbb{R}\). Let \(\langle \cdot,\cdot\rangle\) be the euclidean scalar product on \(\mathbb{R}^\mathcal{J}\). Then it is
\[\Prob(\langle X, v\rangle=\lambda)=0\]
for any \(\lambda\in\mathbb{R}\), \(v\in\mathbb{R}^\mathcal{J}\setminus\{ 0\}\).
\end{lem}
\begin{proof}
Let \(\lambda\in\mathbb{R}\) and \(v\in\mathbb{R}^\mathcal{J}\setminus\{ 0\}\).
Since \(v\neq 0\), there exists an index \(j\in\mathcal{J}\) such that \(v_j\neq0\). It follows that
\begin{align}
\Prob(\langle X, v\rangle=\lambda)
&=\Prob_\Hash\left(X_j=\lambda-\frac{1}{v_j}\sum_{j'\neq j}v_{j'}X_{j'}\right)
\\
&=\int_{\mathbb{R}^{\mathcal{J}}}\mathds{1}_{x_j=\lambda-\frac{1}{v_j}\sum_{{j'}\neq j}x_{j'}v_{j'}}\mathrm{d}\Prob \circ X^{-1}(x)
\\
&=\int_{\mathbb{R}^{\mathcal{J}\setminus\{j\}}}
\int_{\mathbb{R}}\mathds{1}_{x_j=\lambda-\frac{1}{v_j}\sum_{{j'}\neq j}x_{j'}v_{j'})}\mathrm{d}\Prob\circ X_j^{-1}(x_j)\mathrm{d}\Prob \circ X_{-j}^{-1}(x_{-j})
\\
&=\int_{\mathbb{R}^{\mathcal{J}\setminus\{j\}}}\Prob\left (
X_j=\lambda-\frac{1}{v_j}\sum_{{j'}\neq j}x_{j'}v_{j'}
\right)\mathrm{d}\Prob \circ X_{-j}^{-1}(x_{-j})
\\\label{eq:500a}
&=\int_{\mathbb{R}^{\mathcal{J}\setminus\{j\}}}0\;\mathrm{d}\Prob \circ X_{-j}^{-1}(x_{-j})
\\&=0,
\end{align}
where we use the assumption \(\Prob(X_j=\lambda)=0\) for any \(\lambda\in\mathbb{R}\) in (\ref{eq:500a}).
\end{proof}

\begin{proof}[Proof of Proposition~\ref{prop-almost-surely-tie-breaking-function}]
To check that \(\chi^\Hash\) always admits a maximum among the OP-optimal policies, which is part (a), (i) of Definition~\ref{defn-tie-breaking-functions}, let \(\DecP\in\DecPSet\) arbitrary and let \(\Policy\equiv\Policy'\in\PolicySet^\DecP\) be equivalent policies. By the definition of \(\equiv\), it is \(\Psi^\DecP(\Policy)=\Psi^\DecP(\Policy')\) (*).
Hence,
it is
\begin{multline}
\chi^\Hash(\DecP,\Policy)
=
\E_{\Psi^{\DecP}(\Policy)}\left[
    \E_{\Isom\sim\U( \SymNorm(\HistoryRV^\DecP))}\left[\Hash(\Isom^*\DecP,\Isom\HistoryRV^{\DecP})\right]
    \right]
\\
\overset{\text{(*)}}{=}
\E_{\Psi^{\DecP}(\Policy')}\left[
    \E_{\Isom\sim\U( \SymNorm(\HistoryRV^\DecP))}\left[\Hash(\Isom^*\DecP,\Isom\HistoryRV^{\DecP})\right]
    \right]
=
\chi^\Hash(\DecP,\Policy').
\end{multline}
This shows that \(\chi^\Hash(\DecP,\Policy')=\chi^\Hash(\DecP,\Policy)\) for any \(\Policy \in \PolicySet^\DecP_\OP\) and \(\Policy'\in[\Policy]\). Hence, we can define the function
\[\tilde{\chi}^\Hash(\DecP,\cdot)\colon \faktor{\PolicySet^\DecP_\OP}{\equiv}\rightarrow [0,1],[\Policy]\mapsto\chi^\Hash(\DecP,\Policy).\]
Then, using assumption (i), \(\tilde{\chi}^\Hash(\DecP,\cdot)\) is maximized by some equivalence class \([\Policy]\in\faktor{\PolicySet^\DecP_\OP}{\equiv}\). By definition of \(\tilde{\chi}^\Hash(\DecP,\cdot)\), it follows that
for any \(\Policy'\in\PolicySet_\OP^\DecP\), it is
\[\chi^\Hash(\DecP,\Policy')=\tilde{\chi}^\Hash(\DecP,[\Policy'])\leq\tilde{\chi}^\Hash(\DecP,[\Policy])=\chi^\Hash(\DecP,\Policy),\]
so \(\Policy\) maximizes \(\chi^\Hash(\DecP,\cdot)\) on the set \(\PolicySet_\OP^\DecP\).

To prove part (a), (ii) of the definition, consider two non-equivalent policies \(\Policy,\Policy'\in\PolicySet^\DecP\), i.e., assume that
\(\Psi(\Policy)\neq\Psi(\Policy')\). We now want to show that \(\Prob_\Hash\left(\chi^\Hash(\DecP,\Policy)=\chi^\Hash(\DecP,\Policy')\right)=0\). If this is true, then \(\chi^\Hash\) is \(\Prob_\Hash\)-almost surely a tie-breaking function.

First, note that actions of automorphisms on histories are group actions (see Section~\ref{appendix-relation-to-group-theory}). Thus, we can consider a partition of the set of histories into orbits, \(\mathcal{J}:=\{\Aut(\DecP)\History\mid\History\in\HistorySet^\DecP\}\). Moreover, note that since \(\Psi^\DecP(\Policy),\Psi^\DecP(\Policy')\) are both invariant to automorphism by Corollary~\ref{corollary-psi-isom-commute}, using Theorem~\ref{lem-pull-back-isomorphism-compatibility}, we can conclude that for any \(\History\in\HistorySet^\DecP\) and \(\tilde{\History}\in\Aut(\DecP)\History\), it is \(\Prob_{\Psi(\Policy)}(\HistoryRV=\History)=\Prob_{\Psi(\Policy)}(\HistoryRV=\tilde{\History})\). The same holds for \(\Psi(\Policy')\). This step only works since we sample from \(\Psi(\Policy)\), instead of from the original policy \(\Policy\). As a result, it follows that we can define a vector \(v\in\mathbb{R}^\mathcal{J}\) such that \(v_j:=|\Aut(\DecP)\History|\Prob_{\Psi(\Policy)}(\HistoryRV=\History)\) for \(j\in\mathcal{J}\), where \(\History\in\HistorySet^\DecP\) is arbitrary such that \(\Aut(\DecP)\History=j\). Define \(v'\) analogously for \(\Policy'\). By assumption (ii), \(\Psi(\Policy)\) and \(\Psi(\Policy')\) induce different distributions over histories, so it must also be \(v\neq v'\), since the map from orbit-invariant distributions to vectors \(v\) is injective.

Second, note that by Lemma~\ref{lem-normal-form-invariant}, it is
\[\E_{\Isom\sim\U( \SymNorm(\History))}\left[\Hash(\Isom^*\DecP,\Isom\History)\right]=\E_{\Isom\sim\U( \SymNorm(\Auto\History))}\left[\Hash(\Isom^*\DecP,\Isom\Auto\History)\right]\]
for any \(\History\in\HistorySet^\DecP\) and \(\Auto\in\Aut(\DecP)\), and thus for any \(j\in\mathcal{J}\) we can define the random variable \(X_j:=\E_{\Isom\sim\U( \SymNorm(\History))}\left[\Hash(\Isom^*\DecP,\Isom\History)\right]\), where \(\History\) is an arbitrary history such that \(j=\Aut(\DecP)\History\).

We now want to show that \(\Prob_\Hash(\langle X, v-v'\rangle = 0)=0\), where \(\langle \cdot,\cdot \rangle\) is the euclidean scalar product on \(\mathbb{R}^\mathcal{J}\). This will then allow us to conclude that \(\Prob_\Hash\left(\chi^\Hash(\DecP,\Policy)=\chi^\Hash(\DecP,\Policy')\right)=0\). We already know that \(v-v'\neq 0\). To be able to apply Lemma~\ref{lem-sum-probability-zero} to show this, it remains to prove that for any two \(j\neq j'\in\mathcal{J}\), we also have two independent variables \(X_j,X_{j'}\), and that \(\Prob_\Hash(X_j=\lambda)=0\) for any \(\lambda\in\mathbb{R}\).

To prove independence, let \(j=\Aut(\DecP)\History,j'=\Aut(\DecP)\History'\), and assume towards a contradiction that there exist \(\Isom\in \SymNorm(\History),\IsomSecond\in\SymNorm(\History')\) such that \(\Isom^*\DecP=\IsomSecond^*\DecP\) and \(\Isom\History=\IsomSecond \History'\). Then it follows that \(\History=\Isom^{-1}\IsomSecond\History'\) and thus \(j=j'\), which is a contradiction. It follows that the sets of variables \(\{\Hash(\Isom^*\DecP,\Isom\History)\mid \Isom\in \SymNorm(\History)\}\) and \( \{\Hash(\IsomSecond^*\DecP,\IsomSecond\History')\mid \IsomSecond\in\SymNorm(\History')\}\) are disjoint. Since all the contained random variables are independent, also 
 \(\E_{\Isom\sim\U( \SymNorm(\History))}\left[\Hash(\Isom^*\DecP,\Isom\History)\right]\) and \(\E_{\Isom\sim\U( \SymNorm(\History'))}\left[\Hash(\Isom^*\DecP,\Isom\History')\right]\) are independent random variables.
 
 Next, let \(j\in\mathcal{J}\). To prove that \(\Prob_\Hash(X_j=\lambda)=0\) for any \(\lambda\in\mathbb{R}\), let \(\History\in\HistorySet^\DecP\) such that \(j=\Aut(\DecP)\History\) arbitrary and let \(\lambda\in\mathbb{R}\). Let \(M:=\left|\bigcup_{\Isom\in\SymNorm(\History)}\{\Hash(\Isom^*\DecP,\Isom\History)\}\right|\) and define \(Y_1,\dotsc,Y_M\) as random variables such that \(\{Y_1,\dotsc,Y_M\}=\bigcup_{\Isom\in\SymNorm(\History)}\{\Hash(\Isom^*\DecP,\Isom\History)\}\). Define \(w\in[0,1]^M\) via \(w_m:=\E_{\Isom\sim\U(\SymNorm(\History))}[\delta_{Y_m,\Hash(\Isom^*\DecP,\Isom\History)}]\) for \(m=1,\dotsc,M\). By assumption (\ref{eq:condition-hash-function}), we have \(\Prob_\Hash(Y_m=\lambda')=0\) for any \(\lambda'\in\mathbb{R}\) and \(Y_m,Y_{m'}\) are by definition independent variables for \(m\neq m'\in\{1,\dotsc,M\}\). We can hence apply Lemma~\ref{lem-sum-probability-zero} to conclude that
 \begin{align}\label{eq:503}
 \Prob_\Hash(X_j=\lambda)
& =\Prob_\Hash\left(\E_{\Isom\sim\U(\SymNorm(\History))}[\Hash(\Isom^*\DecP,\Isom\History)]=\lambda\right)
 \\& =\Prob_\Hash\left(\sum_{m=1}^MY_m\E_{\Isom\sim\U(\SymNorm(\History))}[\delta_{Y_m,\Hash(\Isom^*\DecP,\Isom\History)}]=\lambda\right)
   \\& =\Prob_\Hash\left(\sum_{m=1}^MY_mw_m=\lambda\right)
   \\ &    =\Prob_\Hash\left(\langle Y, w\rangle=\lambda\right)
\\&=0.\end{align}
 
Since we have shown above that \(v\neq v'\), that \(X_j,X_{j'}\) are independent for \(j\neq j'\in\mathcal{J}\), and that \({\Prob_\Hash(X_j=\lambda)=0}\) for any \(\lambda\in\mathbb{R}\) and \(j\in\mathcal{J}\), we can apply Lemma~\ref{lem-sum-probability-zero} again to get
\begin{equation}\label{eq:505}\Prob_\Hash(\langle X, v-v'\rangle = 0)=0.\end{equation}
 
Finally, for \(j\in\mathcal{J}\) let \(\History^{(j)}\in\HistorySet^\DecP\) arbitrary such that \(j=\Aut(\DecP)\History^{(j)}\).
Then 
\begin{multline}
\langle X, v\rangle=
\sum_{j\in\mathcal{J}}|\Aut(\DecP)\History^{(j)}|\Prob_{\Psi^{\DecP}(\Policy)}(\HistoryRV^\DecP=\History^{(j)})X_j
\\=
\sum_{j\in\mathcal{J}}|\Aut(\DecP)\History^{(j)}|\Prob_{\Psi^{\DecP}(\Policy)}(\HistoryRV^\DecP=\History^{(j)})\E_{\Isom\sim\U( \SymNorm(\History^{(j)}))}\left[\Hash(\Isom^*\DecP,\Isom\History^{(j)})\right]
\\=
\sum_{\History\in\HistorySet^\DecP}\Prob_{\Psi^{\DecP}(\Policy)}(\HistoryRV^\DecP=\History)\E_{\Isom\sim\U( \SymNorm(\History))}\left[\Hash(\Isom^*\DecP,\Isom\History)\right]
=
\chi^\Hash(\DecP,\Policy)
\end{multline}
and analogously
\begin{equation}
\langle X, v'\rangle=
\chi^\Hash(\DecP,\Policy').
\end{equation}
Hence, it is
\begin{equation}\label{eq:504}\{\chi^\Hash(\DecP,\Policy)=\chi^\Hash(\DecP,\Policy')\}=\{\langle X, v\rangle =  \langle X, v'\rangle\}.\end{equation}
It follows that
\[\Prob_\Hash(\chi^\Hash(\DecP,\Policy)=\chi^\Hash(\DecP,\Policy'))\overset{(\ref{eq:504})}{=}\Prob_\Hash(\langle X, v\rangle =  \langle X, v'\rangle)
=\Prob_\Hash(\langle X, v-v'\rangle = 0)\overset{(\ref{eq:505})}{=}0.\]
This concludes the proof.
\end{proof}

\fi

\end{document}